%% file: main_icml.tex
\documentclass{article}

\usepackage[accepted]{icml2021_style/icml2021}
\icmltitlerunning{Annealed Flow Transport Monte Carlo}

\usepackage[T1]{fontenc}
\usepackage[latin9]{inputenc}
\usepackage{amstext}
\usepackage{amsthm}

\usepackage{amssymb}
\usepackage{dsfont}
\usepackage{color}
\usepackage[dvipsnames]{xcolor}
\makeatletter

\usepackage[T1]{fontenc}
\usepackage[american]{babel}
\usepackage{url}
\usepackage{booktabs}
\usepackage{amsfonts}
\usepackage{nicefrac}
\usepackage{csquotes}
\usepackage{enumitem}
\usepackage{etoolbox}
\usepackage{stmaryrd}
\usepackage{amstext}
\usepackage{amsmath}
\usepackage{amssymb}

\usepackage{xcolor}

\usepackage{appendix}  
\usepackage{graphicx}
\usepackage{epstopdf}
\usepackage{subcaption}

\usepackage{algorithmic}

\usepackage{algorithm}
\usepackage[colorlinks,linkcolor={red!80!black},citecolor={blue},allbordercolors={1 1 1},urlcolor={blue!80!black},hypertexnames=false]{hyperref} 
\usepackage[noabbrev,capitalize]{cleveref}

\crefformat{equation}{(#2#1#3)}
\crefrangeformat{equation}{(#3#1#4) to~(#5#2#6)}
\crefmultiformat{equation}{(#2#1#3)}{ and~(#2#1#3)}{, (#2#1#3)}{ and~(#2#1#3)}

\crefname{appsec}{Appendix}{Appendices}

\usepackage{autonum}

\newcommand{\manote}[1]{\textcolor{red}{#1}}
\usepackage{automatic_resize}
\newdelimcommand{parens}{(}{)}
\newdelimcommand{brakets}{[}{]}

\newcommand{\R}{\mathbb R}

\newcommand{\X}{\mathcal X}
\newcommand{\G}{G}

\makeatletter
\newcommand\given{\@ifstar{\mathrel{}\middle|\mathrel{}}{\mid}}

\DeclareRobustCommand{\abs}{\@ifstar\@abs\@@abs}

\DeclareRobustCommand{\norm}{\@ifstar\@norm\@@norm}

\DeclareRobustCommand{\inner}{\@ifstar\@inner\@@inner}

\makeatother

\newcommand*\diff{\mathop{}\!\mathrm{d}}

\usepackage{mathtools}

\DeclarePairedDelimiter\floor{\lfloor}{\rfloor}

\newtheorem{lem}{Lemma}
\newtheorem{thm}{Theorem}

\newtheorem{prop}{Proposition}

\newlist{assumplist}{enumerate}{1}
\setlist[assumplist]{label=(\textbf{\Alph*})}
\Crefname{assumplisti}{Assumption}{Assumptions}
\newlist{assumplist2}{enumerate}{2}
\setlist[assumplist2]{label=(\textbf{\alph*})} 
\Crefname{assumplist2i}{Assumption}{Assumptions}

\newlist{assumplistobs}{enumerate}{3}
\setlist[assumplistobs]{label=(\textbf{$\mathcal{F}$-\alph*})} 
\Crefname{assumplistobs}{Assumption}{Assumptions}

\hypersetup{
  colorlinks = true,
  citecolor = blue,
}

\makeatother

\usepackage{babel}

\crefname{lem}{Lemma}{Lemmas}
\Crefname{lem}{Lemma}{Lemmas}
\crefname{thm}{Theorem}{Theorems}
\Crefname{thm}{Theorem}{Theorems}
\crefname{prop}{Proposition}{Propositions}
\Crefname{prop}{Proposition}{Propositions}

\begin{document}

\twocolumn[
\icmltitle{Annealed Flow Transport Monte Carlo}

\icmlsetsymbol{equal}{*}
\icmlsetsymbol{a}{}

\begin{icmlauthorlist}
\icmlauthor{Michael Arbel}{equal,Gatsby}
\icmlauthor{Alexander G. D. G. Matthews}{equal,DeepMind}
\icmlauthor{Arnaud Doucet}{a,DeepMind}
\end{icmlauthorlist}

\icmlaffiliation{Gatsby}{Gatsby Computational Neuroscience Unit, University College London}
\icmlaffiliation{DeepMind}{DeepMind}

\icmlcorrespondingauthor{Michael Arbel}{michael.n.arbel@gmail.com}
\icmlcorrespondingauthor{Alexander G. D. G. Matthews}{alexmatthews@google.com}
\icmlcorrespondingauthor{Arnaud Doucet}{arnauddoucet@google.com}

\icmlkeywords{Machine Learning, ICML}

\vskip 0.3in
]

\printAffiliationsAndNotice{\icmlEqualContribution}
\begin{abstract}
Annealed Importance Sampling (AIS) and its Sequential Monte Carlo (SMC) extensions are state-of-the-art methods for estimating normalizing constants of probability distributions. We propose here a novel Monte Carlo algorithm, Annealed Flow Transport (AFT), that builds upon AIS and SMC and combines them with normalizing flows (NFs) for improved performance. This method transports a set of particles using not only importance sampling (IS), Markov chain Monte Carlo (MCMC) and resampling steps - as in SMC, but also relies on NFs which are learned sequentially to push particles towards the successive annealed targets. We provide limit theorems for the resulting Monte Carlo estimates of the normalizing constant and expectations with respect to the target distribution. Additionally, we show that a continuous-time scaling limit of the population version of AFT is given by a Feynman--Kac measure which simplifies to the law of a controlled diffusion for expressive NFs. We demonstrate experimentally the benefits and limitations of our methodology on a variety of applications.
\end{abstract}

\section{Introduction}
Let $\pi$ be a target density on $\X\subseteq \R^d$ w.r.t. the Lebesgue measure known up to a normalizing constant $Z$. We want to estimate $Z$ and approximate expectations with respect to $\pi$. This has applications in Bayesian statistics but also variational inference (VI) \cite{mnih2016variational} and compression \cite{li2019rate,huang2020evaluating} among others. AIS \citep{neal2001annealed} and its SMC extensions \cite{Del-Moral:2006} are state-of-the art Monte Carlo methods addressing this problem which rely on a sequence of annealed targets $\pi_k \propto \pi_0^{1-\beta_k}\pi_K^{\beta_k}$ bridging smoothly an easy-to-sample distribution $\pi_0$  to $\pi_K:=\pi$ for $0=\beta_0<\beta_1< \cdots <\beta_K=1$ and MCMC kernels of invariant distributions $\pi_k$  \cite{ZhouJohansen2016,llorente2020marginal}. In their simplest instance, SMC samplers propagate $N$ particles approximating $\pi_k$ at time $k$. These particles are reweighted according to weights proportional to $\pi_{k+1}/\pi_k$ at time $k+1$ to build an IS approximation of $\pi_{k+1}$, then one resamples $N$ times from this approximation and finally mutate the resampled particles according to MCMC steps of invariant distribution $\pi_{k+1}$. This procedure can provide high-variance estimators if the discrepancy between $\pi_k$ and $\pi_{k+1}$ is significant as the resulting IS weights then have a large variance and/or if the MCMC kernels mix poorly. This can be reduced by increasing $K$ and the number of MCMC steps at each temperature but comes at an increasing computational cost.

An alternative approach is to build a transport map $T:\X \rightarrow \X$ to ensure that if $X \sim \pi_0$ then the distribution of $X'=T(X)$ denoted $T_{\#}\pi_0$ is approximately equal to $\pi$. In \cite{Marzouk2012bayesian}, this map is parameterized using a polynomial chaos expansion and learned by minimizing a regularized  Kullback-Leibler (KL) divergence between $T_{\#}\pi_0$ and $\pi$; see also \cite{marzouk2016sampling}. \citet{taghvaei2020diffusion} and \citet{olmez2020deep} obtain transport maps by solving a Poisson equation. However, they do not correct for the discrepancy between $T_{\#}\pi_0$ and $\pi$ using IS. Doing so would incur a $O(d^3)$ cost when computing the Jacobian. Normalizing Flows (NFs) are an alternative flexible class of diffeomorphisms with easy-to-compute Jacobians \citep{Rezende:2015}. These can be used to parameterize $T$ and are also typically learned by minimizing $\textup{KL}(T_{\#}\pi_0||\pi)$ or a regularized version of it. This approach has been investigated in many recent work; see e.g. \cite{gao2020flow,nicoli2020asymptotically,noe2019boltzmann,wirnsbergertargetetNF2020}. Although it is attractive, it is also well-known that optimizing this `mode-seeking' KL can lead to an approximation of the target $T_{\#}\pi_0$ which has thinner tails than the target $\pi$ and ignore some of its modes; see e.g. \cite{domke2018importance}.

In this paper, our contributions are as follows.

\begin{itemize}[topsep=0pt,itemsep=0pt,partopsep=1pt,parsep=1pt,listparindent=2pt,leftmargin=12pt]
\item We propose Annealed Flow Transport (AFT), a methodology that takes advantages of the strengths of both SMC and NFs. Given particles approximating $\pi_k$ at time $k$, we learn a NF $T_{k+1}$ minimizing the KL between $(T_{k+1})_{\#}\pi_k$ and $\pi_{k+1}$. As $\pi_k$ is closer to $\pi_{k+1}$ than $\pi_0$ is from $\pi_K=\pi$, learning such a NF is easier and less prone to mode collapse. Additionally the use of MCMC steps in SMC samplers allows the particles to diffuse and further prevent such collapse. Having obtained $T_{k+1}$, we then apply this mapping to the particles before building an IS approximation of $\pi_{k+1}$ and then use resampling and MCMC steps. 

\item We establish a weak law of large numbers and a Central Limit Theorem (CLT) for the resulting Monte Carlo estimates of $Z$ and expectations w.r.t. $\pi$. Available CLT results for SMC \citep{chopin2004central,del2004feynman,kunsch2005recursive,Beskos:2016} do not apply here as the transport maps are learned from particles. 

\item When one relies on Unadjusted Langevin algorithm (ULA) kernels to mutate particles, a time-rescaled population version of AFT without resampling is shown to converge as $K \rightarrow \infty$ towards a Feynman--Kac measure. For NFs expressive enough to include exact transport maps between successive distributions, this measure corresponds to the measure induced by a controlled Langevin diffusion.

\item We demonstrate the performance of AFT on a variety of benchmarks, showing that it can improve over SMC for a given number of temperatures.
\end{itemize}

\textbf{Related Work.}
The use of deterministic maps with AIS  \citep{vaikuntanathan2011escorted} and  SMC \citep{akyildiz2020nudging,everitt2020sequential,heng2015gibbs} has already been explored. However, \citet{everitt2020sequential} and \citet{vaikuntanathan2011escorted} do not propose a generic methodology to build such maps while \citet{akyildiz2020nudging} introduce mode-seeking maps and do not correct for the incurred bias. \citet{heng2015gibbs} rely on quadrature and a system of time-discretized nonlinear ordinary differential equations: this can be computationally cheaper than learning NFs but is application specific. NFs benefit from easy-to-compute Jacobians and a large and quickly expanding literature \citep{papamakarios2019normalizing}; e.g., as both MCMC and NFs on manifolds have been developed, our algorithm can be directly extended to such settings. 

Evidence Lower Bounds (ELBOs) based on unbiased estimators of $Z$ have also been mentioned in \citep{salimans2015markov,goyal2017variational,caterini2018hamiltonian,huang2018improving,wunoe2020stochastic,thin2021MCVAE}.
These estimators generalize AIS, and are obtained using sequential IS, transport maps and MCMC.  However, when MCMC kernels such as Metropolis--Hastings (MH) or Hamiltonian Monte Carlo (HMC) are used, accept/reject steps lead to high variance estimates of ELBO gradients \cite{thin2021MCVAE}. Moreover, while SMC (i.e. combining sequential IS and resampling) can also be used to define an ELBO, resampling steps correspond to sampling discrete distributions and lead to high variance gradient estimates; see e.g. \citep{maddison2017filtering,le2017auto,naesseth2017variational} in the context of state-space models. The algorithm proposed here does not rely on the ELBO, so it can use arbitrary MCMC kernels and exploit the benefits of resampling. Moreover, it only requires a single pass through the $K+1$ annealed distributions: there is no need to iteratively run sequential IS or SMC for estimating $Z$ and an ELBO gradient estimate.

Optimal control ideas have also been proposed to improve SMC by introducing an additive drift to a time-inhomogeneous ULA to improve sampling; see \citet{richard2007efficient,kappen2016adaptive,guarniero2017iterated,heng2017controlled}. The proposed iterative algorithms require estimating value functions but, to be implementable, the approximating function class has to be severely restricted. The algorithm proposed here is much more widely applicable and can use sophisticated MCMC kernels.

Finally, alternative particle methods based on gradient flows in the space of probability measures have been proposed to provide an approximation of $\pi$, such as  Stein Variational Gradient Descent (SVGD) \cite{liu2016stein,liu2019understanding,wang2019accelerated,zhu2020variance,reich2021fokker}. However, their consistency results require both $K$, the number of time steps, and $N$, the number of particles, to go to infinity. In contrast, AFT only needs $N \rightarrow \infty$.
 Moreover, they require specifying a suitable Reproducing Kernel Hilbert Space or performing kernel density estimation, which can be challenging in high dimension. Additionally, contrary to AFT, these methods do not provide an estimate of $Z$. One recent exception is the work of \citet{han2017stein} which combines SVGD with IS to estimate $Z$ but this requires computing Jacobians of computational cost $O(d^3)$.

\section{Sequential Monte Carlo samplers}\label{sec:SMCsamplers}
We provide here a brief overview of SMC samplers and their connections to AIS. More details can be found in \citep{Del-Moral:2006,dai2020invitation}.

We will rely on the following notation for the annealed densities $(\pi_k)_{0\leq k\leq K}$ targeted by SMC: 
\begin{align}\label{eq:path}
	\pi_{k}(x) =  \frac{\gamma_k(x)}{Z_k}=\frac{\exp(-V_{k}(x)}{Z_k},
\end{align}
where $Z_0=1$ so $\pi_0(x) = \gamma_0(x)$ and $V_k(x) =(1-\beta_k) V_0 + \beta_k V_K$ for $0=\beta_0<\beta_1<\cdots<\beta_K=1$. However, we could use more generally any sequence of distributions bridging smoothly $\pi_0$ to $\pi_K=\pi$.

\subsection{Sequential importance sampling}\label{sec:SIS}
Let us first ignore the key resampling steps used by SMC. In this case, SMC boils down to a sequential IS technique where one approximates $\pi_k$ at time $k$. We first sample $X_0\sim \pi_0$ at time $k=0$, then at time $k\geq 1$, obtain a a new sample $X_k\sim M_{k}(X_{k-1},\cdot)$ using a Markov kernel $M_{k}$. For the distribution of $X_k$ to be closer to $\pi_k$ than the one of $X_{k-1}$, $M_k$ is typically selected as a  MCMC kernel of invariant density $\pi_k$ such as MH or HMC, or of approximate invariant density $\pi_k$ such as ULA. Hence, by construction, the joint density of $X_{0:k}$ is
\begin{equation}\label{eq:joint_chain}
	\textstyle{\bar{\eta}_{k}(x_{0:k}) = \pi_0(x_0) \prod_{l=1}^{k} M_l(x_{l-1}, x_l).}
\end{equation}
 
The resulting marginal $\eta_k$ of $X_k$ under $\bar{\eta}_k$ usually differs from $\pi_{k}$. If one could evaluate $\eta_k$ pointwise, then IS could be used to correct for the discrepancy between $\eta_k$ and $\pi_k$ using the IS weight $w_k(x_k)= \gamma_{k}(x_k)/\eta_{k}(x_k)$. Unfortunately, $\eta_k$ is intractable in all but toy scenarios. 
Instead, SMC samplers introduce joint target densities $\bar{\pi}_k(x_{0:k})$ to compute tractable IS weights  $w_k(x_{0:k})$ over the whole path $X_{0:k}$ defined by 
\begin{align}\label{eq:jointtarget}
\textstyle{\bar{\pi}_k(x_{0:k}) = \pi_{k}(x_k)	\prod_{l=0}^{k-1} L_{l}(x_{l+1}, x_{l}),}
\end{align}
here $L_{l}$ are ``backward'' Markov kernels moving each sample $X_{l+1}$ into a sample $X_{l}$ starting from a \textit{virtual} sample $X_k$ from $\pi_k$\footnote{As in \citep{crooks1998nonequilibrium,neal2001annealed,Del-Moral:2006,dai2020invitation}, we do not use measure-theoretic notation here but it should be kept in mind that the kernels $M_l$ do not necessarily admit a density w.r.t. Lebesgue measure; e.g. a MH kernel admits an atomic component. For completeness, a formal measure-theoretic presentation of the results of this section is given in Appendix \ref{sec:measuretheoreticpresentation}.}. Hence by construction $\pi_k$ is the marginal of $\bar{\pi}_k$ at time $k$. The backward kernels $L_{k-1}$ are chosen so that the following \textit{incremental IS weights} are well-defined
\begin{equation}\label{eq:incrementalweightSMC}
    G_k(x_{k-1},x_k)=\frac{\gamma_k(x_k)L_{k-1}(x_k,x_{k-1} )}{\gamma_{k-1}(x_{k-1})M_k(x_{k-1},x_k)},
\end{equation}
and, from \eqref{eq:joint_chain} and \eqref{eq:jointtarget}, one obtains
\begin{align}\label{eq:unnormalized_importance_weights}
w_{k}(x_{0:k}) &:= \frac{\bar{\gamma}_k(x_{0:k})}{\bar{\eta}_k(x_{0:k})} = \prod_{l=1}^k G_l(x_{l-1},x_l),
\end{align}
where $\bar{\gamma}_k(x_{0:k})=Z_k \bar{\pi}_k(x_{0:k})$ is the unnormalized joint target. Using IS, it is thus straightforward to check that 
\begin{align}\label{eq:generalizedCrooksJarzynski}
Z_k &={\bar{\eta}_k[w_{k}]},\quad \bar{\pi}_k[f]=\frac{\bar{\eta}_k[w_{k}f]}{\bar{\eta}_k[w_{k}]},
\end{align}
where $f(x_{0:k})$ is a function of the whole trajectory $x_{0:k}$ and $\mu[g]$ is a shorthand notation for the expectation  $\mathbb{E}_{X\sim\mu}[g(X)]$. 
As $\pi_k$ is a marginal of $\bar{\pi}_k$, we can also estimate expectations w.r.t. to $\pi_k$ using $\bar{\pi}_k[f]=\pi_k[f]$ for $f(x_{0:k})=f(x_k)$. From \cref{eq:generalizedCrooksJarzynski}, it is thus possible to derive consistent estimators of $Z_k$ and $\pi_k[f]$ by sampling $N$ `particles' $X^i_{0:k}\sim \bar{\eta}_k$ where $i=1,...,N$ and using 
\begin{equation}\label{eq:SISestimates}
    Z_k^N =\frac{1}{N}\sum_{i=1}^N w_{k}(X^i_{0:k}),\quad \pi^N_k[f]=\sum_{i=1}^N W_k^i f(X_k^i),
\vspace{-.3cm}
\end{equation}
where $W_k^i\propto  w_{k}(X^i_{0:k})$, $\sum_{i=1}^N W_k^i=1$.

When the kernels $M_k$ are $\pi_k$-invariant and we select $L_{k-1}$ as the reversal of $M_k$, i.e. $\pi_k(x) M_k(x,x')=\pi_k(x') L_{k-1}(x',x)$, it is easy to check that $G_l(x_{l-1},x_l)=\gamma_{l}(x_{l-1})/\gamma_{l-1}(x_{l-1})$. In that case, \cref{eq:generalizedCrooksJarzynski} corresponds to AIS \citep{neal2001annealed} and is also known as the Jarzynski--Crooks identity \citep{jarzynski1997nonequilibrium,crooks1998nonequilibrium}. When $\pi_k$ is a sequence of posterior densities, a similar construction was also used in  \citep{maceachern1999sequential,gilks2001following,chopin2002sequential}. The generalized identity \cref{eq:generalizedCrooksJarzynski} allows the use of more general dynamics, including deterministic maps which will be exploited by our algorithm.

In practice, the choice of the backward transition kernels has a large impact on the variance of the estimates \eqref{eq:SISestimates}. \cite{Del-Moral:2006} identified the backward kernels minimizing the variance of the IS weights \eqref{eq:incrementalweightSMC}-\eqref{eq:unnormalized_importance_weights} and proposed various approximations to them.

\subsection{Sequential Monte Carlo}\label{sec:SMC}
To reduce the variance of the IS estimators \eqref{eq:SISestimates}, SMC samplers combine sequential IS steps with resampling steps. Given an IS approximation $\pi^N_{k-1}=\sum_{i=1}^N W_{k-1}^i \delta_{X_{k-1}^i}$ of $\pi_{k-1}$ at time $k-1$, one resamples $N$ times from $\pi^{N}_{k-1}$ to obtain particles approximately distributed according to $\pi_{k-1}$. This has for effect of discarding particles with low weights and replicating particles with high weights, this helps focusing subsequent computation on ``promising'' regions of the space. Empirically, resampling usually provides lower variance unbiased estimates of normalizing constants and is computationally very cheap; see e.g. \cite{chopin2002sequential,hukushima2003population,Del-Moral:2006,Rousset:2006,ZhouJohansen2016,barash2017gpu}. The resampled particles are then evolved according to $M_k$, weighted according to $G_k$ and resampled again.

\section{Annealed Flow Transport Monte Carlo}\label{sec:smc_normalizing_flow}
We now introduce AFT, a new flexible adaptive Monte Carlo method that leverages NFs.
Given the particle approximations $\pi_{k-1}^N := \sum_{i=1}^N W_{k-1}^{i} \delta_{X_{k-1}^{i}}$ and $Z_{k-1}^{N}$ at time $k-1$, AFT computes an approximation $\pi_k^{N}$ and $Z_{k}^{N}$ by performing four main sub-steps: \textit{Transport, Importance Sampling, Resampling and Mutation}, as summarized in \break \cref{alg:SMC}. Whenever the index $i$ is used in the algorithm, we mean `for all $i \in \{1,...,N\}$'. These four sub-steps are now detailed below. 

\begin{algorithm}
\caption{Annealed Flow Transport}\label{alg:SMC}
	\begin{algorithmic}[1]
		\STATE \textbf{Input:} number of particles $N$, unnormalized annealed targets $\{\gamma_k\}_{k=0}^K$ such that $\gamma_0=\pi_0$ and $\gamma_K=\gamma$, resampling threshold $A\in\left[1/N,1\right)$.
		\STATE \textbf{Ouput:} Approximations $\pi^N_K$ and $Z_K^N$ of $\pi$ and $Z$.
		\STATE Sample $X^i_0 \sim \pi_0$ and set $W_0^i = \frac{1}{N}$ and  $Z_0^N=1$.
		\FOR{$k=1,\dots, K$}
			\STATE Compute $\mathcal{L}^N_k(T)$ using \cref{eq:empirical_loss}. 
			\STATE Solve
				$T_k  \leftarrow \text{argmin}_{T \in \mathcal{T}}~\mathcal{L}^N_k(T)$ using e.g. SGD.
			\STATE Transport particles: $\widetilde{X}^{i}_{k}=T_k(X^{i}_{k-1})$.
			\STATE Estimate normalizing constant $Z_k$: \\$Z_k^N \leftarrow Z_{k-1}^N \parens{\sum_{i=1}^N W_{k-1}^{i}G_{k,T_k}(X_{k-1}^{i})}$.
			\STATE Compute IS weights: \\$w^{i}_k \leftarrow W_{k-1}^i \G_{k,T_k}(X_{k-1}^i)$ // unnormalized\\
				$W^{i}_k \leftarrow  \frac{w_{k}^{i}}{\sum_{j=1}^N w_{k}^{j}}$ // normalized
			\STATE Compute effective sample size $\textup{ESS}_k^N$  using \cref{eq:ESS}.
			\IF{$\textup{ESS}^{N}_{k}/N \leq A$} 
				\STATE Resample $N$ particles denoted abusively also $\widetilde{X}^i_k$ according to the weights $W_k^i$, then set $W_k^i = \frac{1}{N}$. 
			\ENDIF
			\STATE  Sample $X_{k}^i \sim K_{k}(\widetilde{X}_{k}^i,\cdot)$. // MCMC
		\ENDFOR
	\end{algorithmic}
\end{algorithm}
\subsection{Transport map estimation}
In this step, we learn a NF $T_{k}$ that moves each sample $X_{k-1}$ from $\pi_{k-1}$ to a sample $\widetilde{X}_k = T_k(X_{k-1})$ as close as possible to $\pi_k$ by minimizing an estimate of $\textup{KL}(T_{\#}\pi_{k-1}||\pi_k)$ over a set  $\mathcal{T}$ of NFs. This $\textrm{KL}$ can be decomposed as a sum of a loss term $\mathcal{L}_k(T)$ and a term $ \log\frac{Z_{k}}{Z_{k-1}}$ that can be ignored as it is independent of the NF $T$.
A simple change of variables allows us to express the loss term $\mathcal{L}_k(T)$ as an expectation under $\pi_{k-1}$ of some tractable function $x\mapsto h_T(x)$:
\begin{align}\label{eq:kl_loss_jacobian}
\begin{aligned}
	\mathcal{L}_k(T):=&  \pi_{k-1}\brakets{ h_T},  \\
	h_T(x) :=&  V_{k}\parens{T(x)} -V_{k-1}\parens{x} - \log \verts{ \nabla T(x) }.
	\end{aligned}
\end{align}
The Jacobian determinant of $T$ in \cref{eq:kl_loss_jacobian} can be evaluated efficiently for NFs while the expectation under $\pi_{k-1}$ can be estimated using  $\pi_{k-1}^N$ thus yielding the empirical loss:
\begin{align}\label{eq:empirical_loss}
	\textstyle{\mathcal{L}^N_k(T) :=  \sum_{i=1}^N W_{k-1}^i h_T(X_{k-1}^i).} 
\end{align}
In practice, \cref{eq:empirical_loss} is optimized over the NF parameters using gradient descent.
The resulting NF $T_k$ is then used to transport each particle $X_{k-1}^{i}$ to $\widetilde{X}_{k}^{i}= T_k(X_{k-1}^{i})$\footnote{We should write $T^N_k$ to indicate the dependence of our estimate of $N$ but do not to simplify notation.}.
However, the loss \cref{eq:empirical_loss} being not necessarily convex, the solution $T_k$ is likely to be sub-optimal. This is not an issue, since 
IS is used to correct for such approximation error as we will see next. We also emphasize that the convergence results for this scheme presented in \cref{sec:Asymptoticanalysis} do not require finding a global minimizer of this non-convex optimization problem.
\subsection{Importance Sampling, Resampling and Mutation}
\paragraph{Importance Sampling.} This step corrects for the NF $T_k$ being only an approximate transport between $\pi_{k-1}$ and $\pi_k$. In this case, we have $M^{\textup{trans}}_{k}(x,x')=\delta_{T_k(x)}(x')$ and by selecting $L^{\textup{trans}}_{k-1}(x,x')=\delta_{T^{-1}_k(x')}(x)$ then the incremental IS weight \cref{eq:incrementalweightSMC} is given by a simple change-of-variables formula
\begin{align}\label{eq:discrete_importance_weight}
 	\G_{k,T_k}(x_{k-1}) =\frac{\gamma_{k}(T_k(x_{k-1}))\verts{\nabla T_k(x_{k-1}) }}{\gamma_{k-1}(x_{k-1})}.
\end{align}
 Using \cref{eq:discrete_importance_weight}, we can update the weights  $w_{k}^{i} =W_{k-1}^{i} \G_{k,T_k}(X_{k-1}^i)$ to account for the errors introduced by $T_k$.
When $T_k$ are exact transport maps from $\pi_{k-1}$ to $\pi_k$, 
the incremental weight in \cref{eq:discrete_importance_weight} becomes constant and equal to the ratio $Z_k/Z_{k-1}$.
Thus, introducing the NF $T_k$ can be seen as a way to reduce the variance of the IS weights in the SMC sampler.
\paragraph{Resampling.} As discussed in \cref{sec:SMC}, resampling can be very beneficial but it should only be performed 
when the variance of the IS weights is too high \cite{liuchen1995} as measured by the Effective Sample Size (ESS)
\begin{equation}\label{eq:ESS}
	\textup{ESS}^{N}_{k} = \left(\sum_{i=1}^N \parens{ W_k^{i}  }^2\right)^{-1},
\end{equation}
which is such that $\textup{ESS}^{N}_{k}\in\left[1,N\right]$. When $\textup{ESS}^{N}_{k}/N$ is smaller than some prescribed threshold $A\in\left[1/N,1\right)$ (we use $A=0.3$), resampling is triggered and each particle $\widetilde{X}_k^i$ is then resampled without replacement from the set of $N$ available particles $\{\widetilde{X}_k^i\}_{i\in[1:N]}$ according to a multinomial distribution with weights $\{W_{k}^i\}_{i\in[1:N]}$.
The weights are then reset to uniform ones; i.e. $W_{k}^i = \frac{1}{N}$. More sophisticated lower variance resampling schemes have also been proposed; see e.g. \cite{kitagawa1996monte,chopin2004central}. 
\paragraph{Mutation.} The final step consists in mutating the particles using a $\pi_{k}-$invariant MCMC kernel $K_{k}$ , i.e. using $X_k^{i}\sim K_{k}(\widetilde{X}_{k}^i,\cdot)$. This allows particles to better explore the space. 

Note that if the transport maps $T_k$ were known, \cref{alg:SMC} could be reinterpreted as a specific instance of a SMC as detailed in Section \ref{sec:SMCsamplers} where at each time $k\geq 1$ we perform two time steps of a standard SMC sampler by applying first a transport step $M^{\textup{trans}}_{k}(x,x')=\delta_{T_k(x)}(x')$ then a mutation step $M^{\textup{mut}}_{k}(x,x')=K_k(x,x')$; see Appendix \ref{subsec:noncollapse} for details.
\setlength{\textfloatsep}{0.3cm}
\subsection{Variants and Extensions}\label{sec:extensions}
Contrary to standard SMC, the estimates $Z_k^N$ returned by \cref{alg:SMC} are biased because of the dependence of the NF $T_k$ on the particles. To obtain unbiased estimates of $Z_k$ and to avoid over-fitting of the NF to the $N$ particles, a variant of \cref{alg:SMC} described in \cref{alg:SMC-detailed} (see  \cref{sec:variant_algo}) is used in the experimental evaluation. This variant employs three sets of particles: the \emph{training set} is used to evaluate the loss \cref{eq:empirical_loss}, the \emph{validation set} is used in a stopping criterion when learning the NF and the \emph{test set} is independent from the rest and is computed sequentially using the learned NFs.
It would also be possible to combine AFT with various extensions to SMC that were already proposed in the literature.
For example, we can select adaptively the annealing parameters $\beta_k$ to ensure the ESS only decreases by a pre-determined percentage \citep{jasra2011inference,chopinschafer2013sequential,Beskos:2016,ZhouJohansen2016} or use the approximation of $\pi_k$ obtained at step 13 of \cref{alg:SMC} to determine the parameters of the MCMC kernel $K_k$ \cite{delmoral2012adaptive,buchholz2020adaptive}.

\section{Asymptotic analysis}\label{sec:Asymptoticanalysis}
We establish here a law of large numbers and a CLT for the particle estimates $\pi_{k}^{N}[f]$ and $Z_k^N$ of $\pi_k[f]$ and $Z_k$. We denote by $\xrightarrow{P}$ convergence in probability and by  $\xrightarrow{\mathcal{D}}$ convergence in distribution. 
\subsection{Weak law of large numbers}
\cref{thm:WLLN} shows that $\pi_{k}^{N}[f]$ and $Z_k^N$  are consistent estimators of $\pi_k[f]$ and $Z_k$, hence of $\pi[f]$ and $Z$ at time $k=K$.
\begin{thm}[weak law of large numbers]\label{thm:WLLN}
Let $f$ be a function s.t. $ \vert f(x)\vert\leq C (1+ \Vert x \Vert^4) $ for all $x\in \X$ and for some $C>0$. Under \cref{assumption:kernel,assumption:transport,assumption:moments,assumption:uniformly_bounded} and for any $k\in{0,...,K}$:
\begin{equation}\label{eq:recursion_consistency}
   (\mathcal{R}_{k}):\qquad \pi_k^{N}[f]\xrightarrow{P} \pi_k[f],\quad Z_k^N\xrightarrow{P} Z_k.
\end{equation}
\end{thm}
The result is proven in \cref{sec:proof:wlln} and relies on four assumptions stated in \cref{sec:assumptions_clt}: \ref{assumption:kernel} on the smoothness of the Markov kernels $K_k$, \ref{assumption:moments} on the moments of $\pi_{k}$, \ref{assumption:transport} on the smoothness of the family of NFs and \ref{assumption:uniformly_bounded} on the boundedness of the incremental IS weight $G_{k,T}(x)$. 
Perhaps surprisingly, \cref{thm:WLLN} does not require the NFs to converge as $N\rightarrow \infty$. This is a consequence of \cref{prop:general_Glivenko_cantelli} in \cref{sec:proof:wlln} which ensures uniform consistency of the particle approximation regardless of the choice of the NFs. However, convergence of the NFs is required to obtain a CLT result as we see next.
\cref{thm:WLLN_2} of \cref{sec:proof:wlln} states a similar result for \cref{alg:SMC-detailed} of \cref{sec:variant_algo}.
\subsection{Central Limit theorem}
Besides assumptions \ref{assumption:kernel} to \ref{assumption:uniformly_bounded}, we make five assumptions stated in \cref{sec:assumptions_clt}: \ref{assumption:kernel_2} on the Markov kernels $K_k$ strengthens \ref{assumption:kernel} and is satisfied by many commonly used Markov kernels as shown in \ref{sec:kernel_assump}.  The smoothness assumptions \ref{assumption:parametrization} and \ref{assumption:potential} on the family $\mathcal{T}$ of NFs and potentials $V_k$ are also standard. Finally, \ref{assumption:approximate_minimizer} and \ref{assumption:bassin} describe the asymptotic behavior of $T_k$. We do not require $T_k$ to be a global minimizer of the loss $\mathcal{L}^N_k$, neither do we assume it to be an exact local minimum of $\mathcal{L}^N_k$. 
Instead, \ref{assumption:approximate_minimizer} only needs $T_k$ to be an approximate local minimum of $\mathcal{L}^N_k$ and \ref{assumption:bassin} implies that $T_k$ converges in probability towards a \emph{strict local} minimizer $T_{k}^{\star}$ of $\mathcal{L}_k$ as $N \rightarrow \infty$.

Before stating the CLT result, we need to introduce the asymptotic incremental variance $\mathbb{V}_k^{\textup{inc}}[f]$ at iteration $k$. To this end, consider the set of limiting re-sampling times $ \mathcal{K}_{\textup{opt}} := \{ k_{0},...k_{P}\}\subset \{0,...,K\}$  defined recursively by $k_{p+1}  := \inf\{ k_p <k: \textup{nESS}_{k}  \leq A \} $ and $k_{P+1}:=K+1$ where $\textup{ESS}_{k}^N/N \overset{N\rightarrow \infty}{\rightarrow} \textup{nESS}_{k}$ with
\begin{align}
    \textup{nESS}_{k} = \frac{\pi_{k_p}\brakets{ \mathbb{E}
		\brakets{  w^{\star}_{k}
			 \middle| X_{k_p}
			}   
		}^2}{\pi_{k_p}\brakets{ \mathbb{E}
		\brakets{  \parens{w^{\star}_{k}}^2 \middle| X_{k_p}
			}   
		}},
\end{align} 
the expectation being w.r.t. to $X_{s}\sim K_s(T^{\star}_s(X_{s-1}),\cdot)$ for $ k_{p}+1 \leq s \leq k $, while $X_{k_p}\sim \pi_{k_p}$  and $w^{\star}_{k} = \prod_{s=k_p+1}^k G_{s,T^{\star}_s}\parens{X_{s-1}}$
is the product of the incremental IS weights using the locally optimal NFs $T^{\star}_{s}$. 
The variance $\mathbb{V}_k^{\textup{inc}}[f]$ at time $k$ is given by:
\begin{align}
    \mathbb{V}_k^{\textup{inc}}[f] =
    \begin{cases}
        Z_k^2 \textup{Var}_{\pi_k}[f],&\hspace{-0.30cm}k\in \mathcal{K},\\
        Z_{k_p}^2\pi_{k_p}\brakets{ \mathbb{E}
		\brakets{\parens{w^{\star}_{k}}^2 \mathcal{G}_{k}\brakets{f} \middle| X_{k_p}}},
		&\hspace{-0.30cm}k_p<k<k_{p+1},
    \end{cases}
\end{align}
with $\mathcal{G}_{k}\brakets{f} := K_k\brakets{f^2}\parens{ T_k^{\star}(X_{k-1}) }  - K_k\brakets{f}^2\parens{ T_k^{\star}(X_{k-1}) }$.
\begin{thm}[Central limit theorem]\label{thm:CLT}
Let $f$ be a real valued function s.t., for some $C>0$,  $f(x)\leq C(1 + \Vert x \Vert^2)$ and 
\begin{align}
	\Verts{f(x)-f(x')}\leq C\parens{1 +  \Verts{x}^{3} + \Verts{x'}^{3}    } \Verts{x-x'}.
\end{align}
Then, under \cref{assumption:kernel,assumption:potential,assumption:kernel_2,assumption:potential,assumption:transport,assumption:moments,assumption:parametrization,assumption:uniformly_bounded,assumption:approximate_minimizer,assumption:bassin} and for $0\leq k \leq K$:
\begin{align}\label{eq:CLT}
	(CLT_k):\quad 
\begin{cases}
	\sqrt{N}\parens{\pi_k^N[f] - \pi_k[f]}&\xrightarrow{\mathcal{D}} \mathcal{N}(0,  \mathbb{V}^{\pi}_k[f]),\\
	\sqrt{N}\parens{Z_k^N - Z_k}&\xrightarrow{\mathcal{D}} \mathcal{N}(0,  \mathbb{V}^{\gamma}_k[1]).
	\end{cases}
\end{align}
$\mathbb{V}^{\gamma}_k[f]$ and $\mathbb{V}^{\pi}_k[f]$ are defined recursively with  $\mathbb{V}^{\gamma}_0[f]=\textup{Var}_{\pi_0}[f]$ and
\begin{align}\label{eq:assymptotic variance}
	 \mathbb{V}^{\gamma}_k\brakets{f} &=  \mathbb{V}_k^{\textup{inc}}\brakets{f} + \mathbb{V}^{\gamma}_{k-1}\brakets{Q_{k,T_k^{\star}}[f]},\\
	 \mathbb{V}^{\pi}_k \brakets{f} &=Z_k^{-2}\mathbb{V}^{\gamma}_k\brakets{f -\pi_{k}[f]},
\end{align}
where $Q_{k,T}(x,\diff y) := G_{k,T}(x)K_k(T(x),\diff y)$.
\end{thm}
The asymptotic variances $\mathbb{V}^{\gamma}_k$ and $\mathbb{V}^{\pi}_k$ depend only on the  maps $T_k^{*}$ and not on the local variations of the family $\mathcal{T}$ around $T_k^*$. This is a consequence of the particular form of the IS weights which provide an exact correction regardless of the NF selected as summarized by the following identity:
\begin{align}\label{eq:IS_relation}
	\pi_k\brakets{f}= \frac{\pi_{k-1}\brakets{Q_{k,T}\brakets{f}}}{\pi_{k-1}\brakets{G_{k,T}}},\qquad \forall T\in \mathcal{T}.
\end{align}
In the ideal case when $T^{\star}_k$ are exact transport maps from $\pi_{k-1}$ to $\pi_k$, the ESS resampling criterion $\textup{ESS}_k^N/N$ is always equal to $1$ and thus resampling is never triggered. Moreover, a direct computation shows that the asymptotic variance $\mathbb{V}_{k}^{\pi}[f]$ is exactly equal to $\textup{Var}_{\pi_k}[f]$. This illustrates the benefit of introducing NFs to improve SMC. A proof is provided in \cref{proof:clt} along with a similar result (\cref{thm:CLT_2}) for \cref{alg:SMC-detailed}.
\section{Continuous-time scaling limit}\label{sec:scalinglimit}

We consider the setting where $\pi_k$ arise from the time-discretization of a continuous-time path $(\Pi_t)_{[0,1]}$ of densities connecting $\pi_0$ to $\pi$; i.e. $\pi_k$ is of the form $\pi_k = \Pi_{t_k}$ with  $t_k= k\lambda $ and $\lambda = \frac{1}{K}$. We write $V_t(x)$ and $Z_t$ to denote the potential and unknown normalizing constant of $\Pi_t$
 and $\Gamma_t(x) = \exp(-V_t(x))$. We are here interested in identifying the ``population'' behavior of AFT (i.e. $N\rightarrow \infty$) as $\lambda \rightarrow 0$ when ULA kernels are used and no resampling is performed as in AIS. To simplify the analysis, we further consider in this Section the ideal situation where $T_k$ is an exact minimizer of the population loss $\mathcal{L}_k$.
 Rigorous proofs of the results discussed here can be found in \cref{sec:proofs_continuous_time}.
\subsection{Settings}\label{sec:continuous_time_setting}
Without resampling, the population version of AFT behaves as a sequential IS algorithm as defined in \cref{sec:SIS} where it is possible to \textit{collapse} the transport step and mutation step into one single Markov kernel $M_k(x,x') =K_k(T_k(x),x')$.
Similarly we can collapse the corresponding backward kernels and the resulting extended target distributions $\bar{\pi}_k$ are still given by \eqref{eq:generalizedCrooksJarzynski} with modified IS weights 
\begin{align}\label{eq:extendedtargetcollapse}
w_{k}(x_{0:k})=\underbrace{\prod_{l=1}^k  \frac{\gamma_l(x_l)} {\gamma_l K_l(x_l)}}_\text{$r_k(x_{1:k})$} \prod_{l=1}^k G_{l,T_l}(x_{l-1}),
\end{align}
where $r_k(x_{1:k})=1$ for $\pi_l$-invariant MCMC kernels $K_l$ as used in \cref{alg:SMC}; see \cref{subsec:collapse} for a derivation.
To ensure that the laws $\bar{\eta}_k$ and $\bar{\pi}_k$ of the Markov chain $X_{0:k}$ converge to some continuous-time limits, $K_k$ are chosen to be ULA kernels\footnote{The random walk MH algorithm also admits a Langevin diffusion as scaling limit when $\lambda\rightarrow 0$ \citep{Gelfand:1991,Choi:2019} but the technical analysis is much more involved.}; i.e. $K_{k}(x, x') $ is a Gaussian density in $x'$ with mean $x -\lambda V_{k}(x)$ and covariance $2\lambda I$. 
In this case, $\gamma_{k}K_k(x)=  \int \gamma_{k}(y)K_{k}(y,x) \diff y$ is  intractable and so is $r_k(x_{1:k})$. This is not an issue as we are only interested here in identifying the theoretical scaling limit. 
To ensure $\bar{\eta}_k$ and $\bar{\pi}_k$ admit a limit, we also consider NFs of the form: 
\begin{align}\label{eq:main_nf_continuous_time}
	T(x) = x + \lambda A_{\theta}(x),
\end{align}
where $(\theta,x)\mapsto A_{\theta}(x)$  is from $\Theta\times \mathcal{X}$ to $\mathcal{X}$ and $\Theta$  is a compact parameter space.
The continuous-time analogues of NFs sequences $(T_k)_{k\in\{1,...,K\}}$ are represented by a set $\mathcal{A}$ of time-dependent controls of the form $\alpha_t(x)=A_{\theta_t}(x)$, where $t\mapsto\theta_t$ is a 1-Lipschitz trajectory in $\Theta$. To any control $\alpha$ corresponds an NFs sequence $(T_k)_{k\in\{1,...,K\}}$ defined by $T_k(x) = x + \lambda \alpha_{t_k}(x) $.
\subsection{Continuous-time limits}\label{sec:scaling_limit_feynamn_kac}
\paragraph{Limiting forward process.}
Using a similar approach to \cite{Dalalyan:2014}, 
the Markov chain $X_{0:K}$ under $\bar{\eta}_K$ converges towards a stochastic process $X_{[0,1]}$ defined by the following Stochastic Differential Equation (SDE)
\begin{align}\label{eq:controled_sto}
	\diff X_t = \parens{\alpha_{t}(X_t)- \nabla V_t(X_t) }\diff t + \sqrt{2}\diff B_t,
\end{align}
where $X_0 \sim \pi_0$ and $(B_t)_{t\geq0}$ is a standard Brownian motion. 
We denote by $\bar{\Lambda}^{\alpha}_t$ the joint distribution of this process up to time $t$ and by $\Lambda^{\alpha}_t$ its marginal at time $t$.
\paragraph{Limiting weights.}
The weight $w_K(X_{0:K})$ in \eqref{eq:extendedtargetcollapse} is such that $r_K(X_{1:K})\rightarrow 1$ as the invariant distribution of the ULA kernel $K_k$ converges to $\pi_k$ when $\lambda \rightarrow 0$ while the logarithm of the product of  $G_{l,T_l}(X_{l-1})$ is a Riemann sum whose limiting value is the following integral:
\begin{align}
\sum_{l=1}^K \log(G_{l,T_l}(X_{l-1})) \xrightarrow[\lambda \rightarrow 0 ]{}  \int_0^1 g_{s}^{\alpha}(X_s)\diff s,
\end{align}
with $X_{[0,1]}$ defined in \cref{eq:controled_sto} and $g^{\alpha}_{t}(x)$ being the dominating term in the Taylor expansion of $\log(G_{l,T_l}(x))$ w.r.t. time:
\begin{align}
	g^{\alpha}_{t}(x) =  \nabla \cdot \alpha_t(x) - \nabla_x V_t(x)^{\top}\alpha_t(x) - \partial_t V_t(x).
\end{align}
 The limit of IS weights $w_k(X_{0:k})$ is thus identified as
\begin{align}\label{eq:continuous_time_IS}
	w^{\alpha}_{t}(X_{\left [0,t\right]}) = \exp\parens{ \int_0^t g^{\alpha}_{s}(X_s)\diff s }.
	\end{align}
In the context of \textit{non-equilibrium dynamics}, $g^{\alpha}_{t}(x)$ is known as \textit{instantaneous work} \cite{Rousset:2006} and is constant in the ideal case where $\Pi_t = \Lambda^{\alpha}_t $.
\paragraph{Limiting objective.}
To identify a non-trivial limiting loss, we consider the following aggregation of all $\mathcal{L}_k(T_k)$ 
\begin{align}\label{eq:discrete_loss}
	\mathcal{L}^{tot}_{\lambda}(\alpha) := \lambda^{-1} \sum_{k=1}^{K} \mathcal{L}_{k}(T_k). 
\end{align}
The next result shows that \cref{eq:discrete_loss} converges towards a non-trivial loss $\mathcal{M}(\alpha)$ as $\lambda \rightarrow 0$ under three assumptions stated in \cref{sec:assumptions_continuous}: \ref{assump_cont:lipschitz_potential} and \ref{assump_cont:Lipschiz_control} on the smoothness of $V_t(x)$ and $A_{\theta}(x)$ and \ref{assump_cont:moment} on the moments of $\Pi_t$.
\begin{prop}\label{prop:convergence_loss_main}
Under \cref{assump_cont:lipschitz_potential,assump_cont:Lipschiz_control,assump_cont:moment}, for $\lambda$ small enough, it holds that for all $\alpha\in \mathcal{A}$
	\begin{align}
		\verts{ \mathcal{L}^{tot}_{\lambda}(\alpha) -  \mathcal{M}(\alpha)  }\leq \lambda C,
		\end{align}
		where $C$ is independent of $\lambda$ and
	\begin{align}\label{eq:variance_work}
	\mathcal{M}(\alpha) = \frac{1}{2}\int_0^1 \parens{\Pi_{t}\brakets{\parens{g_{t}^{\alpha}}^2} - \Pi_{t}\brakets{g_{t}^{\alpha}}^2}\diff t.
\end{align}
\end{prop}
 The optimal NFs $(T_k)_{1:K}$ are thus expected to converge towards some $\alpha^{\star}$ minimizing $\mathcal{M}(\alpha)$ over $\mathcal{A}$ as made precise in \cref{prop:convergence_controls} of \cref{sec:convergence_optimal_control}. 
Moreover when the class of NFs is expressive, i.e. $\mathcal{A}$ is rich enough, then $\mathcal{M}(\alpha^{\star})=0$ and thus $g^{\alpha}_t$ are constant and $\alpha^{\star}$ satisfies the Partial Differential Equation (PDE)
\begin{align}
	0 = \nabla \cdot \alpha^{\star}_t(x) - \nabla_x V_t(x)^{\top}\alpha^{\star}_t(x) - \partial_t V_t(x)+\Pi_t\brakets{\partial_t V_t}.
	\end{align} 
This PDE has appeared, among others, in \citet[pp. 273--275]{lelievreroussetstoltz2010free} and \citep{vaikuntanathan2008escorted,reich2011,heng2015gibbs}. Its solution defines a deterministic flow $\alpha^{\star}_t$ that transports mass along the path $(\Pi^{\alpha}_t)_{[0,1]}$; i.e. if $X_t$ is a solution to an ODE of the form $\dot{X_t} = \alpha^{\star}_t(X_t)$ with initial values $X_0 \sim \Pi_0$, then $X_t \sim \Pi_t$. 
\paragraph{Feynman--Kac measure.} Given a control $\alpha$, we consider the Feynman--Kac measure $\overline{\Pi}_t$ defined for any bounded continuous functional $f$ of the process $X_{[0,t]}$ in \cref{eq:controled_sto}
\begin{align}\label{eq:feynman-kac}
	\overline{\Pi}^{\alpha}_t[f] = \frac{\overline{\Lambda}^{\alpha}_t\brakets{w^{\alpha}_{t} f}}{\overline{\Lambda}^{\alpha}_t\brakets{w^{\alpha}_{t}}}.
\end{align}
By a similar argument as in \cite{Rousset:2006}, we show in \cref{prop:controlled_non_equilibrium_dynamics_with_birth_death} of \cref{sec:continuoustimeIS} that $\overline{\Pi}^{\alpha}_t$ admits $\Pi_t$ as a marginal at time $t$ regardless of the choice of  $\alpha$.  Using the optimal control $\alpha^{\star}$ in \cref{eq:controled_sto,eq:feynman-kac} gives rise to $\overline{\Lambda}^{\star}$ and $\overline{\Pi}^{\star}_t$ which are equal when $\mathcal{M}(\alpha^{\star})=0$. Next, we show that $\overline{\Pi}^{\star}_t$ is the scaling limit of $\overline{\pi}_k$.

\subsection{Convergence to the continuous-time limit}\label{sec:continuous_time_convergence}
As the measures $\overline{\pi}_k$ and $\overline{\Pi}^{\star}_t$ are defined on different spaces, we construct a sequence of interpolating measures $\bar{\Pi}^{\lambda}_t$ defined over the same space as $\bar{\Pi}^{\star}_t$ and whose marginal at the joint times $\{t_0,...,t_K\}$ is exactly equal to $\bar{\pi}_k$; see \cref{sec:notation} for details. 
\cref{prop:discretization_error_main} provides a convergence rate for the interpolating measures $\bar{\Pi}^{\lambda}_t$ towards $\overline{\Pi}^{\star}_t$ as $\lambda \rightarrow 0$, thus establishing $\overline{\Pi}^{\star}_t$ as the scaling limit of $\overline{\pi}_k$; see  \cref{sec:convergence_optimal_control} for the proof.
\begin{thm}\label{prop:discretization_error_main}
Under \cref{assump_cont:lipschitz_potential,assump_cont:Lipschiz_control,assump_cont:moment,assump_cont:moment_2,assump_cont:separated_optimum,assump_cont:existence_interpolating_controls,assum_cont:bounded_weights}, then for $\lambda$ small enough there exists a finite $C$ such that for any $t\in [0,1]$:\begin{align}
	  \textup{KL}(\bar{\Pi}^{\star}_t||\bar{\Pi}^{\lambda}_t)  \leq C\sqrt{\lambda}.
\end{align}
\end{thm}
This result relies on \cref{assump_cont:moment_2,assump_cont:separated_optimum,assump_cont:existence_interpolating_controls,assum_cont:bounded_weights} in addition to \cref{assump_cont:lipschitz_potential,assump_cont:Lipschiz_control,assump_cont:moment} which are also stated in \cref{sec:assumptions_continuous}. \ref{assump_cont:moment_2} strengthens assumption \ref{assump_cont:moment} on the moments of $\Pi_t$. \ref{assump_cont:separated_optimum} guarantees the existence of a solution $\alpha^{\star}$ in $\mathcal{A}$ minimizing $\mathcal{M}$ and controls the local behavior of $\mathcal{M}$ near $\alpha^{\star}$. \ref{assump_cont:existence_interpolating_controls} guarantees the existence of solutions  $\alpha^{\lambda}$ in $\mathcal{A}$ minimizing $\mathcal{L}^{tot}_{\lambda}(\alpha)$ for any $\lambda = \frac{1}{K}$. Finally, \ref{assum_cont:bounded_weights} ensures the optimal control $\alpha^{\star}$ induces bounded IS weights.
\section{Applications}\label{sec:appli}
\input{applications}
\section{Conclusion}
We proposed Annealed Flow Transport which combines SMC samplers and normalizing flows. We studied its asymptotic behavior and showed the benefit of introducing learned flows to reduce the asymptotic variance. We identified the scaling limit of AFT as a controlled Feynman--Kac measure whose optimal control solved a flow transport problem in an idealized setting.
Empirically we found multiple cases where trained AFT gave lower variance estimates than SMC for the same number of transitions, showing that we can combine the advantages of both SMC and normalizing flows. We believe AFT will be particularly useful in scenarios where it is both difficult to design fast mixing MCMC kernels and very good flows so that neither SMC nor VI provide low variance estimates.

\section{Acknowledgements}

The authors would like to thank Danilo Rezende, Arthur Gretton and Taylan Cemgil.

\bibliographystyle{apalike}
\bibliography{references,biblio}
\clearpage
\onecolumn
\appendix
\section{Using measure-theoretic notation}\label{sec:measuretheoreticpresentation}
 The Markov transition kernel $M_k$ is defined as a map $M_{k}:\X \times \mathcal{B}(\X)\rightarrow [0,1]$ where $\mathcal{B}(\X)$ are the Borel sets, $L_{k-1}$ is defined similarly. The joint distribution of the non-homogeneous Markov chain of initial distribution $X_0\sim \pi_0$ and Markov transition kernel $X_l\sim M_l(X_{l-1},\cdot)$ at time $l$ is given at time $k$ by
\begin{align}\label{eq:proposaljoint}
	\bar{\eta}_{k}(\diff x_{0:k}) = \pi_0(\diff x_0)\prod_{l=1}^{k} M_l(x_{l-1},\diff x_l).
\end{align}
SMC samplers rely on the following target distribution of the form
\begin{align}\label{eq:targetjoint}
\bar{\pi}_k(\diff x_{0:k}) = \pi_{k}(\diff x_k)	\prod_{l=0}^{k-1} L_{l}(x_{l+1},\diff x_{l}),
\end{align}
and $\gamma_k(\diff x_{0:k})=Z_k \bar{\pi}_k(\diff x_{0:k})$. When $\pi_l(\diff x') L_{l-1}(x',\diff x)$ is absolutely continuous w.r.t. $\pi_l(\diff x) M_l(x,\diff x')$, then we can define its Radon-Nikodym derivative and the incremental importance weight $G_l$ through
\begin{align}\label{eq:incrementalRN}
G_l(x_{l-1},x_l)=\frac{Z_l}{Z_{l-1}}\frac{\pi_l(\diff x_l) L_{l-1}(x_l,\diff x_{l-1})}{ \pi_{l-1}(\diff x_{l-1}) M_l(x_{l-1},\diff x_l)}=\frac{\gamma_l(\diff x_l) L_{l-1}(x_l,\diff x_{l-1})}{ \gamma_{l-1}(\diff x_{l-1}) M_l(x_{l-1},\diff x_l)}.
\end{align}
If $G_l$ is defined for $l=1,...,k$, then $\bar{\pi}_k$ is absolutely continuous w.r.t. $\bar{\eta}_{k}$ so we can write
\begin{align}\label{eq:ideal_importance_weights_RN}
	\quad \bar{\pi}_k[f]=\frac{\bar{\eta}_k[w_{k}f]}{\bar{\eta}_k[w_{k}]},\quad\text{where}\quad w_{k}(x_{0:k}) = Z_k\frac{\diff \bar{\pi}_k}{\diff \bar{\eta}_k}(x_{0:k})= \frac{\diff \bar{\gamma}_k}{\diff \bar{\eta}_k}(x_{0:k}) = \prod_{l=1}^k G_l(x_{l-1},x_l).
\end{align}
If $M_k$ is $\pi_k$-invariant then \citep{crooks1998nonequilibrium,neal2001annealed} select $ L_{k-1}$ at the reversal of $M_k$, that is the kernel satisfying $\pi_k(\diff x) M_k(x,\diff x')=\pi_k(\diff x') L_{k-1}(x',\diff x)$ and in this case
\begin{equation}
G_l(x_{l-1},x_l)=\frac{Z_l}{Z_{l-1}} \frac{\diff \pi_l}{\diff \pi_{l-1}}(x_{l-1})=\frac{\diff \gamma_l}{\diff \gamma_{l-1}}(x_{l-1}).
\end{equation}
\section{Extended proposal and target of AFT algorithm}\label{sec:extendedtargetproposalNFSMC}
In this section, assuming the transport maps $T_k$ are here fixed, we write explicitly the extended proposal and target distributions used implicitly by the AFT algorithm if no resampling was used. 
\subsection{Non-collapsed version}\label{subsec:noncollapse}
In this case, we sample $X_0\sim \pi_0(\cdot)$ at $k=0$ then use $\tilde{X}_k=T_k(X_{k-1})$ followed by $X_k=K_k(\tilde{X}_k,\cdot)$ at time $k\geq 1$. Hence, using the notation  $M^{\textup{trans}}_{l}(x,x')=\delta_{T_l(x)}(x')$ and $M^{\textup{mut}}_{l}(x,x')=K_l(x,x')$, the proposal at time $k$ after the transport step is of the form
\begin{equation}
\bar{\eta}_k(x_{0:k-1},\tilde{x}_{1:k})=\pi_0(x_0) \left(\prod_{l=1}^{k-1} M^{\textup{trans}}_{l}(x_{l-1},\tilde{x}_l) M^{\textup{mut}}_{l}(\tilde{x}_l,x_l)\right)  M^{\textup{trans}}_{k}(x_{k-1},\tilde{x}_k) ,
\end{equation}
and the target is
\begin{equation}
\bar{\pi}_k(x_{0:k-1},\tilde{x}_{1:k})=\pi_k(\tilde{x}_k) L^{\textup{trans}}_{k-1}(\tilde{x}_k,x_{k-1})\left(\prod_{l=0}^{k-2} L^{\textup{mut}}_{l}(x_{l+1},\tilde{x}_{l+1}) L^{\textup{trans}}_{l}(\tilde{x}_{l+1},x_l)\right),
\end{equation}
where $L^{\textup{trans}}_{l-1}(x,x')=\delta_{T^{-1}_l(x)}(x')$ and $L^{\textup{mul}}_{l-1}(x,x')=\pi_l(x')M^{\textup{mul}}_l(x',x)/\pi_l(x)$.
After the mutation step at time $k$, the proposal is
\begin{equation}
\bar{\eta}_k(x_{0:k},\tilde{x}_{1:k})=\bar{\eta}_k(x_{0:k-1},\tilde{x}_{1:k}) M^{\textup{mut}}_{k}(\tilde{x}_k,x_k)= \pi_0(x_0)\left(\prod_{l=1}^k M^{\textup{trans}}_{l}(x_{l-1},\tilde{x}_l) M^{\textup{mut}}_{l}(\tilde{x}_l,x_l)\right).
\end{equation}
and the target is 
\begin{equation}
\bar{\pi}_k(x_{0:k},\tilde{x}_{1:k})=\pi_k(x_k)\left(\prod_{l=0}^{k-1} L^{\textup{mut}}_{l}(x_{l+1},\tilde{x}_{l+1}) L^{\textup{trans}}_{l}(\tilde{x}_{l+1},x_l)\right).
\end{equation}
Hence the incremental weight after a transport term at time $k$ is of the form 
\begin{equation}
\frac{\bar{\pi}_k(x_{0:k-1},\tilde{x}_{1:k})}{\bar{\eta}_k(x_{0:k-1},\tilde{x}_{1:k})}=\frac{\bar{\pi}_{k-1}(x_{0:k-1},\tilde{x}_{1:k-1})}{\bar{\eta}_{k-1}(x_{0:k-1},\tilde{x}_{1:k-1})}\underbrace{\frac{\pi_k(\tilde{x}_k)L^{\textup{trans}}_{k-1}(\tilde{x}_k,x_{k-1})}{\pi_{k-1}(\tilde{x}_{k-1})M^{\textup{trans}}_{k}(x_{k-1},\tilde{x}_k)}}_{\textup{incremental weight}=\frac{Z_{k-1}}{Z_k}G_{k,T_k}(x_{k-1})},
\end{equation}
while after the mutation step it is of the form
\begin{equation}
\frac{\bar{\pi}_k(x_{0:k},\tilde{x}_{1:k})}{\bar{\eta}_k(x_{0:k},\tilde{x}_{1:k})}=\frac{\bar{\pi}_k(x_{0:k-1},\tilde{x}_{1:k})}{\bar{\eta}_k(x_{0:k-1},\tilde{x}_{1:k})}\underbrace{\frac{\pi_k(x_k)L^{\textup{mul}}_{k-1}(x_k,\tilde{x}_k)}{\pi_{k}(\tilde{x}_k)M^{\textup{mul}}_{k}(\tilde{x}_k,x_k)}}_{\textup{incremental weight}=1}.
\end{equation}
\subsection{Collapsed version}\label{subsec:collapse}
When no resampling is used, there is no use for the introduction of the random variables $\tilde{X}_{1:k}$ in the previous derivation and they can be integrated out. In this case, we collapse the transport step and mutation step into one single Markov kernel
\begin{align}\label{eq:collapsedkernel}
M_k^{\textup{col}}(x,x') &=  \int M^{\textup{trans}}_{k}(x,\tilde{x}) M^{\textup{mut}}_{k}(\tilde{x},x')\diff \tilde{x}\\
                &= \int \delta_{T_k(x)}(\tilde{x})K_k(\tilde{x},x')\diff \tilde{x}\\
                &= K_k(T_k(x),x').
\end{align}
Similarly we collapse the backward kernels used to defined the extended target distributions $\bar{\pi}_k$
\begin{align}\label{eq:collapsedbackwardkernel}
L_{k-1}^{\textup{col}}(x,x') &=\int \frac{\pi_k(\tilde{x}) K_k(\tilde{x},x)}{\pi_k K_k(x)} \delta_{T^{-1}_k(\tilde{x})}(x')\diff \tilde{x}\\
&=\frac{\pi_k(T_k(x')) |\nabla T_k(x')| K_k(T_k(x'),x)}{\pi_k K_k(x)}.\\
\end{align}
Contrary to \cref{subsec:noncollapse}, we consider the more general scenario here where $K_k$ might not be $\pi_k$ invariant discussed in \cref{sec:scalinglimit}. From \cref{eq:collapsedkernel} and \cref{eq:collapsedbackwardkernel}, $\bar{\pi}_k$ is thus given by \eqref{eq:generalizedCrooksJarzynski} for
\begin{align}\label{eq:extendedtargetcollapse_2}
 w_{k}(x_{0:k})&=\underbrace{\prod_{l=1}^k  \frac{\gamma_l(x_l)} {\gamma_l K_l(x_l)}}_\text{$r_k(x_{1:k})$} \prod_{l=1}^k G_{l,T_l}(x_{l-1}),\\
\end{align}
where $r_k(x_{1:k})=1$ for $\pi_l$-invariant MCMC kernels $K_l$ as used in \cref{alg:SMC}.
\section{Proof of the asymptotic results}
We consider the unnormalized empirical measure $\gamma_k^N$ defined as:
\begin{align}
	\gamma_k^{N} =  Z_k^N \pi_{k}^N. 
\end{align}
We will provide the consistency and CLT results for both $\gamma_k^N[f]$ and $\pi_k^N[f]$ which imply the results on the normalizing constant $Z_k^N$ as $Z_k^N = \gamma_k^N[1]$. We denote by $\mathcal{F}_{k}^N$ the filtration generated by the particles $X^{i}_k$ and the NFs $T_{k+1}$ up to time $k$ and write  $\mathbb{E}_{k}[f]= \mathbb{E}[f|\mathcal{F}_k^N ]$. This accounts for possible randomness coming from the optimization of the NFs. 
We also consider $\mathcal{C}_p$ the class of continuous functions $f$ on $\mathcal{X}$ with growth in $x$ of at most $p$, for some non-negative integer $p$, i.e.
\begin{align}\label{eq:growth_class}
	 \mathcal{C}_p(C) =\{ f:\mathcal{X}\rightarrow \mathbb{R}~|~f \text{ is continuous and } \forall x\in \mathcal{X}: \Verts{f(x)}\leq C\parens{1 +  \Verts{x}^p } \}.
\end{align}
In addition, we denote by $\mathcal{L}\mathcal{C}_p$ the class of functions in $\mathcal{C}_p$ that are locally Lipschitz and with local Lipschitz constant satisfying a growth condition:
\begin{align}\label{eq:lip_growth_class}
		 \mathcal{LC}_p(C) =\braces{ f \in\mathcal{C}_p(C) ~|~ \forall x,x'\in \mathcal{X}: \Verts{f(x)-f(x')}\leq C\parens{1 +  \Verts{x}^{p+1} + \Verts{x'}^{p+1}    } \Verts{x-x'} }.
\end{align}  
For ease of notation we also introduce the unnormalized transition kernel $Q_{k,T}$ which acts on functions $f$ by:
\begin{align}
	Q_{k,T}[f](x) := \G_{k,T}(x) \int f(y)K_k(T(x),\diff y).
\end{align}
Moreover, we overload the notation and write $ \mathcal{L}_k(\theta):= \mathcal{L}_k(\tau_{\theta}) $ and $ \mathcal{L}^N_k(\theta):= \mathcal{L}^N_k(\tau_{\theta})$.

\subsection{Assumptions}\label{sec:assumptions_clt}
The following assumptions are needed for both \cref{thm:WLLN,thm:CLT}.
\begin{assumplist}
	\item\label{assumption:kernel} The Markov kernel $K_k$ preserves the class $\mathcal{C}_p$ for any $p$, meaning that $K_k(f)\in \mathcal{C}_p $ whenever $f$ in  $\mathcal{C}_p$.
	\item \label{assumption:moments} $\pi_{k}$ admits $8$th order moments.
	\item \label{assumption:transport} The normalizing flows in $\mathcal{T}$ are of the form $T(x) = \tau_{\theta}(x)$  where  $\theta$ is a finite dimensional vector in a compact convex set  $\Theta$. Moreover,  the maps $x\mapsto\tau_{\theta}(x)$  are $L$-Lipschitz and jointly continuous in $x$ and $\theta$.
	\item \label{assumption:uniformly_bounded} The importance weights $G_{k,T}(x)$ are uniformly bounded over $x$ and $T$.
\end{assumplist}
In addition to the previous assumptions, we will need additional assumptions for the CLT result in \cref{thm:CLT}. 
First, we strengthen \cref{assumption:kernel}
\begin{assumplist}[resume]
	\item \label{assumption:kernel_2} 
	The Markov kernel $K_k$ preserves the class $\mathcal{LC}_p$ for any $p$, i.e. $K_k(f)\in \mathcal{C}_p $ for any $f$ in  $\mathcal{LC}_p$.
\end{assumplist}
We then make additional assumptions on the smoothness of the potentials $V_k$ and the parameterization of the normalizing flows $\tau_{\theta}$:
\begin{assumplist}[resume]
	\item\label{assumption:parametrization}   The flow $(\theta,x)\mapsto\tau_{\theta}(x)$ admits derivatives $\nabla_{\theta}\tau_{\theta}(x) $,  $\partial_{\theta_i}\partial_{x_l} \tau_{\theta}(x)$ and $\partial_{\theta_i}\partial_{\theta_j}\partial_{x_l} \tau_{\theta}(x)$, $H_x \tau_{\theta}(x)$  with at most linear growth in $x$ uniformly in $\theta$. Moreover, all singular values of $\nabla_{x}\tau_{\theta}(x)$ are lower-bounded by a positive constant $c>0$ uniformly in $x$ and $\theta$. 
	\item\label{assumption:potential} The potentials are twice continuously differentiable and their gradients are $L$-Lipchitz, i.e. $\Vert \nabla V_k(x) - \nabla V_{k}(x') \Vert \leq L \Vert x-x' \Vert $.
\end{assumplist}

Finally, we make two assumptions on the algorithm used to find $\theta^N_k$. We denote by $\Theta^{\star}_k$ the set of local minimizers of the population loss $\theta\mapsto\mathcal{L}_k(\theta)$.
\begin{assumplist}[resume]
	\item \label{assumption:approximate_minimizer} The estimator $\theta_k^N$ satisfies: 
	\begin{align}\label{eq:local_min}
		\nabla \mathcal{L}^N_k(\theta^N_k) &= o_{\mathbb{P}}(1),\\
		H \mathcal{L}^N_k(\theta^N_k) &\geq o_{\mathbb{P}}(1). \\
	\end{align}
	\item \label{assumption:bassin} There exists a local minimizer $\theta^{\star}_k\in \Theta^{\star}_k$ of the population loss $\mathcal{L}_k(\theta)$ such that
	\begin{align}
		\mathbb{P}\brakets{  \theta^{\star}_k\in \arg\min_{\theta\in \Theta^{\star}_k} \Verts{\theta^N_k-\theta} }\rightarrow 1.
	\end{align}
\end{assumplist}
\cref{assumption:approximate_minimizer} states that the algorithm finds an approximate local minimizer of the empirical loss $\mathcal{L}^N_k(\theta)$.  This condition depends only on how well the algorithm is able to find a local optimum of the empirical loss accurately. In the ideal case where $\theta^N_k$ is an exact local minimizer of $\mathcal{L}^N_k(\theta)$, then the condition holds by definition. \cref{assumption:bassin} states that as $N$ increases $\theta_k^{N}$ remains within the basin of attraction of a single local optimum $\theta^{\star}_k$ and does not jump between different solutions. For instance, in the case of gradient descent, this assumption can be satisfied if the algorithm starts from the same initial $(\theta_{k})_0$ for all values of $N$ and is iterated to obtain an estimator $\theta_k^{N}$. Hence, as $N$ increases the empirical loss will have the same basins of attraction as the population loss and the choice of the solution $\theta_k^{\star}$ is determined only by the initial condition $(\theta_{k})_0$.

\subsection{Kernels satisfying \cref{assumption:kernel,assumption:kernel_2}  }\label{sec:kernel_assump}
Here we provide examples of generic transition kernels $K_k$ that satisfy \cref{assumption:kernel,assumption:kernel_2}.
 In \cref{sec:ULA}, we show that the kernel used in the Unadjusted Langevin Algorithm  (ULA kernel) satisfies \cref{assumption:kernel,assumption:kernel_2} under mild assumptions on $\pi_k$. While this kernel is not exactly invariant w.r.t. $\pi_k$, we will use it in \cref{sec:HM} to construct a kernel invariant w.r.t. $\pi_k$ and satisfying \cref{assumption:kernel,assumption:kernel_2}.

\subsubsection{Unadjusted Langevin Kernel}\label{sec:ULA}
We consider a slightly generalized version of the ULA kernel whose density $g(x,y)$ is given by:
\begin{align}\label{eq:density_ULA}
	g(x,y) \propto \exp\parens{-\frac{1}{4\tau} \Verts{y-x-\alpha\tau \nabla \log\pi_k(x)  }^2 },
\end{align} 
with $0\leq \alpha\leq 1$ and $\tau>0$. When $\alpha = 0$, one recovers the random walk kernel, while setting $\alpha=1$ gives back the ULA kernel with discretization step-size $\tau$.
\begin{prop}\label{prop:ULA_kernel_growth}
	Under \cref{assumption:potential}, the density  $g(x,y)$ in \cref{eq:density_ULA} satisfies
	\begin{align}
		\Verts{\nabla_x \log g(x,y)}\lesssim \parens{1+\Verts{y} + \Verts{x}},\qquad \Verts{\nabla_y \log g(x,y)}\lesssim \parens{1+\Verts{y} + \Verts{x}}.
	\end{align}
Moreover, the ULA kernel with density $g(x,y)$ satisfies 	\cref{assumption:kernel,assumption:kernel_2}. 
\end{prop}
\begin{proof}
	The estimate in \cref{eq:growth_ULA} is obtained by direct computation of the gradient of the logarithm of $g(x,y)$
	\begin{align}
		\Verts{\nabla_x g(x,y)} &= \frac{1}{\tau}\Verts{\parens{y-x+ \tau\alpha \nabla V_k(x)}^{\top}\parens
		{-I + \tau\alpha H_x V_k(x) }}\\
		&\lesssim \parens{1+\Verts{x} +\Verts{y}},
	\end{align} 
	where we used that $\nabla V_k(x)$ has at most a linear growth in $x$ and $H_x V_k(x)$ is bounded by \cref{assumption:potential}. The second assertion is obtained similarly by directly computing the gradient w.r.t $y$.
	
	To show that the ULA kernel satisfies \cref{assumption:kernel}, consider a function $f$ in $\mathcal{C}_p$, we can then write after a change of variables:
	\begin{align}
		K_k(f)(x) &= \int f\parens{y + x - \tau\alpha \nabla V_k(x)} \mathcal{N}(y,0, 2\tau I )\diff y\\
				&\lesssim \int \Verts{y}^{p}\mathcal{N}(y,0, 2\tau I ) \diff y+ \Verts{x}^{p} + (\tau\alpha)^p \Verts{\nabla V_k(x)} \lesssim \parens{1+ \Verts{x}^p}, 
	\end{align}
	where we get the last inequality by \cref{assumption:potential}. 	It is easy to see that $x\mapsto K(f)(x)$ is continuous by smoothing with a Gaussian and recalling that $\nabla V_k(x)$ is continuous. Hence, we can conclude that $K(f)\in \mathcal{C}_p$. 
To show that \cref{assumption:kernel_2} holds, we consider a function $f$ in $\mathcal{LC}_p$ and control the difference $\verts{K_k(f)(x)- K_k(f)(x')}$. For concision, we introduce $r(x,y) = y + x - \tau\alpha \nabla V_k(x)$ and write:
	\begin{align}\label{eq:control_K}
		\verts{K_k(f)(x)- K_k(f)(x')} &\leq \int \verts{f\parens{r(x,y)}- f\parens{r(x',y)}  }\mathcal{N}(y,0, 2\tau I )\diff y\\
		&\lesssim \int \Verts{r(x,y)-r(x',y)}\parens{\Verts{r(x,y)}^{p+1} + \Verts{r(x',y)}^{p+1}}\mathcal{N}(y,0, 2\tau I )\diff y.
	\end{align}
	Using \cref{assumption:potential}, we clearly have:
	\begin{align}\label{eq:control_r}
		\Verts{r(x,y)}\lesssim \parens{1 + \Verts{x} + \Verts{y}},\qquad  \Verts{r(x,y)- r(x',y)}\lesssim \parens{x-x'}.
	\end{align}
	We get the desired result by using the previous bounds in \cref{eq:control_K} and using the convexity of the power function.
\end{proof}

\subsubsection{Metropolis--Hastings kernel}\label{sec:HM}
For a target density $\pi_k$, we consider a Metropolis--Hasting kernel $K_k(x,\diff y)$ of the form:
\begin{align}\label{eq:HM_kernel}
	K_k(x,\diff y) =  g(x,y )\alpha(x,y)\diff y + \delta_{x}(\diff y)\int \parens{1-\alpha(x,u)}g(x,u)\diff u,
\end{align}
where $g(x,y)$ is the density of a proposal kernel and $\alpha(x,y)$ is the acceptance ratio:
\begin{align}
	\alpha(x,y) = \min\parens{1, \frac{\pi_{k}(y)g(y,x)}{\pi_k(x)g(x,y)} }.
\end{align}
We are in particular interested in proposals $g$ that satisfy the growth condition:
	\begin{align}\label{eq:growth_ULA}
		\Verts{\nabla_x \log g(x,y)}\lesssim \parens{1+\Verts{x} + \Verts{y}}.
	\end{align}
By \cref{prop:ULA_kernel_growth}, the above condition is satisfied if $g$ is a ULA kernel and if the potential $V_k$ satisfies \cref{assumption:potential}. 

In the next proposition, we show that \cref{assumption:kernel,assumption:kernel_2} hold under mild assumptions on $\pi_k$ and when the proposal $g$ satisfies \cref{eq:growth_ULA}. 
\begin{prop}
	Assume that Assumptions \ref{assumption:moments} and \ref{assumption:potential} hold for $\pi_k$ and that $g$ satisfies \cref{assumption:kernel} then the MH kernel in \cref{eq:HM_kernel} satisfies \cref{assumption:kernel}.
	
	If, in addition, $g$ satisfies the growth condition in \cref{eq:growth_ULA}, then the MH kernel in \cref{eq:HM_kernel} satisfies \cref{assumption:kernel_2}.	
\end{prop}
\begin{proof}
	For the first part of the proof, we consider a function $f$ in $\mathcal{C}_p$ and write:
	\begin{align}
		\verts{K_k\brakets{f}(x)} &= \verts{f(x)\int \parens{1-\alpha(x,u)}g(x,u)\diff u + \int f(y) \alpha(x,y) g(x,y)\diff y}\\
		&\leq 
		\verts{f(x)} + \verts{\int f(y)g(x,y)\diff y},
	\end{align}
	where we used that $0\leq \alpha(x,y)\leq 1$ to get the inequality. Since $g$ satisfies \cref{assumption:kernel} and $f\in \mathcal{C}_p$ we directly conclude that:
	\begin{align}
		\verts{K_k\brakets{f}}\lesssim 1 + \Verts{x}^p.
	\end{align}
	
	To prove the second part, consider a function $f$ in $\mathcal{LC}_p$. We need to control the difference $\verts{K_k\brakets{f}(x) -K_k\brakets{f}(x')}$:
	\begin{align}
		K_k\brakets{f}(x) -K_k\brakets{f}(x') =& f(x)-f(x') + \int \parens{f(y)-f(x)}g(x,y)\alpha(x,y)\diff y  - \int \parens{f(y)-f(x')}g(x',y)\alpha(x',y)\diff y\\
		=& \underbrace{\parens{f(x)-f(x')}\int g(x,y)\alpha(x,y)\diff y}_{A}\\
		&+ \underbrace{\int \parens{f(y)-f(x')}\parens{g(x,y)-g(x',y)}\alpha(x,y)\diff y}_{B}\\
		&+ \underbrace{\int \parens{f(y)-f(x')}g(x',y)\parens{\alpha(x,y)-\alpha(x',y)}\diff y}_{C}.
	\end{align}
We will control each term $A$, $B$ and $C$ independently. Since $0\leq \alpha \leq 1$, and $f\in \mathcal{LC}_p$, we directly have $\verts{A}\lesssim  \parens{1 +  \Verts{x}^{p+1} + \Verts{x'}^{p+1} } \Verts{x-x'}$. To control the second term $B$, we use the fundamental theorem of calculus which yields
\begin{align}
	\verts{g(x,y)- g(x',y)} = \verts{\int_0^1 g(x_t,y) \nabla_x \log g(x_t,y)^{\top}\parens{x-x'}   },
\end{align} 
where $x_t := (1-t)x +  tx'$. Since $g$ satisfies \cref{eq:growth_ULA} by assumption, we can directly write:
\begin{align}
	\verts{g(x,y)- g(x',y)}\leq \Verts{x-x'}\int_0^1 g(x_t,y)\parens{1 + \Verts{y}  +  \Verts{x_t} }\diff t.
\end{align}
Plugging the above inequality in $B$ and using that $f\in \mathcal{C}_p$ yields:
\begin{align}
	\verts{B} \lesssim \Verts{x-x'}\int_0^1 \int \parens{ 1 + \Verts{y}^p + \Verts{x'}^p }(1+\Verts{y} + \Verts{x_t})g(x_t,y)\diff t.  
\end{align}
Since $g$ satisfies \cref{assumption:kernel}, we can directly conclude that $\verts{B}\lesssim  \parens{1 +  \Verts{x}^{p+1} + \Verts{x'}^{p+1} } \Verts{x-x'}$.
Finally, to control $C$, we first define the function $b(x,y) = V_k(x)-V_k(y) + \log \frac{g(y,x)}{g(x,y)} $ so that the acceptance ratio can be written as $\alpha(x,y) = \min\parens{1, e^{b(x,y)}} $. Using \cref{lem:lip_alpha}, we directly have:
\begin{align}
	\verts{\alpha(x,y)- \alpha(x',y)}&\leq \verts{b(x,y)-b(x',y)}\\
	&\leq \verts{V_k(x) - V_k(x')  +    \log g(y,x) - \log g(y,x') + \log g(x',y) - \log g(x,y) }\\
	&\leq \verts{\int_0^1  \parens{\nabla V_k(x_t) + \nabla_y \log g(y,x_t) -  \nabla_x \log g(x_t,y) }^{\top}\parens{x-x'}\diff t }\\
	&\lesssim \Verts{x-x'}\int_0^1 \parens{1 + \Verts{x_t} + \Verts{y}}\diff t
	\lesssim \Verts{x-x'}\parens{1 + \Verts{x} + \Verts{x'} + \Verts{y}}.
\end{align}
We can therefore use the above inequality to upper-bound $\verts{C}$ as follows
\begin{align}
	\verts{C}&\lesssim  \Verts{x-x'}\int \parens{\verts{f(x')} + \verts{f(y) }}\parens{1 + \Verts{x} + \Verts{x'} + \Verts{y}}g(x',y)\diff y\\
	  &\lesssim \Verts{x-x'}\parens{1 + \Verts{x} + \Verts{x'}},
\end{align}
where we used that $f$ belongs to $\mathcal{LC}_p$ and thus to $\mathcal{C}_p$.
\end{proof}

\begin{lem}\label{lem:lip_alpha}
The following holds for any $a$, $a'$ in $\mathbb{R}$:
	\begin{align}
		\verts{\min\parens{1,e^a} - \min\parens{1,e^{a'}}  }\leq \verts{a-a'}
	\end{align}
\end{lem}
\begin{proof}
	Fix $a$ and $a'$ in $\mathbb{R}$. We distinguish $4$ cases:

	{\bf Case $1$:  $a\leq 0$ and $a'\leq 0$.}
	\begin{align}
		\verts{\min\parens{1,e^a} - \min\parens{1,e^{a'}}}  
		&= \verts{e^{a}- e^{a'}} = \verts{a-a'}\int_0^1 e^{(1-t)a + ta'}\diff t\\
		&\leq \max\parens{e^{a},e^{a'}}\verts{a-a'}\leq \verts{a-a'}
	\end{align}
	where we used that $e^{a}\leq 1$ and $e^{a'}\leq 1$.

	\begin{itemize}
		\item Case $2$:  $a\geq 0$ and $a'\geq 0$
	\end{itemize}
	We directly have $\verts{\min\parens{1,e^a} - \min\parens{1,e^{a'}}} = \verts{1-1} =  0\leq \verts{a-a'}$.

	{\bf Case $3$:  $a\leq 0$ and $a'\geq 0$.}
	\begin{align}
		\verts{\min\parens{1,e^a} - \min\parens{1,e^{a'}}}   = \verts{e^a - 1} = \verts{a}\int_{0}^1 e^{ta}\diff t 
	\end{align}
	Recalling that $a\leq 0$ we have $e^{ta}\leq 1$ and $\verts{a} = -a$. Moreover, since $a'\geq 0$ we can write
	\begin{align}
		\verts{\min\parens{1,e^a} - \min\parens{1,e^{a'}}}\leq -a \leq a'-a =  \verts{a'-a}
	\end{align}
	
{\bf Case $4$:  $a'\leq 0$ and $a\geq 0$.} This case is the same as case $3$ by switching the roles of $a$ and $a'$.

\end{proof}

\subsection{Weak law of large numbers}\label{sec:proof:wlln}
For simplicity, we provide a proof of \cref{proof:wlln} when  resampling is performed at each step. This can easily be extended to adaptive resampling using techniques from \cite{Douc:2007,Del-Moral:2012}. We denote  by $\xrightarrow{P}$ convergence in probability. 
\begin{proof}{\textbf{of \cref{thm:WLLN}}.}\label{proof:wlln}
	We proceed by induction. The result clearly holds for $k=0$ by the regular law of large numbers. By induction, we assume $\mathcal{R}_{l}$ holds for $0\leq l\leq k-1$  and we will prove that $\mathcal{R}_k$ holds as well. Let $f$ be a function in $\mathcal{C}_4$. We use the decomposition $ \pi_k^{N}[f]- \pi_k[f] =   A_N +B_N$ with:
	\begin{align}
		A_N &= \pi_k^{N}[f]- \mathbb{E}_{k-1}[\pi_k^{N}[f]],\\
		B_N &= \mathbb{E}_{k-1}[\pi_k^{N}[f]]- \pi_k[f].
	\end{align}
	\cref{prop:B,prop:A} show that both $A_N$ and $B_N$ converge in probability to $0$ and imply that $\pi_{k}^N[f]-\pi_k[f]\xrightarrow{P} 0$. It remains to show that $\gamma_{k}^N[f]-\gamma_k[f] \xrightarrow{P} 0$. We recall that $\gamma_k^N[f] =  \gamma_k^N[1]\pi_{k}^N[f] $ and $\gamma_k[f] = \gamma_k[1]\pi_k[f]$. Thus we only need to show that $\gamma_k^N[1]-\gamma_k[1]\xrightarrow{P} 0$. Recall that $\gamma_k^N[1]=  \prod_{l=1}^{k}\pi_{l-1}^{N}[\G_{l,T_{l}}]$
	and by \cref{prop:B} we know that $\pi_{l-1}^{N}[\G_{l,T_{l}}] \xrightarrow{P}  \frac{Z_{l}}{Z_{l-1}} $, thus we directly have  $\gamma_k^N[1]\xrightarrow{P}  \prod_{l=1}^{k} \frac{Z_{l}}{Z_{l-1}} = \frac{Z_k}{Z_0}$. This directly implies  $\gamma_k^N[1]\xrightarrow{P}\frac{Z_k}{Z_0} = \gamma_k[1]$ since $Z_0=1$ by construction. Finally, we conclude that $\gamma_k^N[f] -\gamma_k[f] \xrightarrow{P} 0 $ using Slutsky's lemma.
\end{proof}

\begin{thm}[Weak law of large numbers for \cref{alg:SMC-detailed}]\label{thm:WLLN_2}
Let $f$ be a function s.t. $ \vert f(x)\vert\leq C (1+ \Vert x \Vert^4) $ for all $x\in \X$ and for some $C>0$. Under \cref{assumption:kernel,assumption:transport,assumption:moments,assumption:uniformly_bounded} and for any $k\in{0,...,K}$:
\begin{equation}\label{eq:recursion_consistency}
   (\mathcal{R}_{k}):\qquad \pi_k^{N_{\textup{test}}}[f]\xrightarrow{P} \pi_k[f],\quad Z_k^{N_{\textup{test}},\textup{test}}\xrightarrow{P} Z_k,
\end{equation}
where $\pi_k^{N_{\textup{test}}}$ and $Z_k^{N_{\textup{test}},\textup{test}}$ are given by \cref{alg:SMC-detailed}.
\end{thm}
\begin{proof}
	The proof is a direct consequence of consistency of SMC samplers \cite{chopin2004central,del2004feynman}. 
	Indeed, the test particles $\braces{X_{k}^{i,\textup{test}}}_{1:N}$  are independent from the train  and validation particles $\braces{X_{k}^{i,\textup{train}}}_{1:N}$ and $\braces{X_{k}^{i,\textup{val}}}_{1:N}$ used to learn the flows $T_k$. 
	Moreover, by \cref{prop:independence_from_T}, the importance weights $w_{k}^{i,\textup{test}}$ correct exactly for the discrepancy between $\pi_k$ and $\parens{T_k}_{\#}\pi_{k-1}$. Hence, knowing the train and validation particles, \cref{alg:SMC-detailed} is a standard SMC sampler with Markov transition kernel given by $M_k(x,. ) = \parens{T_k}_{\#}K_k(x,.)$. Therefore consistency holds and $Z_{k}^{N_{\textup{test}},\textup{test}}$ is an unbiased estimator of $Z_k$.
\end{proof}

\begin{prop}\label{prop:B}
	Under \cref{assumption:kernel,assumption:transport,assumption:moments,assumption:uniformly_bounded} and whenever the recursion assumption $\mathcal{R}_{k-1}$ holds for a given $k>0$, it also holds that
	\begin{align}
		B_N :=\mathbb{E}_{k-1}\brakets{\pi_k^N[f]}-\pi_k[f] \xrightarrow{P} 0,\quad \quad
		\pi_{k-1}^N\brakets{\G_{k,T_k}}  \xrightarrow{P} \frac{Z_k}{Z_{k-1}},
	\end{align}
	for all  functions $f$  in $\mathcal{C}_4$.
\end{prop}
\begin{proof}
	We use the following decomposition for $B_N = B^{(1)}_N + B^{(2)}_N+ B^{(3)}_N +B^{(4)}_N$:
	\begin{align}
		B^{(1)}_N &:=  \frac{1}{\pi^N_{k-1}\brakets{\G_{k,T_k}}}\parens{ \pi_{k-1}\brakets{Q_{k,T_k}\brakets{f}} - \pi_{k-1}\brakets{Q_{k,T^{\star}_k}\brakets{f}}},   \\
		B^{(2)}_N &:=   \frac{\pi_{k-1}\brakets{Q_{k,T^{\star}_{k}}\brakets{f}}}{\pi^{N}_{k-1}\brakets{\G_{k,T_k}}\pi_{k-1}\brakets{\G_{k,T^{\star}_k}}}\parens{\pi_{k-1}\brakets{G_{k,T^{\star}_k}} - \pi_{k-1}\brakets{\G_{k,T_k}} },\\
		B^{(3)}_N &:= \frac{1}{\pi^N_{k-1}\brakets{\G_{k,T_k}}}\parens{ \pi^N_{k-1}\brakets{Q_{k,T_k}\brakets{f}} - \pi_{k-1}\brakets{Q_{k,T_k}\brakets{f}}  },\\
		B^{(4)}_N &:=   \frac{\pi_{k-1}\brakets{Q_{k,T^{\star}_{k}}\brakets{f}}}{\pi^{N}_{k-1}\brakets{\G_{k,T_k}}\pi_{k-1}\brakets{\G_{k,T^{\star}_k}}}\parens{\pi_{k-1}\brakets{\G_{k,T_k}} - \pi^{N}_{k-1}\brakets{\G_{k,T_k}}}.
	\end{align} 
\cref{prop:independence_from_T} states that $\pi_{k-1}[Q_{k,T}[f]] $ is independent of the choice of $T$. Thus, the first two terms $B^{(1)}_N$ and $B^{(2)}_N$ are exactly $0$.  

We know by \cref{prop:growth} that $Q_{k,T}[f](x)$ belongs to  $\mathcal{C}_4$, uniformly over $\mathcal{T}$. Moreover, the family $\mathcal{F} = (Q_{k,\tau_{\theta}}[f](x))_{\theta\in \Theta}$ is continuously indexed by the compact set $\Theta$. We can therefore apply \cref{prop:general_Glivenko_cantelli} under the recursion assumption $\mathcal{R}_{k-1}$ to the family $\mathcal{F}$. This ensures that
\begin{align}
	\sup_{T\in \mathcal{T}} \verts{\pi^{N}_{k-1}[Q_{k,T}[f]]-  \pi_{k-1}[Q_{k,T}[f]]} \xrightarrow{P} 0.
\end{align}
In particular, we have 
\begin{align}
	\pi^{N}_{k-1}\brakets{Q_{k,T_k}\brakets{f}}-  \pi_{k-1}\brakets{Q_{k,T_k}\brakets{f}} &\xrightarrow{P} 0,\\
	\pi_{k-1}^N\brakets{\G_{k,T_k}} - \pi_{k-1}\brakets{\G_{k,T_k}}&\xrightarrow{P} 0,
\end{align}
where the last equation is obtained simply by choosing $f=1$. This directly implies that $B^{(3)}_N$ and $B^{(4)}_N$ converge to $0$ in probability using Slutsky's lemma.
\end{proof}

\begin{prop}\label{prop:A}
	Under \cref{assumption:kernel,assumption:transport,assumption:moments,assumption:uniformly_bounded} and whenever the recursion assumption $\mathcal{R}_{k-1}$ holds for a given $k>0$, it holds that:
	\begin{align}
		A_N:=  \pi_k^{N}[f]- \mathbb{E}_{k-1}[\pi_k^{N}[f]]\xrightarrow{P} 0,
	\end{align}
	for all measurable functions $f$ in $\mathcal{C}_p$.
\end{prop}
\begin{proof}
We will show that the characteristic function of $A_N$ denoted $\phi_{A_N}(t)$ converges towards $1$ for all $t\in \mathbb{R}$. It is easy to see that $A_N$ can be expressed as:
\begin{align}
	A_N = \frac{1}{N}\parens{\sum_{i=1}^N U_{N,i} },\qquad U_{N,i} =   f(X_k^i) -\mathbb{E}_{k-1}[f(X_{k}^{i})],
\end{align}
where, conditionally on $\mathcal{F}_{k-1}^N$, the variables $X_k^i$ are independent and identically distributed according to:
\begin{align}
	\mathbb{P}\parens{X_k\in \diff x \middle|\mathcal{F}^N_{k-1}} = \sum_{i=1}^N \frac{\G_{k,T_k}\parens{X_{k-1}^i}}{\sum_{j=1}^N \G_{k,T_k}\parens{X_{k-1}^j} } K_{k}\parens{T_{k}\parens{X_{k-1}^i},\diff x}.
\end{align}
Let us introduce the conditional characteristic function $\hat{\phi}_{U_{N,1}}(t)$ knowing $\mathcal{F}^N_{k-1}$ defined by:
\begin{align}
	\hat{\phi}_{U_{N,1}}(t) = \mathbb{E}_{k-1}\brakets{e^{it U_{N,1} }  }.
\end{align}
This allows to express $\phi_{A_N}(t)$ in terms of $\hat{\phi}_{U_{N,1}}(t)$ as $ \phi_{A_N}(t) = \mathbb{E}[\hat{\phi}_{U_{N,1}}(\frac{t}{N})^{N}]$. Thus, we only need to prove that $\hat{\phi}_{U_{N,1}}(\frac{t}{N})^{N} \xrightarrow{P} 1$.

We will rely on the following expression $\hat{\phi}_{U_{N,1}}(\frac{t}{N}) = 1 + \frac{t}{N} E_N(\frac{t}{N})$ with $E_N(u) :=\int_0^1 \hat{\phi}_{U_{N,1}}'(su)\diff s$. We only need to prove that $ E_N(\frac{t}{N})\xrightarrow{P} 0$  as $N\rightarrow \infty$. Let us introduce the families of functions indexed by $u\in [-\vert t\vert ,\vert t\vert]$:   
\begin{align}
		g_u(x) =  \int_0^1 e^{isu f(x)}\diff s,\qquad\widehat{g}_u = \int_0^1 e^{-isu \mathbb{E}_{k-1}[f(X_k)]}\diff s,  \qquad f_u(x) = f(x) g_u(x).
\end{align}
Using the dominated convergence theorem, we can further express $E_N(u)$ as follows:
\begin{align}
	E_N(u) &= i\mathbb{E}_{k-1}\brakets{ U_{N,1}\int_0^1 e^{i suU_{N,1}}\diff s  }\\
		 &=   i\mathbb{E}_{k-1}\brakets{f_u(X_k)}\widehat{g}_u-i\mathbb{E}_{k-1}\brakets{f(X_k)}\mathbb{E}_{k-1}\brakets{g_u(X_k)}\widehat{g}_u.
\end{align}
Each expectation is of the form $\mathbb{E}_{k-1}[h(X_k)] = \frac{\pi_{k-1}^N[Q_{k,T_k}[h]  ]}{\pi_{k-1}^{N}[\G_{k,T_k}]}$. Using the induction hypothesis $\mathcal{R}_{k-1}$, and recalling that each function $f_u$, $f$ and $g_u$ belongs to  $\mathcal{C}_4$, it follows that each conditional expectation converges in probability towards $\pi_{k}[f_u]$, $\pi_{k}[f]$ and $\pi_k[g_u]$ while $\widehat{g}_u \xrightarrow{P} \int_0^1 e^{-isu \pi_k[f] }\diff s $. 
Moreover, using the fact that the functions $g_u$ and $f_u$ are continuously indexed by $u$ over the compact set $[-\vert t\vert ,\vert t\vert]$, we can apply \cref{prop:general_Glivenko_cantelli} to ensure that convergence is uniform over this set. This allows us to prove in particular that $E_N(\frac{t}{N})\xrightarrow{P} 0$. We have shown so far that $\hat{\phi}_{U_{N,1}}(\frac{t}{N}) = 1 + o_{\mathbb{P}}(\frac{t}{N}) $ which allows to conclude that $\hat{\phi}_{U_{N,1}}(\frac{t}{N})^{N}\xrightarrow{P} 1 $ as $N \rightarrow \infty$ and thus $A_N\xrightarrow{P} 0$.
\end{proof}
\begin{prop}\label{prop:independence_from_T}
		The following holds for any admissible $T$ in $\mathcal{T}$ and function $f$ such that  $K_k[f](x)<\infty$ and $\pi_{k}[f]<\infty$:
		\begin{align}
			\pi_{k-1}\brakets{Q_{k,T}\brakets{f}} = \frac{Z_{k}}{Z_{k-1}}\pi_{k}[f].
		\end{align}
		In particular, we have $\pi_{k-1}[\G_{k,T}] = \frac{Z_{k}}{Z_{k-1}}$.
\end{prop}
\begin{proof}
	For any admissible map $T$ we have that:
\begin{align}
	\pi_{k-1}\brakets{Q_{k-1,T}\brakets{f}} &= \frac{Z_{k}}{Z_{k-1}}\int \frac{\pi_k(T(x))}{\pi_{k-1}(x)}\vert \nabla T(x) \vert \left( \int f(y)K_k(T(x),\diff y)\right) \pi_{k-1}(x)\diff x \\
	  &=  \frac{Z_{k}}{Z_{k-1}} \int \int \pi_k(z)f(y)K_k(z,\diff y)\diff z\\
	  &= \frac{Z_{k}}{Z_{k-1}} \pi_k[f].
\end{align}
The second line is obtained by a change of variables $z= T(x)$ and using that $K_k$ is invariant w.r.t $\pi_k$. The last inequality is obtained by choosing $f=1$.
\end{proof}

\begin{prop}\label{prop:growth}
Let $f$ be a measurable function in  $\mathcal{C}_p$ for  $0 \leq p\leq 4$. Then, under \cref{assumption:kernel,assumption:transport,assumption:uniformly_bounded}, the function $x\mapsto Q_{k,T}[f](x)$ belongs to $\mathcal{C}_p$ uniformly over $\mathcal{T}$. In other words, there exists a positive constant $C>0$ such that:
\begin{align}
	\vert  Q_{k,T}[f](x) \vert \leq C(1  + \Vert x\Vert^p ),\qquad \forall x\in \X, \forall T\in \mathcal{T}.
\end{align}
\end{prop}
\begin{proof}
By \cref{assumption:kernel}, we have that $\vert K_k[f](x)\vert\leq C(1 + \Vert x \Vert^p)$. Moreover, using \cref{assumption:transport} we know that $T(x)$ has a linear growth in $x$, $\Vert T(x) \Vert \leq C'(1 +  \Vert x\Vert )$ with the same constant $C'$ for all $T\in \mathcal{T}$.
Therefore, there exists a positive constant $M>0$, such that $\vert K_k[f](T(x))\vert\leq M(1 + \Vert x \Vert^p)$ for any $T\in \mathcal{T}$ and $x\in \X$. Finally, we know by \cref{assumption:uniformly_bounded} that  $\G_{k,T}(x)$ is bounded uniformly over $x$ and $T$. This allows us to conclude that $Q_{k,T}[f]$ has the desired growth in $x$ which is uniform over $T \in \mathcal{T}$.
\end{proof}

\begin{prop}\label{prop:Glivenko_cantelli_braketing}
	Let $1\leq p\leq 4$, $C>0$ and $\mathcal{F}$ be a class of measurable functions in  $\mathcal{C}_p(C)$ such that the bracketing number $\mathcal{N}_{[]}(\epsilon, \mathcal{F}, L_2(\pi_k)) $ is finite for any $\epsilon>0$.
	Then under \cref{assumption:moments} and the recursion assumption $\mathcal{R}_{k}$ the following uniform convergence holds in probability
	\begin{align}
		\sup_{f\in F} \vert \pi_{k}^N[f] -\pi_{k}[f] \vert   \xrightarrow{P} 0.
	\end{align}
	\end{prop}
\begin{proof}
First consider the envelope  $F(x):=  \sup_{s\in S} \vert f(x)\vert $ which has a growth of at most $p$ in $x$ by assumption on $\mathcal{F}$. Moreover, $F$ is $\pi_k$-integrable by \cref{assumption:moments}.  
Fix $\epsilon>0$.  Since the bracketing number $N_{[]}(\epsilon, \mathcal{F}, L_1(\pi_{k}))$ is finite, there exists finitely many $\epsilon$-brackets $([l_i,u_i])_{1\leq i \leq I}$ whose union contains $\mathcal{F}$ and such that $\pi_{k}(u_i-l_i)< \epsilon$ for every $i\in \braces{1,I}$. Moreover, the functions $l_i$ and $u_i$ can be chosen to have a growth of at most $p$ in $x$, since $\mathcal{F}$ belongs to $\mathcal{C}_p(C)$.
 Hence, for every $f\in \mathcal{F}$, there is a bracket $[l_i,u_i]$ such that:
\begin{align}
	\parens{\pi_{k}^N-\pi_{k}}[f]\leq \parens{\pi_{k}^N-\pi_{k}}\brakets{u_i} + \pi_{k}\brakets{u_i-f}\leq \parens{\pi_{k}^N-\pi_{k}}\brakets{u_i}+\epsilon.
\end{align} 	
Hence, we have:
\begin{align}
	\sup_{f\in \mathcal{F}} \parens{\pi_{k}^N-\pi_{k}}\brakets{f}\leq \max_i \parens{\pi_{k}^N-\pi_{k}}\brakets{u_i}+\epsilon.
\end{align}
Since $\mathcal{R}_k$ holds and $u_i\in \mathcal{C}_p(C)$, the right hand side converges in probability towards $\epsilon$. Similarly, it is possible to show that:
\begin{align}
	\inf_{f\in \mathcal{F}} \parens{\pi_{k}^N-\pi_{k}}\brakets{f} \geq  \min_i \parens{\pi_{k}^N-\pi_{k}}\brakets{l_i}-\epsilon.
\end{align}
with r.h.s. converging towards $-\epsilon$ in probability. This allows us to conclude.
\end{proof}

\begin{prop}\label{prop:general_Glivenko_cantelli}
	Let $\mathcal{F}$ be a class of measurable functions in  $\mathcal{C}_4(C)$ for some $C>0$:
	\begin{align}
		\verts{ f(x)} \leq C(1 + \Vert x \Vert^4),\qquad \forall x\in \X,~\forall f\in \mathcal{F}.
	\end{align} 
	Assume that $\mathcal{F}$ is continuously indexed by a compact set $S$, i.e.  $s\mapsto f_s(x)$ is continuous for any $x\in \X$, where $f_s$ is an element in $\mathcal{F}$ indexed by $s\in S$.
	Then under \cref{assumption:moments} and the recursion assumption $\mathcal{R}_{k}$ the following uniform convergence holds in probability
	\begin{align}
		\sup_{f\in F} \vert \pi_{k}^N[f] -\pi_{k}[f] \vert   \xrightarrow{P} 0.
	\end{align}
	\end{prop}
\begin{proof}
First consider the envelope  $F(x):=  \sup_{s\in S} \vert f_{s}(x)\vert $ which has a growth of at most $4$ in $x$ by assumption over the class $\mathcal{F}$. Moreover, $F$ is $\pi_k$-integrable by \cref{assumption:moments}.  Since $\mathcal{F}$ is continuously indexed by a compact set and has an integrable envelope $F$ w.r.t. $\pi_{k}$, this implies that its bracketing number $N_{[]}(\epsilon, \mathcal{F}, L_1(\pi_{k}))$ is finite for every $\epsilon>0$ \citep[Example 19.8]{Van-der-Vaart:2000}. We can directly apply \cref{prop:Glivenko_cantelli_braketing} to get the desired result.
\end{proof}

\subsection{Proof of the central limit theorem}\label{proof:clt}
As shown in \citep[Theorem 10]{Douc:2007} and \citep[Section 6]{Del-Moral:2012}, the fluctuations of the SMC sampler with adaptive resampling admit the same asymptotic variance as the ideal SMC sampler with resampling at the optimal times in $\mathcal{K}_{\textup{opt}}$. Therefore, it is enough to prove this result for the case when resampling is triggered exactly at times $k$ in  $\mathcal{K}$. 
\begin{proof}{\textbf{of \cref{thm:CLT}.}}
We will proceed by induction. For $k=0$, the samples $(X_0^i)_{1\leq i \leq N}$ are i.i.d. thus one can directly apply the standard central limit theorem to show that the result holds at $k=0$, i.e. we write $CLT_0$ holds. By induction, let us assume that $CLT_{k-1}$ holds for some $k>0$, we will then show that $CLT_{k}$ holds as well. Let $f$ be a measurable real-valued function over $\X$ with at most quadratic growth in $x$. For conciseness, we first define $E_N$ and $F_N$
	\begin{align}
	E_N =   \sqrt{N}\parens{ \gamma_k^N[f] - \gamma_k[f]},\qquad
	F_N = \sqrt{N}\parens{ \pi_k^N[f] - \pi_k[f]}.
\end{align}
We need to show that $E_N$ and $F_N$ converge to centered Gaussians with variances $\mathbb{V}^{\gamma}_k[f]$ and $\mathbb{V}^{\pi}_k[f]$. Starting with $E_N$, we use the decomposition  $E_N=  R_N + P_N$ with
\begin{align}\label{eq:decomposition_clt}
		P_{N} := \mathbb{E}_{k-1}\brakets{E_N}, \qquad R_N :=E_N - \mathbb{E}_{k-1}\brakets{E_N}.
\end{align}
Using the recursion assumption and \cref{prop:C}, we can show that $P_{N}$ converges in distribution towards a centered Gaussian with variance $\mathbb{V}_{k-1}^{\gamma}[Q_{k,T^{\star}}[f]]$. Moreover, by \cref{prop:D,prop:D_no_resampling}, we know that $\mathbb{E}_{k-1}[e^{it R_N }] \xrightarrow{P} \exp(-\frac{t^2}{2} \mathbb{V}_k^{inc}\brakets{f} )$. This allows us to conclude that the characteristic function of $E_N$ converges pointwise to the characteristic function of a centered Gaussian distribution with variance $\mathbb{V}^{\gamma}_k[f]$: 
	\begin{align}
		\phi_{E_N}(t) = \mathbb{E} \brakets{e^{ it P_N} \mathbb{E}_{k-1}\brakets{ e^{itR_N } } } \rightarrow \exp\parens{-\frac{t^2}{2}\mathbb{V}^{\gamma}_k[f]},
	\end{align}
	 where by definition $\mathbb{V}^{\gamma}_k[f]= \mathbb{V}_k^{\textup{inc}}\brakets{f} + \mathbb{V}_{k-1}^{\gamma}[Q_{k,T^{\star}}[f]]$. We can then conclude using L\'evy continuity theorem that $E_N$ converges in distribution towards a centered Gaussian with variance $\mathbb{V}^{\gamma}_k[f]$ for any $f$ with at most quadratic growth.
	 For $F_N$, we can use the following identity:
	 \begin{align}
	 	F_N  = \frac{\sqrt{N}}{\gamma_k^N[1]}\parens{\gamma_k^N- \gamma_k}\brakets{f-\pi_k[f]  }.
	 \end{align}
	 Recalling that $\gamma_k^N[1] \xrightarrow{P} \gamma_k[1]$ by \cref{thm:WLLN}, we can directly conclude using Slutsky's lemma that
	 \begin{align}
	 	F_N \xrightarrow{\mathcal{D}} \mathcal{N}(0,\mathbb{V}_k^{\pi}[f])
	 \end{align}
where convergence is in distribution and where, by definition, $\mathbb{V}_k^{\pi}[f] =  \frac{\mathbb{V}_k^{\gamma}[f-\pi_k[f]]}{\gamma_k[1]^2}$. This concludes the proof.
\end{proof}

\begin{thm}[Central limit theorem for \cref{alg:SMC-detailed} ]\label{thm:CLT_2}
Let $f$ be a real valued function s.t.  $f(x)\leq C(1 + \Vert x \Vert^2)$ for some $C>0$. Then, under \cref{assumption:potential,assumption:kernel,assumption:kernel_2,assumption:potential,assumption:transport,assumption:moments,assumption:parametrization,assumption:uniformly_bounded} and for any $k\in{0,...,K}$ the same CLT result as in \cref{thm:CLT} holds when using the particles produced by \cref{alg:SMC-detailed} instead of \cref{alg:SMC}.   
\end{thm}
\begin{proof}
	The proof proceeds by recursion exactly as in \cref{thm:CLT}. The only difference is that the flow is estimated using the training and validation particles instead of the test ones. This does not affect the proof by recursion since we condition w.r.t. the sigma algebra  $\mathcal{F}_{k-1}$  generated by both  test particles and by the flow at time $k$. We only need to make sure that $\theta_k^N$ produced by \cref{alg:LearnFlow} satisfies \cref{assumption:bassin,assumption:approximate_minimizer}. 
	
	First, the validation criterion $\mathcal{L}_k^{N_{val}}(\theta)$ converges uniformly in $\theta$ in probability towards $\mathcal{L}_k$ and so does the training criterion $\mathcal{L}_k^{N_{train}}(\theta)$ by \cref{prop:general_Glivenko_cantelli} . Hence, returning the flow with smallest validation error is asymptotically equivalent to returning the flow with smallest training error as $N_{val}$ and $N_{train}$ increase. 
Moreover, since $\theta_k^N$ is obtained by performing gradient descent over $\mathcal{L}_k^{N_{train}}$, the final iterate will be the one that minimizes $\mathcal{L}_k^{N_{train}}$. Recalling now that gradient descent converges to a local minimizer, it follows that \cref{assumption:approximate_minimizer} holds provided that the number of iterations $J$ is large enough as $N_{train}$ and $N_{val}$ increase.
	
	Second, since the flows are all initialized to the identity for any number of particles and since both the  training loss and its gradient are uniformly converging in probability towards  $\mathcal{L}_k(\theta)$ and $\nabla_{\theta} \mathcal{L}_k(\theta)$ then the optimization trajectories obtained using   $\nabla_{\theta} \mathcal{L}^{N}_k(\theta)$ also converge uniformly to the one obtained using $\nabla_{\theta} \mathcal{L}_k(\theta)$. Hence, for $N$ large enough, $\theta^N$ is approaching a single local minimizer $\theta^{\star}$. Therefore \cref{assumption:bassin} also holds.
\end{proof}

\begin{prop}\label{prop:C}
	Let $f$ be a function in  $\mathcal{C}_2$. Under the induction assumption $CLT_{k-1}$, we have that:
	\begin{align}
		P_N \xrightarrow{\mathcal{D}} \mathcal{N}\parens{0, \mathbb{V}_{k-1}^{\gamma}\brakets{Q_{k,T^{\star}}[f]}},
	\end{align}
	where $P_N$ is defined in \cref{eq:decomposition_clt}.  
\end{prop}  
We defer the proof of \cref{prop:C} to \cref{sec:proof_C} as it relies on asymptotic stochastic equicontinuity of a suitable pocess which will be proven later in \cref{sec:equicontinuity}.

\begin{prop}\label{prop:IS_weights_rec}
Assume that resampling is only performed at the ideal resampling times in $\mathcal{K}$. For $k\notin \mathcal{K}$, denote by $k_p$ the largest integer in $\mathcal{K}$ such that $k_p < k$. Then, the importance weights $W_k^i$ are given by:
\begin{align}
    W_k^{i} = \frac{1}{N} \frac{Z_{k_p}^N}{Z_k^N} w_{k}^{i},
\end{align}
with $w_k^{i}$ given by:
\begin{align}
    w_k^{i} = \prod_{s=k_p+1}^k  G_{s,T_s}\parens{X^{i}_{s-1}}.
\end{align}
\end{prop}
\begin{proof}
This is a simple consequence of the recursion expression of the IS weights and normalizing constants for $k>k_p$:
\begin{align}
    W_k^i &= \frac{W_{k-1}^{i}}{\pi_{k-1}^N\brakets{G_{k,T_k}}}G_{k,T_k}(X_{k-1}^{i}),\\
    Z_k^N &= Z_{k-1}^N \pi_{k-1}^N\brakets{G_{k,T_k}},
\end{align}
with $W^i_{k_p}=\frac{1}{N}$.
\end{proof}
\begin{prop}\label{prop:D_no_resampling}
Let $f$ be a function in $\mathcal{C}_2$ and $R_N$ is defined in \cref{eq:decomposition_clt} and consider the conditional characteristic function
\begin{align}
	\hat{\phi}_{R_N}(t) =  \mathbb{E}_{k-1}[ e^{it R_N}  ].
\end{align}
Assume that resampling is not performed at iteration $k$ and let $k_p$ be the largest integer in $\mathcal{K}$ such that resampling is performed at time $k_p$ and $k_p < k$. Recall the expression of the asymptotic incremental variance $\mathbb{V}^{inc}_k$ when $k\not \in \mathcal{K}$:
\begin{align}
    \mathbb{V}^{\textup{inc}}_k\brakets{f} = Z_{k_p}^2\pi_{k_p}\brakets{ \mathbb{E}
		\brakets{  \mathcal{G}_{k}\brakets{f} \middle| X_{k_p}}}, 
\end{align}
with $\mathcal{G}_{k}\brakets{f} := K_k\brakets{f^2}\parens{ T_k^{\star}(X_{k-1}) }  - K_k\brakets{f}^2\parens{ T_k^{\star}(X_{k-1}) }$.
Then, under the recursion assumption $CLT_{k-1}$, we have:
\begin{align}
	\hat{\phi}_{R_N}(t)\xrightarrow{P} \exp\parens{-\frac{t^2}{2} \mathbb{V}^{\textup{inc}}_k\brakets{f}}.
\end{align}
\end{prop}
\begin{proof}

First, define $U_{N,i} :=  \sqrt{N}Z_{k}^{N} W_k^{i} f(X_k^i)$ and note that $R_N$ is expressed in term of $U_{N,i}$ as:
\begin{align}
    R_N = \sum_{i=1}^N U_{N,i} - \mathbb{E}_{k-1}\brakets{U_{N,i}}.
\end{align}
We will use the same approach as in the proof of \citep[Theorem 2]{Douc:2007}. 
For that purpose, we will show the following equations hold
\begin{align}\label{eq:Lind_1}
    &\sum_{i=1}^N \mathbb{E}_{k-1}\brakets{U_{N,i}^2}- \parens{\mathbb{E}_{k-1}\brakets{U_{N,i}}}^2 \xrightarrow{P} \mathbb{V}^{\textup{inc}}_k\brakets{f},
\end{align}
\begin{align}\label{eq:Lind_2}
    & \sum_{i=1}^N \mathbb{E}_{k-1}\brakets{U_{N,i}^2 \mathds{1}_{\braces{ \verts{U_{N,i}}\geq \epsilon }} } \xrightarrow{P} 0,\qquad \text{for any~} \epsilon>0.
\end{align}
The result will follow directly by application of \citep[Theorem A.3]{Douc:2007}. Using the expression of $W_k^i$ given by \cref{prop:IS_weights_rec}, we have that:
\begin{align}
    \sum_{i=1}^N \mathbb{E}_{k-1}\brakets{U_{N,i}^2} = \parens{Z_{k_p}^{N}}^2 \frac{1}{N}\sum_{i=1}^N \parens{w^{i}_{k}}^2 K_k\brakets{f^2}\parens{T_k\parens{X_{k-1}^i}} &\xrightarrow{P} Z_{k_p}^2\pi_{k_p}\brakets{\mathbb{E}\brakets{\parens{w^{\star}_{k}}^2K_k\brakets{f^2}\circ T_k^{\star}}}\\
    \sum_{i=1}^N \mathbb{E}_{k-1}\brakets{U_{N,i}}^2 = \parens{Z_{k_p}^{N}}^2 \frac{1}{N}\sum_{i=1}^N \parens{w^{i}_{k}}^2 \parens{K_k\brakets{f}\parens{T_k\parens{X_{k-1}^i}}}^2 &\xrightarrow{P} Z_{k_p}^2\pi_{k_p}\brakets{\mathbb{E}\brakets{\parens{w^{\star}_{k}}^2\parens{K_k\brakets{f}\circ T_k^{\star}}^2}}.
\end{align}
The above expressions are a result of the consistency of the particles trajectories $X^{i}_{k_p:k-1}$ by the recursion assumption. 
This shows \cref{eq:Lind_1}. The proof of \cref{eq:Lind_2} is the same as in \citep[Theorem 2]{Douc:2007}.
\end{proof}

\begin{prop}\label{prop:D}
Let $f$ be a function in $\mathcal{C}_2$ and $R_N$ is defined in \cref{eq:decomposition_clt} and consider the conditional characteristic function
\begin{align}
	\hat{\phi}_{R_N}(t) =  \mathbb{E}_{k-1}[ e^{it R_N}  ].
\end{align}
Under the recursion assumption $CLT_{k-1}$, and if resampling is performed at iteration $k$ we have:
\begin{align}
	\hat{\phi}_{R_N}(t)\xrightarrow{P} \exp\parens{-\frac{t^2}{2}\gamma_{k}[1]^2  \textup{Var}_{\pi_k}[f] }.
\end{align}
\end{prop}
\begin{proof}
First note that $R_{N}$ is expressed as a sum of the form:
\begin{align}
	R_{N} = \frac{1}{\sqrt{N}}\sum_{i=1}^N B_{N,i},\qquad B_{N,i}= \gamma_{k}^N[1]\parens{ U_{N,i} -\mathbb{E}_{k-1}[U_{N,i}] },
\end{align}
where, conditionally on $\mathcal{F}_{k-1}^N$, the variables $(U_{N,i})_{1\leq i\leq N}$ are independent and identically distributed as $f(X_{k})$ with:
\begin{align}
	P(X_k\in \diff x |\mathcal{F}^N_{k-1}) = \sum_{i=1}^N \frac{\G_{k,T_k}(X_{k-1}^i)}{\sum_{j=1}^N \G_{k,T_k}(X_{k-1}^j) } K_{k}\parens{ T_{k} \parens{X_{k-1}^i},\diff x}.
\end{align}
Hence, we have that $\hat{\phi}_{R_N}(t) = ( \hat{\phi}_{B_{N,1}}(\frac{t}{\sqrt{N}}) )^N  $. 
We will start by proving the following asymptotic decomposition for $\hat{\phi}_{B_{N,1}}$:
\begin{align}
	\hat{\phi}_{B_{N,1}}\parens{\frac{t}{\sqrt{N}}} = 1-\frac{t^2}{2N}\gamma_{k}^2[1]\textup{Var}_{\pi_k}\brakets{f} + \frac{t^2}{N} o_{\mathbb{P}}(1).
\end{align}
Since $B_{N,1}$ is centered conditionally on $\mathcal{F}_{k-1}^N$ and admits a finite second-order moment, the function $s\mapsto \hat{\phi}_{B_{N,1}}(s)$ is twice differentiable and satisfies $\hat{\phi}_{B_{N,1}}(0)=1$  and $\hat{\phi}_{B_{N,1}}^{'}(0)=0$. Moreover, by application of the dominated convergence theorem, we have that $\parens{\hat{\phi}_{B_{N,1}}}^{''}(s) = -\mathbb{E}_{k-1}\brakets{B^2_{N,1}e^{isB_{N,1} }}$. We can therefore  apply \cref{lem:integral_expression} to  $s\mapsto \hat{\phi}_{B_{N,1}}(s)$ which yields the identity:
\begin{align}
\hat{\phi}_{B_{N,1}}(s) 
&= 1 + s^2\int_0^1\int_0^1 u \parens{\hat{\phi}_{B_{N,1}}}^{''}(suv)  \diff u \diff v\\
&= 1- s^2\int_0^1\int_0^1 u \mathbb{E}_{k-1}\brakets{B^2_{N,1}e^{i suv B_{N,1} }}  \diff u \diff v.
\end{align}
Choosing $s = \frac{t}{\sqrt{N}} $ for a fixed $t\in \mathbb{R}$, we have:
\begin{align}
\hat{\phi}_{B_{N,1}}(s) 
&= 1- \frac{t^2}{N}\int_0^1\int_0^1 u \mathbb{E}_{k-1}\brakets{B^2_{N,1}e^{i \frac{t}{\sqrt{N}}uv B_{N,1} }}  \diff u \diff v.
\end{align}
By \cref{lem:uniform_convergence_characteristic} and for any $s\in \mathbb{R}$,  we know that $\mathbb{E}_{k-1}\brakets{B^2_{N,1}e^{i sB_{N,1}}}$ converges in probability towards $A_{\infty}(s)$ defined as:
\begin{align}
		 A_{\infty}(s) :=  \gamma_k^2\brakets{1} \pi_k\brakets{\parens{f-\pi_k\brakets{f}}^2 e^{is\gamma_k\brakets{1}\parens{f-\pi_k\brakets{f}} } }.
\end{align}
Moreover, \cref{lem:uniform_convergence_characteristic} also ensures this convergence to be uniform in $s$ over the interval $[-\verts{t},\verts{t}]$. Hence, we can write:
\begin{align}\label{eq:expansion_characteristic}
	\hat{\phi}_{B_{N,1}}(s) 
= 1- \frac{t^2}{N}\int_0^1\int_0^1 u A_{\infty}(\frac{t}{\sqrt{N}}uv)\diff u \diff v  + \frac{t^2}{N} o_{\mathbb{P}}(1).
\end{align}
By the dominated convergence theorem we know that $A_{\infty}(\frac{t}{\sqrt{N}}uv) \xrightarrow[N]{} A_{\infty}(0) $ for any fixed $t$, $u$ and $v$. Moreover, since $A_{\infty}$ is bounded, we can apply the dominated convergence theorem a second time to conclude that:
\begin{align}\label{eq:convergence_A_limit}
	\int_0^1\int_0^1 u A_{\infty}(\frac{t}{\sqrt{N}}uv)\diff u \diff v \xrightarrow[N]{} \int_0^1\int_0^1 u A_{\infty}(0)\diff u \diff v = \frac{1}{2} \gamma_k^2\brakets{1} \pi_k\brakets{\parens{f^2-\pi_k\brakets{f}^2}}.
\end{align}
Using \cref{eq:convergence_A_limit} in   \cref{eq:expansion_characteristic}, we have shown so far that:
\begin{align}
		\hat{\phi}_{B_{N,1}}(s) 
= 1- \frac{t^2}{2N}\gamma_k^2\brakets{1} \textup{Var}_{\pi_k}\brakets{f}  + \frac{t^2}{N} o_{\mathbb{P}}(1).
\end{align}
Recalling that $(1 + \frac{x}{N} + \frac{x}{N}o(1)  )^N \rightarrow  e^{x}$ as $N\rightarrow \infty$, we can therefore conclude that:
\begin{align}
	\hat{\phi}_{R_N}(t) = ( \hat{\phi}_{B_{N,1}}(\frac{t}{\sqrt{N}}) )^N \xrightarrow{P} \exp\parens{-\frac{t^2}{2}\gamma_{k}[1]^2  \textup{Var}_{\pi_k}[f] }.
\end{align}
which is the desired result.

\end{proof}
\begin{lem}\label{lem:integral_expression}
Let $f:\mathbb{R}\rightarrow \mathbb{C} $ be a function that is twice differentiable and that $f(0)=0$ and $f'(0) = 0$. Then the following identity holds:
\begin{align}
	f(s) = 1  + s^2\int_0^1 \int_0^1 u f^{''}(suv)\diff u\diff v.
\end{align}
\end{lem}
\begin{proof}
	The identity follows by direct integration.
\end{proof}

\begin{lem}\label{lem:uniform_convergence_characteristic}
Let $f$ be a function in $\mathcal{C}_2$  and define:
\begin{align}
 B_{N,i}= \gamma_{k}^N[1]\parens{ U_{N,i} -\mathbb{E}_{k-1}[U_{N,i}] },
\end{align}
where, conditionally on $\mathcal{F}_{k-1}^N$, the variables $(U_{N,i})_{1\leq i\leq N}$ are independent and identically distributed as $f(X_{k})$ with:
\begin{align}
	P(X_k\in \diff x |\mathcal{F}^N_{k-1}) = \sum_{i=1}^N \frac{\G_{k,T_k}(X_{k-1}^i)}{\sum_{j=1}^N \G_{k,T_k}(X_{k-1}^j) } K_{k}\parens{ T_{k} \parens{X_{k-1}^i},\diff x}.
\end{align}
Define the limiting function
\begin{align}\label{eq:limit_variance}
	 A_{\infty}(s) :=  \gamma_k^2\brakets{1} \pi_k\brakets{\parens{f-\pi_k\brakets{f}}^2 e^{is\gamma_k\brakets{1}\parens{f-\pi_k\brakets{f}} } }
\end{align}
then
\begin{align}
	\mathbb{E}_{k-1}\brakets{B^2_{N,1} e^{isB_{N,1}} }\xrightarrow{P} A_{\infty}(s),
\end{align}
where convergence is in probability and is uniform in $s$ over any compact interval. 
\end{lem}
\begin{proof}
	By definition of $U_{N,1}$, we have
	\begin{align}
	    \mathbb{E}_{k-1}\brakets{U_{N,1}} = \frac{\pi_{k-1}^N\brakets{Q_{k,T_k}\brakets{f}}}{\pi_{k-1}^N[G_{k,T_k}]}:= \bar{f}^N,
	\end{align}
	where we introduced $\bar{f}^N$ as a shorthand notation.
Since $f$ belongs to $\mathcal{C}_2$ we can apply the weak law of large numbers in \cref{thm:WLLN} which implies that
\begin{align}
    \bar{f}^N \xrightarrow{P} \frac{\pi_{k-1}\brakets{Q_{k,T_k}\brakets{f}}}{\pi_{k-1}[G_{k,T_k}]} = \pi_{k}[f],
\end{align}
	where we used \cref{prop:independence_from_T} to get the last equality. We also have $\gamma^N_{k}[1]  \xrightarrow{P}  \gamma_k[1] $ by \cref{thm:WLLN}.
Furthermore, by definition of $B_{N,1}$ and $U_{N,1}$, we can write
\begin{align}
    \mathbb{E}_{k-1}\brakets{B^2_{N,1} e^{isB_{N,1}} } =  \parens{\gamma^N_{k}[1]}^2 \frac{\pi_{k-1}^N\brakets{Q_{k,T_k}\brakets{\parens{f - \bar{f}^N}^2 e^{is \gamma_k^N[1]\parens{f-\bar{f}^N}} } }   }{\pi_{k-1}^N\brakets{G_{k,T_k}}}.
\end{align}
Recalling that $f^2$ belongs to $\mathcal{C}_4$, we can again apply the weak law of large numbers in \cref{thm:WLLN} along with the continuous mapping theorem to conclude that
\begin{align}
    \mathbb{E}_{k-1}\brakets{B^2_{N,1} e^{isB_{N,1}} } &\xrightarrow{P} \gamma_k[1]^2 \frac{\pi_{k-1}\brakets{Q_{k,T_k}\brakets{\parens{f - \pi_k\brakets{f}}^2 e^{is \gamma_k[1]\parens{f-\pi_k\brakets{f}}}  } }   }{\pi_{k-1}\brakets{G_{k,T_k}}}\\
    &= \gamma_k[1]^2 \pi_{k}\brakets{\parens{f - \pi_k\brakets{f}}^2 e^{is \gamma_k[1]\parens{f-\pi_k\brakets{f}}}  } = A_{\infty}(s),  
\end{align}
where the second line is obtained by application of \cref{prop:independence_from_T}. Moreover, using \cref{prop:general_Glivenko_cantelli} (as in the proof of \cref{prop:A}), we can conclude that convergence is uniform over $s\in[-\verts{t},\verts{t}]$.
\end{proof}

\subsection{Convergence of the flow transport}

\begin{prop}\label{prop:convergence_parameter}
	Under \cref{assumption:kernel,assumption:moments,assumption:uniformly_bounded,assumption:bassin,assumption:potential,assumption:parametrization,assumption:transport,assumption:approximate_minimizer}, it holds that:
\begin{align}
	\theta_k^N \xrightarrow[]{P} \theta_k^{\star}
\end{align}
\end{prop}
\begin{proof}
To simplify notations, we write $d(\theta, \Theta^{\star}) =  \inf_{\theta'\in \Theta^{\star}} \Verts{\theta- \theta'} $. We will first show that $\theta_k^N$ approaches the set of critical points $\Theta^{\star}$, i.e.   $d\parens{\theta^N_k, \Theta^{\star} }\xrightarrow[]{P} 0$. 
We assume by contradiction that $d\parens{\theta^N_k, \Theta^{\star} }$ does not converge to $0$ in probability. Hence,  there exist $\epsilon>0$ and $\eta>0$ as well as a subsequence of $\theta^{\phi_N}_k$ with $\phi_N\rightarrow + \infty$ such that:
\begin{align}\label{eq:contradiction_assumption}
	\mathbb{P}\brakets{d\parens{\theta^{\phi_N}_k, \Theta^{\star} } \geq \epsilon }\geq \eta, \forall N\in \mathbb{N}
\end{align}
	However, we also know that the sequence $\theta^{\phi_N}$ is tight as it is supported on $\Theta$ which is compact by  \cref{assumption:transport}. Hence, it admits a subsequence that converges in distribution towards a r.v. $\theta^{\infty}_k$. Without loss of generality, we assume $\theta^{\phi_N}$ to be such convergent subsequence. Since $\theta\mapsto d(\theta, \Theta^{\star})$ is continuous we have by the continuous mapping theorem that 
	\begin{align}\label{eq:convergence_in_dist_opt}
		d(\theta^{\phi_N}_k, \Theta^{\star})\xrightarrow[]{d} d(\theta^{\infty}_k, \Theta^{\star})
	\end{align}
	
	 We will now show that $\theta^{\infty}_k$ must be supported on $\Theta^{\star}$, the set of local minima of $\theta\mapsto \mathcal{L}_{\theta}$.  This would imply $d(\theta^{\infty}_k, \Theta^{\star})=0$.
	
We know by \cref{assumption:potential,assumption:parametrization} that $\theta\mapsto \nabla \mathcal{L}(\theta)$ and $\theta\mapsto H \mathcal{L}(\theta)$ are continuous functions, hence using the continuous mapping theorem, it holds that 
	\begin{align}\label{eq:conv_local_opt_cond}
		\nabla \mathcal{L}(\theta_k^{\phi_N})\xrightarrow[]{d} \nabla \mathcal{L}(\theta^{\infty}_k),\\
		H \mathcal{L}(\theta_k^{\phi_N})\xrightarrow[]{d} H \mathcal{L}(\theta^{\infty}_k).
	\end{align}
	Moreover, by \cref{lem:approx_minimizer} we can express the approximate local optimality assumption \ref{assumption:approximate_minimizer} in terms of the population loss $\theta\mapsto\mathcal{L}_k(\theta)$ instead of the empirical loss $\theta\mapsto\mathcal{L}^{N}_{k}(\theta)$:
	\begin{align}\label{eq:pop_approx_minimizer}
		\nabla \mathcal{L}(\theta^{\phi_N})&=o_{\mathbb{P}}(1),\\
		H \mathcal{L}(\theta^{\phi_N})&\geq o_{\mathbb{P}}(1).
	\end{align}
	Combining \cref{eq:conv_local_opt_cond,eq:pop_approx_minimizer} if follows that $\nabla \mathcal{L}(\theta_k^{\infty})= 0$ and $H \mathcal{L}(\theta_k^{\infty})\geq 0 $. This  precisely means that $\theta_k^{\infty}$ is supported on the set of local minimizers $\Theta^{\star}$ so that $d(\theta^{\infty}_k, \Theta^{\star}) =  0$. Hence, \cref{eq:convergence_in_dist_opt} implies that $d(\theta^{\phi_N}_k, \Theta^{\star})$ converges in distribution to a deterministic value $0$. This, in turn, means convergence in probability
	\begin{align}\label{eq:conv_proba_distance}
		d(\theta^{\phi_N}_k, \Theta^{\star})\xrightarrow[]{P} 0.
	\end{align}
	We have extracted a subsequence that satisfies both \cref{eq:contradiction_assumption,eq:conv_proba_distance}, which is contradictory. We can therefore conclude that $d\parens{\theta^N_k, \Theta^{\star} }\xrightarrow[]{P} 0$.  We introduce now the decomposition
\begin{align}\label{eq:dist_decomposition}
	\Verts{\theta^N_k- \theta_k^{\star} } = \parens{\Verts{\theta^N_k- \theta_k^{\star} } - d\parens{\theta^N_k, \Theta^{\star}}}   +  d\parens{\theta^N_k, \Theta^{\star}}.
\end{align}
We already know that the second term in \cref{eq:dist_decomposition} converges to $0$ in probability. Moreover, we know by \cref{assumption:bassin} that $\theta_k^{\star}$ is asymptotically the closest point in $\Theta^{\star}$ to $\theta_k^N$, hence,  the first term also converges to $0$ in probability, concluding the proof.   

\end{proof}
\begin{lem}\label{lem:approx_minimizer}
	Under \cref{assumption:kernel,assumption:moments,assumption:uniformly_bounded,assumption:transport,assumption:parametrization,assumption:potential,assumption:approximate_minimizer}
	It holds that:
	\begin{align}
		\nabla \mathcal{L}(\theta^{N})&=o_{\mathbb{P}}(1),\\
		H \mathcal{L}(\theta^{N})&\geq o_{\mathbb{P}}(1).\\
	\end{align}

\end{lem}
\begin{proof} 
	For simplicity, we introduce the function:
	\begin{align}
		l_{\theta}(x)= V_k(\tau_{\theta}(x)) - V_{k-1}(x) - \log \vert \nabla_x \tau_{\theta}(x) \vert + \log(\frac{Z_k}{Z_{k-1}}).  
	\end{align}
	Hence, by definition of $\mathcal{L}(\theta)$ and $\mathcal{L}^N(\theta)$, we have $\mathcal{L}(\theta) = \pi_{k-1}\brakets{l_{\theta} }$ and $\mathcal{L}^N(\theta) = \pi_{k-1}^N\brakets{l_{\theta} }$. Under \cref{assumption:potential,assumption:parametrization}  the gradient $\nabla_{\theta} l_{\theta}(x) $  and Hessian $H_{\theta} l_{\theta}(x)$ are well defined and admit a growth of at most $2$ in $x$ uniformly in $\theta$. Since, $\pi_{k-1}$ admit a finite second order moment by \cref{assumption:moments}, we can apply the dominated convergence theorem to write
	\begin{align}\label{eq:tcd}
		\parens{\pi_{k-1}^{N}-\pi_{k-1}}\brakets{\nabla_{\theta}l_{\theta}} &=   \nabla_{\theta}\mathcal{L}^N(\theta)  - \nabla_{\theta}\mathcal{L}(\theta),\\
		\parens{\pi_{k-1}^{N}-\pi_{k-1}}\brakets{H_{\theta}l_{\theta} } &=   H_{\theta}\mathcal{L}^N(\theta)  - H_{\theta}\mathcal{L}(\theta).
	\end{align}
	Moreover, recalling that, under \cref{assumption:kernel,assumption:moments,assumption:uniformly_bounded,assumption:transport},  the particle estimator $\pi^N_k$ is consistent by \cref{thm:WLLN} and the families of functions $\parens{\nabla_{\theta} l_{\theta}(x)}_{\theta\in \Theta}$ and $ \parens{  H_{\theta} l_{\theta}(x)}_{\theta\in \Theta}$ are indexed by a compact set $\Theta$ by \cref{assumption:transport} and admit a quadratic growth in $x$, we can apply \cref{prop:general_Glivenko_cantelli} to both families of functions to get a uniform convergence in probability
	\begin{align}\label{eq:uniform_convergence_hessian}
		\sup_{\theta}\verts{\parens{\pi_{k-1}^N - \pi_{k-1}}\brakets{\nabla_{\theta}l_{\theta}}} &\xrightarrow{P} 0,\\
		\sup_{\theta}\verts{\parens{\pi_{k-1}^N - \pi_{k-1}}\brakets{H_{\theta}l_{\theta}}} &\xrightarrow{P} 0.
	\end{align}
	Therefore, combining \cref{eq:uniform_convergence_hessian,eq:tcd} it follows that:
	\begin{align}
		\sup_{\theta}  \nabla_{\theta}\mathcal{L}^N(\theta)  - \nabla_{\theta}\mathcal{L}(\theta) &\xrightarrow{P} 0,\\
		\sup_{\theta}H_{\theta}\mathcal{L}^N(\theta)  - H_{\theta}\mathcal{L}(\theta) &\xrightarrow{P} 0.
	\end{align}
	We can rely on \cref{assumption:approximate_minimizer} to directly write
	\begin{align}
		\nabla \mathcal{L}(\theta_k^N) = \parens{\nabla \mathcal{L}(\theta_k^N) - \nabla \mathcal{L}^N(\theta_k^N)} + o_{\mathbb{P}}(1) = o_{\mathbb{P}}(1),\\
		H \mathcal{L}(\theta_k^N) \geq \parens{H \mathcal{L}(\theta_k^N) - H \mathcal{L}^N(\theta_k^N)} + o_{\mathbb{P}}(1) = o_{\mathbb{P}}(1).
	\end{align}
\end{proof}

\section{Asymptotic stochastic equi-continuity}\label{sec:equicontinuity}
In this section, we establish asymptotic stochastic equi-continuity (ASEC) of a process defined by the fluctuations of the particle approximation when applied to a suitable class of functions $\mathcal{G}$. More precisely, we would like to establish ASEC for the empirical process $E^N$ indexed by a class of functions $\mathcal{G}$ and defined as follows:
\begin{align}
E_k^N:  &\mathcal{G}\rightarrow \mathbb{R}
	  &f \mapsto E_k^N\brakets{f} =\sqrt{N}\parens{\pi_k^N\brakets{f}-\pi_k\brakets{f}}.
\end{align}
This property will be be useful for proving the CLT result in \cref{thm:CLT}. We start by introducing some notions used in this section.

\subsection{Class of functions with finite locally uniform entropy}
For some positive constant $C$, we consider $\mathcal{G}$ a subset of $\mathcal{LC}_p(C)$ or $\mathcal{C}_p(C)$  with a measurable envelope function $F(x) :=  \sup_{f\in \mathcal{G}}  \verts{f(x)} $. 
For $\epsilon>0$ and a probability distribution $\mathbb{P}$, we denote by $\mathcal{N}\parens{\epsilon, \mathcal{G}, L_2(P)}$ the covering number of $\mathcal{G}$ w.r.t.  $L_2(P)$ , meaning the smallest number $p$ of $L_2(P)$-balls centered on functions $f_1, ...,f_p$ and of radius smaller than $\epsilon$ needed to cover $\mathcal{G}$. 
The uniform covering number defined by taking the supremum  of $\mathcal{N}\parens{\epsilon, \mathcal{G}, L_2(P)}$ over all probability distributions $P$. Unlike in \citep[Lemma 9.6.1.]{del2004feynman}, we will not use the uniform covering number as this quantity will be infinite in our setting. Instead we will consider \textit{local} version of it, which we define next. We first consider the functions  $H_2(x) = 1 + \Verts{x}^{p+2}$  and the set  $S_R :=  \braces{Q\in \mathcal{P}: Q\brakets{H_2^2}\leq R^2  }$ and define the \textit{locally uniform covering number} $\mathcal{N}^{LU}(\epsilon, \mathcal{G}, R)$ to be
\begin{equation}\label{eq:unif_covering_number}
	\mathcal{N}^{LU}(\epsilon, \mathcal{G}, R) := 
	\begin{cases}
	\sup_{Q\in S_R}\mathcal{N}(\epsilon, \mathcal{G}, L_2(Q)), &\qquad S_R\neq \emptyset,  \\
	 1, &\qquad S_R=\emptyset.
	\end{cases}
\end{equation}
We define the \textit{locally uniform entropy} of the class $\mathcal{G}$ to be
\begin{align}\label{eq:bracketing_entropy}
	\mathcal{J}(\delta, \mathcal{G}, R ) := \int_0^{\delta} \sqrt{\log\parens{\mathcal{N}^{LU}\parens{\epsilon, \mathcal{G}, R} } } \diff \epsilon.
\end{align}
We have the following basic properties of the \textit{locally uniform entropy}.
\begin{prop}\label{prop:locally_uniform_entropy}
Assume  the envelope function $F(x)$ of $\mathcal{G}$ satisfies $F(x)\leq C(1+\Verts{x}^p)$  and   $J(\delta, \mathcal{G}, r )<+\infty$ for any $r>1$ and $\delta>0$. Then for any $\delta>0$, $0\leq r\leq r'$, it holds that
\begin{align}
	\mathcal{J}\parens{\delta, \mathcal{G}, r  } &\leq \mathcal{J}\parens{\delta, \mathcal{G}, r'},\\
	\sup_{\delta>0}\mathcal{J}\parens{\delta, \mathcal{G}, r } &<+\infty.\\
\end{align}
\end{prop}
\begin{proof}
	For the first statement, recall that the constraint set $S_r$ defining  $\mathcal{N}^{LU}\parens{\epsilon, \mathcal{G}, r}$ in \cref{eq:unif_covering_number} trivially satisfies $S_r\subset S_{r'} $, hence $\mathcal{N}^{LU}\parens{\epsilon, \mathcal{G}, r}\leq \mathcal{N}^{LU}\parens{\epsilon, \mathcal{G}, r'}$. The result follows directly by definition of the entropy in \cref{eq:bracketing_entropy}.
	
	For the second statement, we first note that $\mathcal{N}^{LU}\parens{c, \mathcal{G}, r}=1$ for any $c\geq 4Cr $. Indeed, either $S_{4Cr} = \emptyset$ in which case $\mathcal{N}^{LU}\parens{c, \mathcal{G}, r}=1$ holds by definition, or $S_{4Cr} \neq \emptyset$. In the later case, for any element $P\in S_{4Cr}$ any $f,g\in \mathcal{G}$ we have $P\brakets{(f-g)^2}^{\frac{1}{2}}\leq 2 P\brakets{F^2}^{\frac{1}{2}}$. Moreover, since $F(x)\leq C(1+ \Verts{x}^p)\leq 2C(1 + \Verts{x}^{p+2})$ we can conclude that $P\brakets{(f-g)^2}^{\frac{1}{2}}\leq 2 P\brakets{F^2}^{\frac{1}{2}}\leq 4Cr\leq c$.  $\leq 2Cr\leq c $. Hence, any $L_2(P)$-ball centered in an element $f$ of $\mathcal{G}$ and of radius $c$ covers $\mathcal{G}$. This directly implies  that $\mathcal{N}^{LU}\parens{c, \mathcal{G}, r}=1$. The result follows directly by definition of the entropy.
\end{proof}

We will be in particular interested in classes of functions $\mathcal{G}$ for which the \textit{locally uniform entropy} $\mathcal{J}(\delta, \mathcal{G}, r )$ is finite for any $r>1$ and $\delta>0$  and satisfies a growth condition $\sup_{\delta>0}\mathcal{J}(\delta, \mathcal{G}, r ) \lesssim r $ . Note that this condition always holds when the uniform entropy is finite and is therefore a weaker requirement. 
Next, we show the stability of this condition when applying the operator $Q_{k,T}$ to functions in $f$ with $T$ varying over $\mathcal{T}$. More precisely, we control the entropy of the set $\mathcal{QG}$ of the form
\begin{align}\label{eq:span_Q}
	\mathcal{QG} := \braces{ Q_{k,\tau_{\theta}}\brakets{f}\middle| \theta\in \Theta, f\in \mathcal{G}}. 
\end{align}
\begin{prop}\label{prop:bracketing_entopy}
Let $\mathcal{G}$ be a subset of $\mathcal{LC}_p(C)$ for some positive constant. Assume $\mathcal{G}$ has a finite bracketing number $\mathcal{N}_{[]}(\epsilon, \mathcal{G}, L_2(\pi_k))$  , that the \textit{locally uniform entropy}  $J(\delta, \mathcal{G}, r )$ defined in \cref{eq:bracketing_entropy} is finite for any $r>1$ and $\delta>0$ and satisfies the linear growth condition $\sup_{\delta>0}J(\delta, \mathcal{G}, r ) \lesssim r $. 
Under \cref{assumption:kernel,assumption:kernel_2,assumption:uniformly_bounded,assumption:potential,assumption:parametrization,assumption:transport}, the class $\mathcal{QG}$ defined in \cref{eq:span_Q} is a subset of $\mathcal{LC}_p(C')$ for some $C'>0$, has a finite bracketing number $\mathcal{N}_{[]}(\delta, \mathcal{QG}, L_2(\pi_{k-1}))$ and a finite \textit{locally uniform entropy}  $J(\delta, \mathcal{QG}, r )$ for any $r>1$ and $\delta>0$ satisfiying the linear growth condition:
\begin{align}
	\sup_{\delta>0}\mathcal{J}(\delta, \mathcal{QG}, r ) \lesssim r. 
\end{align}
In particular, the result holds if $\mathcal{G}$ is a single element in $\mathcal{LC}_p(C)$. 
\end{prop}
\begin{proof}
{\bf Bounding the locally uniform entropy.} 
Fix $\epsilon>0$ and $r>1$. Let $P$ be a probability distribution in $S_r$, ie. such that $P\brakets{H_2^2}\leq r^2$.  We consider an $\epsilon r$-net $f_1,...,f_p$  of $\mathcal{G}$ with respect to $L_2(PK_k)$. We choose a covering of minimal size, i.e. such that $p = \mathcal{N}(\epsilon r, \mathcal{G}, L_2(PK_k))$. We also consider $\theta_1,...,\theta_{p'}$, an $\epsilon$-covering  of $\Theta$  with  with minimum cardinality, i.e. $p' = \mathcal{N}\parens{\epsilon, \Theta, \Vert . \Vert}$ where $\mathcal{N}(\epsilon, \Theta, \Vert . \Vert)$ is the covering number of $\Theta$ with $\epsilon$-balls under the Euclidean distance. We also denote by $\mathcal{J}(\delta, \Theta, \Vert .\Vert )$ the entropy of set $\Theta$ defined as
\begin{align}
	 \mathcal{J}(\delta, \Theta, \Vert .\Vert ):=   \int_0^{\delta}\sqrt{\log\parens{\mathcal{N}\parens{\epsilon, \Theta, \Vert . \Vert }}}.
\end{align}
Since $\Theta$ is bounded and finite dimensional, $\sup_{\delta>0}\mathcal{J}(\delta, \Theta, \Vert .\Vert )$ is finite.

Let $g\in \mathcal{G}$, hence by definition of $\mathcal{G}$, there exists $\theta\in \Theta$ and $f\in \mathcal{G}$ such that $g= Q_{k,\tau_{\theta_i}}\brakets{f}$. 

	By definition of the $\epsilon$-covering of $\Theta$, there exists $j'\in \{1,...,p'\}$ such that $\Verts{\theta-\theta_{j'}}\leq \epsilon$. We can then use \cref{prop:Lipschiz_Q} which holds under \cref{assumption:kernel,assumption:kernel_2,assumption:uniformly_bounded,assumption:potential,assumption:parametrization,assumption:transport} to write:
	\begin{align}\label{eq:bracket_1}
		Q_{k,\tau_{\theta_{j'}}}\brakets{f}(x) - \epsilon c_2H_2(x)\leq g(x)\leq Q_{k,\tau_{\theta_{j'}}}\brakets{f}(x) + \epsilon c_2H_2(x),\qquad \forall x\in \mathcal{X}.
	\end{align}
	
	Moreover, by definition of the $\epsilon r$-net, there exists $j\in \{1,...,q\}$ such that:  $PK_k\brakets{(f-f_j)^2}\leq \epsilon^2 r^2$. Subtracting $Q_{k,\tau_{\theta_{j'}}}\brakets{f_j}(x)$ from \cref{eq:bracket_1} yields:
	\begin{align}\label{eq:bracket_2}
		Q_{k,\tau_{\theta_{j'}}}\brakets{f-f_j}(x) - \epsilon c_2H_2(x)\leq g(x) - Q_{k,\tau_{\theta_{j'}}}\brakets{f_j}\leq Q_{k,\tau_{\theta_{j'}}}\brakets{f-f_j}(x) + \epsilon c_2H_2(x),\qquad \forall x\in \mathcal{X}.
	\end{align}
	We need to quantify $P\brakets{\parens{g-Q_{k,\tau_{\theta^{j'}}}\brakets{f_j}}^2}$. Using \cref{eq:bracket_2} it is easy to see that 
	\begin{align}
		P\brakets{\parens{g-Q_{k,\tau_{\theta^{j'}}}\brakets{f_j}}^2}^{\frac{1}{2}}
		&\leq  MP\brakets{\parens{K_k\brakets{(f-f_j)}}^2}^{\frac{1}{2}} +\epsilon c_2 P\brakets{H_2^2}^{\frac{1}{2}}\\
		&\leq  MPK_k\brakets{(f-f_j)^2}^{\frac{1}{2}} + \epsilon c_2 P\brakets{H_2^2}^{\frac{1}{2}}\\
		&\leq  \epsilon\parens{M r + c_2 P\brakets{H_2^2}^{\frac{1}{2}}}\leq \epsilon r\parens{M + c_2}:=\epsilon r c_3.
	\end{align}
	Hence, $Q_{k,\tau_{\theta_{j'}}}\brakets{f_j}$ form an $\epsilon r c_3$ net of $\mathcal{QG}$. This allows us to write:
	\begin{align}\label{eq:main_covering_ineq}
		\mathcal{N}\parens{\epsilon rc_3 , \mathcal{QG}, L_2(P)}\leq \mathcal{N}\parens{\epsilon r, \mathcal{G}, L_2(PK_k)}\mathcal{N}\parens{\epsilon, \Theta, \Vert .\Vert}.
	\end{align}
	We will now upper-bound $\mathcal{N}\parens{\epsilon r, \mathcal{G}, L_2(PK_k)}$ by a locally uniform covering number. To achieve this we need to find a scalar $r'$ such that $PK_k\in S_{r'}$ whenever $P\in S_r$. By \cref{assumption:kernel}, we know that $K_k$ preserves the classes $\mathcal{C}_p$ and $\mathcal{C}_{p+1}$. Hence, there exists a constant $c_4$ large enough so that  $K_k\brakets{H_2^2}\leq c_4H_2^{2}(x)$, implying that $PK_k\brakets{H_2^2}\leq c_4 P\brakets{H_2^{2}}$.   By choosing $r' = c_4r  $, we are guaranteed that $PK_k\in S_{r'}$. Hence, it follows that
\begin{align}\label{eq:main_covering_ineq_2}
		\mathcal{N}\parens{\epsilon rc_3, \mathcal{QG}, L_2(P)}\leq \mathcal{N}^{LU}\parens{\epsilon r, \mathcal{G}, c_4 r}\mathcal{N}\parens{\epsilon, \Theta, \Vert .\Vert}.
	\end{align}
	Taking the supremum over $S_r$ in the l.h.h. of \cref{eq:main_covering_ineq_2},  we get
	\begin{align}\label{eq:main_covering_ineq_3}
		\mathcal{N}^{LU}\parens{\epsilon r c_3, \mathcal{QG}, r}\leq \mathcal{N}^{LU}\parens{\epsilon r , \mathcal{G}, c_4 r}\mathcal{N}\parens{\epsilon, \Theta, \Vert .\Vert}.
	\end{align}
Taking the logarithm of  \cref{eq:main_covering_ineq_3} and using the inequality $\sqrt{a+b}\leq \sqrt{a}+ \sqrt{b}$ for any non-negative numbers $a$ and $b$, we obtain a bound on the entropy after a simple change of variables
\begin{align}
	\mathcal{J}(\delta, \mathcal{QG}, r)\leq c_3 \mathcal{J}\parens{c_3^{-1}\delta, \mathcal{G},  c_4 r} + c_3r\mathcal{J}\parens{(rc_3)^{-1}\delta, \Theta, \Vert .\Vert }
\end{align}
By assumption, we have that $\sup_{\delta>0}\mathcal{J}\parens{c_3^{-1}\delta, \mathcal{G},  c_4 r}\lesssim r $. Moreover, since $\Theta$ is bounded and finite dimensional, it must hold that $\sup_{\delta>0, r\geq 1}\mathcal{J}\parens{(rc_3)^{-1}\delta, \Theta, \Vert .\Vert } <+\infty $. Thus we have shown that $\sup_{\delta >0}\mathcal{J}(\delta, \mathcal{QG}, r)\lesssim r$.

{\bf Bounding the bracketing number.}
Similarly, fix $\epsilon>0$ and let $\braces{[l_i,u_i]}_{1:p}$ be $\epsilon$-brackets covering $\mathcal{G}$ w.r.t. $L_2(\pi_k)$, i.e.  $ \pi_k\brakets{(l_i-u_i)^2} \leq \epsilon^2$ and for any $f\in\mathcal{G} $ there exists $i\in \braces{1,p}$ such that
\begin{align}
	l_i(x)\leq f(x)\leq u_i(x), \forall x\in \mathcal{X}.
\end{align}
Moreover, we assume the $\epsilon$-brackets to be of minimal size, i.e.  $p = \mathcal{N}(\epsilon, \mathcal{G}, L_{\pi_k} )$. From \cref{eq:bracket_1} we directly have
	\begin{align}\label{eq:bracket_3}
		Q_{k,\tau_{\theta_{j'}}}\brakets{l_j}(x) - \epsilon c_2H_2(x)\leq g(x)\leq Q_{k,\tau_{\theta_{j'}}}\brakets{u_j}(x) + \epsilon c_2H_2(x),\qquad \forall x\in \mathcal{X}.
	\end{align}
Using \cref{prop:Continuity_Q} we deduce that
\begin{align}
	\pi_{k-1}\brakets{\parens{Q_{k,\tau_{\theta_{j'}}}\brakets{u_j-l_j}(x) - 2\epsilon c_2H_2(x)}^2}^{\frac{1}{2}}\lesssim \epsilon\parens{ C +  \pi_{k-1}\brakets{H_2^2}^{\frac{1}{2}} }.
\end{align}
Thus we have shown that the bracketing number $\mathcal{N}_{[]}\parens{\epsilon,\mathcal{QG}, L_{\pi_{k-1}}}$ must be finite.
\end{proof}

\begin{prop}\label{prop:Lipschiz_Q}
	Let $\mathcal{G}$ be a subset $\mathcal{LC}_p(C)$ for some positive constant $C$. Under \cref{assumption:kernel,assumption:kernel_2,assumption:uniformly_bounded,assumption:potential,assumption:parametrization,assumption:transport}, the class $\mathcal{QG}$ belongs to $\mathcal{LC}_p(C')$ for some positive $C'$. Moreover, the following holds for any function $f\in \mathcal{G}$
	\begin{align}
		\verts{Q_{k,\tau_{\theta}}\brakets{f}(x) - Q_{k,\tau_{\theta'}}\brakets{f}(x)}&\lesssim \Verts{\theta-\theta'}H_2(x) 
	\end{align}
	where $H_2(x) =1 + \Verts{x}^{p+2}$.
\end{prop}
\begin{proof}
	Fix $1\leq  k\leq K$. We first start by controlling $G_{k,\tau_{\theta}}$. For conciseness, we write $G_{k,\tau_{\theta}} = e^{h_{\theta}(x)}$ with  $h^{\theta}(x)$ given by
	\begin{align}
		h_{\theta}(x) = \nabla V_{k}(\tau_{\theta}(x)) -  \log\verts{\nabla_x \tau_{\theta}(x)}. 
	\end{align}
	The function $\theta\mapsto h_{\theta}(x)$ is differentiable by \cref{assumption:potential,assumption:parametrization} and its gradient is given by:
	\begin{align}
		\nabla_{\theta}h_{\theta}(x) = \nabla_x V_k\parens{\tau_{\theta}(x)}\nabla_{\theta} \tau_{\theta}(x) - Tr\parens{ \parens{\nabla_x \tau_{\theta}(x)}^{-1}\nabla_{\theta}\nabla_x \tau_{\theta}(x)}.
	\end{align}
	Moreover, by \cref{assumption:parametrization}, we also know that the singular values of $\nabla_x \tau_{\theta}(x)$ are lower-bounded by a positive constant $c$ uniformly in $x$ and $\theta$. Hence $\verts{Tr\parens{ \parens{\nabla_x \tau_{\theta}(x)}^{-1}\nabla_{\theta}\nabla_x \tau_{\theta}(x)}}\leq c^{-1}\Verts{\nabla_{\theta}\nabla_x \tau_{\theta}(x)}$. Moreover, again by  \cref{assumption:potential,assumption:transport,assumption:parametrization}  $\nabla_x V_k(x)$, $\tau_{\theta}(x)$,  $\nabla_{\theta} \tau_{\theta}(x)$, $\nabla_{\theta}\nabla_x \tau_{\theta}(x)$ all have a linear growth in $x$ uniformly in $\theta$. Hence, we deduce that $\nabla_{\theta}h_{\theta}(x)$ has at most a quadratic growth in $x$. In addition,  $G_{k,\tau_{\theta}}$ is uniformly bounded in $\theta$ and $x$ by \cref{assumption:uniformly_bounded}, therefore, we deduce that  
	\begin{align}
		\Verts{\nabla_{\theta}G_{k,\tau{\theta}}(x)} \leq G_{k,\tau{\theta}}(x)\Verts{\nabla_{\theta}h_{\theta}(x)}\leq   C\parens{1 + \Verts{x}^2}.
	\end{align}
	For $\theta, \theta'\in \Theta$ and setting $\theta_t = t\theta + (1-t)\theta'  $ , we use the fundamental theorem of calculus to write:
	\begin{align}\label{eq:lip_weights}
		\verts{G_{k,\tau_{\theta}}(x) - G_{k,\tau_{\theta'}}(x)   }\leq \Verts{\theta - \theta'}\int_0^1 \Verts{\nabla_{\theta}G_{k,\tau_{\theta_t}}(x)}\diff t \lesssim \Verts{\theta - \theta'}\parens{1+\Verts{x}^2}.
	\end{align}
	
	For the second part of the proof, we simply use the following decomposition:
	\begin{align}\label{eq:main_decomposition_Q}
		Q_{k,\tau_{\theta}}\brakets{f}(x) - Q_{k,\tau_{\theta'}}\brakets{f}(x) =  \parens{G_{k,\tau_{\theta}}(x) - G_{k,\tau_{\theta'}}(x)} K_k\brakets{f}(x) +  G_{k,\tau_{\theta}}\parens{K_k\brakets{f}(\tau_{\theta}) -K_k\brakets{f}(\tau_{\theta'})   }.    
	\end{align}
	For the first term in the r.h.s. of \cref{eq:main_decomposition_Q}, we use \cref{eq:lip_weights} and that, by \cref{assumption:kernel},   $K_k\brakets{f}(x)$ belongs to $\mathcal{C}_{p}(C')$ for some constant $C'$ independent of $f$ in $\mathcal{LC}_{p}(C)$. Hence, we deduce that:
	\begin{align}\label{eq:bound_Q_1}
		\verts{\parens{G_{k,\tau_{\theta}}(x) - G_{k,\tau_{\theta'}}(x)} K_k\brakets{f}(x)}\lesssim \Verts{\theta-\theta'}\parens{1 +  \Verts{x}^{p+2} }.
	\end{align}
	For the second term in the r.h.s. of \cref{eq:main_decomposition_Q}, we know by \cref{assumption:kernel_2} that $K_k\brakets{f}(x)$ belongs to $\mathcal{LC}_{p}(C')$ for some constant $C'$ independent of $f$ in $\mathcal{LC}_{p}(C)$. Since $G_{k,\tau_{\theta}}(x)$ is bounded uniformly in $x$ and $\theta$, we get
	\begin{align}
		\verts{G_{k,\tau_{\theta}}\parens{K_k\brakets{f}(\tau_{\theta}) -K_k\brakets{f}(\tau_{\theta'})   }}\lesssim \parens{ 1 +  \Verts{\tau_{\theta}(x)}^{p+1} + \Verts{\tau_{\theta'}(x)}^{p+1} }\Verts{\tau_{\theta}(x) -  \tau_{\theta'}(x) }.
	\end{align}
	Moreover, we know that $\tau_{\theta}(x)$ has at most a linear growth in $x$ uniformly in $\theta$
	and that 
	\begin{align}
		\Verts{\tau_{\theta}- \tau_{\theta'} }\leq  \Verts{\theta'-\theta}\int_0^1 \Verts{\nabla_{\theta}\tau_{\theta_t}(x)}\diff \lesssim \Verts{\theta'-\theta} \Verts{1 + \Verts{x} }.
	\end{align}
	We can therefore deduce that: 
	\begin{align}\label{eq:bound_Q_2}
		\verts{G_{k,\tau_{\theta}}\parens{K_k\brakets{f}(\tau_{\theta}) -K_k\brakets{f}(\tau_{\theta'})   }}\lesssim \Verts{\theta-\theta'}\parens{1 + \Verts{x}^{p+2}}.
	\end{align}
	Combining \cref{eq:bound_Q_1,eq:bound_Q_2} yields the desired result.
\end{proof}

\begin{prop}\label{prop:Continuity_Q}	
	Under \cref{assumption:uniformly_bounded} for any $f\in L_2(\pi_k)$ we have 
	\begin{align}
		\Verts{Q_{k,T}\brakets{f}}_{L_2\parens{\pi_{k-1}}}\leq C \Verts{f}_{L_2\parens{\pi_k}}
	\end{align}
	with $C$ being independent of $f$.
\end{prop}
\begin{proof}
	Let $f$ be a function in $L_2(\pi_{k})$, we can directly write:
	\begin{align}
		\pi_{k-1}\brakets{\parens{Q_{k,T}\brakets{f}}^2} 
		&= \int \pi_{k-1}(x)G^{2}_{k,T}(x)\parens{K_k\brakets{f}\parens{T(x)}}^2 \diff x \\
		&= \frac{Z_k}{Z_{k-1}}\int \pi_k(y)G_{k,T}\parens{T^{-1}(y)}\parens{K_k\brakets{f}(y)}^2\diff x\\
		&\leq M\frac{Z_k}{Z_{k-1}}\int \pi_k(y)\parens{K_k\brakets{f}(y)}^2\diff x\\
		&\leq M\frac{Z_k}{Z_{k-1}}\int \pi_k(y)K_k\brakets{f^2}(y)\\
		&=  M\frac{Z_k}{Z_{k-1}} \pi_{k}\brakets{f^2}.
	\end{align}
	The second line is using the change of variables $y = T(x)$ and the expression of the importance weight $G_{k,T}$. The third line follows by \cref{assumption:uniformly_bounded} stating the weights $G_{k,T}(x)$ are bounded by a positive constant $M>0$ independently from $x\in \mathcal{X}$ and $T\in\mathcal{T}$. The fourth line follows by application of Cauchy--Schwarz while the last line is a consequence of the Markov kernel $K_k$ being invariant under $\pi_k$. Choosing $C =  M\frac{Z_k}{Z_{k-1}} $ gives the desired result.
\end{proof}

\subsection{Asymptotic stochastic equi-continuity}

\paragraph{Definition of asymptotic stochastic equicontinuity} 
For a positive scalar $\delta>0$ we denote by $\mathcal{G}_k(\delta)$ the intersection of  $\mathcal{G}$ with the $L_2(\pi_k)$-ball of radius $\delta$
	\begin{align}
		\mathcal{G}_k(\delta) := \braces{ f\in \mathcal{G}\middle|  \Verts{f}_{L_2(\pi_k)}\leq \delta }.
	\end{align}
Consider a stochastic processes $X$ indexed by $\mathcal{G}$, we define the uniform semi-norm over the set $\mathcal{G}_k(\delta)$ to be
\begin{align}
	\Verts{X}_{\mathcal{G}_k(\delta)} := \sup_{f\in \mathcal{G}_k(\delta)} \verts{X(f)}.
\end{align}
We will always be in the setting where $\Verts{X}_{\mathcal{G}_k(\delta)}$ is a random variable (i.e. measurable). 
Let now $X^N$ be a sequence of stochastic processes that are linear in their index ( i.e. $ f\mapsto X^N\brakets{f} $ is a linear map). In this case the sequence $X^N$ is said to be asymptotically stochastically equicontinuous if for any sequence $\delta^N$ of positive real numbers converging to $0$, it holds that $\Verts{X^N}_{\mathcal{G}_k(\delta^N)}$ converges to $0$ in probability. 

We are now ready to state the main result of this section establishing asymptotic stochastic equicontinuity of the sequence of processes $E_k^N$.

\begin{prop}[Asymptotic stochastic equicontinuity]\label{prop:asymptotic_equi_continuity}
Set $p=2$. Under \cref{assumption:kernel,assumption:kernel_2,assumption:uniformly_bounded,assumption:moments,assumption:potential,assumption:parametrization,assumption:transport}, and for any positive $C>0$ and any subclass $\mathcal{G}$ of  $\mathcal{LC}_p(C)$ with finite bracketing number $\mathcal{N}_{[]}(\delta,\mathcal{G}, L_{2}(\pi_k) )$ and finite locally uniform entropy \cref{eq:bracketing_entropy} $\mathcal{J}(\delta, \mathcal{G}, r)$ satisfying  $\sup_{\delta>0}\mathcal{J}\parens{\delta, \mathcal{G}, r }\lesssim r$, the process $f\mapsto \parens{E_k^N\brakets{f}}$ is asymptotically stochastically equicontinuous. In other words, for any sequence $\delta^N$ of positive real numbers converging to $0$,  $\Verts{E_k^N}_{\mathcal{G}_k(\delta^N)}$ converges to $0$ in probability.
\end{prop}
\begin{proof}
	We proceed by induction over $k$. The result holds by \cref{prop:inedependec_maximam_ineq} for $k=0$. By induction, we assume the property holds for the process $E^N_{k-1}$.  We then consider a function class $\mathcal{G}$ satisfying the conditions of the result. We will show that $\parens{E_{k}^N\brakets{f}}_{\mathcal{G}}$ is asymptotically stochastically equicontinuous. 
	
	To achieve this, we fix a sequence $\delta^N$ converging to $0$ and consider a function $f\in \mathcal{G}_k(\delta^N)$. We then use the following decomposition of $E_k^N\brakets{f}$
	\begin{align}
		E_k^N\brakets{f} = \underbrace{E_k^N\brakets{f}-\mathbb{E}_{k-1}\brakets{E_k^N\brakets{f}}}_{R_k^{N}\brakets{f}} + \underbrace{\mathbb{E}_{k-1}\brakets{E_k^N\brakets{f}}}_{P_k^N\brakets{f}}.
	\end{align}
	Taking the supremum over $\mathcal{G}_k(\delta^N)$ on both sides of the above inequality yields
	\begin{align}
		\Verts{E_k^N}_{\mathcal{G}(\delta^N)} \leq \Verts{R_k^N}_{\mathcal{G}_k(\delta^N)} + \Verts{P_k^N}_{\mathcal{G}_k(\delta^N)}.
	\end{align}
	By \cref{prop:inedependec_maximam_ineq} we have that $\Verts{R_k^N}_{\mathcal{G}_k(\delta^N)}\xrightarrow[N\rightarrow +\infty]{P} 0$. Thus, we only need to prove that $\Verts{P_k^N}_{\mathcal{G}_k(\delta^N)}\xrightarrow[N\rightarrow +\infty]{P} 0$. By direct computation, we have:
	\begin{align}
		P_k^N\brakets{f} = \frac{\gamma_{k-1}^N\brakets{1}}{\gamma_{k}^N\brakets{1}} E_{k-1}^N\brakets{Q_{k,T_k}\brakets{f-\pi_{k}\brakets{f}}}.
	\end{align}
	Now let us call $\mathcal{QG}$ be the set of functions of the form $Q_{k,T}\brakets{f}$ for some $f\in \mathcal{G}$ and $T\in \mathcal{T}$. \cref{prop:Continuity_Q} ensures also that $\pi_{k-1}\brakets{Q_{k,T_k}\brakets{f}}\leq C\pi_{k}\brakets{f}\leq C\delta^N$, hence $Q_{k,T_k}\brakets{f} \in \mathcal{QG}_{k-1}(C\delta^N)$. Therefore, we have shown that
	\begin{align}
		\Verts{P_k^N}_{\mathcal{QG}_k(\delta^N)}\leq  \frac{\gamma_{k-1}^N\brakets{1}}{\gamma_{k}^N\brakets{1}} \Verts{E_{k-1}^N}_{\mathcal{QG}_{k-1}(C\delta^N)}.
	\end{align}
	Finally, by \cref{prop:bracketing_entopy}, we know that $\mathcal{QG}$ belongs to $\mathcal{LC}_p(C')$ for some positive $C'$,  has a finite bracketing number $\mathcal{N}(\epsilon, \mathcal{QG}, L_2(\pi_{k-1}))$ and a finite locally covering entropy $\mathcal{J}(\delta,\mathcal{QG},r)$  for any $\delta>0$ and $r>1$ satisfying $\mathcal{J}(\delta,\mathcal{QG},r)\lesssim r $. Therefore, by the recursion assumption, we know that   $\Verts{E_{k-1}^N}_{\mathcal{QG}_{k-1}(C\delta^N)}\xrightarrow[N\rightarrow +\infty]{P} 0$. Moreover, recalling that $\gamma_{k}^N\brakets{1}\xrightarrow[N\rightarrow +\infty]{P} Z_{k-1}  $ for any $0\leq k\leq K$ by \cref{thm:WLLN}, we can directly conclude that $\Verts{P_k^N}_{\mathcal{G}_k(\delta^N)}\xrightarrow[N\rightarrow +\infty]{P} 0$.
\end{proof}

\begin{prop}\label{prop:inedependec_maximam_ineq}
Let $\mathcal{G}$ be the class defined in \cref{prop:asymptotic_equi_continuity} and denote by $R^N_k$ to be the process index by $\mathcal{G}$ and defined by:
	\begin{align}
		R_k\brakets{f} =  \sqrt{N}\parens{\pi_{k}^N\brakets{f} - \mathbb{E}_{k-1}\brakets{ \pi_{k}^{N}\brakets{f}  } }
	\end{align}
	with the convention that for $k=0$, $\mathbb{E}_{-1}$ is the expectation over the samples $\braces{X_0^{(i)}}_{1:N}$ from $\pi_0$. 
	Under \cref{assumption:kernel,assumption:moments,assumption:uniformly_bounded,assumption:transport} and for any sequence $\delta^N$ converging to $0$ and any $ 0 \leq k\leq K$, it holds that
	\begin{align}
		\Verts{R_k^N}_{\mathcal{G}_k(\delta^N) }\xrightarrow[N\rightarrow +\infty]{P} 0.
	\end{align}
\end{prop}
\begin{proof}
Let  $ 0\leq k\leq K $. We use the convention that $\mathcal{F}_{-1}$ is the empty sigma algebra and that $\mathbb{E}_{-1}$ is the expectation over the samples ${X_0^{(i)}}_{1:N}$ from $\pi_0$. This will allow the same treatment for both $k=0$ and $k>0$.
Note that, $\mathbb{E}_{k-1}\brakets{R_k^N\brakets{f}} = 0$ for any $f\in \mathcal{G}$. Hence, the process is centered conditionally on $\mathcal{F}_{k-1}$. Moreover, the particles $\braces{X_{k}^{(i)}}_{1:N}$ defining $R^N_k$  are i.i.d. conditionally on $\mathcal{F}_{k-1}$. Therefore, we can use a symmetrization approach as in  \citep[Lemma 9.6.1.]{del2004feynman}.    
	Let $\braces{\epsilon^{i}}_{1:N}$ be i.i.d. Bernoulli variables with $\mathbb{P}(\epsilon^{i}=1) = \mathbb{P}(\epsilon^{i}=-1)=\frac{1}{2}$ independent from $\mathcal{F}_{k}$. Conditionally on $\mathcal{F}_{k-1}$ the samples $\braces{X_k^{i}}_{1:N}$ are i.i.d., hence, the symmetrization inequality in \citep[Theorem 3.14]{Sen:2018a} holds conditionally  on $\mathcal{F}_{k-1}$:
	\begin{align}\label{eq:symmetrization_ineq}
		\mathbb{E}_{k-1}\brakets{\Verts{ R^N_k}_{\mathcal{G}_k(\delta^N)}   }\leq 2\sqrt{N}\mathbb{E}_{k-1}\brakets{\Verts{ m_{\epsilon}^{N}}_{\mathcal{G}_k(\delta_N)}}
	\end{align}
	where $m^N_{\epsilon} =  \frac{1}{N} \sum_{i=1}^N \epsilon_i\delta_{X_k^{i}}$. Now by conditioning on $\mathcal{F}_k$, we can apply Hoeffding inequality  \citep[Lemma 3.11]{Sen:2018a} which implies that
	\begin{align}
		\mathbb{P}_{k}\brakets{\sqrt{N}\parens{m_{\epsilon}^{N}\brakets{f}- m_{\epsilon}^{N}\brakets{h}   }>s  } 
		\leq 2 e^{-\frac{s^2}{2 \pi_k^{N}\brakets{\parens{f-h}^2} } }.
	\end{align}
	Following the proof of \citep[Theorem 4.8]{Sen:2018a}, we can use the Maximal inequality for sub-Gaussian processes \citep[Theorem 4.5]{Sen:2018a}, which implies that:
	\begin{align}\label{eq:maximal_ineq}
		\sqrt{N} \mathbb{E}_{k}\brakets{\Verts{ m_{\epsilon}^{N}}_{\mathcal{G}_k(\delta^N)}  } \lesssim \int_0^{\psi^N}\sqrt{\log\parens{\mathcal{N}(r, \mathcal{G}_k(\delta^N), L_2(\pi_{k}^{N})    )}}\diff  r, 
	\end{align} 
	where $\psi^N:=   \sup_{f\in \mathcal{G}(\delta^N)}\pi_{k}^N\brakets{f^2}$. 	
	By definition of the locally uniform covering number in \cref{eq:unif_covering_number}, we have that
\begin{align}
	\mathcal{N}\parens{r, \mathcal{G}_k(\delta^N), L_2(\pi_{k}^{N}) }\leq \mathcal{N}^{LU}\parens{  r, \mathcal{G}, r^N},
\end{align}
with $(r^N)^2 = \pi_{k}^N\brakets{H_2^2} $. 
Hence,  we can upper-bound $\sqrt{N} \mathbb{E}_{k}\brakets{\Verts{ m_{\epsilon}^{N}}_{\mathcal{G}(\delta^N)}  }$ using the locally uniform entropy so that \cref{eq:maximal_ineq} becomes 
\begin{align}\label{eq:maximal_ineq_2}
	\sqrt{N} \mathbb{E}_{k}\brakets{\Verts{ m_{\epsilon}^{N}}_{\mathcal{G}_k(\delta^N)}  }\lesssim \mathcal{J}\parens{\psi^N, \mathcal{G}, r^N  }.
\end{align}
We can now take the expectation of \cref{eq:maximal_ineq_2} conditionally on $\mathcal{F}_{k-1}$ and combine the resulting inequality with \cref{eq:symmetrization_ineq} yielding
	\begin{align}
		\mathbb{E}_{k-1}\brakets{\Verts{ R^N_k}_{\mathcal{G}_k(\delta^N)}   }\lesssim \mathbb{E}_{k-1}\brakets{ \mathcal{J}\parens{\psi^N, \mathcal{G}, r^N  }}.
	\end{align}
We only need to prove that $\mathbb{E}_{k-1}\brakets{ \mathcal{J}\parens{\psi^N, \mathcal{G}, r^N  }} \xrightarrow[N\rightarrow +\infty]{P} 0 $. To this end, we define the event
\begin{align}
	S^N_{\epsilon} = \braces{\verts{\pi_{k}^N\brakets{H_2^2}-\pi_k\brakets{H_2^2} }\geq \epsilon } 
\end{align} 
for some $\epsilon>0$ and its complement $\bar{S}^N_{\epsilon}$. We then decompose the expected entropy into to terms
\begin{align}
	\mathbb{E}_{k-1}\brakets{ \mathcal{J}\parens{\psi^N, \mathcal{G}, r^N  }}  = \mathbb{E}_{k-1}\brakets{ \mathcal{J}\parens{\psi^N, \mathcal{G}, r^N  }\bar{S}^N_{\epsilon}  } + \mathbb{E}_{k-1}\brakets{ \mathcal{J}\parens{\psi^N, \mathcal{G}, r^N  }S^N_{\epsilon}}.  
\end{align}

Conditionally on $\bar{S}^N_{\epsilon} $ we have that $r^N\leq \sqrt{\epsilon +  \pi_k\brakets{H_2^2}} := R$. Hence, by monotonicity of the entropy w.r.t. to the last argument (\cref{prop:locally_uniform_entropy} ) , it holds that $\mathcal{J}\parens{\psi^N, \mathcal{G}, r^N  }\bar{S}^{N}_{\epsilon}\leq \mathcal{J}\parens{\psi^N, \mathcal{G}, R  }$. Conditionally on $S^N_{\epsilon} $ we will rely on the assumption that $\sup_{\delta>0}\mathcal{J}\parens{\delta, \mathcal{G}, r  }\lesssim r $ which implies that $\mathcal{J}\parens{\psi^N, \mathcal{G},r^N  }S^N_{\epsilon}\lesssim  r^N   S^{N}_{\epsilon}$. 
Hence, we can write
\begin{align}\label{eq:main_ineq_entropy}
	\mathbb{E}_{k-1}\brakets{\Verts{ R^N_k}_{\mathcal{G}_k(\delta^N)}   }\lesssim \mathbb{E}_{k-1}\brakets{ \mathcal{J}\parens{\psi^N, \mathcal{G}, r^N  }} & 
	\lesssim
	 \mathbb{E}_{k-1}\brakets{\mathcal{J}\parens{\psi^N, \mathcal{G}, R  }} +  \mathbb{E}_{k-1}\brakets{  \pi_{k}^{N}\brakets{H_2^2}^{\frac{1}{2}}    S^{N}_{\epsilon} }\\
	 &\lesssim 
	 \mathbb{E}_{k-1}\brakets{\mathcal{J}\parens{\psi^N, \mathcal{G}, R  }} +  \mathbb{E}_{k-1}\brakets{ \pi_{k}^{N}\brakets{H_2^2}  }^{\frac{1}{2}} \mathbb{P}_{k-1}\brakets{ S^{N}_{\epsilon} }^{\frac{1}{2}},
\end{align}
where we used Cauchy--Schwarz inequality to get the second line. To control the first term, we first note that $\psi^N\xrightarrow[]{P} 0$. Indeed, by construction, $\pi_{k}[f^2]\leq (\delta^N)^2$ for any $f\in \mathcal{G}_k(\delta^N)$. Moreover, by assumption, $\mathcal{G}_k(\delta^N)$ has a finite bracketing number $\mathcal{N}_{[]}(\epsilon, \mathcal{G}_k(\delta^N), L_2(\pi_k))$ for any $\epsilon>0$ and the functions $f^2$ have a growth of at most $2p$. We can then apply \cref{prop:Glivenko_cantelli_braketing}, which holds under \cref{assumption:kernel,assumption:moments,assumption:uniformly_bounded,assumption:transport} to get the following uniform convergence in probability of $\pi_k^N\brakets{f^2}$  over $\mathcal{G}_k(\delta^N)$, i.e. $\sup_{ f\in \mathcal{G}_k(\delta^N) } \parens{\pi_{k}^{N}-\pi_k}\brakets{f^2}\xrightarrow[]{P}0$. This allows to conclude that 
	\begin{align}
		\psi^N \leq (\delta^N)^2 + \sup_{ f\in \mathcal{G}_k(\delta^N) } \parens{\pi_{k}^{N}-\pi_k}\brakets{f^2}\xrightarrow[]{P}0.
	\end{align}
Second, we know by \cref{prop:locally_uniform_entropy} that $\mathcal{J}\parens{\psi^N, \mathcal{G}, R  } $ is bounded by a finite quantity, i.e.  $ \sup_{\delta>0}\mathcal{J}\parens{\delta, \mathcal{G}, R  } <+\infty$ . Hence, by the dominated convergence theorem, it follows that $\mathbb{E}_{k-1}\brakets{ \mathcal{J}\parens{\psi^N, \mathcal{G}, R  }} \xrightarrow[]{P} 0  $. For the second term we have that $ \mathbb{E}_{k-1}\brakets{ \pi_k^{N}\brakets{H_2^2}}\xrightarrow[]{P} \pi_{k}\brakets{H_2^2}$ by \cref{thm:WLLN} with $\pi_{k}\brakets{H_2^2}$ being finite under \cref{assumption:moments}. Moreover, since we also have $\pi_k^{N}\brakets{H_2^2}\xrightarrow[]{P} \pi_{k}\brakets{H_2^2}$,  this necessarily implies that $\mathbb{E}\brakets{\mathbb{P}_{k-1}\brakets{S^N_{\epsilon}}} = \mathbb{P}\brakets{S_{\epsilon}^N}  \xrightarrow[]{P} 0$ and henceforth that $\mathbb{P}_{k-1}\brakets{ S^{N}_{\epsilon} }\xrightarrow[]{P} 0$.

We have shown that the r.h.s. of \cref{eq:main_ineq_entropy} converges to $0$ in probability. 
This directly implies $\mathbb{E}_{k-1}\brakets{\Verts{ R^N_k}_{\mathcal{G}_k(\delta^N)}   }  \xrightarrow[]{P} 0 $ and henceforth that $ a_{\epsilon}^N := \mathbb{P}_{k-1}\brakets{\Verts{ R^N_k}_{\mathcal{G}_k(\delta^N)}\geq \epsilon}  \xrightarrow[]{} 0 $  for any choice of $\epsilon>0$ by Markov inequality. Noting that $a_{\epsilon}^N$ is bounded for a given $\epsilon>0$, we get that $\mathbb{E}\brakets{a_{\epsilon}^N}\rightarrow 0$ which exactly means  $\Verts{ R^N_k}_{\mathcal{G}_k(\delta^N)}\xrightarrow[]{P} 0 $.

\end{proof}

We then have the following result which is a consequence of \cref{prop:asymptotic_equi_continuity}
\begin{prop}\label{cor:application_equicontinuity}
	Let $f$ be in $\mathcal{LC}_p$ and consider the family of function $\mathcal{QG}$ of the form $S_{\theta}(x) =  Q_{k,\tau_{\theta}}\brakets{f}(x)- Q_{k,\tau_{\theta^{\star}}}\brakets{f}(x)$ indexed by the parameter $\theta\in \Theta$. Under \cref{assumption:kernel,assumption:kernel_2,assumption:uniformly_bounded,assumption:moments,assumption:potential,assumption:parametrization,assumption:transport} and for any random sequence $g^N$ in $\mathcal{QG}$ such that $\pi_{k-1}\brakets{(g^N)^2}\xrightarrow[]{P}0$ it holds that $E^N_{k-1}\brakets{g^N} \xrightarrow[]{P}0$.
\end{prop}

\begin{proof}
  By \cref{prop:bracketing_entopy} we know that $\mathcal{QG}$ belongs to $\mathcal{LC}_p(C)$ for some positive $C$,  has a finite bracketing number $\mathcal{N}(\delta, \mathcal{QG}, L_2(\pi_{k-1}))$, a finite \textit{locally uniform entropy} $\mathcal{J}(\delta, \mathcal{QG}, r)$ for any $r\geq 1$ and $\delta> 0$ and satisfying a linear growth condition $\sup_{\delta>0}\mathcal{J}(\delta, \mathcal{QG}, r) \lesssim r $. Hence, by direct application of \cref{prop:asymptotic_equi_continuity}, we know that the process $g\mapsto  \sqrt{N}\parens{ \pi_{k-1}^N - \pi_{k-1}  }\brakets{g}:= E_{k-1}^N\brakets{g} $  indexed by functions $g\in \mathcal{QG}$ is asymptotically stochastically equicontinuous, meaning that $ \Verts{E^N_{k-1}}_{\mathcal{QG}(\delta^N)}  \xrightarrow[]{P} 0$ for any sequence of positive numbers $\delta^N$ converging to $0$ and where by definition:
  \begin{align}
  	\Verts{E^N_{k-1}}_{\mathcal{QG}(\delta^N)} := \sup_{\substack{g\in \mathcal{QG}\\ \pi_{k-1}\brakets{g^2} \leq \parens{\delta^N}^2 }} \verts{E^N_{k-1}\brakets{ g} }.  
  \end{align}
  Denote by $b^N:= \pi_{k}\brakets{(g^N)^2}$ and fix $\epsilon>0$. We know that for any sequence $R^N$, the following inclusion of events hold:
 \begin{align}
 	\braces{E^N_k\brakets{g^N}>\epsilon} \cap  \braces{b^N >R^N }\subset \braces{b^N >R^N }\\
 	\braces{E^N_k\brakets{g^N}>\epsilon} \cap  \braces{b^N \leq R^N }\subset \braces{\Verts{E^N_{k-1}}_{\mathcal{QG}(R^N)} \geq \epsilon}.
 \end{align}
 Hence, the following inequality holds
 \begin{align}
  	\mathbb{P}\brakets{E^N_k\brakets{g^N}>\epsilon} &= \mathbb{P}\brakets{\braces{E^N_k\brakets{g^N}>\epsilon} \cap  \braces{b^N >R^N } } + \mathbb{P}\brakets{\braces{E^N_k\brakets{g^N}>\epsilon} \cap  \braces{b^N \leq R^N } }\\
  	&\leq  \mathbb{P}\brakets{b^N >R^N } + \mathbb{P}\brakets{ \Verts{E^N_{k-1}}_{\mathcal{QG}(R^N)} \geq \epsilon  }.
 \end{align}
 Moreover, by assumption, we have $b^N:= \pi_{k}\brakets{(g^N)^2}\xrightarrow[]{P}0$. Hence, by \cref{lem:characterization_conv_proba}, there exists a deterministic sequence $R^N$ converging to $0$ and such that
  \begin{align}
  	\mathbb{P}\brakets{b^N>R^N}\leq R^N.
  \end{align}
We therefore have
\begin{align}
	\mathbb{P}\brakets{E^N_k\brakets{g^N}>\epsilon}\leq R^N + \mathbb{P}\brakets{ \Verts{E^N_{k-1}}_{\mathcal{QG}(R^N)} \geq \epsilon  }.
\end{align}
  The first term converges to $0$ by definition of $R^N$ so does the second term by asymptotic stochastic equicontinuity of $E^N_{k-1}$ and since $R^N\rightarrow 0$. Therefore, we have shown that $\mathbb{P}\brakets{E^N_k\brakets{g^N}>\epsilon} \rightarrow 0$.
  
\end{proof}

\begin{lem}\label{lem:characterization_conv_proba}
If $b^N$ is non-negative r.v. such that $b^N\xrightarrow[]{P}0$, then there exists a deterministic sequence $R^N$ converging to $0$ such that:
\begin{align}
	\mathbb{P}\brakets{ b^N>R^N } \leq R^N.
\end{align}
\end{lem}
\begin{proof}
	Using  \citep[Theorem 9.2.2.]{Dudley:2018}, $b^N\xrightarrow[]{P}0$ is equivalent to having $r^N \rightarrow 0  $ where
  \begin{align}
  	r^N := \inf\braces{  \epsilon\geq 0 \middle| \mathbb{P}\brakets{ b^N>\epsilon }\leq \epsilon}.
  \end{align}
  Moreover, by definition of $r^N$,  for any $N$ there exists $a^N\leq \frac{1}{N} $ such that $R^N := r^N + a^N$ satisfies
  \begin{align}
  	\mathbb{P}\brakets{ b^N>R^N }\leq R^N.
  \end{align}  
  Since $R^N\rightarrow 0$, we have constructed a sequence $R^N$ converging slowly enough to $0$ so that $\mathbb{P}\brakets{ b^N>R^N } \rightarrow 0$.
\end{proof}

\subsection{Proof of \cref{prop:C} }\label{sec:proof_C}

\begin{proof}[Proof of \cref{prop:C} ]
We first decompose $P_{N}$ as $P_N = A_N + B_N$ with  
\begin{align}\label{eq:decomposition_unormalized_2}
	A_N = &  \sqrt{N}\gamma_{k-1}^N[1] \parens{\pi_{k-1}^N -\pi_{k-1} }\brakets{Q_{k,T_k}[f]-Q_{k,T_k^{\star}}[f]}, \\ 
	B_N=& \sqrt{N}\parens{ \gamma_{k-1}^N\brakets{Q_{k,T_k^{\star}}[f]} - \gamma_{k-1}\brakets{Q_{k,T_k^{\star}}[f]}  },
\end{align}
where we used the identity $\pi_{k-1}[Q_{k,T_k}[f]]=\pi_{k-1}[Q_{k,T_k^{\star}}[f]]$ from \cref{prop:independence_from_T} to express $A_N$ as a difference in expectations under $\pi_{k-1}^N$ and $\pi_{k-1}$.
By induction, $B_N$ converges in distribution to a centered Gaussian with variance  $\mathbb{V}_{k-1}^{\gamma}[Q_{k,T_k^{\star}}[f]]$. 

We only need to show that $A_N\xrightarrow{P}0$. To this end, for the given function $f\in \mathcal{LC}_p$, we consider the family of functions $\mathcal{QG}$ of the form $S_{\theta}(x) =  Q_{k,\tau_{\theta}}\brakets{f}(x)- Q_{k,\tau_{\theta^{\star}}}\brakets{f}(x)$ indexed by the parameter $\theta\in \Theta$. We will first prove that $\pi_{k-1}\brakets{ S_{\theta^N}^2 }\xrightarrow[]{P}0$ then apply \cref{cor:application_equicontinuity} to show that $ \sqrt{N}\parens{\pi_{k-1}^N - \pi_{k-1}} \brakets{S_{\theta^N}}\xrightarrow[]{P}0$. 
By \cref{prop:Lipschiz_Q}, we have that
	\begin{align}\label{eq:lip_S}
		 \verts{S_{\theta^N}(x)} = \verts{Q_{k, T_k}\brakets{f}(x) - Q_{k,T_k^{\star}}\brakets{f}(x)}\lesssim \Verts{\theta_k^{N}-\theta_k^{\star}}H_2(x)
	\end{align}
	where $H_2(x) = 1+\Verts{x}^{p+2} $. Moreover, we know by \cref{assumption:moments} that $\pi_{k-1}\brakets{H_2^2}$ is finite. 
	We can square \cref{eq:lip_S} and take the expectation under $\pi_{k-1}$ to get
	\begin{align}
		\pi_{k-1}\brakets{ S_{\theta^N}^2 }  \lesssim \Verts{\theta_k^{N}-\theta_k^{\star}}^2\pi_{k-1}\brakets{H_2^2} <+\infty.
	\end{align}
	Since $\Verts{\theta_k^{N}-\theta_k^{\star}}\xrightarrow[]{P}0$ by \cref{prop:convergence_parameter} we conclude that  $\pi_{k-1}\brakets{ S_{\theta^N}^2 }\xrightarrow[]{P}0$. We can then apply \cref{cor:application_equicontinuity} to the sequence $g^N := S_{\theta^N}$ to get that $ \sqrt{N}\parens{\pi_{k-1}^N - \pi_{k-1}} \brakets{S_{\theta^N}}\xrightarrow[]{P}0$. We also know by \cref{thm:WLLN} that $\gamma_{k-1}^N[1]\xrightarrow{P}\gamma_{k-1}[1]$. Therefore, it follows directly that $A_N\xrightarrow{P}0$ concluding the proof. 
\end{proof}

\section{Continuous-time limit}\label{sec:proofs_continuous_time}
\subsection{Notation and General setting}\label{sec:notation}
We start by introducing some notations.
\paragraph{Continuous-time path of probability measures.}
Let $(\Pi_{t})_{[0,1]}$ be a continuous-time path of probabilities connecting $\pi_0$ to $\pi$. We write $V_t(x)$ and $Z_t$ to denote the known potential and unknown normalizing constant of $\Pi_t$.
 We assume that the discrete auxiliary targets $\pi_k$ are of the form $\pi_k = \Pi_{t_k}$  where $t_k = k\lambda$ and $\lambda =  \frac{1}{K}$. In this case, increasing $K$ amounts to decreasing the \textit{step-size} $\lambda$.
\paragraph{Markov Kernels.}
We consider the ULA kernels $K_k$ of the form:
\begin{align}
	K_k(x,\diff y ) = \mathcal{N}(y; x -\lambda V_k(x),2\lambda I )\diff y
\end{align}
where $\mathcal{N}(x;m,\Sigma)$ is the Gaussian density with mean $m$ and covariance $\Sigma$.
\paragraph{Normalizing Flows.}
For a given fixed step-size $\lambda>0$, we will assume that the normalizing flows $T$ are of the form:
\begin{align}\label{eq:normalizing_flow_continuous_time}
	T(x) = x+\lambda A_{\theta}(x)
\end{align}
where  $(\theta,x) \mapsto A_{\theta}(x)$ is a function defined from $\Theta \times \mathcal{X}$ to $\mathcal{X}$, with $\Theta$ being a compact subset of $\mathbb{R}^p$.
We then call $\mathcal{T}_{\lambda}$ the set of normalizing flows defined by varying the parameter $\theta$:
\begin{align}\label{eq:NFs_lambda}
	\mathcal{T}_{\lambda} = \braces{ x\mapsto x+\lambda A_{\theta}(x) \quad \middle |\quad   \theta \in \Theta}.
\end{align} 
To ensure that all functions in  $\mathcal{T}_{\lambda}$ are diffeomorphisms, we will require $A_{\theta}(x)$ to satisfy  \cref{assump_cont:Lipschiz_control} and for $\lambda$ small enough, the following lemma ensures that all functions in  $\mathcal{T}_{\lambda}$ are indeed diffeomorphisms and thus define a valid family of normalizing flows: 
\begin{lem}\label{lem:diffeo}
Under \cref{assump_cont:Lipschiz_control} and if $\lambda\leq \frac{1}{2L}$, then any element in $\mathcal{T}_{\lambda}$ is a diffeomorphism.
\end{lem}
\begin{proof}
Let $\theta$ be  in $\Theta$. The map $T(x) = x+\lambda A_{\theta}(x)$ satisfies $\Verts{T(x)}\rightarrow +\infty$ whenever $\Verts{x}\rightarrow +\infty$ since $x\mapsto \lambda A_{\theta}(x)$ has a growth in $x$ of at most $\frac{1}{2}\Verts{x}$. Therefore, $T$ is proper. Moreover, the Jacobian $\nabla T(x)$ is invertible for any $x\in \mathcal{X}$. Thus by the Hadamard--Caccioppoli Theorem, one can conclude that $T$ is a diffeomorphism.
\end{proof}
\paragraph{Time-dependent control.}
To describe the continuous-time limit, we consider time dependent controls of the form $\alpha_t(x) = A_{\theta_t}(x)$, where $t\mapsto\theta_t$ is a trajectory in parameter space $\Theta$. More precisely, we consider the set  $\textup{Lip}([0,1],\Theta)$ of $1$-Lipschitz functions $t\mapsto \theta_t$ from $[0,1]$ to $\Theta$ and we define the set $\mathcal{A}$ of admissible controls as  
\begin{align}
	\mathcal{A} := \braces{\alpha: (t,x)\mapsto A_{\theta_t}(x)  \quad \middle|\quad  t\mapsto\theta_t \in \textup{Lip}([0,1],\Theta)
		}.
\end{align}
The following smoothness properties  are a direct consequence of \cref{assump_cont:Lipschiz_control}.
\begin{lem}\label{lem:smoothness_control}
	Under \cref{assump_cont:Lipschiz_control}, any control in $\mathcal{A}$ satisfies:
\begin{align}\label{eq:admissibility}
			 \Verts{ \alpha_t(x) - \alpha_{t'}(x')} &\leq L\parens{ \Verts{ x-x'} + \verts{ t-t'}  },\\
		 \Verts{ \nabla\alpha_t(x) - \nabla\alpha_{t'}(x')} &\leq L\parens{ 1 + \Verts{ x} + \Verts{ x'} }\parens{\Verts{ x-x'} + \verts{ t-t'} },\\
		\Verts{\alpha_0(0)}&\leq C,
\end{align}
with $C = \sup_{\theta\in\Theta} \vert A_{\theta}(0) \vert  <+\infty$.
\end{lem}

\paragraph{Optimal controls and Normalizing Flows.}
For a given step-size $ \frac{1}{2L}\geq \lambda>0$, we denote by $T^{\lambda}_k$ an NF that minimizes $\mathcal{L}_k$ over the set $\mathcal{T}_{\lambda}$ defined in \cref{eq:NFs_lambda} and we denote by $\theta_k^{\lambda}$ its corresponding parameter in $\Theta$, 
\begin{align}
	T^{\lambda}_k(x) = x+\lambda A_{\theta_k^{\lambda}}(x).
\end{align}
 The  NF $T^{\lambda}_k$ and parameter $\theta_k^{\lambda}$ exist since the set $\Theta$ is compact and the maps $\theta\mapsto \mathcal{L}_k(I+\lambda A_{\theta})$ is continuous.
 For any $\lambda=\frac{1}{K}$ with $K$ being a positive integer, we will later assume there exists at least a solution $\alpha^{\lambda}\in \mathcal{A}$ that interpolates between the optimal NFs $T_k^{\lambda}$ for all $0\leq k\leq K$, i.e.:
\begin{align}
	T_k^{\lambda}(x) = x+\lambda \alpha_{t_k}^{\lambda}(x), \qquad \forall k\in [1:K].
\end{align}
It is easy to see that the loss $\mathcal{L}_{\lambda}^{tot}(\alpha)$ defined in \cref{eq:L_tot} admits $\alpha^{\lambda}$ as a minimizer over the set of admissible controls $\mathcal{A}$:
\begin{align}\label{eq:L_tot}
	\mathcal{L}_{\lambda}^{tot}(\alpha):= \frac{1}{\lambda} \sum_{k=1}^K \mathcal{L}_k(I+\lambda\alpha_{t_k}).
\end{align}
\paragraph{Feynman--Kac measures.}
We will introduce a family of Feynman--Kac measures indexed by $\mathcal{A}$. For any element $\alpha\in \mathcal{A}$, we denote by $b^{\alpha}_{t}(x)$ the controlled drift:
\begin{align}
	b_{t}^{\alpha}(x) = \alpha_t(x) - \nabla_x  V_t(x).
\end{align} 
We first start by introducing a family of non-homogeneous stochastic differential equations: 
\begin{align}\label{eq:SDE}
	\diff X_t^{\alpha} = b^{\alpha}_{t}(X^{\alpha}_t)\diff t + \sqrt{2}\diff B_t, \qquad X^{\alpha}_0\sim \pi_0,
\end{align}
where $(B_t)_{t\geq 0}$ is a standard multivariate Brownian motion. Here, we emphasize the dependence of the controlled drift $b^{\alpha}_{t}$ and the process $(X_t^{\alpha})_{t\geq 0 }$ on the control $\alpha$. We will sometimes simply write $b_t$ and  $X_t$ when clear from context. 
We denote by $\bar{\Lambda}^{\alpha}_t$ the joint distribution of such process up to time $t$ and by $\Lambda^{\alpha}_t$ its marginal at time $t$ which satisfies the following continuity equation:
\begin{align}\label{eq:appendix_continui_eq}
	\partial_t \Lambda^{\alpha}_t = \nabla \cdot b_t^{\alpha}  + \Delta \Lambda^{\alpha}_t.
\end{align}

We can then consider the following continuous-time importance weights over the process $X_t^{\alpha}$:
\begin{align}\label{eq:un_normalized_weights}
	w^{\alpha}_{t} = \exp\parens{\int_0^t  g^{\alpha}_{s}(X^{\alpha}_s)\diff s  },\qquad  g^{\alpha}_{t} = \nabla\cdot \alpha_t- \nabla_x V_t^{\top}\alpha_t - \partial_t V_t.
\end{align}
As shown later in \cref{prop:controlled_non_equilibrium_dynamics_with_birth_death} in Appendix \ref{sec:continuoustimeIS}, the normalizing constant $\bar{\Lambda}_t^{\alpha}[w_t^{\alpha}]$ of the weights $w_t^{\alpha}$ is given by $\exp\parens{-\int_0^t \Pi_s\brakets{\partial_s V_s} \diff s}$. This allows to define the normalized weights as:
\begin{align}\label{eq:normalized_weights}
	\overline{w}^{\alpha}_{t} = \exp\parens{\int_0^t  \overline{g}^{\alpha}_{s}(X^{\alpha}_s)\diff s  },\qquad  \bar{g}^{\alpha}_{t} = g^{\alpha}_{t} + \Pi_{t}\brakets{\partial_t V_t}.
\end{align}
The function $g^{\alpha}_{t}(x)$ is called the \textit{
instantaneous work} and measures how much $\Lambda_t$ is different from $\Pi_t$.
We can now introduce the family of Feynman--Kac measures defined by the importance weights over the sample path $(X_t^{\alpha})_{[0,1]}$ from the process  defined in \cref{eq:SDE}: 
\begin{align}\label{eq:Feynman-kac_appendix}
	\bar{\Pi}^{\alpha}_t[f] = \frac{\bar{\Lambda}^{\alpha}_t\brakets{ w^{\alpha}_{t}f}}{\bar{\Lambda}^{\alpha}_t\brakets{w^{\alpha}_{t}}} = \bar{\Lambda}_t^{\alpha}\brakets{\overline{w}_t^{\alpha} f },
\end{align}
where $f$ is any bounded continuous functional defined over the space of admissible sample paths $(X_t^{\alpha})_{[0,1]}$. We will see in \cref{prop:controlled_non_equilibrium_dynamics_with_birth_death} that  the marginal of $\bar{\Pi}^{\alpha}_t$ at time $t$ is exactly $\Pi_t$.

We consider now the \textit{expected instantaneous variance} of $g_{t,\alpha}$ defined as:
\begin{align}\label{eq:M_control}
	\mathcal{M}(\alpha) := \frac{1}{2}\int_0^1 \parens{\Pi_t\brakets{\parens{g^{\alpha}_{t}}^2} -  \parens{\Pi_t\brakets{g^{\alpha}_{t} }}^2}\diff t. 
\end{align}
We will show that $\mathcal{M}(\alpha)$ is the limit function of \cref{eq:L_tot}. We will assume that the infimum of $\mathcal{M}(\alpha) $ over $\mathcal{A}$ is achieved for an element $\alpha^{\star}\in \mathcal{A}$:
\begin{align}\label{eq:minimizer_continous_time}
	\mathcal{M}(\alpha^{\star}) = \inf_{\alpha\in \mathcal{A}}\mathcal{M}(\alpha). 
\end{align}
\paragraph{Interpolating measures.}
For a given $\lambda = \frac{1}{K}$ and a control $\alpha\in \mathcal{A}$, we define the following  functions:
\begin{align}\label{eq:discrete_drift}
	\beta_t^{\alpha,\lambda}(x) &= \alpha_{t}(x)-\nabla V_t(x+ \lambda \alpha_{t}(x) ),\\
	\delta_t^{\alpha,\lambda}(x) &= \frac{1}{\lambda}\parens{V_{t-\lambda}(x) -V_t(x+\lambda\alpha_t(x))+\log\verts{I+\lambda \nabla \alpha_t(x)}}
\end{align}
with $\delta_t^{\alpha,\lambda}(x)$ being defined for $t\geq\lambda$.   
 The function $\beta_t^{\alpha,\lambda}$ allows us to introduce the non-anticipative drift function $b_t^{\alpha,\lambda}$ 
which depends on the path of a process  $X_{[0,1]}$ up to time $t$:
\begin{align}
	b^{\alpha,\lambda}_{t}(X) = \sum_{k=1}^{K}  \beta_{t_k}^{\alpha,\lambda}(X_{t_{k-1}}) \mathds{1}_{[t_{k-1},t_{k})}(t).
\end{align}
We can then consider the continuous-time process $(X^{\alpha,\lambda}_{t})_{[0,1]}$ defined as:
\begin{align}\label{eq:interpolating_SDE}
	\diff X^{\alpha,\lambda}_t = b^{\alpha,\lambda}_{t}\parens{X^{\alpha,\lambda}}\diff t + \sqrt{2}\diff B_t,\qquad X_0^{\alpha,\lambda} \sim \pi_0.
\end{align}
We denote by $\bar{\Lambda}^{\alpha,\lambda}_t$ the joint distribution of $X^{\alpha,\lambda}_{[0,t]}$. 
To introduce the interpolating measure, we start by defining the instantaneous work:
\begin{align}\label{eq:discrete_instantaneous_work}
	g_{s}^{\alpha,\lambda}(X) = \delta_{t_{k}}^{\alpha,\lambda}\parens{X_{t_{k-1}}}   ,\qquad  t_{k-1}\leq s< t_k.
\end{align}
It is easy to see that $ \delta_{t_{k}}^{\alpha,\lambda}(x)= \frac{1}{\lambda} \log G_{k,T_{k}^{\alpha,\lambda}}(x)$ where $T_k^{\alpha,\lambda}$ are defined from the control $\alpha$ using $T_k^{\alpha,\lambda} =  x + \lambda \alpha_{t_k}(x)$.
We define the truncated time $\tau_{\lambda}(t) = \lambda \floor{\frac{t}{\lambda}}$ and the corresponding index $k_{\lambda}(t) =  \floor{\frac{t}{\lambda}}$. We then introduce the following functions:
\begin{align}\label{eq:discrete_importance_weights}
	w_t^{\alpha,\lambda}(X_{[0,t]}) :=  \exp\parens{ \int_0^{\tau_{\lambda}(t)} g_{s}^{\alpha,\lambda}(X) \diff s   },\qquad 	r_t^{\lambda}(X_{[0, t]}) = \exp\parens{ -\sum_{k=1}^{k_{\lambda}(t)} h_{t_k}^{\lambda}(X_{t_k}) }
\end{align}
where $h_{t}^{\lambda}$ is defined as
\begin{align}\label{eq:kernel_work}
	h^{\lambda}_{t}(x) &= \log \parens{ \int \exp(V_t(x)-V_t(y)-\frac{1}{4\lambda}\Vert x-y-\lambda \nabla V_t(y) \Vert^2  )\diff y}- \frac{d}{2}\log(4\pi\lambda).
\end{align}
The function $w_t^{\alpha,\lambda}$ represents the correction due to the use of the control $\alpha$, while $r_t^{\lambda}$ is the correction due to the transition kernel $K_k$ being only approximately invariant w.r.t. $\pi_{k}$.
We can now introduce the interpolating measures $\bar{\Pi}^{\lambda}_t$ defined as:
\begin{align}\label{eq:Feynman_Kac_interpolation}
	\bar{\Pi}^{\alpha\lambda}_t[f] = \frac{\bar{\Lambda}^{\alpha,\lambda}_t\brakets{w_t^{\alpha,\lambda} r_t^{\lambda}f}}{\bar{\Lambda}^{\lambda}_t\brakets{w_t^{\alpha,\lambda} r_t^{\lambda}}}. 
\end{align}

\paragraph{Additional notations}
We introduce discrepancy measures between controls $\alpha$ and $\alpha'$:
\begin{align}\label{eq:Sobolev_distances}
	\mathcal{S}(\alpha,\alpha') &= \int_0^1 \Pi_{t}\brakets{\Verts{ \alpha_{t} - \alpha_t' }^2 + \Verts{ \nabla \alpha_{t} - \nabla \alpha_t' }^2  } \diff t, \\
	\mathcal{S}_t^{\lambda}(\alpha,\alpha') &= \lambda \sum_{k=1}^{k_{\lambda}(t)} \Pi_{t_{k-1}}\brakets{\Verts{\alpha_{t_{k}} - \alpha_{t_{k}}' }^2 +\Verts{\nabla\alpha_{t_{k}} - \nabla\alpha_{t_{k}}' }^2}.
\end{align} 

For a function $c: [0,1]\times \mathcal{X}\rightarrow \mathbb{R}$, we define the scalar:
\begin{align}\label{eq:discrete_time_expectations}
	\mathbb{M}^{\lambda}_t\brakets{c} := \lambda \sum_{k=1}^{k_{\lambda}(t)} \Pi_{t_{k-1}}\brakets{c_{t_k}}.
\end{align}
In all what follows, for two real numbers $A$ and $B$, the relation $A \lesssim B $ mean that there exists a positive constant $C>0$ that is uniform over $t\in [0,1]$ and over the set of admissible controls $\mathcal{A}$  and a value $\lambda_0$ such that $A\leq C B$ for all $\lambda\leq \lambda_0$.

\emph{In the rest of the paper, when we write $\mathbb{E}$, the expectation is w.r.t. $\bar{\Lambda}^{\alpha}$ unless stated otherwise. To simplify notation, we also write $X_s$ in place of $X^{\alpha}_s$ for the process satisfying the SDE in \cref{eq:SDE} with control $\alpha$.} 

\subsection{Assumptions}\label{sec:assumptions_continuous}
\begin{assumplist2}\label{assumptions}
	\item\label{assump_cont:lipschitz_potential} For some  fixed $L>0$, $V_t(x)$ is continuously differentiable in $(t,x)$ with  $\nabla V_t(x)$ being $L$-Lipschitz jointly in $(t,x)$. Moreover,  $\partial_t V_t(x)$ satisfies: 
	\begin{align}
	    \verts{\partial_t V_{t}(x) - \partial_t V_{t'}(x')} & \leq L\parens{1 + \Verts{x} + \Verts{ x'}}\parens{ \Verts{ x-x'}+ \verts{ t-t'} }. 
	\end{align}
	\item\label{assump_cont:Lipschiz_control} $A_{\theta}(x)$ is continuously differentiable in $(\theta,x)$,   $L$-Lipschitz jointly in $(\theta,x)$ and satisfies:
\begin{align}
	\Verts{\nabla A_{\theta}(x)- \nabla A_{\theta'}(x') }\leq L\parens{ 1 + \Verts{ x} + \Verts{ x'} }\parens{\Verts{ x-x'} + \verts{ \theta-\theta'} },
\end{align} 
for all $x,x'\in \mathcal{X}$ and $\theta,\theta'\in \Theta$.
	\item\label{assump_cont:moment} $\Pi_t$ admits finite $4$-th order moments uniformly bounded for $t\in [0,1]$. 
\end{assumplist2}

To provide the main convergence result, we need to strengthen the assumption on the moments of $\Pi_t$: 
 \begin{assumplist2}[resume]
 	\item\label{assump_cont:moment_2} There exists $c>0$ such that for any $0\leq \lambda < c$ the expectation  $\Pi_t\brakets{  \exp\parens{\lambda \Verts{\nabla V_t}^2}} $ is finite and uniformly bounded for $t\in [0,1]$.
	\end{assumplist2} 
	We will then assume the existence of admissible solutions $\alpha^{\star}$ minimizing $\mathcal{M}$ and $\alpha^{\lambda}$ minimizing $\mathcal{L}^{tot}_{\lambda}(\alpha)$ for any $\lambda = \frac{1}{K}$. 
	We will also need an assumption on the local behavior of the loss $\mathcal{M}(\alpha)$ near the optimum $\alpha^{\star}$. This local behavior will be controlled in terms of the following discrepancy:
	\begin{align}
	    \mathcal{S}(\alpha,\alpha') = \int_0^1 \Pi_{t}\brakets{\Verts{ \alpha_{t} - \alpha_t' }^2 + \Verts{ \nabla \alpha_{t} - \nabla \alpha_t' }^2  } \diff t.	
	 \end{align}
\begin{assumplist2}[resume]
	\item \label{assump_cont:separated_optimum} The minimizer $\alpha^{\star}$ of $\mathcal{M}(\alpha)$ exists in $\mathcal{A}$ and is unique. Moreover, for any $\delta>0$ it holds that
	\begin{align}
		\mathcal{M}(\alpha^{\star})< \inf_{\mathcal{S}(\alpha,\alpha^{\star})>\delta} \mathcal{M}(\alpha).
	\end{align}
	Finally, there exists $\delta_0>0$ such that for all $\alpha\in \mathcal{A}$ satisfying  $S(\alpha^{\star},\alpha)\leq \delta_0$, it holds that
	\begin{align}
		S(\alpha^{\star},\alpha)\lesssim \mathcal{M}(\alpha)- \mathcal{M}(\alpha^*).
	\end{align}
	\item\label{assump_cont:existence_interpolating_controls}
	For any $\lambda=  \frac{1}{K}$ with $K$ a positive integer, there exists at least a solution $\alpha^{\lambda} \in \mathcal{A}$, such that $\alpha^{\lambda}$ interpolates between the optimal NF $T_k^{\lambda}$ that minimizes $\mathcal{L}_{k}(T)$, i.e.: $T_k^{\lambda}(x) = x + \lambda \alpha_{t_k}^{\lambda}(x)$.
\end{assumplist2}
	Finally, a control $\alpha\in \mathcal{A}$ is said to induce bounded weight if the following assumption hold:
\begin{assumplist2}[resume]
    	\item \label{assum_cont:bounded_weights} The functions $g_{t}^{\alpha^{\star}}$ are bounded from above by a constant $C$ at all time $ t\in [0,1]$.
\end{assumplist2}

\subsection{Continuous-time importance sampling}\label{sec:continuoustimeIS}
\begin{prop}\label{prop:controlled_non_equilibrium_dynamics_with_birth_death}
Under \cref{assump_cont:lipschitz_potential,assump_cont:Lipschiz_control,assump_cont:moment} and if the control $\alpha$ satisfies \cref{assum_cont:bounded_weights}, then the marginal at time $t$ of $\bar{\Pi}^{\alpha}_t$ defined in \cref{eq:Feynman-kac_appendix}  is equal to $\Pi_{t}$.
Moreover, we have $\bar{\Lambda}_t^{\alpha}[w_t^{\alpha}]$ of the weights $w_t^{\alpha}$ is given by $\exp\parens{-\int_0^t \Pi_s\brakets{\partial_s V_s} \diff s}$.
\end{prop}  
\begin{proof} 
We know that $\alpha$ satisfies \cref{eq:admissibility} of \cref{lem:smoothness_control} thanks to \cref{assump_cont:Lipschiz_control}. This, in addition to  \cref{assump_cont:lipschitz_potential,assump_cont:moment}, ensures the SDE defined in \cref{eq:SDE} is well defined and admits finite 4th order moments by \cref{lem:main_bounds}.
	Denote by $\rho_t$ the marginal of  $\bar{\Pi}_t$ at time $t$. Let $f$ be a bounded smooth function $f$ of $X_t$ at time $t$. Then, by definition of $\rho_t$, we have:
\begin{align}\label{eq:importance_sampling}
	\rho_t\brakets{f} = \frac{\bar{\Lambda}_t^{\alpha}\brakets{w_t^{\alpha} f} }{\bar{\Lambda}_t^{\alpha}\brakets{w_t^{\alpha}}}= \frac{\mathbb{E}\brakets{w_t^{\alpha} f(X_t)} }{\mathbb{E} \brakets{w_t^{\alpha}}}.
\end{align}
Such quantity is finite since the importance weights $w_t^{\alpha}$ are bounded by \cref{assum_cont:bounded_weights}. It is sufficient to show that both $\rho_t$ and $\Pi_t$ satisfy the same partial differential equation with the same initial condition. 
Let us first express the time derivative of the process $w_s^{\alpha} f(X_s)$, which is obtained using It\^{o}'s formula:
\begin{align}
	\frac{\diff}{\diff s} \brakets{w_s^{\alpha} f(X_s)} =&\frac{\diff}{\diff s} \brakets{\exp\parens{\int_0^s g_u^{\alpha}\parens{X_u}\diff u } f(X_s)}\\
	=& w_s^{\alpha}\parens{g_s^{\alpha}(X_s) f(X_s) + \nabla f(X_s)^{\top} \parens{\alpha_s(X_s)-\nabla V_s(X_s)} + \Delta f(X_s) }\\
	 &+ \sqrt{2}w_s^{\alpha}\nabla f(X_s)^{\top}\diff B_s.
\end{align}
Integrating the above expression in time on the interval $[t,t+h]$ for some $h>0$ and taking the expectation w.r.t. the process $X_t$ yields:
\begin{align}
	\mathbb{E}[w_{t+h}^{\alpha}f(X_{t+h})]-\mathbb{E}[w_{t}^{\alpha}f(X_{t})] =& \int_t^{t+h} w_s^{\alpha}\parens{\nabla f(X_s)^{\top} \parens{\alpha_s(X_s)-\nabla V_s(X_s)} + \Delta f(X_s) }\diff s\\
	&+ \int_t^{t+h} w_s^{\alpha} g_s^{\alpha}(X_s) f(X_s)\diff s,
\end{align}
where we used $\mathbb{E}[\int_t^{t+h}\sqrt{2}w_s^{\alpha}\nabla f(X_s)^{\top}\diff B_s]=0$. A similar expression can be obtained for $h<0$. By continuity of the integrands of the  r.h.s., we can divide by $h$ and take the limit $h\rightarrow 0$ which yields:
\begin{align}\label{eq:ito_formula}
	\frac{\diff }{\diff t}\mathbb{E}[w_{t}^{\alpha}f(X_{t})] = \mathbb{E}\brakets{w_t^{\alpha}\parens{g_t^{\alpha}(X_t) f(X_t) + \nabla f(X_t)^{\top} \parens{\alpha_t(X_t)-\nabla V_t(X_t)} + \Delta f(X_t) }}.
\end{align}
Now we can compute the time derivative of $\rho_t[f]$:
\begin{align}
	\partial_t\rho_t[f] &= \frac{1}{\mathbb{E}[w_t^{\alpha}]}\frac{\diff }{\diff t}\mathbb{E}[w_{t}^{\alpha}f(X_{t})] - \frac{\mathbb{E}[w_t^{\alpha}f(X_t)]}{\mathbb{E}[w_t^{\alpha}]^2}\frac{\diff }{\diff t}\mathbb{E}[w_{t}^{\alpha}]\\
	&= \rho_t\brakets{g_t^{\alpha} f + \nabla f^{\top} \parens{\alpha_t-\nabla V_t} + \Delta f)}-\rho_t[f]\rho_t[g_t^{\alpha}],
\end{align}
where we used \cref{eq:ito_formula} and the definition of $\rho_t$ to obtain the final expression. Hence, we have the following partial differential equation:
\begin{align}
	\partial_t \rho_t[f] =  -\rho_t \brakets{\nabla f^{\top} \parens{ \nabla V_t  -\alpha_t  }  } + \rho_t \brakets{\Delta f} + \rho_t\brakets{(g_{t}^{\alpha} - \rho_t[g_{t}^{\alpha}])f},\qquad \rho_0\brakets{f}= \Pi_0\brakets{f}.
\end{align}
This implies the following partial differential equation on $\rho_t$ using integration by parts:
\begin{align}
	\partial_t \rho_t =  \nabla \cdot (\rho_t \parens{ \nabla V_t  -\alpha_t  }  ) + \Delta \rho_t + (g_{t}^{\alpha} - \rho_t[g_{t}^{\alpha}])\rho_t,\qquad \rho_0= \Pi_0.    
\end{align}
It is easy to check that $\Pi_t$ satisfies the same partial differential equation. One then concludes by uniqueness of the solution.

To get the expression of the normalizing constant $\mathcal{Z}_t =\bar{\Lambda}_t^{\alpha}[w_t^{\alpha}]$, we take the time derivative of $\mathcal{Z}_t$ and using It\^{o}'s formula, we get:
\begin{align}
	\frac{\diff \mathcal{Z}_t}{\diff t} = \bar{\Lambda}^{\alpha}_t\brakets{g_{t}^{\alpha} w_t^{\alpha}} =\mathcal{Z}_t \bar{\Pi}_t^{\alpha}\brakets{g_t^{\alpha}}.
\end{align}
Since $g_t^{\alpha}$ depends only on the process at time $t$, and since the marginal of $\bar{\Pi}_t$ at time $t$ is equal to $\Pi_t$, we obtain:
\begin{align}
	\frac{\diff \mathcal{Z}_t}{\diff t} = \mathcal{Z}_t\Pi_t\brakets{g_t^{\alpha}}, 
\end{align}
where we have used integration by parts to conclude that $\Pi_t\brakets{ \nabla\cdot \alpha_t  -\nabla_xV_t^{\top}\alpha_t} = 0$. Using the fact that $\mathcal{Z}_0=1$ and solving the above differential equation, we get $\mathcal{Z}_t =\exp\parens{-\int_0^t g_s^{\alpha} \diff s}  =\exp\parens{-\int_0^t \Pi_s\brakets{\partial_s V_s} \diff s}$.

\end{proof} 
\begin{lem}\label{lem:bounded_weights}
	Under \cref{assump_cont:lipschitz_potential,assump_cont:moment}  and if the control $\alpha$ satisfies \cref{assum_cont:bounded_weights}, then there exists a positive constant $C$ such that for all $t\in [0,1]$ and $x\in \mathcal{X}$:
	\begin{align}
		\overline{g}_t^{\alpha}(x)\leq  C.
	\end{align}
	In particular, this implies that the normalized weights $\overline{w}_t^{\alpha}$ are uniformly bounded over $t\in [0,1]$.
\end{lem}
\begin{proof}
	We have by definition $\bar{g}^{\alpha}_{t} = g^{\alpha}_{t} + \Pi_{t}\brakets{\partial_t V_t}$. Since $\alpha$ satisfies \cref{assum_cont:bounded_weights}, this implies that $g^{\alpha}_{t}$ is bounded from above by some constant number $C$. We only need to control $\Pi_{t}\brakets{\partial_t V_t}$ in time. By \cref{assump_cont:lipschitz_potential}, we have the estimate:
	\begin{align}\label{eq:bound_partial_V_t}
		\verts{\partial_t V_t(x)}\leq \verts{\partial_t V_0(0)} + L(1 +2\Verts{x})^2.
	\end{align}
		Moreover, \cref{assump_cont:moment} states that the 4-th, order moments of $\Pi_t$ ((hence the lower order moements) ) are finite and uniformly bounded for $t\in [0,1]$. Therefore, using \cref{eq:bound_partial_V_t} implies  $\verts{\Pi_t\brakets{\verts{\partial_t V_t(x)}}}$ is also finite and uniformly bounded on $[0,1]$. This concludes the proof.  
\end{proof}

The next proposition shows that the interpolating measures $\bar{\Pi}^{\alpha,\lambda}_t$ and $\bar{\Lambda}_t^{\alpha,\lambda}$ admit $\bar{\eta}_k$ and $\bar{\pi}_k$ as marginals. It is a direct consequence of the definition of \cref{eq:interpolating_SDE} and importance weights \cref{eq:discrete_importance_weights} and is thus provided without a proof.
\begin{prop}\label{prop:interpolating_measures}
Let $\alpha$ be a continuous-time control in $\mathcal{A}$. Let $\bar{\Lambda}^{\alpha,\lambda}_t$ and $\bar{\Pi}_{t}^{\alpha,\lambda}$ defined by \cref{eq:interpolating_SDE} and \cref{eq:Feynman_Kac_interpolation}. 
Under \cref{assump_cont:lipschitz_potential,assump_cont:Lipschiz_control,assump_cont:moment} and provided  $\alpha$ satisfies \cref{assum_cont:bounded_weights}, then the joint distribution of the vector $(X_{t_0},...,X_{t_k})$ is equal to $\bar{\eta}_k$, that is the joint distribution of the Markov chain defined by $X_{0:k}$, with NFs $T_k^{\alpha,\lambda}$ given by:
\begin{align}\label{eq:proof_interpolating_maginals_NF}
	T_k^{\alpha,\lambda}(x) = x+\lambda \alpha_{t_k}(x).
\end{align} 
Moreover, consider the joint distribution $\bar{\pi}_k$ defined by
\begin{align}
    \bar{\pi}_k[f]=\frac{\bar{\eta}_k[w_{k}f]}{\bar{\eta}_k[w_{k}]},
\end{align}
where the IS weights $w_{k}$ are given by \cref{eq:extendedtargetcollapse} using the same choice of NFs $T_k$ as in \cref{eq:proof_interpolating_maginals_NF}. Then, for any bounded smooth function $f$ of $(x_{t_0},...,x_{t_k})$ it holds that:
\begin{align}
	\bar{\Pi}_{t_k}^{\alpha,\lambda}\brakets{f} = \bar{\pi}_{k}[f].
\end{align}
\end{prop}

\subsection{Relative entropy computation}
In this section, we are interested in computing the relative entropy between the limit distribution $\bar{\Pi}^{\alpha}_t$ with control $\alpha$ and the interpolating measure $\bar{\Pi}_t^{\alpha',\lambda}$ using the control $\alpha'$. For clarity, we introduce the following notation:
\begin{align}
	\mathcal{D}_t^{\lambda}(\alpha,\alpha') &:= \textrm{KL}(\bar{\Pi}^{\alpha}_t || \bar{\Pi}^{\alpha',\lambda}_t),\\
	\Delta_t^{\lambda}(\alpha,\alpha')  &:= \mathcal{D}_t^{\lambda}(\alpha,\alpha) - \mathcal{D}_t^{\lambda}(\alpha,\alpha').
\end{align}
$\mathcal{D}_t^{\lambda}(\alpha,\alpha')$ represents the relative entropy between the limit distribution $\bar{\Pi}^{\alpha}_t$ with control $\alpha$ and the interpolating measure $\bar{\Pi}_t^{\alpha',\lambda}$ using the control $\alpha'$.  The error term $\Delta_t^{\lambda}(\alpha,\alpha')$ represents how much additional error is introduced by using a different control $\alpha'$ for the interpolating measure instead of the reference control $\alpha$.
In \cref{prop:entropy_computation}, we provide an expression for $\mathcal{D}_t^{\lambda}(\alpha,\alpha')$, then in \cref{prop:bound_diff_KL} we control the difference in relative entropies $\Delta_t^{\lambda}(\alpha,\alpha')$ when the control $\alpha'$ of the interpolating measure is replaced by $\alpha$.
 
\begin{prop}\label{prop:entropy_computation}
	Let $\alpha$ and $\alpha'$ be two continuous time controls  in $\mathcal{A}$. Under \cref{assump_cont:lipschitz_potential,assump_cont:Lipschiz_control,assump_cont:moment} and provided $\alpha$ satisfies \cref{assum_cont:bounded_weights}, then the relative entropy between $\bar{\Lambda}^{\alpha}$ and $\bar{\Lambda}^{\alpha',\lambda}$ is  well defined and given by:
\begin{align}
	\textup{KL}(\bar{\Lambda}^{\alpha}||\bar{\Lambda}^{\alpha',\lambda}) = \frac{1}{4}\mathbb{E}\brakets{\int_0^t\Vert b^{\alpha}_s(X_s)-b^{\alpha',\lambda}_s(X)\Vert^2 \diff s  }.
\end{align}
	 Moreover, the relative entropy between $\bar{\Pi}_t^{\alpha}$  and $\bar{\Pi}_t^{\alpha',\lambda}$ is well defined and given by:
\begin{align}
	\mathcal{D}_t^{\lambda}(\alpha,\alpha') =& \mathbb{E}\brakets{\overline{w}_t^{\alpha} \parens{\int_0^{\tau_{\lambda}(t)} \parens{g_{s}^{\alpha}(X_s) - g_{s}^{\alpha',\lambda}(X)}  \diff s  } } + \mathbb{E}\brakets{\overline{w}^{\alpha}_t \int_{\tau_{\lambda}(t)}^{t}  \bar{g}_s^{\alpha}(X_s)  \diff s  }\\
	 &+ \frac{1}{4}\mathbb{E}\brakets{ \overline{w}_t^{\alpha} \int_0^t \Vert b^{\alpha}_s(X_s)-b^{\alpha',\lambda}_s(X)\Vert^2 \diff s   } +  \frac{1}{\sqrt{2}}\mathbb{E}\brakets{\overline{w}^{\alpha}_t\int_0^t \parens{b^{\alpha}_s(X_s)-b^{\alpha',\lambda}_s(X)}^{\top}\diff B_s  } \\
	 &+
	 	\mathbb{E}\brakets{\overline{w}^{\alpha}_t  \parens{ \sum_{k=1}^{k_{\lambda}(t)} h_{t_k}^{\lambda}(X_{t_k}) } }.
\end{align}
\end{prop}
\begin{proof}
Let us express the relative entropy between $\bar{\Pi}_t$ and $\bar{\Pi}^{\lambda}_t$ defined as the expectation under  $\bar{\Pi}_t$ of the logarithm of the Radon-Nykodim of $\bar{\Pi}_t$ w.r.t. $\bar{\Pi}^{\lambda}_t$. By a simple chaining argument, we have that:
\begin{align}\label{eq:chaining}
	\log\parens{\frac{d\bar{\Pi}^{\alpha}_t}{d\bar{\Pi}^{\alpha',\lambda}_t}} =  \log\parens{ \frac{d\bar{\Pi}^{\alpha}_t}{d \bar{\Lambda}^{\alpha}_t}  } + \log\parens{ \frac{d\bar{\Lambda}^{\alpha}_t}{d \bar{\Lambda}_t^{\alpha',\lambda}} } + \log\parens{  \frac{d  \bar{\Lambda}_t^{\alpha',\lambda}}{d\bar{\Pi}_t^{\alpha',\lambda}} }.
\end{align}
Hence, $\mathcal{D}_t^{\lambda}(\alpha,\alpha')$ is obtained by taking the expectation of \cref{eq:chaining} w.r.t $\bar{\Pi}^{\alpha}_t$, that is the expectation w.r.t $\bar{\Lambda}^{\alpha}_t$ weighted by $\overline{w}^{\alpha}_t$:
\begin{align}\label{eq:proof_kl_4}
	\mathcal{D}_t^{\lambda}(\alpha,\alpha') = \mathbb{E}\brakets{\overline{w}_t^{\alpha} \log\parens{ \frac{d\bar{\Pi}^{\alpha}_t}{d \bar{\Lambda}^{\alpha}_t}  } }  + \mathbb{E}\brakets{\overline{w}_t^{\alpha} \log\parens{ \frac{d\bar{\Lambda}^{\alpha}_t}{d \bar{\Lambda}_t^{\alpha',\lambda}} }  } + \mathbb{E}\brakets{\overline{w}_t^{\alpha} \log\parens{  \frac{d  \bar{\Lambda}_t^{\alpha',\lambda}}{d\bar{\Pi}_t^{\alpha',\lambda}} } }.
\end{align}
We only need to express each term in \cref{eq:chaining}. The first and last terms in \cref{eq:chaining} are given by definition of $\Pi_t^{\alpha}$ and $\Pi^{\alpha',\lambda}_t$: 
\begin{align}\label{eq:proof_kl_1}
	\log\parens{ \frac{d\bar{\Pi}^{\alpha}_t}{d \bar{\Lambda}^{\alpha}_t}  }  &=\int_0^t g_s^{\alpha}(X_s) \diff s - \log\parens{\frac{\mathcal{Z}_0}{\mathcal{Z}_{t}}}, \\
	\log\parens{  \frac{d  \bar{\Lambda}_t^{\alpha',\lambda}}{d\Pi_t^{\alpha',\lambda}} } &= \log\parens{\frac{\mathcal{Z}_0}{\mathcal{Z}_{\tau_{\lambda}(t)}}} - \int_0^{\tau_{\lambda}(t)  }  g_{s}^{\alpha',\lambda}(X) \diff s +   \sum_{k=1}^{k_{\lambda}(t)} h_{t_k}^{\lambda}(X_{t_k}).
\end{align}
The second term is obtained by application of Girsanov's formula, since $\bar{\Lambda}^{\alpha}_t $ and $\bar{\Lambda}^{\alpha',\lambda}_t$  are mutually absolutely continuous as they share the same volatility term by construction:
\begin{align}\label{eq:proof_kl_3}
	\log\parens{\frac{\diff \bar{\Lambda}^{\alpha}_t}{\diff \bar{\Lambda}^{\alpha',\lambda}_t}}:= \frac{1}{\sqrt{2}}\int_0^t (b_s^{\alpha}(X_s)-b^{\alpha',\lambda}_s(X)  )^{\top}\diff B_s + \frac{1}{4}  \int_0^t \Vert b_s^{\alpha}(X_s)-b^{\alpha',\lambda}_s(X)\Vert^2 \diff s.
\end{align} 
We obtain the desired expression for $\mathcal{D}_t^{\lambda}(\alpha,\alpha')$ by plugging \cref{eq:proof_kl_1,eq:proof_kl_3} in \cref{eq:proof_kl_4}.
Finally, the relative entropy between $\bar{\Lambda}^{\alpha}_t$  and $\bar{\Lambda}^{\alpha',\lambda}_t$  is obtained directly by taking the expectation of \cref{eq:proof_kl_3} under $\bar{\Lambda}^{\alpha}_t $.
\end{proof}

\subsection{Relative entropy bounds}
In this section, we provide bounds on the relative entropy $\mathcal{D}^{\lambda}_t(\alpha,\alpha)$  in \cref{prop:discretization_error}  and difference in relative entropies $\Delta^{\lambda}_t(\alpha,\alpha')$ in  \cref{prop:bound_diff_KL}. We start by \cref{prop:Expectations_bound} which will be crucial in the proofs of \cref{prop:bound_diff_KL,prop:discretization_error}. \cref{prop:Expectations_bound} provides estimates of the expectations under $\bar{\Lambda}^{\alpha}_t$ of product of functions, where one of the functions depend only on the value of the process at an earlier time $s$. We defer the proof of \cref{prop:Expectations_bound} to \cref{sec:proof_expecation_bound} which crucially relies on a coupling argument later provided in \cref{lem:main_bounds_gronwal}.
\begin{prop}\label{prop:Expectations_bound}
Let  $0\leq s\leq s'\leq t$. 
Consider a function $f$ of the process $X_s$ at time $s$ such that $\mathbb{E}[\Verts{f(X_s)}^2]<\infty $ and let $h_u(X_u,X_s)$ be a function of the sample path $(X_{t})_{[s,s']}$ satisfying $\mathbb{E}[\int_s^{s'} \Verts{h_u(X_u,X_s)}^2\diff u  ]<\infty$. 
Under \cref{assump_cont:lipschitz_potential,assump_cont:Lipschiz_control,assump_cont:moment} and provided $\alpha$ satisfies \cref{assum_cont:bounded_weights}, the following holds:
\begin{align}
\label{eq:exp_bound_1}
	\Verts{\mathbb{E}\brakets{\overline{w}_t^{\alpha} f(X_s)  }}&\lesssim \Pi_s\brakets{\Verts{f(X_s)}}\\
\label{eq:exp_bound_2}	\verts{ \mathbb{E}\brakets{ \overline{w}_t^{\alpha} f(X_s)^{\top} \int_{s}^{s'}  \diff B_u  }}&\lesssim (s'-s) \Pi_{s}\brakets{\Verts{f}^2}^{\frac{1}{2}}\\
\label{eq:exp_bound_3}	\verts{ \mathbb{E}\brakets{ \overline{w}_t^{\alpha} f(X_s)^{\top} \int_{s}^{s'} h_u(X_u,X_s)\diff u  }}&\lesssim \sqrt{s'-s}  \Pi_{s}\brakets{\Verts{f}^2}^{\frac{1}{2}}\mathbb{E}\brakets{\int_s^{s'} \Verts{h_u(X_u,X_s)}^2\diff u }^{\frac{1}{2}}
\end{align}
In particular, if $  t_{k-1}\leq s \leq s'\leq t_k$  for some $k\leq k_{\lambda}(t)+1$ and  $h_u(x,y) =  b_u^{\alpha}(x)- \beta_{u'}^{\alpha,\lambda}(y)  $ with   $ t_{k-1}\leq u'\leq t_k$, then
\begin{align}
	\verts{ \mathbb{E}\brakets{ \overline{w}_t^{\alpha} f(X_s)^{\top} \int_{s}^{s'} h(X_u,X_s)\diff u + \diff B_u  }}&\lesssim \lambda \Pi_{s}\brakets{\Verts{f}^2}^{\frac{1}{2}}.
\end{align}
\end{prop}
  
We can now provide an upper-bound on the difference in relative entropies $\Delta_t^{\lambda}(\alpha,\alpha')$ in terms of the time discretization step-size $\lambda$ and the time-discrete discrepancy $\mathcal{S}_t^{\lambda}(\alpha,\alpha')$ between the controls $\alpha$ and $\alpha'$, defined in  \cref{eq:Sobolev_distances}.
\begin{prop}\label{prop:bound_diff_KL}
	Let $\alpha$ and $\alpha'$ be two continuous time controls  in $\mathcal{A}$. Under \cref{assump_cont:lipschitz_potential,assump_cont:Lipschiz_control,assump_cont:moment} and provided $\alpha$ satisfies \cref{assum_cont:bounded_weights}, then  the following upper-bound holds: 
	\begin{align}
	\verts{\Delta_t^{\lambda}(\alpha,\alpha')} \lesssim & \mathcal{S}_t^{\lambda}(\alpha,\alpha') + \mathcal{S}_t^{\lambda}(\alpha,\alpha')^{\frac{1}{2}} + \lambda,
	\end{align}
	where $\mathcal{S}_t^{\lambda}$ is defined in  \cref{eq:Sobolev_distances}. 
\end{prop}
\begin{proof}
	By direct computation using the expression of the relative entropy in \cref{prop:entropy_computation}, we have that:
	\begin{align}
	\begin{aligned}
		-\Delta_t^{\lambda}(\alpha,\alpha') =& \mathcal{D}_t^{\lambda}(\alpha,\alpha') - \mathcal{D}_t^{\lambda}(\alpha,\alpha)\\
		=& \mathbb{E}\brakets{\overline{w}_t^{\alpha} \parens{\int_0^{\tau_{\lambda}(t)} \parens{g_{s}^{\alpha,\lambda}(X) - g_{s}^{\alpha',\lambda}(X)}  \diff s  } } \\
	 &+ \frac{1}{4}\mathbb{E}\brakets{ \overline{w}_t^{\alpha} \int_0^t \Verts{ b^{\alpha}_s(X_s)-b^{\alpha',\lambda}_s(X)}^2-\Verts{ b^{\alpha}_s(X_s)-b^{\alpha,\lambda}_s(X)}^2 \diff s   }
	\end{aligned}
	\end{align} 
 We further introduce: 
		\begin{align}
			\Delta\delta_{s}(x) := \delta_{s}^{\alpha,\lambda}(x) -  \delta_{s}^{\alpha',\lambda}(x),\qquad \Delta\beta_s(x) := \beta_{s}^{\alpha,\lambda}(x)  -\beta_{s}^{\alpha',\lambda}(x),
		\end{align}
	\begin{align}
		 h_{u,s}(x,y) = b_u^{\alpha}(x) - \beta_{s}^{\alpha,\lambda}(y). 
	\end{align}
	Recalling the definition of  $g_s^{\alpha,\lambda}(X)$ from \cref{eq:discrete_instantaneous_work}, we can express $\Delta_t^{\lambda}(\alpha,\alpha')$ in terms of $\Delta\delta$, $\Delta\beta$ and $h_{u,s}(x,y)$:
	\begin{align}\label{eq:expression_diff_entropy}
	\begin{aligned}
		-\Delta_t^{\lambda}(\alpha,\alpha')
		=&\lambda \sum_{k=1}^{k_{\lambda}(t)} \mathbb{ E}\brakets{\overline{w}_{t}^{\alpha} \parens{\Delta\delta_{t_k}(X_{t_{k-1}}) + \frac{1}{4}\Verts{\Delta\beta_{t_k}(X_{t_{k-1}})}^2 } } \\
		+&  \frac{1}{2}\sum_{k=1}^{k_{\lambda}(t)} \mathbb{ E}\brakets{\overline{w}_{t}^{\alpha} \parens{\Delta\beta_{t_k}(X_{t_{k-1}})}^{\top}\int_{t_{k-1}}^{t_k}\parens{h_{u,t_k}(X_u,X_{t_{k-1}})\diff u + \sqrt{2}\diff B_u }   }\\
		+& \frac{1}{2}\mathbb{ E}\brakets{\overline{w}_{t}^{\alpha} \parens{\Delta\beta_{\tau_{\lambda}(t)+\lambda}(X_{\tau_{\lambda}(t)})}^{\top}\int_{\tau_{\lambda}(t)}^{t}\parens{ h_{u, \tau_{\lambda}(t)+\lambda }(X_u,X_{\tau_{\lambda}(t)}) \diff u + \sqrt{2}\diff B_u }}.
	\end{aligned}
	\end{align}
	By \cref{prop:Expectations_bound}, we know that the remainder term in the last line of \cref{eq:expression_diff_entropy} is of order $\lambda$. Moreover, \cref{prop:Expectations_bound} allows us to control the first two terms in \cref{eq:expression_diff_entropy} so that the following bound holds:
	\begin{align}\label{eq:main_estimate}
		\verts{\Delta_t^{\lambda}(\alpha,\alpha')}&\lesssim \lambda \sum_{k=1}^{k_{\lambda}(t)} \Pi_{t_{k-1}}\brakets{ \verts{\Delta \delta_{t_k}  } +\Verts{\Delta \beta_{t_k}  }^2   } + \lambda\sum_{k=1}^{k_{\lambda}(t)} \Pi_{t_{k-1}}\brakets{\Verts{\Delta \beta_{t_k}}^2}^{\frac{1}{2}} + \lambda\\
		&\lesssim \lambda \sum_{k=1}^{k_{\lambda}(t)} \Pi_{t_{k-1}}\brakets{ \verts{\Delta \delta_{t_k}  } +\Verts{\Delta \beta_{t_k}  }^2   } + \parens{\lambda\sum_{k=1}^{k_{\lambda}(t)} \Pi_{t_{k-1}}\brakets{\Verts{\Delta \beta_{t_k}}^2}}^{\frac{1}{2}} + \lambda\\
		&\lesssim \mathbb{M}^{\lambda}_t\brakets{ \verts{\Delta \delta} } + \mathbb{M}^{\lambda}_t\brakets{ \Verts{\Delta \beta}^2 } + \mathbb{M}^{\lambda}_t\brakets{ \Verts{\Delta \beta}^2 }^{\frac{1}{2}} +\lambda,
	\end{align}
		where we used Cauchy--Schwarz inequality in the second line  and introduced the notation  $\mathbb{M}^{\lambda}_t$ from \cref{eq:discrete_time_expectations} in the last line,  $\vert\Delta \delta \vert $  and  $\Verts{\Delta \beta}$ being viewed as functions from $[0,1]\times \mathcal{X}$ to $\mathbb{R}$. Thus, we only need to control $\mathbb{M}^{\lambda}_t\brakets{ \verts{\Delta \delta} }$ and $\mathbb{M}^{\lambda}_t\brakets{ \Verts{\Delta \beta}^2 }$.
To control $\mathbb{M}^{\lambda}_t\brakets{ \verts{\Delta \delta} }$, we can rely on the following pointwise estimate from  \cref{lem:appendix_estimate_control}:
\begin{align}
	\verts{\delta_{t_k}^{\alpha,\lambda}(x)- \delta_{t_k}^{\alpha',\lambda}(x)} &\lesssim  \parens{\parens{1 + \Verts{x}}\Verts{\alpha_{t_k}(x)-\alpha_{t_k}'(x)}  + \Vert \nabla \alpha_{t_k}(x) - \nabla\alpha'_{t_k}(x)  \Vert}.
\end{align}
Further defining $\overline{L}: [0,1]\times \mathbb{X} \rightarrow \mathbb{R}$ to be $(s,x)\mapsto \overline{L}_s(x) = 1 + \Verts{x}$, this allows us to write
	\begin{align}
\mathbb{M}^{\lambda}_t\brakets{ \Verts{\Delta \beta}^2 } &\lesssim \mathbb{M}_t^{\lambda}\brakets{\Verts{\alpha-\alpha'}^2}\\
\mathbb{M}^{\lambda}_t\brakets{ \verts{\Delta \delta} }
		&\lesssim \mathbb{M}_t^{\lambda}\brakets{ \overline{L}\Verts{ \alpha - \alpha' }  } + \mathbb{M}_t^{\lambda}\brakets{\Verts{ \nabla\alpha - \nabla\alpha' }} \\
		&\lesssim \mathbb{M}_t^{\lambda}\brakets{ \overline{L}^2}^{\frac{1}{2}} \mathbb{M}_t^{\lambda}\brakets{\Verts{ \alpha - \alpha' }^{2}  }^{\frac{1}{2}} +  \mathbb{M}_t^{\lambda}\brakets{\Verts{ \nabla\alpha - \nabla\alpha' }}
	\end{align}
	where we used Cauchy--Schwarz inequality to get the last inequality. The factor $\mathbb{M}_t^{\lambda}\brakets{ \overline{L}^2}$ is also bounded as it is a Riemann sum and converges towards:
	\begin{align}
		\mathbb{M}_t^{\lambda}\brakets{ \overline{L}^2}\xrightarrow[\lambda\rightarrow 0]{} \int_0^t\Pi_{t}\brakets{\parens{1+\Verts{ X} }^2}\diff t,
	\end{align} 
	which is finite by \cref{assump_cont:moment}. 
	Again using the pointwise estimate from  \cref{lem:appendix_estimate_control}:
\begin{align}
	\Vert{\beta}^{\alpha,\lambda}_t(x)-\beta^{\alpha',\lambda}_t(x)\Vert &\lesssim \Vert \alpha_t(x)-\alpha_t'(x)\Vert,
\end{align}
it follows directly that $\mathbb{M}^{\lambda}_t\brakets{ \Verts{\Delta \beta}^2 } \lesssim \mathbb{M}_t^{\lambda}\brakets{\Verts{\alpha-\alpha'}^2}$. Therefore, we have shown:
	\begin{align}\label{eq:bound_1}
			\mathbb{M}^{\lambda}_t\brakets{ \Verts{\Delta \beta}^2 } &\lesssim \mathbb{M}_t^{\lambda}\brakets{\Verts{\alpha-\alpha'}^2}\\	
		\mathbb{M}^{\lambda}_t\brakets{ \verts{\Delta \delta} }
		&\lesssim
		\mathbb{M}_t^{\lambda}\brakets{\Verts{ \alpha- \alpha' }^{2}  }^{\frac{1}{2}} + \mathbb{M}_t^{\lambda}\brakets{\Verts{ \nabla\alpha - \nabla\alpha' }}. 
	\end{align}
The desired upper-bound follows using \cref{eq:bound_1} in  \cref{eq:main_estimate}.
\end{proof}

Next we control the relative entropy $\mathcal{D}_{t}^{\lambda}(\alpha,\alpha)$  between the Feynman-Kac measure $\bar{\Pi}^{\alpha}_t$ and the interpolating measure $\bar{\Pi}^{\alpha,\lambda}_t$ using the same control $\alpha$.
\begin{prop}\label{prop:discretization_error}
Under \cref{assump_cont:lipschitz_potential,assump_cont:Lipschiz_control,assump_cont:moment,assump_cont:moment_2} and if  $\alpha$ satisfies \cref{assum_cont:bounded_weights}, the following upper bound holds:  
\begin{align}
	  \mathcal{D}_{t}^{\lambda}(\alpha,\alpha):= \textup{KL}(\bar{\Pi}^{\alpha}_t||\bar{\Pi}^{\alpha,\lambda}_t)  \lesssim \sqrt{\lambda}.
\end{align}
\end{prop}
\begin{proof}{\textbf{of \cref{prop:discretization_error}}.}\label{proof:discretization_error}
We use the expression of the relative entropy $\mathcal{D}_{t}^{\lambda}(\alpha,\alpha)$ provided in \cref{prop:entropy_computation}, then recalling that the normalized weights  $\bar{w}_t^{\alpha}$ are bounded by \cref{lem:bounded_weights}  we have
\begin{align}
	\mathcal{D}_t^{\lambda}(\alpha,\alpha') \lesssim & \mathbb{E}\brakets{ \verts{\int_0^{\tau_{\lambda}(t)} \parens{g_{s}^{\alpha}(X_s) - g_{s}^{\alpha',\lambda}(X)}  \diff s  } } + \mathbb{E}\brakets{\verts{\int_{\tau_{\lambda}(t)}^{t}  \bar{g}_s^{\alpha}(X_s)  \diff s } }\\
	 &+ \frac{1}{4}\mathbb{E}\brakets{  \int_0^t \Vert b^{\alpha}_s(X_s)-b^{\alpha',\lambda}_s(X)\Vert^2 \diff s   } +  \frac{1}{\sqrt{2}}\mathbb{E}\brakets{\verts{\int_0^t \parens{b^{\alpha}_s(X_s)-b^{\alpha',\lambda}_s(X)}^{\top}\diff B_s } } \\
	 &+
	 	\mathbb{E}\brakets{\overline{w}^{\alpha}_t  \parens{ \sum_{k=1}^{k_{\lambda}(t)} \verts {h_{t_k}^{\lambda}(X_{t_k})} } }.
\end{align}
Now by a direct application of Cauchy--Schwarz inequality and triangular inequalities, we get the following upper-bound on $\mathcal{D}_{t}^{\lambda}(\alpha,\alpha)$
\begin{align}
	D^{\lambda}_t(\alpha,\alpha) \lesssim \parens{  \mathcal{E}_W + \mathcal{E}_{R}  + \mathcal{E}_H + \frac{1}{4}\mathcal{E}_{KL} + \frac{1}{\sqrt{2}}\mathcal{E}_{KL}^{\frac{1}{2}} }
\end{align}
where $\mathcal{E}_W$, $\mathcal{E}_R$, $\mathcal{E}_H$ and $\mathcal{E}_{KL}$  are given by
\begin{align}
	\mathcal{E}_W &:= \mathbb{E}\brakets{ \verts{ \int_0^{\tau_{\lambda}(t)} g^{\alpha}_s(X_s)- g_{s}^{\alpha,\lambda}(X)  \diff s  }  },\qquad  &\mathcal{E}_{R} &:= \mathbb{E}\brakets{\int_{\tau_{\lambda}(t)}^t  \bar{g}_s^{\alpha}(X_s) \diff s },  \\
	 \mathcal{E}_{KL} &:= \mathbb{E}\brakets{\int_0^t\Verts{b_{s}^{\alpha}(X_s)-b_{s}^{\alpha,\lambda}(X)}^{2} },\qquad &\mathcal{E}_H &:= \mathbb{E}\brakets{w_t^{\alpha}\sum_{k=1}^{k_{\lambda}(t)}\verts{ h_{t_k}^{\lambda}}}.
\end{align}
\paragraph{Bound on $\mathcal{E}_R$. } $\mathbb{E}[\vert \overline{g}_{s}^{\alpha}(X_s)\vert]$ has a quadratic growth by \cref{lem:appendix_estimate_control} and the process $X_s$ has a bounded second moment by \cref{lem:main_bounds}, therefore  $\mathcal{E}_R\lesssim \vert t-\tau_{\lambda}(t)\vert\leq \lambda $.
\paragraph{Bound on  $\mathcal{E}_W$.}
We first start by applying a triangular inequality:
\begin{align}
	\mathcal{E}_{W}&\leq \sum_{k=1}^{k_{\lambda}(t)}\mathbb{E}\underbrace{\brakets{\int_{t_{k-1}}^{t_k}\verts{g_s^{\alpha}(X_s)- \delta_{s}^{\alpha,\lambda}(X_{t_{k-1})}} \diff s}}_{\mathcal{E}_{W,k}}
\end{align}
where by definition of $g_{s}^{\alpha,\lambda}(X) = \delta_{s}^{\alpha,\lambda}(X_{t_{k-1}})$ for $t_{k-1}\leq s\leq t_k$. We then use the point-wise upper-bound on the difference $g_s^{\alpha}(x') - \delta_t^{\alpha,\lambda}(x)$ provided in  \cref{lem:bound_intermediate_terms}:
\begin{align}
	\verts{g_s^{\alpha}(x') - \delta_t^{\alpha,\lambda}(x) }\lesssim \parens{1+\Verts{x}+\Verts{x'} }\Verts{x -x'}+  \lambda\parens{1 +\Verts{x} + \Verts{x'}}^2. 
\end{align}
This allows to control each integral $\mathcal{E}_{W,k}$ after a simple application of Cauchy--Schwarz inequality:
\begin{align}
	\mathcal{E}_{W,k}
	\lesssim & \lambda  \int_{t_{k-1}}^{t_k}\mathbb{E}\brakets{ \parens{1+\Verts{X_s} + \Verts{X_{t_{k-1}}}}^2} \diff s \\
	&+ \parens{\int_{t_{k-1}}^{t_k}\mathbb{E}\brakets{\Verts{X_s-X_{t_{k-1}}}^2}\diff s}^{\frac{1}{2}}\parens{\int_{t_{k-1}}^{t_k}\mathbb{E}\brakets{ \parens{1+\Verts{X_s} + \Verts{X_{t_{k-1}}}}^2 }\diff s}^{\frac{1}{2}}
\end{align}
where we also used Fubini's theorem to exchange the order of the expectation and time integral. 
By \cref{lem:main_bounds}, we know that the second moments of the process $X_{t}$ are bounded over the time interval $[0,1]$ and that $\mathbb{E}\brakets{\Verts{X_s-X_{t_{k-1}}}^2}\lesssim \vert s-t_{k-1} \vert $. Therefore, we get the upper-bound $\mathbb{E}[\mathcal{E}_{W,k}]\lesssim \lambda^2 + \lambda\sqrt{\lambda}$. Finally, summing over $k$ ranging from $1$ to $k_{\lambda}(t)$ yields:
\begin{align}
	\mathcal{E}_{W}\leq \sum_{k=1}^{k_{\lambda}(t)}\mathbb{E}\brakets{\mathcal{E}_{W,k}}\lesssim \sqrt{\lambda}.
\end{align}
\paragraph{Bound on  $\mathcal{E}_{H}$. }
By a direct application of \cref{prop:Expectations_bound}, we know that:
\begin{align}
	\mathcal{E}_{H}\lesssim \sum_{k=1}^{k_{\lambda}(t)} \Pi_{t_k}\brakets{\verts{h_{t_k}^{\lambda}}}.
\end{align}
We just need to control each term $\Pi_{t_k}\brakets{\verts{h_{t_k}^{\lambda}}}$ as a function of $\lambda$.  The technical \cref{lem:kerenl_error_main_lemma} provides an upper-bound on $\vert h^{\lambda}_t(x)\vert$ that is point-wise in $x$ and of the form:
\begin{align}
			\vert h^{\lambda}_t(x) \vert \leq \lambda \sqrt{\lambda} \frac{Q_1(x)}{1 + Q_2(x)} \exp( C \lambda \Vert \nabla V_t(x)\Vert^2 )  
	\end{align}
where $Q_1$ and $Q_2$ are non-negative functions of $\Vert x \Vert$ of polynomial growth and  independent from $\lambda$ and $t$ and $C$ is a non-negative constant. Taking the expectation under $\Pi_t$ and using the integrability condition of \cref{assump_cont:moment_2}, it follows that $\Pi_{t_k}\brakets{\verts{h^{\lambda}_{t_k}}} \lesssim \lambda \sqrt{\lambda}$.
Thus summing over $k$ directly yields $\mathcal{E}_{H} \lesssim \sqrt{\lambda}$.
\paragraph{Bound on $\mathcal{E}_{KL}$.} We start by decomposing $\mathcal{E}_{KL}$ as a sum of $k_{\lambda}(t)+1$ integral terms over intervals of size less or equal to $\lambda$:
	\begin{align}
		\mathcal{E}_{KL} 
		=&   \int_{\tau_{\lambda}(t)}^{t} \mathbb{E}\brakets{ \Vert b_s^{\alpha}(X_s)-\beta^{\alpha,\lambda}_{\tau_{\lambda}(t)+1}(X_{\tau_{\lambda}(t)}) \Vert^2  }\diff s + \sum_{k=1}^{k_{\lambda}(t)} \int_{t_{k-1}}^{t_{k}} \mathbb{E}\brakets{ \Vert b_s^{\alpha}(X_s)-\beta^{\alpha,\lambda}_{t_{k}}(X_{t_{k-1}}) \Vert^2  }\diff s 
		\end{align}
where we used that $b_s^{\alpha,\lambda}(X) = \beta^{\alpha,\lambda}_{t_{k}}(X_{t_{k-1}})$ for $t_{k-1}\leq s\leq t_{k}$ by definition of $b_s^{\alpha,\lambda}(X)$. We can then use  \cref{prop:relative_entropy} which allows to control each term by $\lambda^2$. Hence, after summing, we directly get that $\mathcal{E}_{KL}\lesssim \lambda \verts{t-\tau_{\lambda}(t)} + \sum_{k=1}^{k_{\lambda}(t)} \lambda^2\lesssim \lambda$.

We finally get the desired result by combining upper-bounds on each quantity $\mathcal{E}_W$, $\mathcal{E}_R$, $\mathcal{E}_H$ and $\mathcal{E}_{KL}$  which are all of order $\sqrt{\lambda}$ at least.
\end{proof}

\begin{lem}\label{prop:relative_entropy}
Let $ s $ and $s'$ be such that $ t_{k-1}\leq s\leq s'\leq t_k $ for $0\leq  k\leq  K$ and $t_{k-1}\leq  u' \leq t_k $.
Under \cref{assump_cont:lipschitz_potential,assump_cont:Lipschiz_control,assump_cont:moment} it holds that:
\begin{align} 
		  \mathbb{E}\brakets{\int_s^{s'} \Verts{b_u^{\alpha}(X_u)- \beta_{u'}^{\alpha,\lambda}(X_{s}) }^2 \diff u}\lesssim \lambda\verts{s'-s}.
\end{align} 
\end{lem}
\begin{proof}
By a direct application of \cref{prop:beta_error}, we have
		\begin{align}
		\mathbb{E}\brakets{\int_s^{s'} \Verts{b_u^{\alpha}(X_u)- \beta_{u'}^{\alpha,\lambda}(X_{s}) }^2 \diff u}
		&\lesssim \int_s^{s'} \mathbb{E}\brakets{ \Verts{X_u -X_{s} }^2  +\lambda^2\parens{1+ \Verts{X_{s}}^2 }}\diff u.
		\end{align}
Moreover, by \cref{lem:main_bounds}, we directly know that the second moments of the process are bounded at any time and that $\mathbb{E}[\Vert X_{u'} - X_u \Vert^2  ]\leq \vert u- u' \vert $, hence:
		\begin{align}
			\mathbb{E}\brakets{\int_s^{s'} \Verts{b_u^{\alpha}(X_u)- \beta_{u'}^{\alpha,\lambda}(X_{s}) }^2 \diff u}&\lesssim  \int_{s}^{s'} \parens{\verts{u-s}+ \lambda^2} \diff u\lesssim \lambda\verts{s'-s}.
		\end{align}\end{proof}

\subsection{Convergence towards an optimal Feynman-Kac measure}\label{sec:convergence_optimal_control}
In this section, we show that the optimal interpolating controls $\alpha^{\lambda}$ converge towards the continuous-time limit optimal control $\alpha^{\star}$.
We will first need proposition \cref{prop:convergence_loss} bellow showing that the discrete-time loss $\mathcal{L}^{tot}_{\lambda}(\alpha)$ converges towards $\mathcal{M}(\alpha)$ uniformly over the class of admissible controls with a rate of $\lambda$. \cref{prop:convergence_loss} is a restatement of \cref{prop:convergence_loss_main}, we defer its proof to \cref{appendix_proof_discrete_loss}.
\begin{prop}\label{prop:convergence_loss}
Under  \cref{assump_cont:lipschitz_potential,assump_cont:Lipschiz_control,assump_cont:moment} and for $\lambda\leq \frac{1}{2L}$, it holds that:
	\begin{align}
		\verts{ \mathcal{L}^{tot}_{\lambda}(\alpha) -  \mathcal{M}(\alpha)  }\lesssim \lambda ,
		\end{align}
		where $\mathcal{M}(\alpha)$ and $\mathcal{L}^{tot}_{\lambda}(\alpha)$ are defined in \cref{eq:discrete_loss,eq:variance_work}.  
\end{prop}
Using \cref{prop:convergence_loss} we can show that the interpolating control $\alpha^{\lambda}$, which exists by \cref{assump_cont:existence_interpolating_controls}, converges towards the continuous-time limit optimal control $\alpha^{\star}$. This essentially relies on the local behavior of the objective $\mathcal{M}$ near the optimum $\alpha^{\star}$ as described by \cref{assump_cont:separated_optimum}. \cref{prop:convergence_controls} makes this idea more precise.
\begin{prop}[Convergence of the controls]\label{prop:convergence_controls}
	
	Under \cref{assump_cont:lipschitz_potential,assump_cont:Lipschiz_control,assump_cont:separated_optimum,assump_cont:moment,assump_cont:existence_interpolating_controls}, it holds that:
	\begin{align}\label{eq:convergence_discrete_sum_controls}
		\mathcal{S}^{\lambda}_t(\alpha^{\star},\alpha^{\lambda})\lesssim \lambda,\qquad \mathcal{S}(\alpha^{\star},\alpha^{\lambda})\lesssim \lambda,
	\end{align}
		where $\mathcal{S}$  and $\mathcal{S}^{\lambda}_t$ are defined in \cref{eq:Sobolev_distances}.
\end{prop}
\begin{proof}
	Recall that $\alpha^{\lambda}$ is a minimizer of $\mathcal{L}_{\lambda}^{tot}(\alpha)$ defined in \cref{eq:L_tot} over $\mathcal{A}$, while $\alpha^{\star}$ is a minimizer of $\mathcal{M}(\alpha)$ (defined in \cref{eq:M_control} ) over $\mathcal{A}$. 	By \cref{assump_cont:moment,assump_cont:existence_interpolating_controls,assump_cont:separated_optimum} both exist and belong to $\mathcal{A}$. We will first show that $\verts{\mathcal{M}(\alpha^{\star}) - \mathcal{M}(\alpha^{\lambda})}\lesssim \lambda$, then, we will use the coercivity property of the loss $\mathcal{M}$ ( \cref{assump_cont:separated_optimum}) to obtain convergence rates in terms of the divergence $\mathcal{S}$ between $\alpha^{\star}$ and $\alpha^{\lambda}$. Finally, we establish the convergence of the discrete sums $\mathcal{S}^{\lambda}_t$,  by interpreting it as a Riemann sum.
	\paragraph{Bounds on  $\vert\mathcal{M}(\alpha^{\star}) - \mathcal{M}(\alpha^{\lambda})\vert$ and $\vert\mathcal{L}_{\lambda}^{tot}(\alpha^{\star}) - \mathcal{L}_{\lambda}^{tot}(\alpha^{\lambda})\vert$.}
	By definition of the minimizers, the following lower bound holds $0\leq \mathcal{M}(\alpha^{\lambda}) -\mathcal{M}(\alpha^{\star})$ holds. It remains to find an upper-bound. For this purpose, we use the following decomposition:
	\begin{align}
		\mathcal{M}(\alpha^{\lambda}) -\mathcal{M}(\alpha^{\star})&= \parens{\mathcal{M}(\alpha^{\lambda}) - \mathcal{L}_{\lambda}^{tot}(\alpha^{\lambda})} + \parens{\mathcal{L}_{\lambda}^{tot}(\alpha^{\lambda}) -  \mathcal{L}_{\lambda}^{tot}(\alpha^{\star})} + \parens{\mathcal{L}_{\lambda}^{tot}(\alpha^{\star}) -\mathcal{M}(\alpha^{\star})}\\
	&\leq \parens{\mathcal{M}(\alpha^{\lambda}) - \mathcal{L}_{\lambda}^{tot}(\alpha^{\lambda})} + \parens{\mathcal{L}_{\lambda}^{tot}(\alpha^{\star}) -\mathcal{M}(\alpha^{\star})},
	\end{align}
	where we used that $\mathcal{L}_{\lambda}^{tot}(\alpha^{\lambda}) -  \mathcal{L}_{\lambda}^{tot}(\alpha^{\star})\leq 0$ to get the second line. 
	Moreover, \cref{prop:convergence_loss} provides an error bound of the form $\verts{\mathcal{M}(\alpha)- \mathcal{L}_{\lambda}^{tot}(\alpha)}\lesssim \lambda$ that holds uniformly over the set $\mathcal{A}$. This directly implies $\mathcal{M}(\alpha^{\lambda}) -\mathcal{M}(\alpha^{\star})\lesssim \lambda.$
	Thus, we have shown:
	\begin{align}\label{eq:convergence_loss_control}
		\verts{\mathcal{M}(\alpha^{\lambda}) -\mathcal{M}(\alpha^{\star})}\lesssim \lambda.
	\end{align}
	\paragraph{Convergence in Sobolev norm $\mathcal{S}(\alpha^{\star},\alpha^{\lambda})$.} 
	By \cref{assump_cont:separated_optimum}, we know that the minimizer $\alpha^{\star}$ is well-separated, i.e. for any $\delta>0$ it holds that
	\begin{align}\label{eq:well_separation}
		\mathcal{M}(\alpha^{\star})< \inf_{\mathcal{S}(\alpha,\alpha^{\star})>\delta} \mathcal{M}(\alpha)
	\end{align}
Along with \cref{eq:convergence_loss_control},  this necessarily implies that $\mathcal{S}(\alpha^{\lambda},\alpha^{\star}) \xrightarrow[\lambda \rightarrow 0]{} 0$. Indeed, if by contradiction, $\mathcal{S}(\alpha^{\lambda},\alpha^{\star})$ does not converge to $0$, then there exists a positive number $\delta_0$ and a sequence of elements $(\lambda_{l})_{l\geq 0}$ such that $\lambda_{l}\rightarrow 0$ and $\mathcal{S}(\alpha^{\lambda_l},\alpha^{\star})>\delta_0$.  If we set $M_{\delta_0}= \inf_{\mathcal{S}(\alpha,\alpha^{\star})>\delta_0} \mathcal{M}(\alpha)
 $, then \cref{eq:well_separation} implies that $M_{\delta_0}> \mathcal{M}(\alpha^{\star})$. Hence,  $\mathcal{M}(\alpha^{\lambda_l})> M_{\delta_0}> \mathcal{M}(\alpha^{\star})$. This contradicts the fact that $\mathcal{M}(\alpha^{\lambda_l})\xrightarrow[\lambda\rightarrow 0]{} \mathcal{M}(\alpha^{\star})$.
 \paragraph{Convergence rate in Sobolev norm.} Now, we can get a convergence rate for $\mathcal{S}(\alpha^{\lambda},\alpha^{\star})$. By \cref{assump_cont:separated_optimum}, we know that for $\delta$ small enough, the following local coercivity property holds:
 \begin{align}
 	\mathcal{S}(\alpha,\alpha^{\star})\lesssim \mathcal{M}(\alpha)-\mathcal{M}(\alpha^{\star}),\qquad \forall \alpha\in \mathcal{A}:\quad \mathcal{S}(\alpha,\alpha^{\star})\leq \delta  
 \end{align}
 Since $\mathcal{S}(\alpha^{\lambda},\alpha^{\star}) \xrightarrow[\lambda\rightarrow 0]{} 0 $, there exists $\lambda_0$ small enough such that $\mathcal{S}(\alpha^{\lambda},\alpha^{\star})\leq \delta$ for all $\lambda\leq \lambda_0$. The convergence rate follows directly using \cref{eq:convergence_loss_control}: $\mathcal{S}(\alpha^{\lambda},\alpha^{\star})\lesssim \lambda$ for all $ \lambda\leq \lambda_0$.
 \paragraph{Convergence of discrete sums.}
 We will first show $\mathcal{S}^{\lambda}_t(\alpha^{\star},\alpha^{\lambda})  \lesssim \lambda $ for $t=1$. The result will follow for any $t\leq 1$ since by definition we have $\mathcal{S}^{\lambda}_t(\alpha,\alpha') \leq \mathcal{S}^{\lambda}_1(\alpha,\alpha')$ for any $\alpha$ and $\alpha'$ in $\mathcal{A}$. By the triangular inequality, we have:
 \begin{align}
    \mathcal{S}^{\lambda}_1(\alpha^{\star},\alpha^{\lambda}) &\leq \mathcal{S}(\alpha^{\star},\alpha^{\lambda}) +  \verts{\mathcal{S}^{\lambda}_1(\alpha^{\star},\alpha^{\lambda})  - \mathcal{S}(\alpha^{\star},\alpha^{\lambda})} \\
    &\lesssim \lambda +\verts{\mathcal{S}^{\lambda}_1(\alpha^{\star},\alpha^{\lambda})  - \mathcal{S}(\alpha^{\star},\alpha^{\lambda})}
 \end{align}
where we used that $\mathcal{S}(\alpha^{\star},\alpha^{\lambda})\lesssim \lambda$ to get the second line. Therefore, it suffices to show that:
 \begin{align}
 	\verts{\mathcal{S}^{\lambda}_1(\alpha,\alpha')  - \mathcal{S}(\alpha,\alpha')}\lesssim \lambda.
 \end{align}
 uniformly over $\alpha$ and $\alpha'$ in $\mathcal{A}$. 
The term $\mathcal{S}^{\lambda}_1(\alpha,\alpha')$ is simply a Riemann  sum converging towards $\mathcal{S}(\alpha,\alpha')$ and whose convergence rate depends on the smoothness in time of the integrand. The finiteness of the moments of $\Pi_t$, along with the controls $\alpha^{\star}$ and $\alpha^{\lambda}$ and their gradient being locally Lipschitz in time and the variations of the potential $\dot{V}_t$ having a linear growth allows to get the desired rate.
\end{proof}
We are now ready to prove our main result of \cref{prop:discretization_error_main} which we restate as  \cref{thm:convergence_continuous_time} below for convenience and to keep a consistent notation. The proof simply combines the estimates in \cref{prop:discretization_error,prop:bound_diff_KL} along with the convergence results of the controls in \cref{prop:convergence_controls} and loss (\cref{prop:convergence_loss}).
\begin{thm}\label{thm:convergence_continuous_time}
	Under \cref{assump_cont:lipschitz_potential,assump_cont:Lipschiz_control,assump_cont:separated_optimum,assump_cont:moment,assump_cont:moment_2,assump_cont:existence_interpolating_controls} and if $\alpha^{\star}$ satisfies \cref{assum_cont:bounded_weights}, then it holds that:
	\begin{align}
		\textup{KL}\parens{\bar{\Pi}_{t}^{\alpha^{\star}}\middle|\middle| \bar{\Pi}_{t}^{\alpha^{\lambda},\lambda} }\lesssim \sqrt{\lambda}
	\end{align}
	Moreover, the discrete-time objective $\mathcal{L}_{\lambda}^{tot}(\alpha)$ converges towards $\mathcal{M}(\alpha)$ uniformly over the class $\mathcal{A}$ with at rate  $\lambda$.
\end{thm}
\begin{proof}
	By \cref{assump_cont:existence_interpolating_controls,assump_cont:separated_optimum} , we know that both  $\alpha^{\lambda}$ and $\alpha^{\star}$ exist and belong to the class of admissible controls $\mathcal{A}$. 
	We can then use the following decomposition of the relative entropy $\mathcal{D}^{\lambda}_t(\alpha^{\star},\alpha^{\lambda})$:
	\begin{align}
		\verts{\mathcal{D}^{\lambda}_t(\alpha^{\star},\alpha^{\lambda})}\leq  \verts{\mathcal{D}^{\lambda}_t(\alpha^{\star},\alpha^{\star})} + \verts{\mathcal{D}^{\lambda}_t(\alpha^{\star},\alpha^{\lambda})- \mathcal{D}^{\lambda}_t(\alpha^{\star},\alpha^{\star}) }.
	\end{align}
	\cref{prop:discretization_error} ensures that $\verts{\mathcal{D}^{\lambda}_t(\alpha^{\star},\alpha^{\star})}\lesssim \sqrt{\lambda}$, while \cref{prop:bound_diff_KL} shows that:
	\begin{align}
		\verts{\mathcal{D}^{\lambda}_t(\alpha^{\star},\alpha^{\lambda})- \mathcal{D}^{\lambda}_t(\alpha^{\star},\alpha^{\star}) }\lesssim \lambda + \mathcal{S}_t^{\lambda}(\alpha^{\star},\alpha^{\lambda}) + \mathcal{S}_t^{\lambda}(\alpha^{\star},\alpha^{\lambda})^{\frac{1}{2}}.
	\end{align}
	We can then use \cref{prop:convergence_controls} to ensure the discrepancy $\mathcal{S}^{\lambda}_t(\alpha^{\star},\alpha^{\lambda})\lesssim \lambda  $ for $\lambda$ small enough.
	This directly yields:
	\begin{align}
		\verts{D^{\lambda}_t(\alpha^{\star},\alpha^{\lambda})}\lesssim \sqrt{\lambda}.
	\end{align}
	Finally, \cref{prop:convergence_loss} shows that the discrete time losses $\mathcal{L}_{\lambda}^{tot}$ converge towards $\mathcal{M}$ uniformly over the class $\mathcal{A}$  at rate $\lambda$. This concludes the proof.
\end{proof}

\subsection{Uniform convergence of the objective: Proof of \cref{prop:convergence_loss_main} (\cref{prop:convergence_loss})}\label{appendix_proof_discrete_loss} 
We will first introduce some notations that are used only in this section. In all what follows, $T_k$ are normalizing flows of the form $T_k(x) = x + \lambda \alpha_{t_k}(x)$ for some fixed $\alpha\in \mathcal{A}$. We further denote by $T_{t,s}(x) = x + (s-t)\alpha_s(x)$ and introduce $R_{s,t}$:
\begin{align}
	R_{t,s} := \log(\frac{Z_t}{Z_s})+ \Pi_t\brakets{\parens{ V_s\circ T_{t,s}- V_t  - \log\vert\nabla T_{t,s}\vert) }}
\end{align}	
With the above notations we clearly have $T_k =  T_{t_{k-1},t_{k}}$ and $R_{t_{k},t_{k+1}} =  \textup{KL}((T_{k})_{\#}\Pi_{t_k}|| \Pi_{t_{k+1}})$. 
The discrete time objective can then be expressed in terms of $R$:
\begin{align}\label{eq:Rieman_sum}
	\mathcal{F}_{\lambda}^{tot}(\alpha) = \frac{1}{\lambda}\sum_{k=0}^{K-1} R_{t_k,t_{k+1}}.
\end{align}
We will show that such sum can be written as a double integral of the form:
\begin{align}
	\mathcal{F}_{\lambda}^{tot}(\alpha) = - \frac{1}{\lambda} \sum_{k=0}^{K-1} \int_{t_k}^{t_{k+1}}\int_{t_k}^{s}\partial_t\partial_s R_{t,s}\diff t\diff s.
\end{align}
This allows to view \cref{eq:Rieman_sum} as a Riemann's sum  whose limit is determined by the local behavior of the function $\partial_{t}\partial_s R_{t,s}$ when $\verts{t-s}$ is small. We first show in \cref{prop:flow_error} that $-\partial_{t}\partial_s R_{t,s}$ approaches $\Pi_t[(\overline{g}_t^{\alpha})^2]$ when $t$ and $s$ are close. The proof of \cref{prop:flow_error} is deferred to  \cref{sec:proof_techincal_lemmas} and relies on  \cref{lem:first_partial_derivatives_loss,eq:bounds_loss}. \cref{lem:first_partial_derivatives_loss} provides closed form expressions for the derivatives  $\partial_s R_{t,s}$ and $\partial_{t}\partial_s R_{t,s}$  as expectations of some integrable functions under $\Pi_t$. Then \cref{eq:bounds_loss} shows that the integrand  in $-\partial_{t}\partial_s R_{t,s}$ approaches  $(\overline{g}_{t}^{\alpha})^2(x)$ with an error that is  polynomial in $x$ and proportional to the distance $\verts{t-s}$.
\begin{prop}\label{prop:flow_error}
Under \cref{assump_cont:lipschitz_potential,assump_cont:Lipschiz_control,assump_cont:moment} and for $t$ and $s$  such that  $\vert  t-s\vert \leq \frac{1}{2L}$, it holds that:
	\begin{align}
		\label{eq:bound_loss_main_1}\verts{ \partial_t \partial_s R_{t,s}  +  \Pi_{t}\brakets{\parens{\overline{g}^{\alpha}_{t}}^2} } \lesssim\vert t-s \vert\\
		   \label{eq:bound_loss_main_2}\verts{    \Pi_{t}\brakets{\parens{\overline{g}^{\alpha}_{t}}^2} -  \Pi_{s}\brakets{\parens{\overline{g}^{\alpha}_{s}}^2}} \lesssim\vert t-s \vert 
	\end{align}
	Moreover, when $s=t$, we have:  $\partial_sR_{s,s} = 0$ and  $\partial_t\partial_s R_{t,t}=  \Pi_t\brakets{ \parens{\overline{g}_{t}^{\alpha}}^2  }  $.
\end{prop}
Now that we have described the behavior of $\partial_{t}\partial_s R_{t,s}$ when $\verts{t-s}$ is small, we can proceed to the proof of \cref{prop:convergence_loss}.
\begin{proof}{\textbf{of \cref{prop:convergence_loss}}}
	 By the fundamental theorem of calculus, we can directly write:
	\begin{align}
		R_{t_k,t_{k+1}} = \int_{t_k}^{t_{k+1}}\partial_s R_{t_k,s}\diff s = -\int_{t_k}^{t_{k+1}} \int_{t_k}^{s} \partial_t\partial_s R_{t,s}\diff t \diff s
	\end{align}
	where we used that  $R_{t_k,t_k}=0$ and that $\partial_s R_{s,s}=0$ by \cref{prop:flow_error}. Moreover, define $\mathcal{M}_{\lambda} $ such that:
	\begin{align}
		\mathcal{M}_{\lambda}(\alpha):=   \frac{1}{2}\sum_{k=0}^{K-1}\int_{t_k}^{t_{k+1}}\Pi_{t_k}\brakets{\parens{\overline{g}^{\alpha}_{t_k}}^{2}}\diff t = \frac{1}{\lambda} \sum_{k=0}^{K-1}\int_{t_k}^{t_{k+1}} \int_{t_k}^{s}  \Pi_{t_k}\brakets{\parens{\overline{g}^{\alpha}_{t_k}}^{2}} \diff t\diff s ,
	\end{align}
	where the second expression of $\mathcal{M}_{\lambda}(\alpha)$ is obtained by direct calculation. We will control both errors  $\verts{ \mathcal{F}^{tot}_{\lambda}(\alpha) - \mathcal{M}_{\lambda}(\alpha)}$ and $\verts{ \mathcal{M}_{\lambda}(\alpha) - \mathcal{M}(\alpha) }$. We get an upper-bound on $\verts{ \mathcal{M}_{\lambda}(\alpha) - \mathcal{M}(\alpha) }$ directly using the first expression of $\mathcal{M}_{\lambda}(\alpha)$  and using the Lipschitz smoothness of $\Pi_{t}\brakets{\parens{\overline{g}^{\alpha}_{t}}^{2}}$ as shown in  \cref{eq:bound_loss_main_2} of \cref{prop:flow_error}: 
	\begin{align}
		\verts{ \mathcal{M}_{\lambda}(\alpha) - \mathcal{M}_{\lambda}(\alpha) } &= \verts{ \frac{1}{2}\sum_{k=0}^{K-1}\int_{t_k}^{t_{k+1}}\Pi_{t_k}\brakets{\parens{\overline{g}^{\alpha}_{t_k}}^{2}} -  \Pi_{t}\brakets{\parens{\overline{g}^{\alpha}_{t}}^{2}} \diff t   } \\
		 &\lesssim \sum_{k=0}^{K-1}\int_{t_k}^{t_{k+1}} \verts{ t_k-t} \diff t \lesssim \lambda
	\end{align}
	To control $\verts{ \mathcal{F}^{tot}_{\lambda}(\alpha) - \mathcal{M}_{\lambda}(\alpha)}$ we use the second expression of $\mathcal{M}_{\lambda}(\alpha)$ and rely on the following decomposition:  
	\begin{align}
		\verts{ \mathcal{F}^{tot}_{\lambda}(\alpha) - \mathcal{M}_{\lambda}(\alpha)} & \leq \frac{1}{\lambda} \sum_{k=0}^{K-1} \int_{t_k}^{t_{k+1}} \int_{t_k}^{s} \verts{ \partial_t\partial_s R_{t,s} + \Pi_{t_k}\brakets{\parens{\overline{g}^{\alpha}_{t_k}}^{2}}} \diff t\diff s \\
		 &\leq \frac{1}{\lambda}\sum_{k=0}^{K-1}\int_{t_k}^{t_{k+1}} \int_{t_k}^{s}\parens{ \verts{\partial_t\partial_s R_{t,s} +\Pi_{t}\brakets{\parens{\overline{g}^{\alpha}_{t}}^{2}}} + \verts{ \Pi_{t}\brakets{\overline{g}^2_{t}} - \Pi_{t_k}\brakets{\parens{\overline{g}^{\alpha}_{t_k}}^{2}}}}  \diff t \diff s \\
		&\lesssim \frac{1}{\lambda}\sum_{k=0}^{K-1}\int_{t_k}^{t_{k+1}} \int_{t_k}^{s} \verts{ t-s} + \verts{ t - t_{k}} \diff t \diff s \lesssim \lambda \\
	\end{align}
	where we used the estimates  \cref{eq:bound_loss_main_1,eq:bound_loss_main_2} in  \cref{prop:flow_error} to get the last line. The result follows by direct application of triangular inequality.
\end{proof}

\subsubsection{Proof of the technical lemmas}\label{sec:proof_techincal_lemmas}
We first provide expressions for $\partial_{s}R_{t,s}$ and $\partial_{t,s}R_{t,s}$
\begin{lem}\label{lem:first_partial_derivatives_loss}
Under \cref{assump_cont:lipschitz_potential,assump_cont:Lipschiz_control,assump_cont:moment}, we have the following expressions for $\partial_{s}R_{t,s}$ and $\partial_{t,s}R_{t,s}$:
\begin{align}
	\partial_s R_{t,s} &=	\int  c_{t,s}(y) q_{t,s}(y)\diff y,\\
	\partial_t\partial_s R_{t,s} &= E_{t,s}^{(1)} + E_{t,s}^{(2)},
\end{align}
where we define $c_{t,s}$, $q_{t,s}$, $E_{t,s}^{(1)}$ and $E_{t,s}^{(2)}$ as:
\begin{align}\label{eq:intermediate_densities}
		q_{t,s}(y) &:= \Pi_t(T_{t,s}^{-1}(y))\vert \nabla T_{t,s}\parens{T_{t,s}^{-1}(y)} \vert^{-1} \\
	 	c_{t,s}(y) &:= \dot{V}_s(y)-\Pi_{s}[\dot{V}_s]+ \nabla V_s(y)^{\top}\alpha_{t}\parens{T_{t,s}^{-1}(y)} - Tr\parens{ \nabla T^{-1}_{t,s}(y)\nabla \alpha_t\parens{T_{t,s}^{-1}(y)} }\\
	 				E_{t,s}^{(1)} =& \Pi_{t}\brakets{ \parens{\partial_{t}\alpha_t + \parens{\nabla T_{t,s}}^{-1}\nabla \alpha_t }^{\top} \nabla \log\parens{ \frac{q_{t,s}}{\Pi_s}  }\circ T_{t,s}  }\\
			E_{t,s}^{(2)} =& \Pi_{t}\brakets{\parens{c_{t,s}\circ T_{t,s}} \parens{\partial_{t}\log(q_{t,s})\circ T_{t,s}}}
\end{align}
In particular, when $s=t$, we have:  $\partial_sR_{t,t} = 0$ and  $\partial_t\partial_s R_{t,t}=  \Pi_t\brakets{ \parens{\overline{g}_{t}^{\alpha}}^2  }  $.
\end{lem}
\begin{proof}
\begin{itemize}
	\item Computing $\partial_s R_{t,s}$.
\end{itemize}
First, recall that the control $\alpha$ satisfies \cref{eq:admissibility} of \cref{lem:smoothness_control} under \cref{assump_cont:Lipschiz_control}. In addition and under    \cref{assump_cont:lipschitz_potential,assump_cont:moment}, it is possible to apply the dominated convergence theorem, thus yielding:
	\begin{align}\label{eq:partial_R}
		\partial_s R_{t,s} = -\Pi_{s}\brakets{\dot{V}_s} +  \Pi_{t}\brakets{ \dot{V}_s \circ T_{t,s} + (\nabla V_s\circ T_{t,s})^{\top}\alpha_{t} - Tr\parens{ \nabla T_{t,s}^{-1}\nabla \alpha_t }}
	\end{align}
 We can then perform a change of variables $y = T_{t,s}(x)$ in the above expression to get:
\begin{align}
	\partial_s R_{t,s} =& \int \parens{ \dot{V}_s(y)-\Pi_{s}[\dot{V}_s]+ \nabla V_s(y)^{\top}\alpha_{t}\parens{T_{t,s}^{-1}(y)} - Tr\parens{ \nabla T^{-1}_{t,s}(y)\nabla \alpha_t\parens{T_{t,s}^{-1}(y)} }} q_{t,s}(y)\diff y\\
	=&
	\int  c_{t,s}(y) q_{t,s}(y)\diff y.
\end{align}
When $s=t$, we get  $c_{t,t}(x) =  \overline{g}_t^{\alpha}(x) $  and $q_{t,t}(x) = \Pi_t(x) $, thus  $\partial_s R_{t,t} = 0$.
\begin{itemize}
	\item Computing $\partial_t \partial_s R_{t,s}$.
\end{itemize}
To compute $\partial_t \partial_s R_{t,s}$ we will introduce an auxiliary function $r_{t,t',s}$:
\begin{align}\label{eq:auxiliarty_function}
	r_{t,t',s} :=& \int \parens{ \dot{V}_s(y)-\Pi_{s}[\dot{V}_s] + \nabla V_s(y)^{\top}\alpha_{t}\parens{T^{-1}_{t,s}(y)}}q_{t',s}(y) \diff y  \\
	&- \int\parens{  Tr\parens{ \nabla T^{-1}_{t,s}(y)\nabla \alpha_t\parens{T^{-1}_{t,s}(y)} }}q_{t',s}(y)\diff y.\\
	=& \int  c_{t,s}(y)q_{t,s}(y)\diff y.
\end{align}
Hence, from the expression of $\partial_s R_{t,s}  $, it always holds that  $\partial_s R_{t,s} = r_{t,t,s} $. Moreover, provided the  partial derivatives $\partial_t r_{t,t',s} $ and $\partial_{t'}r_{t,t',s}$ are continuous, the partial derivative $\partial_t \partial_s R_{t,s}$ is simply given by:
\begin{align}
	\partial_t \partial_s R_{t,s} = \partial_t r_{t,t',s}|_{t'=t} + \partial_{t'}r_{t,t',s}|_{t'=t}.
\end{align} 
Thus, we only need to compute each term $E_{t,s}^{(1)}= \partial_t r_{t,t',s}|_{t'=t}$ and $E_{t,s}^{(2)}= \partial_{t'}r_{t,t',s}|_{t'=t}$ separately.
\paragraph{Computing $E_{t,s}^{(1)} = \partial_t r_{t,t',s}|_{t'=t}$.}
First, a simple computation shows that 
\begin{align}
	 \nabla \cdot [\alpha_t(T_{t,s}^{-1}(y))  ] =  Tr(\nabla \alpha_t(T_{t,s}^{-1}(y))\nabla T_{t,s}^{-1}(y))
\end{align}
hence, we can perform integration by parts on the last term of \cref{eq:auxiliarty_function} to obtain a second expression for $r_{t,t',s}$
\begin{align}\label{eq:second_exp_derivative_KL}
	r_{t,t',s} =  \int \parens{\dot{V}_s(y)- \Pi_s\brakets{ \dot{V}_s}} q_{t',s}(y)\diff y  +  \int \alpha_t\parens{T^{-1}_{t,s}(y)}^{\top}  \nabla \log\parens{ \frac{q_{t',s}(y)}{\Pi_s(y)}  } q_{t',s}(y) \diff y.
\end{align}
After differentiating \cref{eq:second_exp_derivative_KL} w.r.t. $t$, only the second integral in \cref{eq:second_exp_derivative_KL} contributes to the expression of the derivative:
\begin{align}
	\partial_t r_{t,t',s} = \int \parens{ \partial_t \alpha_t\parens{T^{-1}_{t,s}\parens{y}} + \nabla T^{-1}_{t,s}\parens{y}\nabla\alpha_t\parens{T^{-1}_{t,s}(y)}}^{\top}   \nabla \log\parens{ \frac{q_{t',s}(y)}{\Pi_s(y)}  }q_{t',s}(y)\diff y.
\end{align}
Taking $t'=t$ and performing again a change of variables $y = T_{t,s}(x)$, it follows that:
\begin{align}
	E_{t,s}^{(1)} &= \int \parens{ \partial_t \alpha_t(x) + \parens{\nabla T_{t,s}(x)}^{-1} \nabla \alpha_t(x)  }^{\top} \nabla \log\parens{ \frac{q_{t,s}}{\Pi_s}  }(T_{t,s}(x))\Pi_{t}(x)\diff x.\\
	&= \Pi_{t}\brakets{ \parens{\partial_{t}\alpha_t + \parens{\nabla T_{t,s}}^{-1}\nabla \alpha_t }^{\top} \nabla \log\parens{ \frac{q_{t,s}}{\Pi_s}  }\circ T_{t,s}  }
\end{align}

\paragraph{Computing $E_{t,s}^{(2)} =\partial_{t'}r_{t,t',s}|_{t'=t}$.}
From \cref{eq:auxiliarty_function} we have that $r_{t,t',s} = \int c_{t,s}(y)q_{t',s}(y)\diff y$. Therefore, by the dominated convergence theorem,
we can write:
\begin{align}
	\partial_{t'}r_{t,t',s} = \int c_{t,s}(y)  \partial_t\log\parens{q_{t',s}(y)} q_{t',s}(y))\diff y.  
\end{align}
Taking $t'=t$ and applying a change of variables $ y = T_{t,s}(x) $, we get:
\begin{align}
	E_{t,s}^{(2)} =& \int c_{t,s}(T_{t,s}(x)) \partial_{t}\log(q_{t,s})(T_{t,s}(x)) \Pi_t(x)\diff x\\
	=& \Pi_{t}\brakets{\parens{c_{t,s}\circ T_{t,s}} \parens{\partial_{t}\log(q_{t,s})\circ T_{t,s}}}.
\end{align}
\end{proof}
Next we show that the functions  $\partial_t \log(q_{t,s}(T_{t,s}(x)))$ and $-c_{t,s}(T_{t,s}(x))$ appearing in the expression of $\partial_t\partial_s R_{t,s}$ provided in \cref{lem:first_partial_derivatives_loss}, approach $\overline{g}_{t}^{\alpha}(x)$ point-wise with an error proportional to $\verts{t-s}$.
\begin{lem}\label{eq:bounds_loss}
Under \cref{assump_cont:lipschitz_potential,assump_cont:Lipschiz_control,assump_cont:moment} and for $t$ and $s$  such that  $\vert  t-s\vert \leq \frac{1}{2L}$, it holds:
\begin{align}
\label{eq:bound_derivative_q}
	\verts{\partial_t \log(q_{t,s}(T_{t,s}(x))) -\overline{g}_{t}^{\alpha}(x) }\lesssim \verts{t-s}\parens{1 + \Verts{x} + \Verts{x}^2 }\\
\label{eq:bound_b}	\verts{c_{t,s}(T_{t,s}(x)) + \overline{g}_{t}^{\alpha}(x)  }\lesssim \verts{t-s}\parens{1 + \Verts{x} + \Verts{x}^2 }\\
\label{eq:bound_score} \Verts{ \nabla \log\parens{ \frac{q_{t',s}}{\Pi_s}}(T_{t,s}(x)) } 
	 \lesssim \vert t-s \vert \parens{1 + \Verts{x} + \Verts{x}^2 } \\ 
\end{align}
\end{lem}
The proof of \cref{eq:bounds_loss} relies on the expressions of the time derivatives $\partial_t T_{t,s}^{-1}(y)$, $\partial_t\nabla T_{t,s}^{-1}(y)$ and $\partial_t \log\verts{  \nabla T_{t,s}^{-1}(y)  } $ which we provide in \cref{lem:derivatives_flow} without proof as they follow by direct calculations.
\begin{lem}\label{lem:derivatives_flow}
Let $y$ be in $\mathcal{X}$ and denote $x_{t,s} := T_{t,s}^{-1}(y) $ and $C_{t,s} =  \nabla T_{t,s}^{-1}(y) $. Under \cref{assump_cont:lipschitz_potential,assump_cont:Lipschiz_control,assump_cont:moment} and using  \cref{eq:admissibility} of \cref{lem:smoothness_control}, it holds that:
	\begin{align}
		\partial_t T_{t,s}^{-1}(y) &= C_{t,s}\parens{\alpha_t(x) -(s-t)\partial_t\alpha_t(x) }\\
		\partial_t\nabla T_{t,s}^{-1}(y) &= C_{t,s}\nabla\alpha_t(x_{t,s})C_{t,s} - (s-t)C_{t,s}\mathcal{G}_{t,s}(y) C_{t,s} \\
				\partial_t \log\verts{  \nabla T_{t,s}^{-1}(y)  } &= tr(C_{t,s} \nabla \alpha_t(x) )-(s-t) Tr(C_{t,s} \mathcal{G}_{t,s}(y) )
	\end{align}
	with 
	\begin{align}\label{function_g}
		\mathcal{G}_{t,s}(y) &:= \partial_t\nabla \alpha_t(x_{t,s}) + H\alpha_t(x_{t,s})\partial_t T_{t,s}^{-1}(y)
	\end{align}
\end{lem}
We can now prove \cref{eq:bounds_loss}.
\begin{proof}{\bf{of \cref{eq:bounds_loss} } }
Let $y\in \mathcal{X}$ and define  $x_{t,s} =  T^{-1}_{t,s}(y) $ and $C_{t,s} =  \nabla T^{-1}_{t,s}(y)$. 

	\paragraph{Equation \cref{eq:bound_derivative_q}.}
	We first express the partial derivative $\partial_t\log(q_{t,s})$:
\begin{align}
	\partial_t \log(q_{t,s}(y)) =& -\parens{\dot{V}_t(x_{t,s}) - \Pi_t\brakets{ \dot{V}_t}   + \nabla V_t(x_{t,s})\partial_t T_{t,s}^{-1}(y) -  \partial_t \log \vert \nabla T_{t,s}^{-1}(y) \vert} \\
	=& -\parens{\dot{V}_t(x_{t,s}) - \Pi_t\brakets{ \dot{V}_t}   + \nabla V_t(x_{t,s}) C_{t,s} \alpha_t(x_{t,s})  - Tr\parens{C_{t,s}\nabla \alpha_t(x_{t,s}) }   }\\
	&+(s-t)\nabla V_t(x_{t,s}) \nabla C_{t,s}\partial_t\alpha_t(x_{t,s})     \\
	& -(s-t)Tr\parens{C_{t,s}\parens{\partial_t\nabla \alpha_t(x_{t,s}) + H\alpha_t(x_{t,s}) C_{t,s} \parens{\alpha_t(x_{t,s})-  (s-t)\partial_t\alpha_t(x_{t,s}) } }} 
\end{align}
where the expressions of the time derivatives $\partial_t T_{t,s}^{-1}(y)$ and $\partial_t \log \vert \nabla T_{t,s}^{-1}(y)\vert$ are given by \cref{lem:derivatives_flow}. We can then evaluate $\partial_t \log(q_{t,s}(y))$ at $y =  T_{t,s}(x) $ for some $x\in \mathcal{X}$:
\begin{align}
	\partial_t \log(q_{t,s}(T_{t,s}(x))) :=& 	-\parens{\dot{V}_t(x) - \Pi_t\brakets{ \dot{V}_t}   + \nabla V_t(x) C_{t,s} \alpha_t(x)  - Tr\parens{C_{t,s}\nabla \alpha_t(x) }   }\\
	&+(s-t)\nabla V_t(x) C_{t,s}\partial_t\alpha_t(x)     \\
	& -(s-t)Tr\parens{C_{t,s}\parens{\partial_t\nabla \alpha_t(x) + H\alpha_t(x) C_{t,s} \parens{\alpha_t(x)-  (s-t)\partial_t\alpha_t(x) } }}
\end{align}
Hence, by definition of $\overline{g}_t^{\alpha}$, we have:
\begin{align}
	\partial_t \log(q_{t,s}(T_{t,s}(x)))-\overline{g}_t^{\alpha}(x) =&
	-\nabla V_t(x)^{\top}\parens{C_{t,s}-I }\alpha_t(x) + Tr\parens{\parens{C_{t,s}-I}\nabla \alpha_t(x)}\\
	&+(s-t)\nabla V_t(x) C_{t,s}\partial_t\alpha_t(x)     \\
	& -(s-t)Tr\parens{C_{t,s}\parens{\partial_t\nabla \alpha_t(x) + H\alpha_t(x) C_{t,s} \parens{\alpha_t(x)-  (s-t)\partial_t\alpha_t(x) } }}
\end{align}
We first note that:
\begin{align}
	C_{t,s}-I &= \nabla T_{t,s}^{-1}(y) -I  =   (\nabla T_{t,s}(x))^{-1}-I \\
	&= \parens{I + (s-t)\nabla \alpha_t(x) }-I \\
	&= (t-s)C_{t,s}\nabla\alpha_t(x).
\end{align}
 This implies:
 \begin{align}
  \partial_t \log(q_{t,s}(T_{t,s}(x)))-\overline{g}_t^{\alpha}(x)  =&
	(s-t)\parens{\nabla V_t(x)^{\top}C_{t,s}\nabla\alpha_t(x)\alpha_t(x) - Tr\parens{C_{t,s}\nabla\alpha_t(x)^2}}\\
	&+(s-t)\nabla V_t(x) C_{t,s}\partial_t\alpha_t(x)     \\
	& -(s-t)Tr\parens{C_{t,s}\parens{\partial_t\nabla \alpha_t(x) + H\alpha_t(x) C_{t,s} \parens{\alpha_t(x)-  (s-t)\partial_t\alpha_t(x) } }}.
 \end{align}
Finally, we know by \cref{lem:Linear_growth} that $\nabla V_t$ and $\alpha_t$ have at most a linear growth in $x$ and by \cref{lem:smoothness_control} we also have that $\nabla\alpha_t$ and $\partial_t \alpha_t$ are bounded, that $H\alpha_t(x)$ and $\partial_t\nabla \alpha_t$ have at most a linear growth in $x$. This directly yields the desired upper-bound:
\begin{align}
	\verts{ \partial_t \log(q_{t,s}(T_{t,s}(x)))-\overline{g}_t^{\alpha}(x)}\lesssim \verts{t-s}\parens{1 + \Verts{x} + \Verts{x}^2}.
\end{align}

\paragraph{Equation \cref{eq:bound_b}}
We will first control the term $c_{t,s}(T_{t,s}(x)) + \overline{g}_s^{\alpha}(x)$, then we will show that $\overline{g}_t^{\alpha}(x) - \overline{g}_s^{\alpha}(x)$ is of the same order. 
By definition of $c_{t,s}(x)$ and $\overline{g}_s^{\alpha}$, we can write:
\begin{align}
	c_{t,s}(T_{t,s}(x)) + \overline{g}_s^{\alpha}(x) =& \parens{\dot{V}_s(T_{t,s}(x))-\dot{V}_s(x)} + \parens{\nabla V_{s}\parens{T_{t,s}(x)}-\nabla V_s(x)}^{\top}\alpha_t(x)\\
	&+ \nabla V_s(x)^{\top}\parens{\alpha_t(x)-\alpha_s(x)}+  Tr\parens{\nabla \alpha_s(x)-\nabla\alpha_t(x)	}\\
	&+ (s-t)Tr\parens{C_{t,s}\nabla \alpha_t(x)^2 }. 
\end{align}
Hence, using \cref{assump_cont:lipschitz_potential} nad \cref{lem:smoothness_control} we have:
\begin{align}
	\Verts{c_{t,s}(T_{t,s}(x)) + \overline{g}_s^{\alpha}(x)}\lesssim \verts{t-s}\parens{1 + \Verts{x} + \Verts{x}^2}.
\end{align}
Similarly, we can show that $\overline{g}_t^{\alpha}(x) - \overline{g}_s^{\alpha}(x)$ satisfies a similar bound. This allows to get \cref{eq:bound_b} using triangular inequality.
\paragraph{Equation \cref{eq:bound_score}}
By direct calculation, we have:
\begin{align}
	\nabla \log\parens{ \frac{q_{t,s}}{\Pi_s}}(y) 
	&= 
	\nabla V_s(y) -  C_{t,s} \nabla V_t(x_{t,s}) - (s-t)Tr(C_{t,s} H\alpha_t(x_{t,s})C_{t,s} ).
\end{align}
In particular, choosing $y = T_{t,s}(x)$, we get:
\begin{align}
	\nabla \log\parens{ \frac{q_{t,s}}{\Pi_s}}(T_{t,s}(x)) =&  \nabla V_s(T_{t,s}(x)) -  C_{t,s} \nabla V_t(x) - (s-t)Tr(C_{t,s} H\alpha_t(x)C_{t,s} )\\
 =& 
 \nabla V_s(T_{t,s}(x))-V_t(x)  -  (C_{t,s}-I)\nabla V_t(x) - (s-t)Tr(C_{t,s} H\alpha_t(x)C_{t,s} ).
\end{align}
The result follows directly by a similar argument as done previously.

\end{proof}

\begin{proof}{\bf{of \cref{prop:flow_error}}}
\paragraph{Equation  \cref{eq:bound_loss_main_1}}
For the first inequality, we use the expression of $\partial_t \partial_s R_{t,s} $ from \cref{lem:first_partial_derivatives_loss}:
	\begin{align}
		\partial_t\partial_s R_{t,s} &= E_{t,s}^{(1)} + E_{t,s}^{(2)}
	\end{align}
	where
	\begin{align}
			E_{t,s}^{(1)} =& \Pi_{t}\brakets{ \parens{\partial_{t}\alpha_t + \parens{\nabla T_{t,s}}^{-1}\nabla \alpha_t }^{\top} \nabla \log\parens{ \frac{q_{t,s}}{\Pi_s}  }\circ T_{t,s}  },\\
			E_{t,s}^{(2)} =& \Pi_{t}\brakets{\parens{c_{t,s}\circ T_{t,s}} \parens{\partial_{t}\log(q_{t,s})\circ T_{t,s}}}.
	\end{align}
Recall that $\nabla \alpha_t$ and  $\partial_{t}\alpha_t$ are bounded by \cref{lem:smoothness_control} and that since $\lambda\leq \frac{1}{2L}$, we also have that $\Verts{\nabla T_{t,s}}^{-1}$ is bounded. Moreover, \cref{eq:bound_score} allows to write:
\begin{align}\label{eq:bound_loss_1}
	\verts{E_{t,s}^{(1)}}\lesssim \verts{t-s} \Pi_{t}\brakets{1+\Verts{X}  + \Verts{X}^2 }. 
\end{align}
For $E_{t,s}^{(2)}$, we use the following decomposition:
\begin{align} 
	E_{t,s}^{(2)} + \Pi_{t}\brakets{\parens{\overline{g}_t^{\alpha}}^2}  =& \Pi_t\brakets{ \parens{c_{t,s}\circ T_{t,s}  + \overline{g}_t^{\alpha}} \parens{\partial_{t}\log(q_{t,s})\circ T_{t,s}}}\\
	&- \Pi_t\brakets{ \overline{g}_t^{\alpha}\parens{\partial_{t}\log(q_{t,s})\circ T_{t,s} -  \overline{g}_t^{\alpha} }}.
\end{align}
The functions $c_{t,s}\circ T_{t,s}$ and $\overline{g}_t^{\alpha}$ have at most a quadratic growth and by \cref{eq:bound_derivative_q,eq:bound_b} we get:
\begin{align}\label{eq:bound_loss_2}
	\verts{E_{t,s}^{(2)} + \Pi_{t}\brakets{\parens{\overline{g}_t^{\alpha}}^2}}\lesssim \verts{t-s}\Pi_{t}\brakets{1+\Verts{X}  + \Verts{X}^2 + \Verts{X}^3 + \Verts{X}^4 }.
\end{align}
Both bound \cref{eq:bound_loss_1,eq:bound_loss_2} are finite due to \cref{assump_cont:moment}, which implies the result.
\paragraph{Equation  \cref{eq:bound_loss_main_2}.} To get the last inequality it suffice to  differentiate in time and use the growth assumptions and integrability conditions.
\end{proof}

\subsection{Diffusion estimates}

\begin{lem}\label{lem:main_bounds}
Under \cref{assump_cont:lipschitz_potential,assump_cont:Lipschiz_control,assump_cont:moment}, the following inequalities hold:
\begin{align}
\mathbb{E}\brakets{\Verts{X_t}^2 } \lesssim 1\qquad \mathbb{E}\brakets{\Verts{X_t}^4}\lesssim 1,\qquad 
\mathbb{E}\brakets{\Verts{X_t- X_s}^2 }\lesssim \verts{t-s}
\end{align}
\end{lem}
\begin{proof}
This is  a direct consequence of the drift $b_{t}(x)$ being jointly Lipschitz in $t$ and $x$ and the initial distribution $\Pi_0$ having finite moments of order $4$.
\end{proof}

\subsubsection{Proof of  \cref{prop:Expectations_bound}}\label{sec:proof_expecation_bound}

\begin{proof}{\bf{of \cref{prop:Expectations_bound}}}
	\paragraph{Bound on Equation \cref{eq:exp_bound_1}}
	We use the decomposition of $w_t^{\alpha}$ as a product of $w_s^{\alpha}$ and a bounded term:
	\begin{align}
		\mathbb{E}\brakets{\overline{w}_t^{\alpha} f(X_s)  } = \mathbb{E}\brakets{\overline{w}_s^{\alpha} f(X_s) \exp\parens{\int_s^{t} \overline{g}^{\alpha}_u \diff u}  }.
	\end{align}
	By \cref{lem:bounded_weights}, we know that the normalized weight $\bar{w}_t^{\alpha}$ is upper-bounded by a term of the form $\exp(C(t-s))$, hence we get
	 	\begin{align}
		\Verts{\mathbb{E}\brakets{\overline{w}_t^{\alpha} f(X_s)  }} &\lesssim \mathbb{E}\brakets{\overline{w}_s^{\alpha} \Verts{f(X_s)}} = 
	\Pi_s\brakets{\Verts{f}}
	\end{align}
	where we used  \cref{prop:controlled_non_equilibrium_dynamics_with_birth_death} to get the final result.

	\paragraph{Bound on Equation \cref{eq:exp_bound_2}}
	We first define the random variable:
\begin{align}
	\Psi_{s,s'}^t := \mathbb{E}\brakets{\exp\parens{\int_s^{t}  \bar{g}_u^{\alpha}(X_u) \diff u } \int_s^{s'} \diff B_u  \middle| \mathcal{F}_s  },
\end{align}
where $\mathcal{F}_{s}$ is the filtration associated to the Brownian motion $B_t$ defining the process $X_{t}$.
Using the tower property for conditional expectations and Cauchy-Schwartz, we have
\begin{align}
\verts{\mathbb{E}\brakets{ \overline{w}_t^{\alpha} f(X_s)^{\top} \int_{s}^{s'}  \diff B_u }} =   \verts{\mathbb{E}\brakets{\overline{w}_t^{\alpha} f(X_s)^{\top} \Psi_{s,s'}^t }}\leq \mathbb{E}\brakets{\overline{w}_t^{\alpha} \Verts{f}^2 }^{\frac{1}{2}} \mathbb{E}\brakets{\overline{w}_t^{\alpha} \Verts{\Psi_{s,s'}^t}^2 }^{\frac{1}{2}}.
\end{align}
Using \cref{eq:exp_bound_1} for both terms on the r.h.s., we obtain 
\begin{align}
	\verts{\mathbb{E}\brakets{ \overline{w}_t^{\alpha} f(X_s)^{\top} \int_{s}^{s'}  \diff B_u }}\lesssim \Pi_s\brakets{\Verts{f}^2}^{\frac{1}{2}} \mathbb{E}\brakets{ \Verts{\Psi_{s,s'}^t}^2 }^{\frac{1}{2}}.
\end{align}
It remains to show that $\mathbb{E}\brakets{ \Verts{\Psi_{s,s'}^t}^2 }\lesssim (s'-s)^2$. To achieve this, we first introduce the notation  $\mathbb{E}_s$ for the conditional expectation knowing the process up to time $s$. We also write $w_{s,s'}(X) = \exp( \int_s^{s'} \overline{g}_u^{\alpha}(X_u)\diff u )$  and $r_{s,t}(X_s) = \mathbb{E}_{s}[w_{s,t}(X)]$.
We consider now a second process $(X_t')_{t\geq 0}$ that is coupled to $(X_t)_{t\geq 0}$ as follows. $X_u=X'_u$ up to time $s$, then on the interval $[s,s']$, $X'_t$ is driven by a Brownian motion $B_t'$ that is independent from $B_t$. Finally, stating from time $s'$ up to time $t$, the process $X'_t$ is driven again by the same Brownian motion $B_t$ as $X_t$. Hence, $(X_{s},X_s^{'})$ defines a coupling of two SDEs.  We denote by $\mathbb{E}_{s',s}$ the conditional expectation knowing the process $(X_{t})$ up to time $s'$ and the process $(X_{t}')$ up to time $s$. Such coupling is shown in \cref{lem:main_bounds_gronwal} to satisfy $\mathbb{E}[\int_s^{t} \Verts{X_u-X'_u}^2\diff u  ]\lesssim \vert s'- s\vert $. 
To use this property, we will express $\Psi_{s,s'}^t$ in terms of the two coupled processes. We start by noting that $\mathbb{E}_s\brakets{ \mathbb{E}_s[w_{s,t}(X')]\parens{B_{s'}-B_s} } = 0  $ by independence of the increments of Brownian motions. Hence, we can write
\begin{align}
	\Verts{\Psi_{s,s'}^t} &= \Verts{\mathbb{E}_{s}\brakets{\parens{w_{s,t}(X) - \mathbb{E}_s[w_{s,t}(X')] }\parens{B_{s'}-B_s}}}\\
	&\leq \mathbb{E}_s\brakets{\verts{\mathbb{E}_{s',s}\brakets{w_{s,t}(X)- w_{s,t}(X') }}\Verts{B_{s'}-B_s} }\\
	&\leq e^{C(t-s)}\mathbb{E}_{s,s}\brakets{\Verts{B_{s'}-B_s}\int_s^{t} \verts{g^{\alpha}_u(X_u) -  g^{\alpha}_u(X'_u)}\diff u     }
\end{align}
	For the last line, we  rely on local Lipschitzness  of the exponential along with the fact  $\bar{g}^{\alpha}_{u}(X_u)\leq C$ by \cref{assum_cont:bounded_weights}. Moreover, \cref{lem:appendix_estimate_control} shows that 
	\begin{align}
		\verts{g_t^{\alpha}(x) - g_t^{\alpha}(x')} &\lesssim \parens{1 + \Verts{x} + \Verts{x'} }\Verts{x-x'},
	\end{align}
	    hence we have:
\begin{align}
	\Verts{\Psi_{s,s'}^t} &\lesssim 
	\mathbb{E}_{s,s}\brakets{\Verts{B_{s'}-B_s}\int_s^{t} \parens{1 + \Verts{X_u}+ \Verts{X'_u}}\Verts{X_u-X'_u}\diff u     }\\
	&\lesssim
	\mathbb{E}_{s,s}\brakets{\Verts{B_{s'}-B_s}^2\int_s^{t} \parens{1 + \Verts{X_u}+ \Verts{X'_u}}^2\diff u     }^{\frac{1}{2}}\mathbb{E}_{s,s}\brakets{\int_s^{t} \Verts{X_u-X'_u}^2\diff u }^{\frac{1}{2}}\\
	&\lesssim \verts{s'-s}\mathbb{E}_{s,s}\brakets{\int_s^{t} \Verts{X_u-X'_u}^2\diff u }^{\frac{1}{2}}
\end{align}
where for the second line we used Cauchy-Schwarz inequality. For the last line, we used the fact that the processes  $X_s$ and $X_s'$ have finite moments of order $4$  by \cref{lem:main_bounds} and that the increment $B_{s'}-B_s$ follows a zero-mean multivariate normal of covariance given by $(s'-s)I_d$. Finally, by application of \cref{lem:main_bounds_gronwal}, we get the desired bound and conclude the proof
\begin{align}
	\mathbb{E}\brakets{\Verts{\Psi_{s,s'}^t}^2} &\lesssim 
	\verts{s'-s}\mathbb{E}\brakets{\int_s^{t} \Verts{X_u-X'_u}^2\diff u  }\lesssim \verts{s'-s}^2.
\end{align}

	\paragraph{Bound on Equation \cref{eq:exp_bound_3}}
	We first apply Cauchy-Schwarz inequality:
	\begin{align}\label{eq:first_inequality}
		 \verts{\mathbb{E}\brakets{ \overline{w}_t^{\alpha} f(X_s)^{\top} \int_{s}^{s'} h_u(X_u,X_s)\diff u }}&\leq \mathbb{E}\brakets{\overline{w}_t^{\alpha} \Verts{f(X_s)}^2}^{\frac{1}{2}} \mathbb{E}\brakets{\overline{w}_t^{\alpha} \Verts{\int_{s}^{s'}  h_u(X_u,X_s)\diff u}^2 }^{\frac{1}{2}}\\
		 &\lesssim \Pi_s\brakets{\Verts{f}^2}^{\frac{1}{2}} \mathbb{E}\brakets{\Verts{\int_{s}^{s'}  h_u(X_u,X_s)\diff u}^2 }^{\frac{1}{2}}\\
		 &\lesssim \sqrt{s'-s}\Pi_s\brakets{\Verts{f}^2}^{\frac{1}{2}} \mathbb{E}\brakets{\int_{s}^{s'}  \Verts{h_u(X_u,X_s)}^2\diff u }^{\frac{1}{2}}
	\end{align}
	To get the second line in the above inequality, we use \cref{eq:exp_bound_1} for the first expectation in the r.h.s.  and use that $\bar{w}_t^{\alpha}$ is bounded by \cref{lem:bounded_weights} to bound the second expectation. The last line follows by direct application of Cauchy-Schwarz inequality in time and yields the desired result.
In the particular case when $h_u(x,y) =  b_u^{\alpha}(x)- \beta_{u'}^{\alpha,\lambda}(y)$ we further use \cref{prop:relative_entropy} which yields:
	\begin{align} 
		 \mathbb{E}\brakets{\int_s^{s'}\Verts{h_{u}(X_u,X_s)}^2}= \mathbb{E}\brakets{\int_s^{s'} \Verts{b_u^{\alpha}(X_u)- \beta_{u'}^{\alpha,\lambda}(X_{s}) }^2 \diff u}\lesssim \lambda\verts{s'-s}
\end{align} 
The result then follows directly by applying this bound to \cref{eq:exp_bound_3} and recalling that $\verts{s'-s}\leq \lambda$ by assumption. 
\end{proof}

\begin{lem}\label{lem:main_bounds_gronwal}
Let $s$ and $s'$ be two numbers in  $[0,t]$ with $s\leq s'$.
Let $(X_t)_{t\geq 0}$ and $(X'_t)_{t\geq 0}$ be two coupled processes following the SDE \cref{eq:SDE} and such that $X_t=X_t'$ up to time $s$, then on the interval $[s,s']$ the two processes are driven with two independent motions $B_t$ and $B_t'$ and finally, starting from time $s'$, the two processes are again diffused with the same Brownian motion.
Then, under \cref{assump_cont:lipschitz_potential} and \cref{lem:smoothness_control}, we have the following:
\begin{align}
	\mathbb{E}\brakets{ \int_{s}^{t}\Verts{X_u-X_u'}^{2}\diff u  }\lesssim \vert s-s' \vert
\end{align}
\end{lem}
\begin{proof}
By definition of the $X_t$ and $X_t'$, since they share the same Brownian motion starting from time $s'$, we have
	\begin{align}
		X_u &= X_{s'} + \int_{s'}^u b_l(X_l)\diff l + \sqrt{2}\int_{s'}^u\diff B_l,\\
		X_u' &= X_{s'}' + \int_{s'}^u b_l(X_l')\diff l + \sqrt{2}\int_{s'}^u\diff B_l.
	\end{align}
	Hence, 
	\begin{align}
		\Verts{X_u-X_u'} \leq \Verts{X_{s'}-X_{s}'} + \int_{s'}^u \Verts{b_l(X_l)-b_l(X_l')}\diff l.
	\end{align}
	By \cref{assump_cont:lipschitz_potential} and  \cref{lem:smoothness_control}, we have that $\Vert b_l(x)- b_l(x') \Vert\lesssim \Vert x- x'\Vert $, hence we have:
	\begin{align}
		\Verts{X_u-X_u'} \lesssim \Verts{X_{s'}-X_{s}'} + \int_{s'}^u \Verts{X_l-X_l'}\diff l.
	\end{align}
	By application of Gronwall's lemma, we have for any $s'\leq u\leq t$:
	\begin{align}\label{eq:error_gronwal_1}
		\Verts{X_u-X_u'} \lesssim \Verts{X_{s'}-X_{s}'}.
	\end{align}
	Now, for $s\leq u\leq s'$, we have 
	\begin{align}
		X_{u} &= X_s +\int_{s}^{u} b_l(X_l)\diff l + \sqrt{2}\int_{s}^{u}\diff B_l,\\
		X'_{u} &= X_s +\int_{s}^{u} b_l(X_l')\diff l + \sqrt{2}\int_{s}^{u}\diff B_l'.
	\end{align}
	Hence, we have the following bound:
	\begin{align}
		\Verts{X_{u}-X_{u}'}\lesssim \int_{s}^{u}  \Verts{X_l-X_l'}\diff l + \sqrt{2}\parens{\Verts{B_{u}-B_s} + \Verts{B'_{u}-B'_s}}.
	\end{align}
	This allows to upper-bound the expectation $\mathbb{E}[\Vert X_{u}-X_{u}'\Vert^2]^{\frac{1}{2}}$:
	\begin{align}
		\mathbb{E}\brakets{\Verts{X_{u}-X_{u}'}^2}^{\frac{1}{2}}\lesssim \int_s^{u}\mathbb{E}\brakets{\Verts{X_{l}-X_{l}'}^2}^{\frac{1}{2}} + 2\sqrt{2}\sqrt{(u-s)d}
	\end{align}
	Using Gronwall's lemma a second time, we get for any $s\leq u\leq s'$:
	\begin{align}\label{eq:error_gronwal_2}
		\mathbb{E}\brakets{\Verts{X_{u}-X_{u}'}^2}^{\frac{1}{2}}\lesssim \sqrt{u-s}.
	\end{align}
	For $u>s'$, we can then use \cref{eq:error_gronwal_1} along with  \cref{eq:error_gronwal_2} to write
	\begin{align}\label{eq:error_gronwal_3}
		\mathbb{E}\brakets{\Verts{X_{u}-X_{u}'}^2}\lesssim s'-s.
	\end{align}
	Finally,  using \cref{eq:error_gronwal_2,eq:error_gronwal_3} and integrating over $u$ on the interval $[s,t]$, we get the desired result after applying Fubini's theorem to exchange the order of the expectation and time integral:
	\begin{align}
		\mathbb{E}\brakets{\int_s^{t}\Verts{X_u-X_u'}^2 \diff u }\lesssim \vert s'-s\vert.
	\end{align} 
\end{proof} 

\subsection{Pointwise estimates}

\begin{lem}\label{lem:bound_intermediate_terms}
Under \cref{assump_cont:lipschitz_potential,assump_cont:Lipschiz_control,assump_cont:moment}, and for any $ t_{k-1}\leq s\leq t_k $ and $0\leq u \leq 1$,  it holds that
\begin{align}
	\verts{g_s^{\alpha}(x') - \delta_t^{\alpha,\lambda}(x) }\lesssim \parens{1+\Verts{x}+\Verts{x'} }\Verts{x -x'}+  \lambda\parens{1 +\Verts{x} + \Verts{x'}}^2.
\end{align}
\end{lem}

\begin{proof}
By definition of $\delta_{t}^{\alpha,\lambda}(x)$ in \cref{eq:discrete_drift} and using the fundamental theorem of calculus we have:
\begin{align}
	\delta_{t}^{\alpha,\lambda}(x) 	=&    \int_{0}^1  Tr\parens{\parens{\parens{I + \lambda u \nabla\alpha_{t}(x)}^{-1}-I}\nabla\alpha_{t}(x)} +  Tr(\nabla \alpha_t(x))  \diff u\\
			  &- \int_{0}^1 \parens{\dot{V}_{t-\lambda u}( y_u )+ \nabla V_{t-\lambda u}(y_u)^{\top}\alpha_{t}(x)}\diff u,
\end{align}
where we introduced $y_u = x+\lambda(1-u)\alpha_t(x)$.
This directly yields:
\begin{align}\label{lem:bound_intermediate_terms:eq_0}
	  g_{s}^{\alpha}(x) - \delta_t^{\alpha,\lambda}(x') = \int_0^1 \parens{B_{u} + C_{u} + D_{u} +E_{u}}  \diff u
\end{align}
where:
\begin{align}\label{eq:intermediate_terms}
	B_{u} &=  Tr(\nabla\alpha_s(x'))- Tr(\nabla \alpha_t(u)), \\
	C_{u} &= -Tr\parens{\parens{\parens{I + \lambda u \nabla\alpha_{t}(x)}^{-1}-I}\nabla\alpha_{t}(x)}, \\
	D_{u} &= \nabla V_{t-\lambda u}(y_u)^{\top}\alpha_{t}(x) - \nabla V_s(x')\alpha_s(x'),\\
	E_{u} &=  \dot{V}_{t-\lambda u}( y_u ) -\dot{V}_s(x').
\end{align}
\paragraph{Bound on $B_{u}$.}
Under \cref{assump_cont:Lipschiz_control}, \cref{lem:smoothness_control} applies and we directly have:
\begin{align}\label{lem:bound_intermediate_terms:eq_2}
	\verts{B_{u}}&\lesssim \parens{1+\Verts{x}+ \Verts{x'}}\parens{\Verts{x-x'}+\verts{t-s}}\\
	&\lesssim \parens{1+\Verts{x}+ \Verts{x'}}\parens{\Verts{x-x'}+\lambda}
\end{align}
where we used that $ t-\lambda \leq  s\leq t$ to get the last line.
\paragraph{Bound on $\vert C_{u}\vert$.}
By direct calculation we get:
\begin{align}
	C_{u} &= Tr\parens{\parens{\parens{I + \lambda u \nabla\alpha_{t}(x)}^{-1}-I}\nabla\alpha_{t}(x)} \\
	&=\lambda u Tr\parens{\parens{I + \lambda u \nabla\alpha_{t}(x)}^{-1} \nabla\alpha_{t}(x)^2}
\end{align}
By \cref{lem:smoothness_control}, we know that $\nabla\alpha_t(x)$ is bounded by a constant term $L$. Hence, when $\lambda<\frac{1}{2L}$, it follows that:
\begin{align}
	\Vert(I  +  u\lambda \nabla \alpha_{t}(x) )^{-1} \Vert\leq \frac{1}{1-u\lambda L}\leq 2. 
\end{align}
This directly implies that:
\begin{align}\label{lem:bound_intermediate_terms:eq_1}
	\verts{C_{u}}\leq 2L^2\lambda.
\end{align}
\paragraph{Bound on $D_{u}$.}
\begin{align}
	\verts{D_{u}} &= \verts{\parens{\nabla V_{t-\lambda u }(y_u)^{\top}-\nabla V_s(x')}^{\top}\alpha_t(x) + \nabla V_s(x')^{\top}\parens{\alpha_t(x)-\alpha_s(x')}}\\
	&\leq \Verts{\nabla V_{t-\lambda u }(y_u)^{\top}-\nabla V_s(x')}\Verts{\alpha_t(x)} + \Verts{\nabla V_s(x')}\Verts{\alpha_t(x)-\alpha_s(x')}\\
	&\lesssim \parens{\verts{t-s-\lambda u } + \Verts{y_u-x'}}\Verts{\alpha_t(x)} + \Verts{\nabla V_s(x')}\parens{ \Verts{x-x'} + \verts{t-s} }\\
	&\lesssim \parens{\lambda  + \Verts{y_u-x'} }\Verts{\alpha_t(x)} + \Verts{\nabla V_s(x')}\parens{ \Verts{x-x'} + \lambda }
\end{align}
where we used \cref{assump_cont:lipschitz_potential} and   \cref{eq:admissibility} of \cref{lem:smoothness_control} for the third line and that  $   \vert t-s - \lambda u\vert \leq \lambda $  and $\vert t-s \vert\leq \lambda$ to get the last line.  Moreover, by \cref{lem:Linear_growth}, we have that:
\begin{align}\label{eq:linear_growth_bound_intermediate_terms}
	\Verts{\alpha_t(x)} \lesssim (1+ \Verts{x} ),\qquad \Verts{\nabla V_t(x)}\lesssim (1+\Verts{x}). 
\end{align}
This allows to further write:
\begin{align}\label{lem:bound_intermediate_terms:eq_3}
	\verts{D_{u}} &\lesssim  \parens{\lambda(1+\Verts{x})  + \Verts{x-x'} }\parens{1+\Verts{x}}. 
\end{align}

\paragraph{Bound on $E_{u}$.}
Using \cref{assump_cont:lipschitz_potential}, we directly have:
\begin{align}
	\verts{E_{u}}&\lesssim \parens{1+\Verts{y_u}+\Verts{x'}}\parens{ \Verts{y_u-x'} + \verts{t-s-\lambda u} }\\
	&\lesssim
	\parens{1+\Verts{x}+\Verts{x'} +\lambda u\Verts{\alpha_t(x)} }\parens{ \Verts{x -x'} + \lambda u\Verts{\alpha_t(x)}  + \verts{t-s-\lambda u } }\\
	&\lesssim\parens{1+\Verts{x}+\Verts{x'} +\lambda u\Verts{\alpha_t(x)} }\parens{ \Verts{x -x'}+  \lambda u\Verts{\alpha_t(x)}   + \lambda },
\end{align}
	where we used  that  $   \vert t-s - \lambda u\vert \leq \lambda $  since $t-\lambda \leq  s\leq t$ and $0\leq u\leq 1$. Moreover, using that $\alpha$ has a linear growth in $x$ (by \cref{eq:linear_growth_bound_intermediate_terms})  we have:
	\begin{align}\label{lem:bound_intermediate_terms:eq_4}
		\verts{E_{u}}\lesssim \parens{1+\Verts{x}+\Verts{x'} }\parens{ \Verts{x -x'}+  \lambda\parens{1 +\Verts{x}} }.
	\end{align}
\paragraph{Final bound.} Using \cref{lem:bound_intermediate_terms:eq_1,lem:bound_intermediate_terms:eq_2,lem:bound_intermediate_terms:eq_3,lem:bound_intermediate_terms:eq_4} in \cref{lem:bound_intermediate_terms:eq_0}, it follows:
\begin{align}
	\verts{g_s^{\alpha}(x) - \delta_t^{\alpha,\lambda}(x') }\lesssim \parens{1+\Verts{x}+\Verts{x'} }\parens{ \Verts{x -x'}+  \lambda\parens{1 +\Verts{x}} }.
\end{align}
\end{proof}
\begin{prop}\label{prop:beta_error}
Let $t$ and $s$ be in $[0,1]$ such that $\vert t-s\vert \leq \lambda$. Under \cref{assump_cont:lipschitz_potential,assump_cont:Lipschiz_control}, it holds that:
\begin{align}
	\Verts{b_s^{\alpha}(x')-\beta_t^{\alpha,\lambda}(x) }\lesssim \Verts{x-x'} + \lambda \parens{1+\Verts{x}}.
\end{align}
\end{prop}
\begin{proof}
	By definition of $b_s^{\alpha}(x')$  and $\beta_t^{\alpha,\lambda}(x)$, the following identity holds:
	\begin{align}
		\Verts{b_s^{\alpha}(x')-\beta_t^{\alpha,\lambda}(x) } = \Verts{\alpha_s(x')-\alpha_t(x) - \nabla V_s(x')+\nabla V_t(x+\lambda \alpha_t(x))  }.
	\end{align}
	Therefore, since $\alpha_t(x)$ and $\nabla V_t(x)$ are Lipschitz by \cref{assump_cont:lipschitz_potential} and  \cref{lem:smoothness_control}, we have:
	\begin{align}
		\Verts{b_s^{\alpha}(x')-\beta_t^{\alpha,\lambda}(x) }&\leq L\parens{2\Verts{x-x'} + 2\verts{t-s}  + \lambda \Verts{\alpha_t(x)}}\\
		&\leq L\parens{2\Verts{x-x'} + 2\lambda  + \lambda \Verts{\alpha_t(x)}}.
	\end{align}
	Moreover, \cref{lem:Linear_growth} below ensures that $\Vert \alpha_t(x)\Vert \lesssim (1+\Verts{x}) $, therefore:
		\begin{align}
		\Verts{b_s^{\alpha}(x')-\beta_t^{\alpha,\lambda}(x) }&\lesssim \Verts{x-x'} + \lambda\parens{1 + \Verts{x}}. 
	\end{align}
\end{proof}

\begin{lem}\label{lem:Linear_growth}
	Under \cref{assump_cont:lipschitz_potential} the gradient of the potential has a linear growth in $x$, i.e:
	\begin{align}\label{eq:linear_growth_potential}
		\Verts{\nabla_x V_t(x)}\lesssim \parens{1+\Verts{x}},\qquad \forall x\in \mathcal{X}, \forall t\in [0,1]. 
	\end{align}
	Under \cref{assump_cont:Lipschiz_control}, there exists a positive constant $D>0$, such that for all $\alpha\in \mathcal{A}$.
	\begin{align}\label{eq:linear_growth_control}
		\Verts{\alpha_t(x)}\leq D\parens{1+\Verts{x}},\qquad \forall x\in \mathcal{X}, \forall t\in [0,1].
	\end{align}
\end{lem}
\begin{proof}
The equation \cref{eq:linear_growth_potential} is a simple consequence of the Lipschitz assumption on $\nabla V_t$. For \cref{eq:linear_growth_control}, we recall that under \cref{assump_cont:Lipschiz_control}, \cref{lem:smoothness_control} holds. Therefore we have
	\begin{align}
		\Verts{\alpha_t(x)}&\leq \Verts{\alpha_0(0)} +   \Verts{\alpha_t(x)- \alpha_0(0)}\\
		&\leq C + L\parens{\Verts{x} +  \verts{t} }\\
		&\leq C + L\parens{1+ \Verts{x}}.
	\end{align}
	The result follows since the constants $C$ and $L$ are independent of the choice of the control $\alpha$ in $\mathcal{A}$.
\end{proof}

\begin{lem}\label{lem:appendix_estimate_control}
	Under \cref{assump_cont:lipschitz_potential,assump_cont:Lipschiz_control}, the following bounds hold:
	\begin{align}
			\Verts{\beta_{t}^{\alpha,\lambda}(x)- \beta_{t}^{\alpha',\lambda}(x) }&\lesssim \Verts{\alpha_t(x)- \alpha'_t(x)},\\
			\verts{\delta_{t}^{\alpha,\lambda}(x)- \delta_{t}^{\alpha',\lambda}(x)}&\lesssim  \parens{\parens{1 + \Verts{x}}\Verts{\alpha_{t}(x)-\alpha_{t}'(x)}  + \Vert \nabla \alpha_{t}(x) - \nabla\alpha'_{t}(x)  \Vert },\\
	\verts{g_t^{\alpha}(x) - g_t^{\alpha}(x')} &\lesssim \parens{1 + \Verts{x} + \Verts{x'} }\Verts{x-x'}.
\end{align}
\end{lem}
\begin{proof}
{\bf Bound on $ \vert\beta_{t}^{\alpha,\lambda}- \beta_{t}^{\alpha',\lambda}\vert$.}
By \cref{assump_cont:lipschitz_potential} and \cref{lem:smoothness_control} we have that:
\begin{align}
	\Verts{\beta_{t}^{\alpha,\lambda}(x)- \beta_{t}^{\alpha',\lambda}(x) }\leq \parens{1+L\lambda}\Verts{\alpha_t(x)- \alpha'_t(x)}.
\end{align}
{\bf Bound on $\vert\delta_{t}^{\alpha,\lambda}- \delta_{t}^{\alpha',\lambda}\vert$.} For the second inequality, we use that:
\begin{align}
	\lambda\parens{\delta_{t}^{\alpha,\lambda}(x)- \delta_{t}^{\alpha',\lambda}(x)} =&  V_{t}\parens{x + \lambda \alpha_{t}'(x)} - V_{t}\parens{x+\lambda \alpha_{t}(x) }\\ 
	&+ \log\verts{I+\lambda \nabla\alpha_{t}(x) } -  \log\verts{I+\lambda \nabla \alpha_{t}'(x) }.
\end{align}
Hence, for $\lambda\leq \frac{1}{2L}$, we get by \cref{assump_cont:lipschitz_potential} and  \cref{lem:smoothness_control} that
\begin{align}
	\lambda\verts{\delta_{t}^{\alpha,\lambda}(x)- \delta_{t}^{\alpha',\lambda}(x)}\lesssim & 
	\lambda\brakets{\parens{1+\Verts{x}}  + \lambda \parens{\Verts{\alpha_{t}(x)}+ \Verts{\alpha_{t}'(x)}} }\Verts{\alpha_{t}(x)-\alpha_{t}'(x)}\\
	&+ \lambda \Vert \nabla \alpha_{t}(x) - \nabla\alpha'_{t}(x)  \Vert.
\end{align}
Moreover, by \cref{lem:Linear_growth}, we know that $\alpha$ and $\alpha'$ have at most a linear growth. This allows to further write: 
\begin{align}
	\verts{\delta_{t}^{\alpha,\lambda}(x)- \delta_{t}^{\alpha',\lambda}(x)}\lesssim  \parens{\parens{1 + \Verts{x}}\Verts{\alpha_{t}(x)-\alpha_{t}'(x)}  + \Vert \nabla \alpha_{t}(x) - \nabla\alpha'_{t}(x)  \Vert }.
\end{align}
{\bf Bound on $\vert g_t^{\alpha}(x) - g_t^{\alpha}(x')\vert$.} It follows by direct computation. 	

\end{proof}

\begin{lem}\label{lem:kerenl_error_main_lemma}
	Under \cref{assump_cont:lipschitz_potential,assump_cont:moment} there exists a constant $\lambda_0$ independent on $t$ such that for $\lambda\leq \lambda_0$, there exists $M>0$ such that for all $0\leq t\leq 1$ it holds that:
	\begin{align}
			\vert h^{\lambda}_t(x) \vert \leq \lambda \sqrt{\lambda} \frac{Q_1(x)}{1 + Q_2(x)} \exp( C \lambda \Vert \nabla V_t(x)\Vert^2 )  
	\end{align}
where $Q_1$ and $Q_2$ are non-negative functions of $\Vert x \Vert$ of polynomial growth and  independent from $\lambda$ and $t$ and $C$ is a non-negative constants.
\end{lem}
\begin{proof}
Let us perform a change of variables $ u = \frac{x-y}{\sqrt{2\lambda}} $ in the integral appearing in the definition of $h^{\lambda}_t(x)$:
	\begin{align}
			h^{\lambda}_{t}(x):= \log\int  h(x,u,\lambda) \diff u,
	\end{align} 
	where we introduce the function $h(x,u,\lambda)$:
	\begin{align}
		h(x,u,\lambda) = \frac{1}{(\sqrt{2\pi})^d}\exp\parens{V_t(x) - V_t(y_{\lambda}) -\frac{1}{2}\Verts{ u +\sqrt{\frac{\lambda}{2}} \nabla V_t(y_{\lambda}) }^2},
	\end{align}
and where $y_{\lambda} := x - \sqrt{2\lambda}u $ for conciseness. By the fundamental theorem of calculus and using the dominated convergence theorem we can write:
\begin{align}\label{eq:fundamental_eq}
	h^{\lambda}_t(x) = \lambda \int_0^1\frac{\int \partial_{\lambda} h(x,u,s\lambda) \diff u }{\int  h(x,u,s\lambda) \diff u} \diff s.
\end{align}
\paragraph{Upper-bound on $\vert \int \partial_{\lambda} h(x,u,s\lambda) \diff u \vert$.} We now find an expression  for $\int  \partial_{\lambda} h(x,u,s\lambda)\diff u$ which exhibits an explicit factor $\sqrt{\lambda}$. By simple computation, we have that $\partial_{\lambda} h(x,u,s\lambda)$ is  given by:
\begin{align}
	\partial_{\lambda} h(x,u,\lambda) =&\frac{1}{2}\parens{\frac{1}{\sqrt{2\lambda}} \nabla V_t(y_{\lambda})^{\top}u + u^{\top}H V_t(y_{\lambda})u - \frac{1}{2}\Vert \nabla V_t(y_{\lambda}) \Vert^2 } h(x,u,\lambda),
	\end{align}
	where $H V_t$ is the Hessian of $V$. Integrating w.r.t $u$, we get:
\begin{align}
	\int \partial_{\lambda} h(x,u,\lambda) \diff u = \frac{1}{2}\frac{1}{\sqrt{2\lambda}} \phi_{\lambda}(1) + \frac{1}{2} \int \parens{ u^{\top}H V_t(y_{\lambda})u  - \frac{1}{2}\Vert \nabla V_t(y_{\lambda}) \Vert^2 }h(x,u,\lambda)\diff u,
\end{align}
where the function $s\mapsto \phi_{\lambda}(s)$ is defined for $0\leq s\leq 1$ as
\begin{align}
	\phi_{\lambda}(s) = \int \nabla V_t(y_{s\lambda})^{\top}u h(x,u,s\lambda)\diff u,
\end{align}
with  $\phi_{\lambda}(0)= 0$. By the fundamental theorem of calculus and using the dominated convergence theorem, we have that:
\begin{align}
	\phi_{\lambda}(1)  =& \int_0^1 \partial_s \phi_{\lambda}(s)\diff s\\
	=& -\sqrt{2\lambda} \int_0^1 \frac{1}{2\sqrt{s}} \int u^{\top} H V_t(y_{s\lambda})u h(x,u,s\lambda) \diff u  \diff s\\
	 &+ \frac{\lambda}{2} \int_{0}^1 \frac{1}{\sqrt{2s\lambda}}\int (\nabla V_t(y_{s\lambda})^{\top}u)^2h(x,u,s\lambda)\diff u \diff s\\
	 &+  \frac{\lambda}{2} \int_{0}^1 \int \nabla V_t(y_{s\lambda})^{\top}u  \parens{u^{\top} H V_t(y_{s\lambda})u -  \frac{1}{2}\Vert \nabla V_t(y_{s\lambda})\Vert^2   }h(x,u,s\lambda) \diff u \diff s.
\end{align}
We can further use that $h(x,u,s\lambda) = h(x,u,\lambda) + \lambda \int_s^1 \partial_{\lambda}h(x,y,s'\lambda)\diff s'$ and express $\int \partial_{\lambda} h(x,u,\lambda) \diff u$ as:
\begin{align}
	\int \partial_{\lambda} h(x,u,\lambda) \diff u = A_{\lambda}(x) + B_{\lambda}(x) + C_{\lambda}(x) + D_{\lambda}(x),
	\end{align} 
where $A_{\lambda}$, $B_{\lambda}$, $C_{\lambda}$ and $D_{\lambda}$ are given by:
\begin{align}
	A_{\lambda}(x) &=   \frac{1}{2}\int_0^1\int \frac{1}{2\sqrt{s}} u^{\top}\parens{ H V_t(y_{\lambda}) -H V_t(y_{s\lambda}) }u h(x,u,\lambda) \diff u \diff s,\\
	B_{\lambda}(x) &=   \frac{1}{4}\int_0^1\frac{1}{2\sqrt{s}} \int \parens{(\nabla V_t(y_{s\lambda})^{\top}u)^2-\Vert \nabla V_t(y_{\lambda}) \Vert^2}h(x,u,\lambda)\diff u\diff s,\\
	C_{\lambda}(x) &= \frac{1}{4\sqrt{2}}\sqrt{\lambda} \int_0^1 \int \nabla  V_t(y_{s\lambda})^{\top}u  \parens{ u^{\top} H V_t(y_{s\lambda})u - \frac{1}{2}\Vert \nabla V_t(y_{s\lambda})\Vert^2  }h(x,u,s\lambda) \diff u \diff s,\\
	D_{\lambda}(x) &= \frac{\lambda}{4}\int_0^1 \int \int_s^1\partial_{\lambda}h(x,u,s'\lambda)\parens{(\nabla V_t(y_{s\lambda})^{\top}u)^2 - 2 HV_t(y_{s\lambda})  }\diff u \diff s'\diff s. 
\end{align}
We can further decompose $B_{\lambda}(x)$ in two terms  $B^{(1)}_{\lambda}(x)$ and $B^{(2)}_{\lambda}(x)$  by adding and subtracting $\nabla V_t(x)$ inside each square:
\begin{align}
	B_{\lambda}(x) := B^{(1)}_{\lambda}(x) + B^{(2)}_{\lambda}(x),
	\end{align}
with:
\begin{align}
	B^{(1)}_{\lambda}(x) :=& \frac{1}{4}\int_0^1\frac{1}{2\sqrt{s}} \int  u^{\top}\parens{(\nabla V_t(y_{s\lambda})-\nabla V_t(x)}\parens{\nabla V_t(y_{s\lambda})-\nabla V_t(x)}^{\top}u \diff u \diff s \\
	&- \frac{1}{4}\int_0^1\frac{1}{2\sqrt{s}} \int \parens{\nabla V_t(y_{\lambda})-\nabla V_t(x)}^{\top}\parens{\nabla V_t(y_{\lambda})+\nabla V_t(x)} \diff u \diff s,\\
	B^{(2)}_{\lambda}(x) :=& \frac{1}{4} \int  \parens{ \nabla V_t(x)^{\top}u }^2 (h(x,u,\lambda)-h(x,u,0) )\diff u  \\
	=&   \frac{\lambda}{4} \int_0^1\frac{1}{2\sqrt{s}}\int  \parens{ \nabla V_t(x)^{\top}u }^2 \partial_{\lambda}h(x,u,s\lambda) \diff u\diff s.
\end{align}
For $B^{(2)}_{\lambda}(x)$ we used that $ \int   \parens{ \nabla V_t(x)^{\top}u }^2  h(x,u,0)  \diff u = \Vert \nabla V_t(x) \Vert^2 $ and then applied the fundamental theorem of calculus for the difference  $(h(x,u,\lambda)-h(x,u,0)$ to exhibit a factor $\lambda$. 
Using that $\nabla V_t(x)$ and $HV_t(x)$ are Lipschitz in $x$ and the growth assumptions, it can be shown that each of the terms $A_{\lambda}$, $B^{(1)}_{\lambda}$, $B^{(2)}_{\lambda}$, $C_{\lambda}$ and $D_{\lambda}$:
\begin{align}
	\vert B^{(2)}_{\lambda}(x)\vert +\vert C_{\lambda}(x)\vert  \leq  \sqrt{\lambda} &\int E^{2}_{\lambda}(x,u) \diff u ,\qquad
	\vert A_{\lambda}(x) \vert + \vert B^{(1)}_{\lambda}(x) \vert \leq \sqrt{\lambda}\int  E^{1}_{\lambda}(x,u)\diff u,\\
	&\vert D_{\lambda}(x)\vert  \leq  \sqrt{\lambda} \int E^{3}_{\lambda}(x,u)\diff u,
\end{align}
with $E^{1}_{\lambda}(x,u)$, $E^{2}_{\lambda}(x,u)$ and $E^{3}_{\lambda}(x,u)$ given by:
\begin{align}
	E^{1}_{\lambda}(x,u) &=   Q_1(x,u) \parens{\int_0^1\int_s^1 h(x,u, s'\lambda)\diff s'\diff s}, \\
	E^2_{\lambda}(x,u) &=   Q_2(x,u) \parens{\int_0^1 h(x,u, s\lambda)\diff s}, \\
	E^3_{\lambda}(x,u) &=   Q_3(x,u)  h(x,u, s\lambda).
\end{align}
Here $ Q_1(x,u)$, $ Q_2(x,u)$ and $ Q_3(x,u)$ are polynomials in $\Vert x \Vert$ and $\Vert u \Vert$ with non-negative coefficients independent from $\lambda$ and $t$. Using again the Lipschitz regularity of $\nabla V_t(x)$, we can further find an upper-bound on $h(x,u,\lambda)$ of the form:
\begin{align}
	h(x,u,\lambda) \leq C  \exp(M\lambda  \Vert \nabla V_t(x) \Vert^2-  \alpha(\lambda)\Vert u\Vert^2 +  \beta(\lambda)\Vert u \Vert\Vert \nabla V_t(x) \Vert    ),\\
	h(x,u,\lambda) \geq C'  \exp(M'\lambda  \Vert \nabla V_t(x) \Vert^2-  \alpha'(\lambda)\Vert u\Vert^2 -  \beta(\lambda)\Vert u \Vert\Vert \nabla V_t(x) \Vert    ). 
\end{align}
Here,  $C,C'$ and $M,M'$ are independent of $\lambda$ and $t$. The functions $\alpha(\lambda)$ and $\alpha'(\lambda)$ are positive for $\lambda$ small enough  and converge to $\frac{1}{2}$ with rate $\lambda$ when $\lambda\rightarrow 0$.  Finally $\beta(\lambda)$ converges to $0$ with rate $\sqrt{\lambda}$. Using those bounds for $E^1_{\lambda}(x,u)$, $E^2_{\lambda}(x,u)$ and   $E^3_{\lambda}(x,u)$ and $h(x,y,\lambda)$ and integrating over $u$, it follows that:  
\begin{align}\label{eq:main_inequalities_E}
	\begin{cases}
	\verts{\int  \partial_{\lambda} h(x,u,\lambda) \diff u }\leq \sqrt{\lambda} Q_1(\Vert x\Vert) \exp(L\lambda \Vert \nabla V_t(x)\Vert^2 ), \\
	\int h(x,u,\lambda) \diff u  \geq  (1 + Q_2(\Vert x\Vert)) \exp(L'\lambda \Vert \nabla V_t(x)\Vert^2 ),
		\end{cases}
\end{align}
where $Q_1$ and $Q_2$ are non-negative functions of $\Vert x \Vert$ of polynomial growth and  independent from $\lambda$ and $t$ and $L,L'$ are some positive constants.
\paragraph{Upper-bound on $\vert h^{\lambda}_{t}(x)\vert $.}
Recalling \cref{eq:fundamental_eq}, and using \cref{eq:main_inequalities_E} we directly have:
\begin{align}
	\vert h^{\lambda}_t(x) \vert \leq \lambda \sqrt{\lambda} \frac{Q_1(\Vert x\Vert)}{1 + Q_2(\Vert x\Vert )} \exp(\vert L-L'\vert\lambda \Vert \nabla V_t(x)\Vert^2 ).
\end{align}
\end{proof}

\section{Variant of the algorithm for unbiased estimates}\label{sec:variant_algo}

\cref{alg:SMC-detailed} and \cref{alg:LearnFlow} give details of the full practical implementation described in Section \ref{sec:appli} of the main paper. 

\begin{algorithm}
\caption{Annealed Flow Transport: Detailed Version}\label{alg:SMC-detailed}
	\begin{algorithmic}[1]
		\STATE \textbf{Input:} Number of training, test and validation particles $N_{\textup{train}}$, $N_{\textup{test}}$, $N_{\textup{val}}$, unnormalized annealed targets $\{\gamma_k\}_{k=0}^K$ such that $\gamma_0=\pi_0$ and $\gamma_K=\gamma$, resampling thresholds $A_a\in\left[1/{N_a},1\right)$ for $a \in \{\textup{train},\textup{test},\textup{val}\}$, number of training iterations $J$.
		\STATE \textbf{Ouput:} Approximations $\pi^{N_{\textup{test}}}_K$ and $Z_K^{N_{\textup{test}},\textup{test}}$ of $\pi$ and $Z$.
		\FOR{$a \in \{\textup{train},\textup{test},\textup{val}\}$}
			\STATE Sample  $X^{i,a}_0 \sim \pi_0$ and set $W_0^{i,a} \leftarrow \frac{1}{N_a}$ and  $Z_0^{N,a}\leftarrow 1$.
		\ENDFOR
		\FOR{$k=1,\dots, K$} 
			\STATE Learn the flow
				$T_k  \leftarrow \verb+LearnFlow+\parens{J,\braces{X^{i,\textup{train}}_{k-1} ,W^{i, \textup{train}}_{k-1}}_{i=1}^{N_{\textup{train}}},  \braces{X^{i,\textup{val}}_{k-1}, W^{i, \textup{val}}_{k-1}}_{i=1}^{N_{\textup{val}}}}$
			\FOR{$a \in \{\textup{train},\textup{test},\textup{val}\}$}
			\STATE Transport particles: $\widetilde{X}^{i,a}_{k} \leftarrow T_k(X^{i,a}_{k-1})$.
			\STATE Estimate normalizing constant $Z_k$: \\$Z_k^{N_a,a} \leftarrow Z_{k-1}^{N_a,a} \parens{\sum_{i=1}^{N_a} W_{k-1}^{i,a}G_{k,T_k}(X_{k-1}^{i,a})}$.
			\STATE Compute IS weights: \\$w^{i,a}_k \leftarrow W_{k-1}^{i,a} \G_{k,T_k}(X_{k-1}^{i,a})$ // unnormalized\\
				$W^{i,a}_k \leftarrow  \frac{w_{k}^{i,a}}{\sum_{j=1}^{N,a} w_{k}^{j,a}}$ // normalized
		    \STATE Compute effective sample size $\textup{ESS}_k^{N_a}$\\
		    $\textup{ESS}_k^{N_a} \leftarrow \left(\sum_{i=1}^{N_a} \parens{W_k^{i,a}}^2\right)^{-1}$.
		    \IF{$\textup{ESS}^{N_a}_{k}/N_a \leq A_a$}
			    \STATE Resample $N_a$ particles from split $a$ denoted abusively also $\widetilde{X}^{i,a}_k$ according to the weights $W_k^{i,a}$, 
			    \STATE Set $W_k^{i,a} \leftarrow \frac{1}{N_a}$. 
		    \ENDIF
			    \STATE  Sample $X_{k}^{i,a} \sim K_{k}(\widetilde{X}_{k}^{i,a},\cdot)$. // MCMC
			\ENDFOR
		\ENDFOR
	\end{algorithmic}
\end{algorithm}

\begin{algorithm}
\caption{LearnFlow}\label{alg:LearnFlow}
	\begin{algorithmic}[1]
		\STATE \textbf{Input:} Number of training iterations $J$, training and validation particles and weights $\braces{X^{i,\textup{train}}_{k-1} , W^{i, \textup{train}}_{k-1} }_{i=1}^{N_{\textup{train}}}$ and  $\braces{X^{i,\textup{val}}_{k-1} , W^{i, \textup{val}}_{k-1} }_{i=1}^{N_{\textup{val}}}$.
		\STATE \textbf{Ouput:} Estimated flow $T_k$
		\STATE Initialize flow to identity $T_k= Id$.
		\STATE Initialize list of flows $\mathcal{T}_{opt} \leftarrow \{T_k\}$.
		\STATE Initialize list of validation losses\\
		    $\mathcal{E}\leftarrow  \braces{\sum_{i=1}^{N_{\textup{val}}} W_{k-1}^{i,\textup{val}} h_{T_k}\parens{X_{k-1}^{i,\textup{val}}}}$
		\FOR{$j=1,...,J$}
		    \STATE Compute training loss using \cref{eq:empirical_loss}  \\
		        $ \mathcal{L}_k^{N_{\textup{train}}}(T_k) \leftarrow \sum_{i=1}^{N_{\textup{train}}} W_{k-1}^{i,\textup{train}} h_{T_k}\parens{X_{k-1}^{i,\textup{train}}}$.
		    \STATE Update $T_k$ using SGD to minimize $\mathcal{L}_k^{N_{\textup{train}}}(T_k)$.
		    \STATE Update list of flows $\mathcal{T}_{opt} \leftarrow \mathcal{T}_{opt}\cup \{T_k\}$
		    \STATE Initialize list of validation losses $\mathcal{E}$\\
		    $\mathcal{E}\leftarrow  \mathcal{E}\cup \braces{\sum_{i=1}^{N_{\textup{val}}} W_{k-1}^{i,\textup{val}} h_{T_k}\parens{X_{k-1}^{i,\textup{val}}}}$
		\ENDFOR
	    \STATE Return flow with smallest validation error from the list of flows $\mathcal{T}_{opt}$.
	\end{algorithmic}
\end{algorithm}

\section{Additional experimental details and discussion}\label{section:add_experiment}

The algorithm is implemented in JAX \cite{Bradbury:2018}, Haiku \cite{Hennigan:2020} and Optax \cite{optax2020github}. In all cases we use the Adam optimizer \cite{Kingma:2014} for learning the flow with the Optax defaults for all parameters except for the learning rates which are chosen for each example and are given below. For the MCMC kernels we used the TensorFlow Probability \cite{Dillon:2017} JAX substrate.

Each experimental configuration was performed using an NVIDIA v100 GPU and 4 CPUs. All experimental configurations took under 30 minutes and most were much shorter. 

In all cases $N_{\textup{test}}=2000$ and this was the batch size for SMC and for VI ensuring fair comparison of trained estimators. Note VI is not sensitive to the training batch size because samples are replenished at each training step.

\begin{figure}[ht]
\vskip 0.2in
\begin{center}
\centerline{\includegraphics[width=\columnwidth]{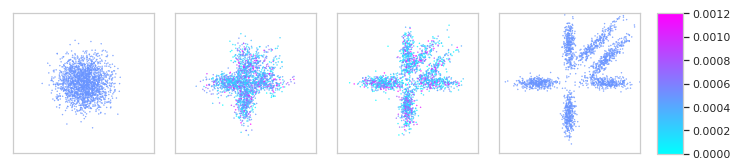}}
\caption{Larger version of Figure \ref{fig:two-dim-series}. Weighted samples for a 2-D target density with AFT. The colours show the normalized weights which are clipped at the 95th percentile for clarity.}
\label{fig:two-dim-series-large}
\end{center}
\vskip -0.2in
\end{figure}

\begin{figure}
\begin{center}
\centerline{\includegraphics[height=0.8\textheight]{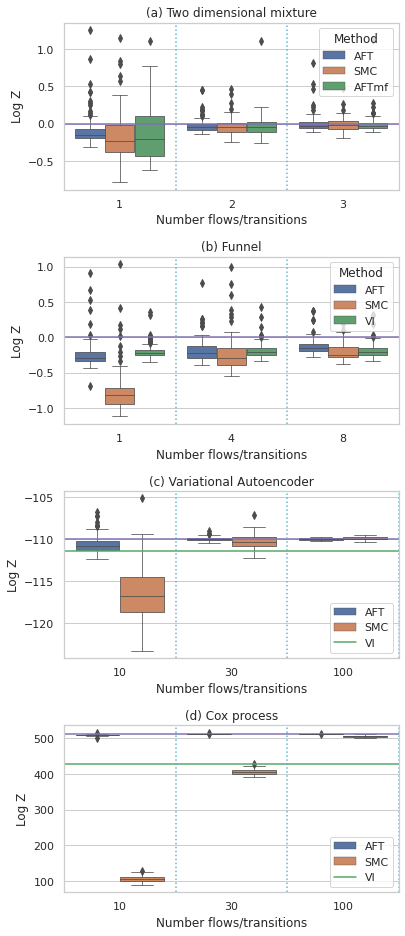}}
\caption{Larger version of results from the four different examples. Cyan lines denote gold standard values of the log normalizing constant. In (c) and (d) green horizontal lines denote the median value for an importance sampling estimate based on variational inference. Note that in (d) the small AFT error bars can make it difficult to see- it can be found next to the gold standard value in each case.}
\label{fig:combined-results-large}
\end{center}
\vskip -0.2in
\end{figure}

\subsection{Two dimensional example}

For each transition, we used 10 iterations of Hamiltonian Monte Carlo with 10 leapfrog steps per iteration.
The rational quadratic splines closely match the implementation described by \cite{Durkan:2019}. We used ten bins and padded with the identity outside of the range $[-4, 4]$. The bins height and widths were parameterized in terms of unconstrained real values. We then took the soft-max of these values and scaled and shifted them so that they had a minimal value of $10^{-4}$ and fitted with the range $[-4, 4]$. The derivatives were again parameterized using unconstrained real values and then made positive using the transformation $\nu + \log (1 + \exp(x))$ where $\nu=10^{-4}$.

The inverse autoregressive flow used the same autoregressive network detailed in the Funnel example below but with a changed input dimensionality and outputting the spline parameters instead of the parameters of an affine transformation.

The Adam learning rate was $10^{-3}$ the training period $J$ was $1000$. AFT flow estimation particle numbers were $N_{\textup{train}}=N_{\textup{val}}=2000$. 

With the total time number of time steps scaled to be in the interval $[0, 1]$ with $0$ being the initial distribution and $1$ the final distribution, the HMC step sizes were interpolated between the times $[0., 0.25, 0.5, 1.]$ using the step sizes $[0.5, 0.5. 0.5, 0.3]$. 

\subsection{Additional details of Funnel example}

We used 1000 steps of slice sampling per temperature with a maximum of 5 step size doublings as defined in the TensorFlow probability interface.   
For the affine inverse regressive flow we used an autoregressive neural network \cite{Germain:15} where the correct autoregressive structure is achieved by masking network weights. We used a Leaky Relu non-linearity. The unmasked network would have $30$ hidden units per input dimension and we used 3 hidden layers. To achieve identity initialisation of the flow we initialised the final weights and biases of the network to zero. The weights of the rest of network where initialised using a truncated normal distribution scaled by the fan-in and biases were initialised to zero. The final output of the MLP was taken and mapped directly into the mean and also the scale of the network after adding one to give the identity transformation overall. 

The Adam learning rate was $10^{-3}$ the training period $J$ was $4000$. AFT flow estimation particle numbers were $N_{\textup{train}}=N_{\textup{val}}=6000$. 

With the total time number of time steps scaled to be in the interval $[0, 1]$ with $0$ being the initial distribution and $1$ the final distribution, the slice sampling step sizes were interpolated between the times $[0., 0.25, 0.5, 0.75, 1.]$ using the step sizes $[0.9, 0.7. 0.6, 0.5, 0.4]$. 

\subsection{Additional details of Variational Autoencoder example}

\subsubsection{Autoencoder training and architecture details}

The VAE encoder architecture was as follows. The encoder, which parameterized the amortized variational distribution started with two convolutional layers each followed by a Rectified Linear non-linearity. The first convolution had kernel shape $4\times 4$, stride $2 \times 2$, 16 output channels and `valid' padding. The second convolution had  $4\times 4$, stride $2 \times 2$, 32 output channels and `valid' padding. In all cases layers where initialized using the Haiku defaults. The output of the second convolution was fed into linear layers which parameterized the variational mean, and the value of the diagonal variational standard deviation. After each of these linear layers we used layer normalization \cite{Ba2016} with the standard additional scale and translation parameters to ensure no loss in expressivity. Positivity of the standard deviation parameter was ensured by transforming the real valued vector through a softplus non-linearity. 

The VAE decoder architecture was as follows. We affinely projected the $30$ dimensional latents into a $7 \times 7 \times 32 = 1568$ dimensional space, which we then followed again by layer normalization. We then reshaped this vector to $7 \times 7 \times 32$ ready for feeding into the deconvolutional layers. There were three such deconvolutional layers separated by two rectified linear non-linearity layers. The first deconvolutional layer had kernel shape $3 \times 3$, a stride of $2 \times 2$ and 64 channels. The second deconvolutional layer had kernel shape $3 \times 3$, a stride of $2 \times 2$ and 32 channels.  The third and final deconvolutional layer had kernel shape $3 \times 3$, a stride of $1 \times 1$ and one output channel to match the target image shape.

Next we describe the training algorithm for learning the VAE, which is distinct from the experiments we ran evaluating the latent space. We used the ADAM optimizer with a small learning rate of $5 \times 10^{-5}$. We used gradient clipping with value of $10^{5}$ and trained for $5 \times 10^{5}$ iterations. We used the reparameterization trick \cite{kingma+welling:2014, rezende+al:2014:icml} for estimating the expectation of the log likelihood and the analytic expression for the KL divergence between the normal approximating distribution and the standard normal prior.  

\subsubsection{Additional details and discussion of VAE experimental results}

Figure \ref{fig:vae-samples} shows reconstructions and samples for the trained autoencoder. Note that we adopt the common practice of plotting the pixel probabilities for the reconstructions and the samples. In the latter case, this has the effect of making the samples look smoother than if we sampled pixel values. 

As discussed in the main text, the experiments on the test set were divided into two types. One group of experiments (Group A) was run over the whole test set. These experiments were used to verify the quality of the autoencoder and identify interesting/challenging examples for the more detailed study. The detailed study (Group B) mirrored the pattern of the other examples. 

We now give more detail of the Group A experiments. For each image in the binarized MNIST test set we ran SMC once with 1000 temperatures. The variance in the log likelihood estimate from these runs was small relative to our the variational inference experiments. For each image in the binarized MNIST test set we did four repeats of the unamortized normal variational approximation with diagonal covariance and associated importance estimates. Figure \ref{fig:vae-scatter} shows log likelihood estimates for SMC and the VI based importance sampler. We have averaged the VI estimates over the four repeats. The plot shows strong correlation- the variation in log likelihood from different digits is the chief source of variance as reported by \cite{Wu:2017}. The four VI repeats were then used to estimate the mean absolute error in the variational approximation relative to SMC as shown in Figure \ref{fig:vae-error-histogram}. These errors were used to identify challenging inference digits. In particular we chose digits on the $99.8$-th percentile of error, which implies about 1 in 500 digits is more challenging than the ones we chose. The three digits selected- one used for the main paper and two for repeats in this Appendix are shown in Figure \ref{fig:second-challenging-digit}. The experiment repeats are shown in Figures \ref{fig:second-challenging-digit} and \ref{fig:third-challenging-digit}. 

As discussed in the main text we found that the variational inference performed well considering the simple form of the approximate posterior. This is likely a consequence of the fact that the training objective favours posteriors that are well matched by the variational approximation. Whilst there can be bias from this effect relative to using an exact marginal likelihood \cite{Hoffman:2017} the ability of a flexible generative network to adapt to this constraint can make the effect less extreme than in simpler statistical models \cite{turner+sahani:2011a}.

\begin{figure}
\begin{center}
\centerline{\includegraphics[width=0.8\textwidth]{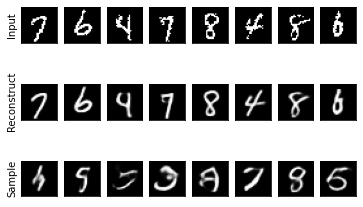}}
\caption{Variational Autoencoder samples and reconstructions. Top line: Images from the binarized MNIST dataset. Middle line: Reconstructions of the same images. Bottom line. Free samples from the model- these have no particular relationship to the images above them.}
\label{fig:vae-samples}
\end{center}
\vskip -0.2in
\end{figure}

\begin{figure}
\begin{center}
\centerline{\includegraphics[width=0.8\textwidth]{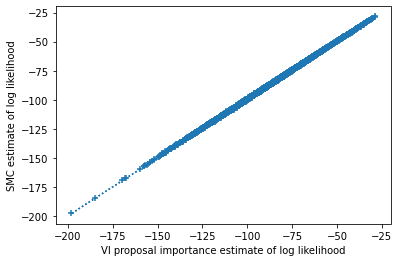}}
\caption{Estimated log likelihood from long run and SMC and mean estimated log likelihood from VI proposal for the binarized MNIST test set. The VI means are estimated based on four repeats of the VI training and sampling. The dotted line shows equality for reference.}
\label{fig:vae-scatter}
\end{center}
\vskip -0.2in
\end{figure}

\begin{figure}
\begin{center}
\centerline{\includegraphics[width=0.8\textwidth]{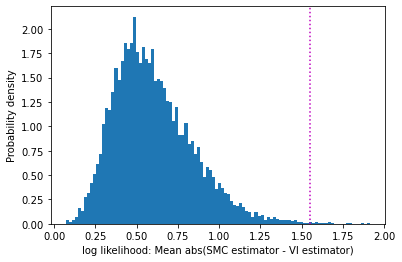}}
\caption{Histogram showing the mean absolute error of the variational proposal importance sampler for the binarized MNIST test set relative to a long run of SMC. The means are estimated based on four repeats of the VI training and sampling. The cyan line shows the discrepancy of the challenging digit chosen for detailed investigation in the main text}
\label{fig:vae-error-histogram}
\end{center}
\vskip -0.2in
\end{figure}

\begin{figure}
\begin{center}
\centerline{\includegraphics[width=0.8\textwidth]{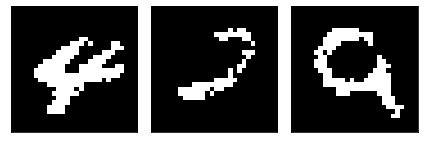}}
\caption{Binarized MNIST digits used in detailed experiments. These digits were selected as challenging for the variational approximation. Left: The digit used in the main paper. Centre and Right: Respectively the digits used for the repeat experiments in Figure \ref{fig:second-challenging-digit} and \ref{fig:third-challenging-digit}.}
\label{fig:vae-scatter}
\end{center}
\vskip -0.2in
\end{figure}

\begin{figure}
\begin{center}
\centerline{\includegraphics[width=0.8\textwidth]{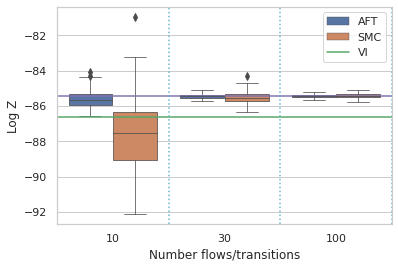}}
\caption{Results for a second challenging digit VAE likelihood, selected from the tail of plot \ref{fig:vae-error-histogram} similar to the result in the main text.}
\label{fig:second-challenging-digit}
\end{center}
\vskip -0.2in
\end{figure}

\begin{figure}
\begin{center}
\centerline{\includegraphics[width=0.8\textwidth]{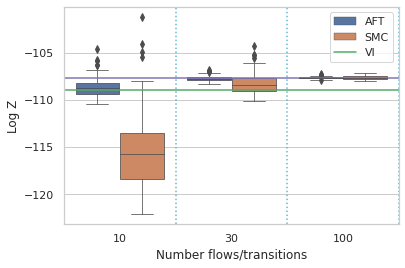}}
\caption{Results for a third challenging digit VAE likelihood, selected from the tail of plot \ref{fig:vae-error-histogram} similar to the result in the main text.}
\label{fig:third-challenging-digit}
\end{center}
\vskip -0.2in
\end{figure}

\subsection{Additional details of Cox process example}

\begin{figure}
\begin{center}
\centerline{\includegraphics[width=0.8\textwidth]{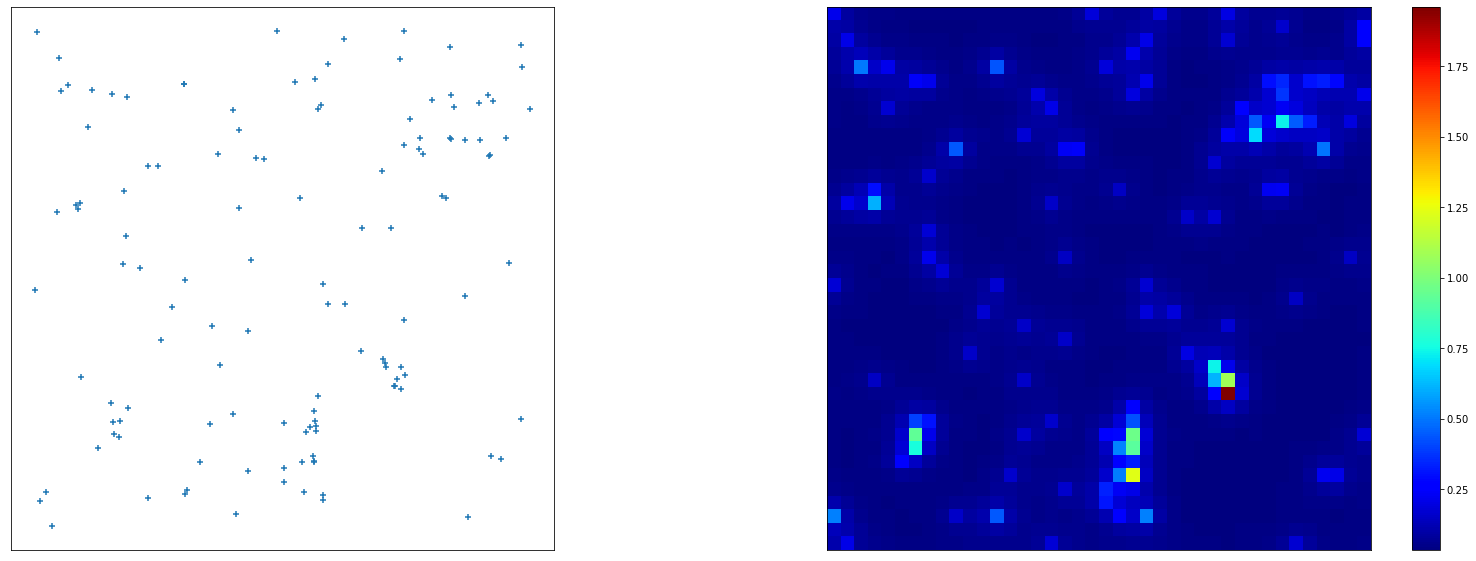}}
\caption{Left: data points for the spatial point process example. Right: Inferred posterior rate for the Cox process with a $40 \times 40$ discretization using AFT. The plot area is defined as having unit area. The samples were taken from a single run of AFT with $30$ temperatures and other parameters as discussed in the text.}
\label{fig:pines-posterior}
\end{center}
\vskip -0.2in
\end{figure}

We experimented with transforming the problem to a whitened representation $\tilde{x}$ where $x = L \tilde{x} + \mu$ and $L$ is the Cholesky decomposition of $K$ \cite{Beskos:2011,Neal2011}. This can also be viewed as choosing a non-identity pre-conditioner or mass matrix. We found that this rendered the problem easier to the extent of no longer being a challenging benchmark, and that in this circumstance there was little that a NF could contribute. To maintain the difficulty of the baseline we therefore focused our comparison on the unwhitened space of the latent function. 

The covariance $K$ is taken to be $K(u,v) = \sigma^2 \exp\left(-\frac{||u-v||_2}{M\beta}\right)$, where $\sigma^2=1.91$ and the mean vector has a constant value of $\log(126) - \sigma^2$, and $a=1/M^2$, matching \cite{Moller:1998}. 

The Adam learning rate was $10^{-2}$ the training period $J$ was $500$. AFT flow estimation particle numbers were $N_{\textup{train}}=N_{\textup{val}}=2000$. 

For each transition, we used ten iterations of Hamiltonian Monte Carlo with 10 leapfrog steps per iteration. With the total time number of time steps scaled to be in the interval $[0, 1]$ with $0$ being the initial distribution and $1$ the final distribution, the slice sampling step sizes were interpolated between the times $[0., 0.25, 0.5, 1.]$ using the step sizes $[0.3, 0.3, 0.2, 0.2]$. The gold standard value for the normalizing constant shown in Figure \ref{fig:combined-results} was found using 1000 repeats of SMC with 1000 temperatures and using the Cholesky whitening. 

We found that to obtain best performance for AFT in this example, it was important to have sufficient HMC updates per transition. 

An example of the samples produced using AFT can be seen in Figure \ref{fig:pines-posterior}.

\end{document}

%% file: applications.tex
In this section we detail the practical implementation of AFT and empirically investigate performance against relevant baselines.

As discussed in Section \ref{sec:extensions}, we use three sets of particles-`train, test and validation' which improves robustness, avoids overfitting the flow to the particles and gives unbiased estimates of $Z$ when using the test set. We initialize our flows to the identity for the optimization at each time step. Algorithm \ref{alg:SMC-detailed}, in the supplement gives a summary. %

We concentrate our empirical value evaluation on the learnt flow, which is equivalent to using the test set particles. The learnt flow is of interest in deploying an efficient sampler on large scale distributed parallel compute resources. It is also of interest for inclusion as a subroutine in a larger system. Since modern hardware enables us to do large computations in parallel, the computation is dominated by algorithmic steps that are necessarily done serially, particularly repeat applications of the Markov kernel \cite{lee2010utility}.

As our primary, strong, baseline for AFT, we use a standard instance of SMC samplers \cite{Del-Moral:2006,ZhouJohansen2016} which corresponds to AIS with adaptive resampling and is also known as population annealing in physics \cite{hukushima2003population,barash2017gpu}. As observed many times in the literature and in our experiments, SMC estimates are of lower variance than AIS estimates. This SMC baseline is closely related to AFT since it corresponds to using AFT with an identity transformation $T_k(x)=x$ instead of a learnt flow.%

We largely use the number of transitions $K$ as a proxy for compute time. This is valid when the cost of evaluating the flow is modest relative to that of the other algorithmic steps, as it is for the trained flows in all non-trivial cases we consider. We only consider flows of no more than a few layers per transition, but deeper flows could start to form an appreciable part of the serial computation. In some cases, we use variational inference (VI) as a measure of behaviour without MCMC. In this case, evaluation time is not comparable and faster. Since we concentrate on trained flows, we do not evaluate training time in the benchmarks considered, though fast training of AFT could be of interest in further work. Both SMC and AFT use the same Markov kernels $K_k$, using HMC except where otherwise stated. We tune the step size to have a reasonable acceptance probability based on preliminary runs of SMC using a modest $K$.  Then for larger $K$ experiments, we linearly interpolate the step sizes chosen on the preliminary runs. We always use a linearly spaced geometric schedule and the initial distribution is always a multivariate standard normal. We repeat experiments 100 times. Further experimental details may be found in Appendix \ref{section:add_experiment}. We plan to make the code available within \url{https://github.com/deepmind}.

\subsection{Illustrative example}

We start with an easily visualized two dimensional target density as shown in Figure \ref{fig:two-dim-series}. All sensible methods should work in such a low dimensional case but it can still be informative. We investigate two families of flows based on rational quadratic splines \cite{Durkan:2019}. The first (termed AFTmf for mean field) operates on the two dimensions separately. The second family (denoted AFT in Figure \ref{fig:two-dim-series}) adds dependence to the splines using inverse autoregressive flows \cite{Kingma:2016}. Figure \ref{fig:two-dim-series} shows weighted samples from AFT as we anneal from a standard normal distribution. Figure \ref{fig:combined-results} (a) shows that AFT reduces the variance of the normalizing constant estimator relative to SMC. Conversely, we see that AFTmf actually \emph{increases} the variance relative to SMC for small numbers of transitions. Since the factorized approximation cannot model the dependence of variables the optimum of the KL underestimates the variance of the target. Later, in Sections \ref{section:vae} and \ref{section:cox}, we discuss examples where even a simple NF leads to an improvement for a modest number of transitions.
\begin{figure}[ht]
\begin{center}
\centerline{\includegraphics[width=\linewidth]{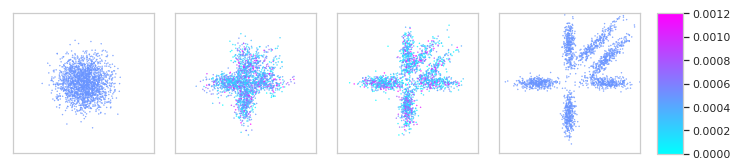}}
\caption{Weighted samples for a 2-D target density with AFT. The colours show the normalized weights which are clipped at the 95th percentile for clarity. The final samples are visually indistinguishable from the target.}
\label{fig:two-dim-series}
\end{center}
\vskip -0.2in
\end{figure}
\begin{figure}[ht!]
\begin{center}
\centerline{\includegraphics[width=\linewidth]{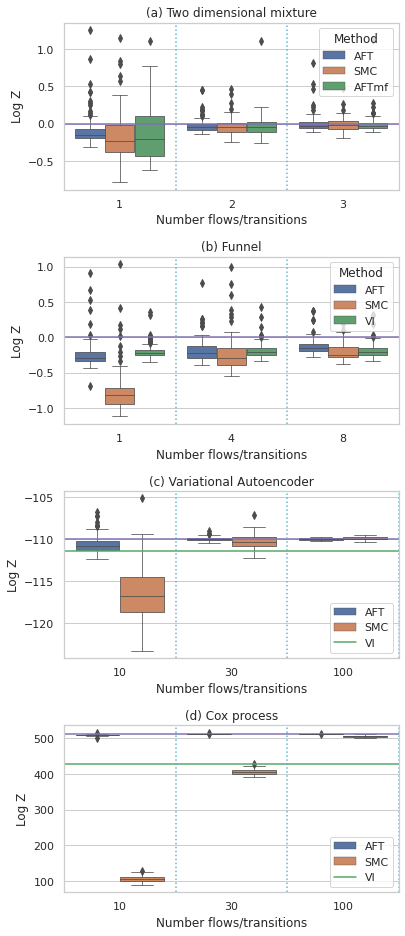}}
\caption{Results from the four different examples. Cyan lines denote gold standard values of the log normalizing constant. In (c) and (d) green horizontal lines denote the median value for an importance sampling estimate based on variational inference. Note that in (d) the small AFT error bars can make it difficult to see - it can be found next to the gold standard value in each case.}
\label{fig:combined-results}
\end{center}
\vskip -0.2in
\end{figure}
\subsection{Funnel distribution}\label{section:funnel}
We next evaluate the performance of the method on Neal's ten-dimensional `funnel' distribution \cite{Neal:2003}:
\begin{align}
x_{0} &\sim  \mathcal{N}(0, \sigma_{f}^2),\quad
x_{1:9} | x_{0} &\sim \mathcal{N}(\mathbf{0}, \exp(x_0) \mathbf{I}).
\end{align}
Here, $\sigma_{f}^2=9$. Many MCMC methods find this example challenging because there is a variety of length scales depending on the value of $x_0$ and because marginally $x_{1:9}$ has heavy tails. We use here slice sampling instead of HMC for the Markov kernels as recommended in \cite{Neal:2003}. For each flow we use an affine inverse autoregressive flow \cite{Kingma:2016}. In this example, we also compare against VI \cite{rezende+al:2014:icml} which uses the same number of flows. We then apply a simple importance correction to the VI samples to give an unbiased estimate of the normalizing constant.
Figure \ref{fig:combined-results} (b) shows the results. We see that for small number of flows/transitions VI performs best, followed by AFT. However, VI shows little further improvement with additional flows and in this regime AFT, SMC and VI perform similarly. 

\subsection{Variational Autoencoder latent space}\label{section:vae}

For our next example, we trained a variational autoencoder \cite{kingma+welling:2014, rezende+al:2014:icml} with convolution on the binarized MNIST dataset \cite{Salakhutdinov:2008} and a normal encoder distribution with diagonal covariance. Using the fixed, trained, generative decoder network we investigated the quality of normalizing constant estimation which in this case corresponds to the likelihood of a data point with the distribution over the 30 latent variables marginalized out \cite{Wu:2017}.

Using long run SMC on the 10000 point test set we estimate that the hold out log-likelihood per data point for the network is -86.3. For each data point we also found the optimal variational normal approximation with diagonal covariance rather than using the amortized variational approximation. Using this optimal normal approximation we investigated its variance when used as an importance proposal for the likelihood. We estimate the mean absolute error for the estimator across the test set was 0.6 nats per data point which indicates that the VI is often performing well. There was a tail of digits where VI performed relatively worse. Since these `difficult' digits constituted a more challenging inference problem, we used one of these, with a VI/SMC error of 1.5 nats, to comparatively benchmark AFT in the detailed manner used in our other examples. 

For the AFT flow we used an affine transformation with diagonal linear transformation matrix. The baseline VI approximation can be thought of the pushforward of a standard normal through this `diagonal affine' flow. Note that since diagonal affine transformations are closed under composition there would obtain no additional expressiveness in the baseline VI approximation from adding more of them. 

Figure \ref{fig:combined-results} (c) shows the results for this example. Both AFT and SMC reduce in variance as the number of temperatures increases and exceed the performance of the variational baseline. AFT has a notably lower variance than SMC for 10 and 30 temperatures- which shows the incorporation of the flows is beneficial in this case. Results for other difficult digits are shown in the appendix where the qualitative trend is similar.

\subsection{Log Gaussian Cox process}\label{section:cox}
We evaluate here the performance of AFT for estimating the normalizing constant of a log Gaussian Cox process applied to modelling the positions of pine saplings in Finland \cite{Moller:1998}. We consider points on a discretized $d=M \times M=1600$ grid. This results in the target density
\begin{equation}
\gamma(x) = \mathcal{N}(x ; \mu, K) \prod_{i \in [1:M]^{2}} \exp(x_{i} y_{i} - a \exp(x_i) ).
\end{equation}
This challenging high-dimensional problem is a commonly used benchmark in the SMC literature \cite{heng2017controlled, buchholz2020adaptive}. 
The mean and covariance function match those estimated by \cite{Moller:1998} and are detailed in the Appendix. The supplement also discusses the effect of pre-conditioners on the mixing of the Markov kernel. For the NF we again used the diagonal affine transformation. The approximating family is the push forward of the previous target distribution and thus even a simple flow can result in a good approximation. It is also fast to evaluate. Figure \ref{fig:combined-results} (d) shows that the baseline VI approximation is unable to capture the posterior correlation and that AFT gives significantly more accurate results than SMC for a given number of transitions. As such, the Markov kernel and flow complement each other in this case.

%% file: main_icml.bbl
\begin{thebibliography}{}

\bibitem[Akyildiz and M{\'\i}guez, 2020]{akyildiz2020nudging}
Akyildiz, {\"O}.~D. and M{\'\i}guez, J. (2020).
\newblock Nudging the particle filter.
\newblock {\em Statistics and Computing}, 30(2):305--330.

\bibitem[Barash et~al., 2017]{barash2017gpu}
Barash, L.~Y., Weigel, M., Borovsk{\`y}, M., Janke, W., and Shchur, L.~N.
  (2017).
\newblock {GPU} accelerated population annealing algorithm.
\newblock {\em Computer Physics Communications}, 220:341--350.

\bibitem[Beskos et~al., 2016]{Beskos:2016}
Beskos, A., Jasra, A., Kantas, N., and Thiery, A. (2016).
\newblock On the convergence of adaptive sequential {M}onte {C}arlo methods.
\newblock {\em The Annals of Applied Probability}, 26(2):1111--1146.

\bibitem[Beskos et~al., 2011]{Beskos:2011}
Beskos, A., Pinski, F., Sanz-Serna, J., and Stuart, A. (2011).
\newblock Hybrid {M}onte {C}arlo on {H}ilbert spaces.
\newblock {\em Stochastic Processes and their Applications}, 121(10):2201 --
  2230.

\bibitem[Bradbury et~al., 2018]{Bradbury:2018}
Bradbury, J., Frostig, R., Hawkins, P., Johnson, M.~J., Leary, C., Maclaurin,
  D., Necula, G., Paszke, A., Vander{P}las, J., Wanderman-{M}ilne, S., and
  Zhang, Q. (2018).
\newblock {JAX}: composable transformations of {P}ython+{N}um{P}y programs.

\bibitem[Buchholz et~al., 2021]{buchholz2020adaptive}
Buchholz, A., Chopin, N., and Jacob, P.~E. (2021).
\newblock Adaptive tuning of {H}amiltonian {M}onte {C}arlo within sequential
  {M}onte {C}arlo.
\newblock {\em Bayesian Analysis to appear - arXiv preprint arXiv:1808.07730}.

\bibitem[Caterini et~al., 2018]{caterini2018hamiltonian}
Caterini, A.~L., Doucet, A., and Sejdinovic, D. (2018).
\newblock Hamiltonian variational auto-encoder.
\newblock In {\em Advances in Neural Information Processing Systems}, pages
  8167--8177.

\bibitem[Choi, 2019]{Choi:2019}
Choi, M.~C. (2019).
\newblock Universality of the {L}angevin diffusion as scaling limit of a family
  of {M}etropolis--{H}astings processes i: fixed dimension.
\newblock {\em arXiv preprint arXiv:1907.10318}.

\bibitem[Chopin, 2002]{chopin2002sequential}
Chopin, N. (2002).
\newblock A sequential particle filter method for static models.
\newblock {\em Biometrika}, 89(3):539--552.

\bibitem[Chopin, 2004]{chopin2004central}
Chopin, N. (2004).
\newblock Central limit theorem for sequential {M}onte {C}arlo methods and its
  application to {B}ayesian inference.
\newblock {\em The Annals of Statistics}, 32(6):2385--2411.

\bibitem[Crooks, 1998]{crooks1998nonequilibrium}
Crooks, G.~E. (1998).
\newblock Nonequilibrium measurements of free energy differences for
  microscopically reversible {M}arkovian systems.
\newblock {\em Journal of Statistical Physics}, 90(5-6):1481--1487.

\bibitem[Dai et~al., 2020]{dai2020invitation}
Dai, C., Heng, J., Jacob, P.~E., and Whiteley, N. (2020).
\newblock An invitation to sequential {M}onte {C}arlo samplers.
\newblock {\em arXiv preprint arXiv:2007.11936}.

\bibitem[Dalalyan, 2017]{Dalalyan:2014}
Dalalyan, A.~S. (2017).
\newblock Theoretical guarantees for approximate sampling from smooth and
  log-concave densities.
\newblock {\em Journal of the Royal Statistical Society: Series \textup{B}},
  3(79):651--676.

\bibitem[Del~Moral, 2004]{del2004feynman}
Del~Moral, P. (2004).
\newblock {\em Feynman-{K}ac Formulae: Genealogical and Interacting Particle
  Approximations}.
\newblock Springer.

\bibitem[Del~Moral et~al., 2006]{Del-Moral:2006}
Del~Moral, P., Doucet, A., and Jasra, A. (2006).
\newblock Sequential {M}onte {C}arlo samplers.
\newblock {\em Journal of the Royal Statistical Society: Series \textup{B}},
  68(3):411--436.

\bibitem[Del~Moral et~al., 2012a]{delmoral2012adaptive}
Del~Moral, P., Doucet, A., and Jasra, A. (2012a).
\newblock An adaptive sequential {M}onte {C}arlo method for approximate
  {B}ayesian computation.
\newblock {\em Statistics and Computing}, 22(5):1009--1020.

\bibitem[Del~Moral et~al., 2012b]{Del-Moral:2012}
Del~Moral, P., Doucet, A., and Jasra, A. (2012b).
\newblock On adaptive resampling strategies for sequential {M}onte {C}arlo
  methods.
\newblock {\em Bernoulli}, 18(1):252--278.

\bibitem[{Dillon} et~al., 2017]{Dillon:2017}
{Dillon}, J.~V., {Langmore}, I., {Tran}, D., {Brevdo}, E., {Vasudevan}, S.,
  {Moore}, D., {Patton}, B., {Alemi}, A., {Hoffman}, M., and {Saurous}, R.~A.
  (2017).
\newblock {TensorFlow Distributions}.
\newblock {\em arXiv preprint arXiv:1711.10604}.

\bibitem[Domke and Sheldon, 2018]{domke2018importance}
Domke, J. and Sheldon, D.~R. (2018).
\newblock Importance weighting and variational inference.
\newblock In {\em Advances in Neural Information Processing Systems}, pages
  4470--4479.

\bibitem[Douc and Moulines, 2008]{Douc:2007}
Douc, R. and Moulines, E. (2008).
\newblock Limit theorems for weighted samples with applications to sequential
  {M}onte {C}arlo methods.
\newblock {\em The Annals of Statistics}, 36(5):2344--2376.

\bibitem[Dudley, 2018]{Dudley:2018}
Dudley, R.~M. (2018).
\newblock {\em Real analysis and Probability}.
\newblock CRC Press.

\bibitem[Durkan et~al., 2019]{Durkan:2019}
Durkan, C., Bekasov, A., Murray, I., and Papamakarios, G. (2019).
\newblock Neural spline flows.
\newblock In {\em Advances in Neural Information Processing Systems}.

\bibitem[El~Moselhy and Marzouk, 2012]{Marzouk2012bayesian}
El~Moselhy, T.~A. and Marzouk, Y.~M. (2012).
\newblock Bayesian inference with optimal maps.
\newblock {\em Journal of Computational Physics}, 231(23):7815--7850.

\bibitem[Everitt et~al., 2020]{everitt2020sequential}
Everitt, R.~G., Culliford, R., Medina-Aguayo, F., and Wilson, D.~J. (2020).
\newblock Sequential {M}onte {C}arlo with transformations.
\newblock {\em Statistics and Computing}, 30(3):663--676.

\bibitem[Gao et~al., 2020]{gao2020flow}
Gao, C., Isaacson, J., and Krause, C. (2020).
\newblock i-flow: High-dimensional {I}ntegration and {S}ampling with
  normalizing flows.
\newblock {\em arXiv preprint arXiv:2001.05486}.

\bibitem[Gelfand and Mitter, 1991]{Gelfand:1991}
Gelfand, S.~B. and Mitter, S.~K. (1991).
\newblock Weak convergence of {M}arkov chain sampling methods and annealing
  algorithms to diffusions.
\newblock {\em Journal of Optimization Theory and Applications},
  68(3):483--498.

\bibitem[Germain et~al., 2015]{Germain:15}
Germain, M., Gregor, K., Murray, I., and Larochelle, H. (2015).
\newblock Made: Masked autoencoder for distribution estimation.
\newblock In Bach, F. and Blei, D., editors, {\em Proceedings of the 32nd
  International Conference on Machine Learning}, volume~37 of {\em Proceedings
  of Machine Learning Research}, pages 881--889, Lille, France. PMLR.

\bibitem[Gilks and Berzuini, 2001]{gilks2001following}
Gilks, W.~R. and Berzuini, C. (2001).
\newblock Following a moving target - {M}onte {C}arlo inference for dynamic
  {B}ayesian models.
\newblock {\em Journal of the Royal Statistical Society: Series B},
  63(1):127--146.

\bibitem[Goyal et~al., 2017]{goyal2017variational}
Goyal, A. G. A.~P., Ke, N.~R., Ganguli, S., and Bengio, Y. (2017).
\newblock Variational walkback: Learning a transition operator as a stochastic
  recurrent net.
\newblock In {\em Advances in Neural Information Processing Systems}, pages
  4392--4402.

\bibitem[Guarniero et~al., 2017]{guarniero2017iterated}
Guarniero, P., Johansen, A.~M., and Lee, A. (2017).
\newblock The iterated auxiliary particle filter.
\newblock {\em Journal of the American Statistical Association},
  112(520):1636--1647.

\bibitem[Han and Liu, 2017]{han2017stein}
Han, J. and Liu, Q. (2017).
\newblock Stein variational adaptive importance sampling.
\newblock {\em Uncertainty in Artificial Intelligence}.

\bibitem[Heng et~al., 2020]{heng2017controlled}
Heng, J., Bishop, A.~N., Deligiannidis, G., and Doucet, A. (2020).
\newblock Controlled sequential {M}onte {C}arlo.
\newblock {\em The Annals of Statistics}, 48(5):2904--2929.

\bibitem[Heng et~al., 2021]{heng2015gibbs}
Heng, J., Doucet, A., and Pokern, Y. (2021).
\newblock Gibbs flow for approximate transport with applications to {B}ayesian
  computation.
\newblock {\em Journal of the Royal Statistical Society Series \textup{B}},
  83(1):156--187.

\bibitem[Hennigan et~al., 2020]{Hennigan:2020}
Hennigan, T., Cai, T., Norman, T., and Babuschkin, I. (2020).
\newblock {H}aiku: {S}onnet for {JAX}.

\bibitem[Hessel et~al., 2020]{optax2020github}
Hessel, M., Budden, D., Viola, F., Rosca, M., Sezener, E., and Hennigan, T.
  (2020).
\newblock Optax: composable gradient transformation and optimisation, in jax.

\bibitem[Hoffman, 2017]{Hoffman:2017}
Hoffman, M.~D. (2017).
\newblock Learning deep latent {G}aussian models with {M}arkov chain {M}onte
  {C}arlo.
\newblock In Precup, D. and Teh, Y.~W., editors, {\em Proceedings of the 34th
  International Conference on Machine Learning}, volume~70 of {\em Proceedings
  of Machine Learning Research}, pages 1510--1519. PMLR.

\bibitem[Huang et~al., 2018]{huang2018improving}
Huang, C.-W., Tan, S., Lacoste, A., and Courville, A.~C. (2018).
\newblock Improving explorability in variational inference with annealed
  variational objectives.
\newblock In {\em Advances in Neural Information Processing Systems}, pages
  9701--9711.

\bibitem[Huang et~al., 2020]{huang2020evaluating}
Huang, S., Makhzani, A., Cao, Y., and Grosse, R. (2020).
\newblock Evaluating lossy compression rates of deep generative models.
\newblock In {\em International Conference on Machine Learning}, pages
  4444--4454. PMLR.

\bibitem[Hukushima and Iba, 2003]{hukushima2003population}
Hukushima, K. and Iba, Y. (2003).
\newblock Population annealing and its application to a spin glass.
\newblock In {\em AIP Conference Proceedings}, volume 690, pages 200--206.
  American Institute of Physics.

\bibitem[Jarzynski, 1997]{jarzynski1997nonequilibrium}
Jarzynski, C. (1997).
\newblock Nonequilibrium equality for free energy differences.
\newblock {\em Physical Review Letters}, 78(14):2690--2963.

\bibitem[Jasra et~al., 2011]{jasra2011inference}
Jasra, A., Stephens, D.~A., Doucet, A., and Tsagaris, T. (2011).
\newblock Inference for {L}{\'e}vy-driven stochastic volatility models via
  adaptive sequential {M}onte {C}arlo.
\newblock {\em Scandinavian Journal of Statistics}, 38(1):1--22.

\bibitem[Kappen and Ruiz, 2016]{kappen2016adaptive}
Kappen, H.~J. and Ruiz, H.~C. (2016).
\newblock Adaptive importance sampling for control and inference.
\newblock {\em Journal of Statistical Physics}, 162(5):1244--1266.

\bibitem[Kingma and Ba, 2014]{Kingma:2014}
Kingma, D.~P. and Ba, J. (2014).
\newblock Adam: A method for stochastic optimization.
\newblock {\em arXiv preprint arXiv:1412.6980}.

\bibitem[Kingma et~al., 2016]{Kingma:2016}
Kingma, D.~P., Salimans, T., Jozefowicz, R., Chen, X., Sutskever, I., and
  Welling, M. (2016).
\newblock Improved variational inference with inverse autoregressive flow.
\newblock In Lee, D., Sugiyama, M., Luxburg, U., Guyon, I., and Garnett, R.,
  editors, {\em Advances in Neural Information Processing Systems}, volume~29.
  Curran Associates, Inc.

\bibitem[Kingma and Welling, 2014]{kingma+welling:2014}
Kingma, D.~P. and Welling, M. (2014).
\newblock Auto-encoding variational {B}ayes.
\newblock {\em ICLR}.

\bibitem[Kitagawa, 1996]{kitagawa1996monte}
Kitagawa, G. (1996).
\newblock Monte {C}arlo filter and smoother for non-{G}aussian nonlinear state
  space models.
\newblock {\em Journal of Computational and Graphical Statistics}, 5(1):1--25.

\bibitem[K{\"u}nsch, 2005]{kunsch2005recursive}
K{\"u}nsch, H.~R. (2005).
\newblock Recursive {M}onte {C}arlo filters: algorithms and theoretical
  analysis.
\newblock {\em The Annals of Statistics}, 33(5):1983--2021.

\bibitem[Le et~al., 2018]{le2017auto}
Le, T.~A., Igl, M., Rainforth, T., Jin, T., and Wood, F. (2018).
\newblock Auto-encoding sequential {M}onte {C}arlo.
\newblock In {\em ICLR}.

\bibitem[Lee et~al., 2010]{lee2010utility}
Lee, A., Yau, C., Giles, M.~B., Doucet, A., and Holmes, C.~C. (2010).
\newblock On the utility of graphics cards to perform massively parallel
  simulation of advanced monte carlo methods.
\newblock {\em Journal of Computational and Graphical Statistics},
  19(4):769--789.

\bibitem[{Lei Ba} et~al., 2016]{Ba2016}
{Lei Ba}, J., {Kiros}, J.~R., and {Hinton}, G.~E. (2016).
\newblock {Layer Normalization}.
\newblock {\em arXiv e-prints}.

\bibitem[Leli\`evre et~al., 2010]{lelievreroussetstoltz2010free}
Leli\`evre, T., Rousset, M., and Stoltz, G. (2010).
\newblock {\em Free Energy Computations: A Mathematical Perspective}.
\newblock World Scientific.

\bibitem[Li and Chen, 2019]{li2019rate}
Li, Q. and Chen, Y. (2019).
\newblock Rate distortion via deep learning.
\newblock {\em IEEE Transactions on Communications}, 68(1):456--465.

\bibitem[Liu et~al., 2019]{liu2019understanding}
Liu, C., Zhuo, J., Cheng, P., Zhang, R., and Zhu, J. (2019).
\newblock Understanding and accelerating particle-based variational inference.
\newblock In {\em International Conference on Machine Learning}, pages
  4082--4092.

\bibitem[Liu and Chen, 1995]{liuchen1995}
Liu, J.~S. and Chen, R. (1995).
\newblock Blind deconvolution via sequential imputations.
\newblock {\em Journal of the American Statistical Association},
  90(430):567--576.

\bibitem[Liu and Wang, 2016]{liu2016stein}
Liu, Q. and Wang, D. (2016).
\newblock Stein variational gradient descent: a general purpose {B}ayesian
  inference algorithm.
\newblock In {\em Advances in Neural Information Processing Systems}.

\bibitem[Llorente et~al., 2020]{llorente2020marginal}
Llorente, F., Martino, L., Delgado, D., and Lopez-Santiago, J. (2020).
\newblock Marginal likelihood computation for model selection and hypothesis
  testing: an extensive review.
\newblock {\em arXiv preprint arXiv:2005.08334}.

\bibitem[MacEachern et~al., 1999]{maceachern1999sequential}
MacEachern, S.~N., Clyde, M., and Liu, J.~S. (1999).
\newblock Sequential importance sampling for nonparametric {B}ayes models: The
  next generation.
\newblock {\em Canadian Journal of Statistics}, 27(2):251--267.

\bibitem[Maddison et~al., 2017]{maddison2017filtering}
Maddison, C.~J., Lawson, J., Tucker, G., Heess, N., Norouzi, M., Mnih, A.,
  Doucet, A., and Teh, Y. (2017).
\newblock Filtering variational objectives.
\newblock In {\em Advances in Neural Information Processing Systems}, pages
  6573--6583.

\bibitem[Marzouk et~al., 2016]{marzouk2016sampling}
Marzouk, Y., Moselhy, T., Parno, M., and Spantini, A. (2016).
\newblock Sampling via measure transport: An introduction.
\newblock {\em Handbook of Uncertainty Quantification}, pages 1--41.

\bibitem[Mnih and Rezende, 2016]{mnih2016variational}
Mnih, A. and Rezende, D. (2016).
\newblock Variational inference for {M}onte {C}arlo objectives.
\newblock In {\em International Conference on Machine Learning}, pages
  2188--2196. PMLR.

\bibitem[M{\o}ller et~al., 1998]{Moller:1998}
M{\o}ller, J., Syversveen, A.~R., and Waagepetersen, R.~P. (1998).
\newblock Log {G}aussian {C}ox processes.
\newblock {\em Scandinavian Journal of Statistics}, 25(3):451--482.

\bibitem[Naesseth et~al., 2018]{naesseth2017variational}
Naesseth, C.~A., Linderman, S.~W., Ranganath, R., and Blei, D.~M. (2018).
\newblock Variational sequential {M}onte {C}arlo.
\newblock In {\em AISTATS}.

\bibitem[Neal, 2011]{Neal2011}
Neal, R. (2011).
\newblock {MCMC} using {H}amiltonian dynamics.
\newblock {\em Handbook of Markov chain Monte Carlo}.

\bibitem[Neal, 2001]{neal2001annealed}
Neal, R.~M. (2001).
\newblock Annealed importance sampling.
\newblock {\em Statistics and Computing}, 11(2):125--139.

\bibitem[Neal, 2003]{Neal:2003}
Neal, R.~M. (2003).
\newblock Slice sampling.
\newblock {\em The Annals of Statistics}, 31(3):705--767.

\bibitem[Nicoli et~al., 2020]{nicoli2020asymptotically}
Nicoli, K.~A., Nakajima, S., Strodthoff, N., Samek, W., M{\"u}ller, K.-R., and
  Kessel, P. (2020).
\newblock Asymptotically unbiased estimation of physical observables with
  neural samplers.
\newblock {\em Physical Review E}, 101(2):023304.

\bibitem[No{\'e} et~al., 2019]{noe2019boltzmann}
No{\'e}, F., Olsson, S., K{\"o}hler, J., and Wu, H. (2019).
\newblock Boltzmann generators: Sampling equilibrium states of many-body
  systems with deep learning.
\newblock {\em Science}, 365(6457):eaaw1147.

\bibitem[Olmez et~al., 2020]{olmez2020deep}
Olmez, S.~Y., Taghvaei, A., and Mehta, P.~G. (2020).
\newblock Deep fpf: Gain function approximation in high-dimensional setting.
\newblock In {\em 59th IEEE Conference on Decision and Control (CDC)}, pages
  4790--4795. IEEE.

\bibitem[Papamakarios et~al., 2019]{papamakarios2019normalizing}
Papamakarios, G., Nalisnick, E., Rezende, D.~J., Mohamed, S., and
  Lakshminarayanan, B. (2019).
\newblock Normalizing flows for probabilistic modeling and inference.
\newblock {\em arXiv preprint arXiv:1912.02762}.

\bibitem[Reich, 2011]{reich2011}
Reich, S. (2011).
\newblock A dynamical systems framework for intermittent data assimilation.
\newblock {\em BIT Numerical Mathematics}, 51(1):235--249.

\bibitem[Reich and Weissmann, 2021]{reich2021fokker}
Reich, S. and Weissmann, S. (2021).
\newblock Fokker--{P}lanck particle systems for {B}ayesian inference:
  Computational approaches.
\newblock {\em SIAM/ASA Journal on Uncertainty Quantification}, 9(2):446--482.

\bibitem[Rezende and Mohamed, 2015]{Rezende:2015}
Rezende, D.~J. and Mohamed, S. (2015).
\newblock Variational inference with normalizing flows.
\newblock In {\em Proceedings of the 32nd {International} {Conference} on
  {Machine} {Learning} - {Volume} 37}, {ICML}'15, pages 1530--1538. JMLR.org.

\bibitem[Rezende et~al., 2014]{rezende+al:2014:icml}
Rezende, D.~J., Mohamed, S., and Wierstra, D. (2014).
\newblock Stochastic backpropagation and approximate inference in deep
  generative models.
\newblock In {\em ICML}, pages 1278--1286.

\bibitem[Richard and Zhang, 2007]{richard2007efficient}
Richard, J.-F. and Zhang, W. (2007).
\newblock Efficient high-dimensional importance sampling.
\newblock {\em Journal of Econometrics}, 141(2):1385--1411.

\bibitem[Rousset and Stoltz, 2006]{Rousset:2006}
Rousset, M. and Stoltz, G. (2006).
\newblock Equilibrium sampling from nonequilibrium dynamics.
\newblock {\em Journal of Statistical Physics}, 123(6):1251--1272.

\bibitem[Salakhutdinov and Murray, 2008]{Salakhutdinov:2008}
Salakhutdinov, R. and Murray, I. (2008).
\newblock On the quantitative analysis of {D}eep {B}elief {N}etworks.
\newblock In {\em Proceedings of the 25th Annual International Conference on
  Machine Learning (ICML 2008)}, pages 872--879.

\bibitem[Salimans et~al., 2015]{salimans2015markov}
Salimans, T., Kingma, D., and Welling, M. (2015).
\newblock Markov chain {M}onte {C}arlo and variational inference: Bridging the
  gap.
\newblock In {\em International Conference on Machine Learning}, pages
  1218--1226.

\bibitem[Sch{\"a}fer and Chopin, 2013]{chopinschafer2013sequential}
Sch{\"a}fer, C. and Chopin, N. (2013).
\newblock Sequential {M}onte {C}arlo on large binary sampling spaces.
\newblock {\em Statistics and Computing}, 23(2):163--184.

\bibitem[Sen, 2018]{Sen:2018a}
Sen, B. (2018).
\newblock A gentle introduction to empirical process theory and applications.
\newblock {\em Lecture Notes, Columbia University}.

\bibitem[Taghvaei et~al., 2020]{taghvaei2020diffusion}
Taghvaei, A., Mehta, P.~G., and Meyn, S.~P. (2020).
\newblock Diffusion map-based algorithm for gain function approximation in the
  feedback particle filter.
\newblock {\em SIAM/ASA Journal on Uncertainty Quantification},
  8(3):1090--1117.

\bibitem[Thin et~al., 2021]{thin2021MCVAE}
Thin, A., Kotelevskii, N., Durmus, A., Panov, M., Moulines, E., and Doucet, A.
  (2021).
\newblock Monte {C}arlo variational auto-encoders.
\newblock {\em International Conference on Machine Learning}.

\bibitem[Turner and Sahani, 2011]{turner+sahani:2011a}
Turner, R.~E. and Sahani, M. (2011).
\newblock Two problems with variational expectation maximisation for
  time-series models.
\newblock In Barber, D., Cemgil, T., and Chiappa, S., editors, {\em Bayesian
  Time Series Models}, chapter~5, pages 109--130. Cambridge University Press.

\bibitem[Vaikuntanathan and Jarzynski, 2008]{vaikuntanathan2008escorted}
Vaikuntanathan, S. and Jarzynski, C. (2008).
\newblock Escorted free energy simulations: Improving convergence by reducing
  dissipation.
\newblock {\em Physical Review Letters}, 100(19):190601.

\bibitem[Vaikuntanathan and Jarzynski, 2011]{vaikuntanathan2011escorted}
Vaikuntanathan, S. and Jarzynski, C. (2011).
\newblock Escorted free energy simulations.
\newblock {\em The Journal of Chemical Physics}, 134(5):054107.

\bibitem[Van~der Vaart, 2000]{Van-der-Vaart:2000}
Van~der Vaart, A.~W. (2000).
\newblock {\em Asymptotic Statistics}.
\newblock Cambridge University Press.

\bibitem[Wang and Li, 2019]{wang2019accelerated}
Wang, Y. and Li, W. (2019).
\newblock Accelerated information gradient flow.
\newblock {\em arXiv preprint arXiv:1909.02102}.

\bibitem[Wirnsberger et~al., 2020]{wirnsbergertargetetNF2020}
Wirnsberger, P., Ballard, A.~J., Papamakarios, G., Abercrombie, S.,
  Racani\`ere, S., Pritzel, A., Rezende, D., and Blundell, C. (2020).
\newblock Targeted free energy estimation via learned mappings.
\newblock {\em The Journal of Chemical Physics}, 153(14):144112.

\bibitem[Wu et~al., 2020]{wunoe2020stochastic}
Wu, H., K{\"o}hler, J., and No{\'e}, F. (2020).
\newblock Stochastic normalizing flows.
\newblock In {\em Advances in Neural Information Processing Systems}.

\bibitem[Wu et~al., 2017]{Wu:2017}
Wu, Y., Burda, Y., Salakhutdinov, R., and Grosse, R.~B. (2017).
\newblock On the quantitative analysis of decoder-based generative models.
\newblock In {\em 5th International Conference on Learning Representations,
  {ICLR} 2017, Toulon, France, April 24-26, 2017, Conference Track
  Proceedings}.

\bibitem[Zhou et~al., 2016]{ZhouJohansen2016}
Zhou, Y., Johansen, A.~M., and Aston, J.~A. (2016).
\newblock Toward automatic model comparison: An adaptive sequential {M}onte
  {C}arlo approach.
\newblock {\em Journal of Computational and Graphical Statistics},
  25(3):701--726.

\bibitem[Zhu et~al., 2020]{zhu2020variance}
Zhu, M., Liu, C., and Zhu, J. (2020).
\newblock Variance reduction and quasi-{N}ewton for particle-based variational
  inference.
\newblock In {\em International Conference on Machine Learning}, pages
  11576--11587.

\end{thebibliography}
